\documentclass[twoside,11pt]{article}
\usepackage{tikz}
\usetikzlibrary{arrows.meta,positioning}
\usepackage{blindtext}

\usepackage[preprint]{jmlr2e}

\usepackage{graphicx} % For images
\usepackage{tcolorbox}
\usepackage{algorithm}
\usepackage{subcaption} 
\usepackage{float}
\usepackage{xcolor}
\usepackage{algorithmic}
\usepackage{url}
\usepackage{booktabs}
\usepackage{soul}
\usepackage{multirow}
\usepackage{xcolor}
\usepackage{pifont}
\usepackage{enumitem}
\usepackage{mathtools}
\usepackage{bm}

\usepackage{lastpage}
\jmlrheading{XX}{2026}{1-\pageref{LastPage}}{X/XX; Revised X/XX}{X/XX}{21-0000}{Jianhong Chen, Naichen Shi and Xubo Yue}

\ShortHeadings{SCD}{Chen, Shi, and Yue}
\firstpageno{1}

\begin{document}

% \title{Causal Discovery for Dynamical Systems}
\title{Causal Discovery from Heteroscedastic Stochastic Dynamical Systems under Imperfect Physical Models}

\author{\name Jianhong Chen \email chen.jianho@northeastern.edu \\
       \addr Department of Mechanical \& Industrial Engineering\\ Northeastern University\\ Boston, MA, USA
       \AND
       \name Naichen Shi \email naichen.shi@northwestern.edu \\
       \addr Department of Industrial Engineering and Management Sciences \\Department of Mechanical Engineering \\ Northwestern University \\ IL, USA
       \AND
       \name Xubo Yue \email x.yue@northeastern.edu \\
       \addr Department of Mechanical \& Industrial Engineering\\ Northeastern University\\ Boston, MA, USA}
\editor{My editor}

\maketitle

\vspace{-1em}
\begin{abstract}
Causal discovery is a data-driven paradigm for analyzing complex real-world systems, while physics-based models, such as ordinary differential equations (ODEs), provide mechanistic structure for real-world dynamical processes. Integrating these paradigms can use physical knowledge as an inductive bias, improving identifiability, stability, and robustness. However, real dynamical systems often exhibit feedback, cyclic interactions, and nonstationarity, whereas many causal discovery methods rely on acyclicity, stationarity, or equilibrium assumptions. We propose an integrative causal discovery framework for dynamical systems that leverages partial physical knowledge through stochastic differential equations (SDEs). The drift term encodes known ODE dynamics, while the diffusion term captures unknown causal couplings beyond the prescribed physics. We develop a scalable sparsity-inducing maximum quasi-likelihood estimator with a theoretically justified stabilization technique. Under mild conditions, we establish causal graph recovery guarantees for both stable and unstable SDE regimes. We also analyze robustness to ODE misspecification and clarify how the stabilization constant balances numerical stability and statistical recoverability. Experiments on linear SDEs and nonlinear benchmarks, including Lotka-Volterra and Lorenz dynamics with acyclic and cyclic structures, show improved graph recovery and robustness over data-driven baselines. We also demonstrate practical utility on real-world epidemic data by reconstructing stochastic SIR dynamics within our causal discovery framework.
\end{abstract}

\section{Introduction}
Understanding cause-effect relationships is fundamental for explaining complex real-world phenomena, such as gene regulation in genetics \citep{Feng2023, glymour2019review}, anthropogenic climate change in earth science \citep{Runge2019InferringCF}, and brain connectivity in neuroscience \citep{sanchez2019estimating}. Establishing reliable causal relationships ideally relies on randomized controlled experiments (RCTs), which are often considered the gold standard. However, RCTs can be prohibitively expensive, time-consuming, or even infeasible in many settings. Consequently, inferring causal structure from purely observational data and, when available, interventional data, known as \emph{causal discovery}, has attracted much attention \citep{chang2024causal}. A central challenge in causal discovery is that, without additional identifying assumptions, multiple causal models can be statistically indistinguishable. Classical approaches address this issue by imposing identifying assumptions or inductive biases, such as acyclicity \citep{pamfil2020dynotears}, non-Gaussianity \citep{shimizu06a}, or temporal precedence \citep{wang2024neural}, often yielding identifiability only up to an equivalence class. While these assumptions have enabled substantial methodological progress, the resulting models are typically formulated at a generic statistical level and do not explicitly incorporate mechanistic knowledge of the underlying system. As a result, the inferred causal graph may capture statistical dependence patterns without faithfully reflecting the physical, biological, or dynamical mechanisms that govern the observed process. 
% In contrast, many physical systems are governed by mechanistic principles that impose structure far beyond generic statistical regularities. Conservation laws, symmetries, and dynamical constraints restrict how system variables can evolve over time and space. Differential equations, including ordinary and partial differential equations, provide a compact mathematical language for representing such structure by encoding the mechanisms through which variables interact and evolve. This observation motivates a natural question:

In contrast, many physical and biological systems are often governed by mechanistic principles that impose structure far beyond generic statistical regularities. These principles constrain how system variables interact and evolve over time, and are often represented through differential equations such as ordinary or partial differential equations. Such equations provide a compact mathematical language for encoding known mechanisms. This observation motivates a natural question:

\begin{tcolorbox}
[width=\linewidth, colback=white!95!black]
\textit{Research Question: Can physical knowledge be incorporated as an inductive bias for causal discovery, and if so, what are its theoretical consequences?} 
\end{tcolorbox}

This perspective is closely related to the idea emphasized by \citet{CRLYoshua} that learning reusable mechanisms is an important direction in causal discovery, with physics providing a natural source of such mechanisms. Ordinary differential equations (ODEs), in particular, are a core modeling tool for characterizing mechanistic dynamics in many scientific and engineering domains, including neuroscience \citep{Vignesh2024ARO}, infectious disease \citep{alqudah2025study}, and manufacturing \citep{han2025physics}. Despite the fact that both tools serve as powerful approaches for modeling real-world processes,  \textit{there remains a clear need for a framework that bridges causal discovery with available physical knowledge}. 

% \textbf{An Intuitive (but Faulty) Solution.} A natural strategy is to use available physical knowledge as a partial structural prior, thereby reducing causal discovery to a partial structure learning problem in which only the unknown edges must be inferred. While this can be useful, it still treats mechanistic knowledge and causal discovery as two separate components: the physical model is specified first, and the remaining causal structure is inferred afterward. This separation can be problematic when the known mechanism interacts nontrivially with the unknown causal component, or when stochastic effects obscure a clean decomposition between the two. In such cases, treating the physical model merely as a preprocessing step may distort the downstream graph-learning problem. \textit{This further highlights the need for a framework that more directly incorporates given physical knowledge into causal discovery}.

Developing such a model raises several challenges. \emph{First}, many widely used causal discovery methods are built around acyclicity assumptions, which can be overly restrictive for modeling real-world dynamical mechanisms. \emph{Second}, there is no clear unified framework for integrating these two sources of information. More concretely, data-driven causal discovery is often formulated through a structural causal model (SCM), that is $x_i(t) = f_i\bigl(x_{\mathrm{pa}(i)}(t)\bigr),$ whereas mechanistic prior knowledge is typically expressed through differential equations of the form $\dot{x}_i(t) = g_i\bigl(t,\bm{x}(t),\bm{\gamma}\bigr)$. Because these formulations operate at different modeling levels, namely functional dependence and continuous-time evolution, it is nontrivial to combine them within a common modeling framework without introducing an explicit bridge between mechanistic constraints and causal graph learning.  

To address these difficulties, we propose a multiplicative-noise stochastic differential equation (SDE) framework that bridges available physical knowledge and data-driven causal discovery. The key modeling idea is that multiplicative SDEs naturally extend ODE models by accounting for process noise and unmodeled inputs, while representing functional dependence under stochastic uncertainty. This yields a unified formulation in which available physical knowledge and causal dependence can be incorporated in a common model. Building on this formulation, Section~\ref{sec:notion} introduces a causal notion adapted to the stochastic dynamical setting considered in this paper.

Framing causal discovery within an SDE model offers several important benefits. \emph{First}, the Brownian-motion component in an SDE (i.e., $W_t$) provides a mathematically tractable way to represent accumulated random perturbations over time, so that the model describes a distribution over trajectories rather than a single deterministic path. This is important in real systems, where uncertainty propagates through time and across coupled variables. Moreover, the diffusion term provides a principled representation of process noise and intrinsic dynamical variability, supporting likelihood-based inference and uncertainty quantification. \emph{Second}, the multiplicative SDE framework provides a natural way to address the central problem of causal discovery from nonstationary time series, since its state-dependent noise generates heteroskedastic and time-evolving dynamics that are incompatible with standard stationarity assumptions. \emph{Third}, since SDEs form a well-established mathematical framework, they provide a principled foundation for deriving theoretical guarantees on the statistical properties of the proposed algorithm. 

\subsection{Main Contributions} In this study, we propose a framework that bridges known dynamics and causal discovery, supported by theoretical guarantees and empirical validation. Our key contributions are:
 
\textbf{A novel formulation.} We introduce an SDE-based causal discovery framework that integrates mechanistic (physics-based) dynamics with data-driven graph learning. In our formulation, the drift term encodes known ODE dynamics, while the diffusion term captures unknown cross-variable causal couplings beyond the prescribed physics. The resulting model accommodates nonstationary time series with state-dependent stochastic variability.

\textbf{A statistically robust algorithm.} We propose \texttt{SCD}, a proximal-gradient method for causal graph recovery from observed trajectories by optimizing an $\ell_1$-penalized quasi-likelihood induced by Euler-Maruyama discretization. To ensure numerical stability and well-posedness of the structure learning problem, we incorporate a stabilization constant into the objective. This stabilization also yields several regularity properties that facilitate the theoretical analysis, including local restricted strong convexity. The estimation strategy applies to both linear models and nonlinear functional forms through feature expansion.

\textbf{Theoretical guarantees.} In summary, our theoretical analysis has three main components. First, we prove a high-probability support recovery guarantee. Second, we establish robustness of the recovery result to ODE misspecification. Third, we characterize the role of the stabilization constant, including its effect on both theory and practice. More specifically, these three components are as follows. 
\begin{enumerate}
    \item  For the support recovery analysis, we consider two SDE regimes: a stable regime, as shown in Theorem~\ref{Thm: LASSO}, and an unstable regime, as shown in Theorem~\ref{Thm:lasso_uns}. Under the stable SDE setting, we prove that for each variable \(x_i\), if $n \gtrsim C\tilde{c}^2 s^3 \log p$ and $\lambda_n \gtrsim \sqrt{\log p/n},$ where \(s \coloneqq \max_i s_i\), \(s_i\) denotes the parent-set size of variable \(x_i\), and \(C,\tilde c>0\) are constants, then the resulting estimator achieves high-probability support recovery (i.e., consistent parent-set selection). Analogous support recovery results are established for the unstable SDE setting.
    
    \item  For robustness, we establish that the support recovery guarantee persists under numerical ODE misspecification, as shown in Theorem~\ref{Thm:Lasso_miss}, after appropriately adjusting the regularization level and sample-size requirements. More precisely, when the mechanistic drift is misspecified at level \(\delta\), the recovery conditions depend explicitly on \(\delta\). We derive the corresponding regularization and sample-size requirements, and show that exact support recovery still holds for a suitable range of misspecification levels.
    
    \item For the stabilization constant, we show that the stabilized estimator still recovers the support of the population parameter under an adjusted minimum signal strength condition as shown in Corollary~\ref{cor:exact-support-original}. Then, we characterize a feasible admissible range for the stabilization constant \(c\) and derive a theory-driven optimal choice, as shown in Proposition~\ref{prop:fesibale_c}, which makes explicit the trade-off between numerical stabilization and statistical recoverability. 
\end{enumerate}

To our knowledge, this is among \textit{the first} support recovery results for nonconvex causal graph estimation under multiplicative SDE noise without imposing acyclicity or additive-noise assumptions. Moreover, we establish robustness to ODE misspecification and provide a theoretically justified characterization of the stabilization strategy.

\textbf{Experimental study.} We validate the linear version of \texttt{SCD} on time-series data generated from SDEs whose diffusion structure includes both acyclic and cyclic directed graphs. We vary the dimension \(p\) under both stable and unstable dynamics, and compare \texttt{SCD} against four benchmark methods. We further conduct simulations under ODE misspecification and empirically study the admissible range of the stabilization constant \(c\), as well as the theory-predicted tolerance range, thereby providing empirical support for the corresponding theoretical results. 

\textbf{Application.} We further validate \texttt{SCD} using Lotka Volterra and chaotic Lorenz dynamics as known physical priors. In this setting, we evaluate both the nonlinear and linear versions of our method under diffusion structures that include both acyclic and cyclic directed graphs. Finally, we apply our algorithm to real-world COVID data to recover a stochastic SIR model under known SIR dynamics.

\section{Literature Review}
\textbf{Causal Discovery from Time-series Data.} Our work is related to causal discovery from observational time-series data, as it studies causal structure in dynamical systems. A comprehensive survey of time-series causal discovery can be found in \citet{chang2024causal}. Existing approaches are typically categorized into (i) constraint-based methods relying on conditional independence tests, (ii) score-based structure optimization, and (iii) Granger-causality-based methods grounded in predictability. While our framework departs from these paradigms by incorporating mechanistic ODE knowledge as an inductive bias, it nevertheless operates on time-series observations. Constraint-based methods infer causal structure via conditional independence tests. PCMCI \citep{Runge2019PCMCI} combines the PC algorithm \citep{Spirtes2000} with momentary conditional independence tests to identify lagged parents, while PCMCI+ \citep{runge20a} further accounts for instantaneous effects. Extensions such as LPCMCI \citep{reiser2022causaldiscoverytimeseries} handle latent confounders, and CD-NOD \citep{Zhang2017a} targets nonstationary, heterogeneous time series. Without additional functional assumptions, these methods typically identify graphs only up to a Markov equivalence class. Score-based approaches instead posit an explicit SEM for the temporal process and learn the graph by optimizing a likelihood or score. \citet{aapo2010} leverage identifiability of additive noise models \citep{Hoyer2008a} to build linear autoregressive SEMs with non-Gaussian noise, and DYNOTEARS \citep{pamfil2020dynotears} adapts the continuous acyclicity constraint of \citet{Zheng2018} for differentiable structure learning. Granger causality rules out backward-in-time and instantaneous effects, and is often operationalized via vector autoregressive prediction \citep{Shojaie2010}. Building on this principle, \citet{Tank2021NGC} propose \texttt{Neural-GC}, which uses regularized neural networks for nonlinear Granger discovery; more recent work strengthens neural Granger methods via Jacobian-based regularization \citep{zhou2024}.

\noindent \textbf{Interpretation of Differential Equations From Causal Perspective.} One line of work that is seemingly related to our problem involves reinterpreting differential equation models from a causal perspective. For instance, \citet{mooij2013} show that, under suitable regularity conditions, the equilibrium behavior of a first-order ODE system can be mapped to a structural causal model, providing a causal interpretation of equilibrium relations, even in the presence of feedback cycles. Moreover, \citet{Hansen_2014} formalize how causal interventions act on SDEs by explicitly defining postintervention SDEs. Their framework shows that intervening on an SDE corresponds to the continuous-time limit of interventions in structural equation models, and establishes conditions under which postintervention distributions are identifiable from the observational SDE. While these works clarify how causality is defined and interpreted in differential equation models, they do not address the problem of learning causal graph structure from observational data.

\noindent \textbf{SDE-based Causal Discovery.} The closest methodological line to our work is SDE-based causal discovery, which models causal structure through stochastic dynamical systems. \citet{bellot2022neural} propose \texttt{NGM} based on the additive-noise SDE model $dx(t)=f(x(t)) dt + dW(t)$, which uses Neural ODEs \citep{chen2018node} for mean process $f(x(t))$ and extracts graphical structure from the first layer of $f(x(t))$. This formulation assumes data drawn from a stationary distribution induced by additive noise, which constrains its applicability to a limited class of dynamical systems.
% However, their method is only applicable to data generated from a stationary distribution, which follows from the assumed linear SDE with additive noise $dW(t)$. 
% and is explicitly stated as a key Assumption~ 
To address this limitation, \citet{wang2024neural} propose \texttt{SCOTCH}, a variational-inference framework that infers posterior distributions over candidate graphs. It adopts a general SDE formulation \(dx(t)=f(x(t))\,dt+g(x(t))\,dW(t)\) and uses neural networks to parameterize both $f$ and $g$, with the graph defined by the induced dependencies among state variables. However, their emphasis is on continuous-time modeling and forecasting from irregular observations. Although they discuss identifiability and consistency, their method mainly learns a latent continuous-time representation to improve prediction rather than performing classical structural identification. In addition, their diffusion term $g$ is constrained to a nonzero diagonal form. Related efforts also include constraint-based causal discovery formulated in SDE form \citep{manten2025signature}, which is currently limited to directed acyclic graphs, and optimal-control-based approaches such as graphical continuous Lyapunov models (GCLMs) \citep{varando2020gclm,dettling}, which rely on additive-noise dynamics. 

Although these SDE-based works provide a continuous time description of stochastic dynamics, existing approaches often restrict attention to additive-noise formulations, acyclicity, and frequently rely on stationarity-type conditions for analysis or inference. Moreover, ODEs and SDEs are often used primarily as data-generation tools, with limited discussion of how mechanistic ODE structure relates to the corresponding SDE representation. Importantly, in many real-world systems, ODE-level physical mechanisms constitute reusable domain knowledge. Causal discovery should be grounded in reliable prior knowledge and, when possible, refined into reusable knowledge that supports future discoveries and interventions \citep{CRLYoshua}. How to systematically incorporate such partial mechanistic information as an inductive bias for causal structure learning remains an open challenge.

\section{Causal Coupling Beyond Known Physical Dynamics }

In this section, we first present our formulation. We then provide a heuristic motivation for why SDEs can be formulated within a physics-informed causal discovery framework. Next, we formalize the corresponding notion of causality. Finally, we describe our problem setting and its extension.

\textbf{Our Formulation.} 
Before presenting our formulation, we introduce some necessary notation. For \(p\in\mathbb N\), let \([p]:=\{1,\ldots,p\}\). For \(i\in[p]\), \(i\) indexes component \(x_i\), \(pa(i)\) denotes the parent set of \(x_i\), and \(x_{pa(i)}(t)\) denotes the subvector of \(x(t)\in\mathbb{R}^p\) for \(t\in[0,T]\) and $T>0$ corresponding to the parents of \(x_i\). When partial physics knowledge is available, we place the known mechanistic component in the drift and model the remaining unresolved variation through a parent-dependent diffusion term. We consider the following stochastic modulation formulation
\begin{equation}
\label{Eq:heuristic_for}
dx_i(t)\coloneqq\underbrace{g_i(x(t),\bm \gamma,t)}_{\text{ODE Information}}\,dt + \underbrace{h_i(x_{\mathrm{pa}(i)}(t))}_{\text{Causal Information}}\,dW_{i,t},
\end{equation}
where \(W_{i,t}\) denotes a standard Brownian motion for the \(i\)-th component; \(h_i\) is a diffusion function depending on the parents of node \(i\); \(g_i\) denotes the ODE model for the \(i\)-th component; and $\bm \gamma$ denotes the ODE parameters. This formulation gives the diffusion-based structure two complementary causal interpretations. First, the known drift term anchors the model in mechanistic dynamics, so the inferred structure captures parent-dependent variation beyond the prescribed deterministic dynamics. Second, the dependence of \(h_i\) on \(x_{\mathrm{pa}(i)}(t)\) induces a stochastic information flow from the parent variables to \(x_i\), since these variables modulate the conditional variability of the future increment of \(x_i(t)\). Notably, learning the structure through the diffusion term provides an uncertainty-aware perspective on causal discovery by identifying parent variables through their modulation of unresolved stochastic variation.

% \textbf{Our Formulation.} When partial physics knowledge is available, we place the known mechanistic component in the drift and model the remaining unresolved variation through a parent-dependent diffusion term. This leads to the following parent-dependent stochastic modulation ansatz
% \begin{equation}
% \label{Eq:heuristic_for}
%     dx_i(t)\coloneqq \underbrace{g_i(x(t),\gamma,t)}_{\text{ODE Information}}dt+ \underbrace{h_i(x_{\mathrm{pa}(i)}(t))}_{\text{Causal Information}}dW_{i,t},
% \end{equation}
% where \(W_{i,t}\) denotes a standard Brownian motion, $h_i$ is a parent-dependent diffusion function In this way, the ODE component captures known deterministic mechanistic dynamics, while the stochastic component incorporates additional parent-driven variation and uncertainty not explained by the deterministic model alone.

% To make this practical, we can approximate $h_i$ using a dictionary of candidate functions $ \phi(x_{\mathrm{pa}(i)}(t))$, which serves as a surrogate for the unknown parent-driven dynamics. For example, through a parameterization such as $h_i(x_{\mathrm{pa}(i)}) \coloneqq\nabla f_i\bigl(x_{\mathrm{pa}(i)}(t)\bigr)^\top\phi(x_{\mathrm{pa}(i)}(t))$. 

\subsection{A Heuristic Motivation for the SDE Formulation}
\label{Sec:mot}
In this section, we provide a heuristic motivation for using a parent-dependent diffusion term, where the diffusion modulation depends on the states of the parent variables, when combining partial physics knowledge with causal structure. We first recall the causal formulation and the ODE formulation, which are commonly written as
\begin{equation}
\label{Eq:botheq}
x_i(t) = f_i\bigl(x_{\mathrm{pa}(i)}(t)\bigr) + u_i(t), \qquad \frac{d}{dt}x_i(t) = g_i\bigl(x(t),\bm \gamma,t\bigr)
\end{equation}
where \(f_i\) denotes a causal functional relationship and $u_i(t)$ denotes a residual latent process. These two formulations arise from different modeling viewpoints: the first encodes causal dependence, while the second specifies deterministic temporal dynamics. Thus, they provide two complementary sources of information, but they are \textbf{not directly compatible} in general. To integrate these two sources of information, we note that the SDE framework provides a natural bridge between ODE-based dynamics and stochastic uncertainty. We emphasize that the following discussion is a modeling analogy, and its purpose is to motivate the ansatz in Eq.\eqref{Eq:heuristic_for}, which we adopt as our modeling choice.

\textbf{A tale from SDE.} To motivate this connection, we first consider the following formal SDE formulation \citep{oksendal2010stochastic}:
\begin{equation}
    \frac{d}{dt}X_t=b(X_t,t)+\sigma(X_t,t)\epsilon(t),
\end{equation}
where $\epsilon(t)$ denotes Gaussian white noise. The first term \(b(\cdot,\cdot)\) represents the deterministic dynamics, while the second term \(\sigma(X_t,t)\epsilon(t)\) represents stochastic perturbations. The function \(\sigma(X_t,t)\) is the diffusion coefficient, which is allowed to depend on both the state and time. Here, \(X_t\) denotes the same state process as \(x(t)\), written in standard SDE notation. Equivalently, in rigorous It\^o form, $dX_t = b(X_t,t)\,dt + \sigma(X_t,t)dW_t$, where \(W_t\) denotes a standard Brownian motion.
% dX_t = b(X_t,t)\,dt + \sigma(X_t,t)\underbrace{dW_t}_{\epsilon(t)dt}.
When the stochastic term vanishes, this formulation reduces formally to an ODE. An ODE can be seen as the deterministic counterpart of an SDE. 

\textbf{A tale from SCM.} We also obtain a useful dynamical interpretation of the causal formulation. By differentiating the causal formulation, as shown on the left-hand side of Eq.\eqref{Eq:botheq}, with respect to time and applying the chain rule, we can write
\begin{equation}
\frac{d}{dt}x_i(t)=\nabla f_i\bigl(x_{\mathrm{pa}(i)}(t)\bigr)^\top\frac{d}{dt}x_{\mathrm{pa}(i)}(t)+ \frac{d}{dt}u_i(t),
\end{equation}
where \(\nabla f_i\) denotes the gradient of \(f_i\) with respect to its parent variables. This relation indicates that the temporal evolution of \(x_i(t)\) depends on its parent variables through the sensitivity term \(\nabla f_i\). \textit{This suggests, by analogy, modeling any residual fluctuation (i.e., $\frac{d}{dt}u_i(t)$) in the dynamics of \(x_i\) as also depending on those parent variables.} In particular, if \(j \in pa(i)\), then \(f_i\) depends nontrivially on \(x_j\); in the smooth setting, this is reflected by \(\partial_{x_j} f_i\) not being identically zero. 

Moreover, if \(i\) has no parents, then the parent-dependent term vanishes and the remaining variability is absorbed into \(u_i(t)\), so that \(\frac{d}{dt}u_i(t) = \frac{d}{dt}x_i(t)\). We model the known deterministic evolution of component \(i\) through the physics-prior drift \(g_i(x(t),\gamma,t)\). Deviations from this baseline, which are not explained by the prior physical model, are treated as stochastic perturbations and encoded in the diffusion component of the SDE.

Since an SDE reduces to an ODE when the diffusion term vanishes, it naturally contains deterministic mechanistic dynamics as a special case. Conversely, differentiating a deterministic structural relation of the form \(x_i(t)=f_i(x_{\mathrm{pa}(i)}(t))\) suggests that the temporal evolution of \(x_i(t)\) is governed by its parent variables through the sensitivity term \(\nabla f_i(x_{\mathrm{pa}(i)}(t))\). Hence, from these two perspectives, Eq.~\eqref{Eq:heuristic_for} provides a natural bridge between known mechanistic dynamics and parent-dependent stochastic modulation, thereby connecting partial physics knowledge with causal structure learning.

% Since an SDE reduces to an ODE when the stochastic term vanishes, and since differentiating the structural equation shows that the temporal evolution of \(x_i(t)\) depends on its parent variables through \(\nabla f_i(x_{pa(i)}(t))\), 

% Hence, from these two perspectives, Eq.\eqref{Eq:heuristic_for} provides a natural bridge between known mechanistic dynamics and parent-dependent stochastic modulation, thereby connecting partial physics knowledge with causal structure learning.

\subsection{Characterization of Causal Influence}
\label{sec:notion}

Building on the heuristic discussion above, we interpret the stochastic term as an
\textit{uncertain information channel} whose effect is modulated by the current system state. In an idealized setting where all relevant mechanisms are fully known and correctly modeled, this stochastic component may be negligible, in which case the deterministic dynamics alone are sufficient to describe the system evolution. Under partial mechanistic knowledge, however, unresolved or exogenous influences may remain, and these are captured through a state-dependent stochastic term. Motivated by this viewpoint, we impose the generic SDE representation for $\bm X_t \in \mathbb R^p$ such that
\begin{equation}
    d\bm X_t = g(\bm X_t,\bm\gamma,t)dt + H(\bm X_t,t)\odot dW_t,
\end{equation}
where $W_t = (W_{1,t},W_{2,t},\ldots W_{p,t}) \in \mathbb R^p$ and $\odot$ denotes the componentwise Hadamard product, and the stochastic differential is interpreted in the It$\hat o$ sense. \(g(\bm X_t,\bm \gamma,t)\) denotes the known \textit{physics as prior knowledge}, which is given partially, and $H(\bm X_t,t) = \bigl(H_1(\bm X_t,t),\ldots,H_p(\bm X_t,t)\bigr)^\top$ represents a \textit{state-dependent stochastic modulation} term driven by parent information and not accounted for by the known dynamics. More precisely, following the causal model, $H: \mathbb R^p \times [0,T] \to \mathbb R^p $.  Under this formulation, it is important to distinguish two notions. First, the condition \(H_i(X_t,t)\not\equiv 0\) indicates that component \(i\) contains residual stochastic variation beyond the prescribed deterministic dynamics. Second, additional causal coupling from variable \(X_j\) to component \(X_i\) is present only when the stochastic modulation of component \(i\) depends nontrivially on \(X_j\). Under this intuition, we can define causality as in \textbf{Definition \ref{Defn:causal}}. 

\begin{definition}[Causal coupling through stochastic modulation]
\label{Defn:causal}
For \(i,j\in[p]\), we say that \(X_j\) exerts \emph{causal coupling beyond the known dynamics} on component \(X_i\) if the \(i\)-th stochastic modulation term depends nontrivially on \(X_j\), that is, if there exist \(t\in[0,T]\) and two states \(x,x'\in\mathbb{R}^p\) such that
\[
x_{-j}=x'_{-j}, \qquad x_j\neq x'_j, \qquad H_i(x,t)\neq H_i(x',t).
\]
Equivalently, when \(H_i\) is smooth, this means that \(\partial_{x_j}H_i(x,t)\) is not identically zero.
\end{definition}
If \(H_i(X_t,t)\not\equiv 0\) but \(H_i\) does not depend on any \(X_j\) for \(j\in [p]\), then component \(i\) contains unresolved stochasticity but no additional cross-component causal coupling beyond the known deterministic mechanism.

At the system level, the dynamics are said to be \emph{causally closed under the known mechanism} if $H(X_t,t)\equiv 0$. In this case, the system reduces to the known deterministic model and no additional state-dependent coupling beyond the prescribed mechanism is present. This definition characterizes causality as additional dynamical coupling not captured by the known ODE component. It is consistent with an information-flow interpretation \citep{GC1969}: if a variable enters the stochastic modulation term of another variable, then perturbations to the former can propagate through the system and alter downstream state trajectories. In this sense, the stochastic term encodes unresolved causal effects beyond the known mechanistic dynamics. This viewpoint complements, rather than replaces, predictive notions such as Granger causality, and it naturally allows higher-order effects through state-dependent modulation.

\textbf{Example: Predator--prey behavior.} Consider a predator-prey system (e.g., foxes and rabbits), whose population dynamics can be modeled by the Lotka-Volterra (LV) equations. In an idealized setting with no additional external drivers, the dynamics are fully determined over long time horizons given the initial conditions and fixed parameters. In practice, however, exogenous factors inevitably enter the system, for example, human hunting that reduces the predator population, habitat changes, or the introduction of a new species. These interventions can be interpreted as causal inputs that perturb the original dynamics. Consequently, the original LV model alone is no longer sufficient to capture the observed behavior, and the system must be augmented to account for these new causal factors. In our formulation, the known ecological mechanism is encoded in the drift term represented by $ \bm g(t, \bm x(t), \bm \gamma)$ (e.g., the LV component in the predator-prey example), while unresolved influences are represented through the stochastic term $H(X_t,t) \odot dW_t$, whose dependence on parent variables captures additional causal coupling beyond the known ODE.

In the next subsection, we instantiate this general characterization in the linear setting, where the stochastic modulation term is parameterized by the causal matrix \(A\).

\subsection{Problem Setup}
\label{sec: Setup}
In this study, we focus primarily on the linear causal fitting case for clarity of mathematical presentation. However, the formulation can be extended to nonlinear settings via dictionary expansion. For the linear case of Eq.\eqref{Eq:heuristic_for}, the causal functional relationship is written as $f_i\bigl(X_{\mathrm{pa}(i)}(t)\bigr)=A_i^\top X(t),$ where $A_i^\top \in \mathbb R^p$ denotes the $i$-row of $A \in \mathbb R^{p \times p}$. We want to note that the non-parent entries of $A_i$ are zero. Then, the dictionary function \(\phi\) is taken to be linear. The continuous-time SDE for each $i\in[p]$ can be written as:
\begin{equation}
\label{Eq: SDE_introd}
    dX_i(t) =  g_i(t, \bm X(t), \bm \gamma)dt + \big(A^\top_i \bm X(t)\big) dW_i(t),
\end{equation}
where $\{W_i(t)\}_{i=1}^p$ are independent standard Brownian motions. As a result, in $p$-dimensional, we can have the following SDE:
\begin{equation}
\label{GSDE}
\begin{split}
&d\bm X_t = \bm{g}(t, \bm X_t, \bm\gamma) dt+S_A(\bm X_t) dW_t, \\
 & \qquad S_A(\bm X_t) =\mathrm{diag}\big([A^\top\bm X_t]_1,[A^\top\bm X_t]_2,\ldots,[A^\top\bm X_t]_p\big), 
\end{split}
\end{equation}
where $\bm g(t,\bm X(t),\gamma) = \bigl(g_1(t,\bm X(t),\gamma),\dots,g_d(t,\bm X(t),\gamma),0,\dots,0\bigr)^\top\in \mathbb{R}^p$, so that the known mechanistic dynamics are encoded in the first \(d\) components, with \([d]\subset [p]\), and \(W_t=(W_1(t),\dots,W_p(t))^\top\). Then, we view the observed $p$-dimensional trajectory $\bm{x}(t)$ as a (discretely sampled) realization of an underlying stochastic process $\bm{X}(t)=(X_1(t),\ldots,X_p(t))$ on $[0,T]$. We approximate the SDE dynamics on a time grid $I=\{t_0,t_1,\ldots,t_n\}\subset[0,T]$ using the Euler-Maruyama scheme \citep{kloeden2011numerical}. For each component $i\in[p]$ and time index $k\in\{0,\ldots,n-1\}$, with step size $\Delta t_k=t_{k+1}-t_k$, we have
\begin{equation}
x^{i}_{k+1} = x^{i}_{k} + g_i(t_k, \bm x_k, \bm \gamma)\Delta t_k + \big(A_i^\top \bm x_k\big)\Delta W_k^i,
\label{eq:em}
\end{equation}
where $\bm x_k =\bm x(t_{k})$ and $\Delta W_k^i = \sqrt{\Delta t_k}\,z_{k+1}^i$, where $z^{i}_{k+1}\sim\mathcal{N}(0,1)$. Here, $g_i(t_k,\bm x_k,\bm \gamma)$ denote the deterministic ODE for the \(i\)-th variable and define the drift term. In contrast, $A^\top_i$ collects the coefficients associated with multiplicative noise and defines the diffusion term. For the concentration analysis of the discretely observed sequence, we impose the following dependence condition on the Euler–Maruyama observations.

\begin{assumption}[Geometric $\alpha$-mixing]
\label{ass:alpha-mixing-EM}
For the concentration analysis, we consider a fixed discretization step $\Delta>0$ and the associated Euler--Maruyama observation sequence $\{X_k\}_{k\ge 1}$ defined by $X_k := X^{\mathrm{EM}}(k\Delta)$. We define the strong mixing coefficients by
\[
\alpha_X(m) := \sup_{r\ge 1} \sup_{A\in \sigma(X_1,\dots,X_r),\, B\in \sigma(X_{r+m},X_{r+m+1},\dots)} \bigl|\mathbb{P}(A\cap B)-\mathbb{P}(A)\mathbb{P}(B)\bigr|,
\]
where $m\ge1$. Moreover, there exist constants $A_\alpha,a_\alpha>0$ such that $\alpha_X(m)\le A_\alpha e^{-a_\alpha m}$.
\end{assumption}
Assumption~\ref{ass:alpha-mixing-EM} is a standard high-level short-memory condition on the discretely observed sequence. It requires temporal dependence to decay geometrically with the lag, which is weaker than assuming independent observations, yet strong enough to permit Bernstein-type concentration inequalities for dependent data. Such mixing assumptions are widely used in modern nonstationary time-series theory (see, e.g., \cite{Fryzlewicz_2011, Vogt_2012}), especially in high-dimensional analysis where one needs nonasymptotic deviation bounds and restricted-eigenvalue-type control under serial dependence. Moreover, geometric \(\alpha\)-mixing can be justified in certain model-specific multiplicative-noise SDEs closely related to our setting, such as stochastic Lotka--Volterra systems, as shown in \S\ref{sec:mix}.

% However, because the present work is formulated for a general class of SDEs, verifying Assumption~\ref{ass:alpha-mixing-EM} requires model-specific arguments for the underlying mechanistic model, which is out of the current scope. We therefore leave its verification for specific mechanistic models as future work.

\subsection{Nonlinear Extension}
\label{Sec: Nonlinear}
% , and a symbolic representation \(\Psi = (\psi_1,\psi_2,\ldots, \psi_p)\), where $\psi_i : \mathbb R^{p\cdot J} \to \mathbb R$. For each $i\in [p]$,

In this section, we show that the formulation in Eq.\eqref{Eq:heuristic_for} also admits a nonlinear dictionary extension. For simplicity and interpretability, we consider an additive nonlinear representation based on  coordinate-wise basis functions (e.g., $\phi_1(u)=u^2$ and $\phi_2(u)=\sin(u)$). Under this construction, the stochastic modulation term is approximated through a dictionary of candidate nonlinear functions. We define $\Phi(X_i(t)) = \bigl(\phi_1(X_i(t)),\phi_2(X_i(t)),\dots,\phi_J(X_i(t))\bigr)^\top \in \mathbb{R}^J$. Applying the dictionary to all \(p\) variables yields
\[
\Phi(\bm X(t)) =\begin{bmatrix}
\Phi(X_1(t)) & \Phi(X_2(t)) & \cdots & \Phi(X_p(t))
\end{bmatrix}\in \mathbb{R}^{J\times p}.
\]
In the nonlinear dictionary setting, the effect of variable $X_j(t)$ on component $i$ is represented by the coefficient block associated with $\Phi(X_j(t))$; accordingly, a directed edge $j\to i$ is present when this block is nonzero. Let $\bar\Phi(\bm X(t)) := \operatorname{vec}(\Phi(\bm X(t))) \in \mathbb{R}^{pJ}.$ We then model the nonlinear stochastic modulation for component \(i\) as $S_{\tilde A}^i(\bm X_t) := \tilde A_i^\top \bar\Phi(\bm X(t)),$ where \(\tilde A_i \in \mathbb{R}^{pJ}\) is the coefficient vector associated with the \(i\)-th component, and \(\tilde A \in \mathbb{R}^{p\times (pJ)}\). This yields the nonlinear extension for the \(i\)-th component as $dX_i(t) = g_i(t, \bm X(t),\bm \gamma)dt + \bigl(\tilde A_i^\top \bar\Phi(\bm X(t))\bigr)dW_i(t).$ In vector form, this can be written as
\begin{equation}
\begin{split}
       & d\bm X_t = g(t, \bm X(t),\bm \gamma)\,dt + S_{\tilde A,\Phi}(\bm X_t)dW_t \\
       &\qquad S_{\tilde A,\Phi}(\bm X_t) = \operatorname{diag}\bigl( \tilde A_1^\top \bar\Phi(\bm X(t)),\dots,\tilde A_p^\top \bar\Phi(\bm X(t)) \bigr).
\end{split}
\end{equation}
The linear model is recovered as a special case by taking $J = 1$ and taking \(\Phi(X_i(t))=X_i(t)\), or equivalently \(\bar\Phi (\bm X(t))=\bm X(t)\) up to the chosen stacking convention. For a more general nonlinear extension, we believe this is a meaningful avenue for future work, and we provide additional discussion in \S\ref{Sed:Dis}.

% \newpage
% Moreover, our $S_A(\bm X_t)$ idea can be extend to nonlinear setting. To model the nonlinear relationship and nonlinear function, we introduce the following representation
% \begin{equation}
%     S^\Phi_g(\bm X_t) = g(\Phi(\bm X_t))
% \end{equation}
% where $\Phi$ is the dictionary of the function representation, and $g$ is the symbolic representation. As an example, given $g_i\coloneqq a_i^\top \Phi_i(\bm X_k)$ as additive relationship of function of $\Phi_i(\bm X_k)$ and each $ X^i_k$ encoded a function such as $\phi_i \coloneqq\sin(X^i_k)$. Given the dictionary size $j \coloneqq|\Phi|$, we can parametrize $g \coloneqq \tilde{A} \in \mathbb R^{p\times (p\cdot j)}$, and we introduce candidate signal $\bm {\tilde X}_k\in \mathbb R^{(p\cdot j)}$

\section{Causal Graph Estimation}
\label{sec: method}

Before moving to the learning objective, we first define the graph formally.
We consider a directed graph $G\coloneqq (V,E)$, where $V\coloneqq \{v_1,\ldots,v_p\}$ with $v_i \coloneqq x_i$. We define $e_{ij}\coloneqq (v_j \to v_i)$ and include $e_{ij}\in E$ if and only if $a_{ij}\neq 0$, where $a_{ij}$ denotes the $(i,j)$th entry of the adjacency matrix $A$. The graph structure is determined by ODE knowledge and adjacency matrix $A$. In this problem, our main task is to learn the causal graph and new edges from the unknown mechanism variables to ODE variables, encoded by the adjacency matrix $A$, and we assume that the graph is sparse; that is,

\begin{assumption}[Sparsity]
\label{sparse}
 There exists some $\kappa \in [0,1)$ for all node $j\in V\setminus{V_{phy}}$ such that, as $|V|\to\infty$, $\max_{j\in V\setminus{V_{phy}}} |pa(j)| = O(|V|^{\kappa}),$
\end{assumption}
where $V_{phy} \coloneqq \{x_1, x_2, \ldots, x_q\} \subseteq V$ denotes the set of variables for which physics information is available. Sparsity is also necessary for identifiable and statistically tractable recovery in high-dimensional statistics. Given the data from the Euler–Maruyama discretization for each $i\in [p]$, as defined in Eq.\eqref{eq:em}, we have that, conditional on \(x_k\), the one-step transition follows $x_{k+1}^i \mid \bm x_k \sim \mathcal N\left( x_k^i + g_i(t_k,\bm x_k,\bm\gamma)\Delta t_k,\Delta t_k (A_i^\top \bm x_k)^2 \right)$. Let $\mu_{i,k} := x_k^i + g_i(t_k,\bm x_k,\bm\gamma)\Delta t_k$. From the Euler--Maruyama Gaussian transition, we have the centered and rescaled transition satisfies $x_{k+1}^i - \mu_{i,k} \mid\bm x_k \sim \mathcal{N}\!\left(0, \Delta t_k (A_i^\top \bm x_k)^2\right)$, and hence $\frac{x_{k+1}^i - \mu_{i,k}}{\sqrt{\Delta t_k}} \Big| \bm x_k \sim \mathcal{N}\left(0,(A_i^\top \bm x_k)^2\right).$ Therefore, we define the normalized residual as
\begin{equation}
\label{Eq:r_ik}
    r_{i,k} := \frac{x_i(k+1)-x_i(k)-g_i(t_k,\bm x_k,\bm\gamma)\Delta t_k}{\sqrt{\Delta t_k}},
\end{equation}
so that $r_{i,k}\mid \bm x_k\sim \mathcal{N}\!\left(0,\; (A_i^\top\bm x_k)^2\right)$. Hence, we form the negative Gaussian quasi-log-likelihood by treating \((A_i^\top x_k)^2\) as the conditional variance. This leads to the objective
\[
-\log p(r_{i,k}\mid \bm x_k; A_i) \propto\frac12\left[\log\big((A_i^\top \bm x_k)^2\big) + \frac{r_{i,k}^2}{(A_i^\top \bm x_k)^2} \right].
\]
Since our primary goal is structure identification, it is possible, according to Definition~\ref{Defn:causal}, that node \(i\) has no parent variables, that is, it is not causally perturbed by the ODE over \(t\in[0,T]\). In such a case, \(A_i^\top \bm{x}_k = 0\), so the conditional variance vanishes and the Gaussian quasi-log-likelihood objective becomes ill-defined. Hence, we introduce a \emph{stabilization constant} $c > 0$ to ensure numerical stability and to make the objective well-defined, that is $(A_i^\top \bm x_k)^2 \to(A_i^\top \bm x_k)^2 +c$.  Moreover, this stabilization facilitates our theoretical analysis about the optimization landscape (Sec.\ref{sec:lemma}). Then, by summing over variables $i\in [p]$ and all time steps $k \in [n]$,  we have the following optimization problem.
\begin{equation}
\label{Eq:A_learning}
\begin{split}
\hat{A} = \arg\min_{A}  
&\frac{1}{2n} \sum_{i=1}^p\sum_{k=0}^{n-1}\Big[\log\big((A_i^\top \bm x(k))^2+c\big) +\frac{r_{i,k}^2}{(A_i^\top \bm x(k))^2+c}\Big] + \lambda\|A\|_1,\\
 & r_{i,k} \coloneqq \frac{x_{i}(k+1)-x_{i} (k) -g_i(t_k, \bm x_k, \bm\gamma) \Delta t_k}{\sqrt{\Delta t_k}},
\end{split}
\end{equation}
where $\lambda > 0$ is a penalty. Moreover, we can see that the diffusion estimator depends on the drift through $r_{i,k}$. Any misspecification of the ODE knowledge would introduce an additive bias term $\Delta t_k e_{i,k}^2$ in $\mathbb{E}[r_{i,k}^2\mid \bm x(k)]$,  which inflates the estimated diffusion magnitude. In later sections, we show both theoretically and empirically that our algorithm remains robust to moderate numerical ODE misspecification. Because the problem reduces to parent selection for each $i\in[p]$, we consider the following simplified formulation:
\begin{equation}
\label{Eq:linear_pa}
\begin{split}
\hat \theta_i =\arg\min_{\theta_i}\frac{1}{2n} \sum_{k=0}^{n-1}
\bigl(\log\big((\theta_i^\top\bm x(k))^2+c\big)
+\frac{r_{i,k}^2}{(\theta_i^\top \bm x(k))^2+c}\bigr)+\lambda\|\theta_i\|_1.
\end{split}
\end{equation}
Here, $\theta_i$ represents the parent vector of node $i$ (i.e., $A_i^\top$). We optimize each row-wise penalized objective using a nonconvex proximal-gradient scheme with monotone backtracking. At each iteration, we take a gradient step on the smooth quasi-likelihood term and then apply the proximal operator of the $\ell_1$ penalty, i.e., soft-thresholding. Because the objective is nonconvex, in implementation we use multiple random initializations and retain the solution with the smallest attained objective value. Since this is a standard approach for solving nonconvex optimization problems (see, e.g., \cite{NIPS2015li}), we defer the convergence discussion to Appendix~\ref{sec:opt_con} to keep the main text concise. A summary of the optimization procedure is given in Algorithm~\ref{alg:picd}.
\begin{algorithm}[ht]
\caption{SCD: Causal Discovery in Stochastic Dynamical Systems}
\label{alg:picd}
\small
\begin{algorithmic}[1]
\REQUIRE Data $\{x_k\}_{k=0}^n$, drift $g_i(t_k,\bm x_k,\bm \gamma)$, penalty $\lambda$, stabilization $c$
\FOR{$i=1,\dots,p$}
    \STATE Compute residuals for each step $k\in [n]$ by following Eq.\eqref{Eq:r_ik}
    \STATE Solve the row-wise penalized objective in Eq.\eqref{Eq:linear_pa} by proximal gradient with backtracking.
    \STATE Store the resulting solution as the $i$th row of $\hat A$.
\ENDFOR
\STATE \textbf{Result:} $\hat A\in \mathbb R^{p\times p}$
\end{algorithmic}
\end{algorithm}
For the nonlinear dictionary case, the same estimation strategy applies after replacing the linear feature vector $x(k)\in\mathbb{R}^p$ by the stacked dictionary feature $\bar{\Phi}(x(k))\in\mathbb{R}^{pJ}$, and replacing the row parameter $\theta_i\in\mathbb{R}^p$ by $\tilde{\theta}_i\in\mathbb{R}^{pJ}$. More specifically, for each node $i\in[p]$, we consider the row-wise objective
\begin{equation}
\hat{\tilde{\theta}}_i = \arg\min_{\tilde{\theta}_i} \frac{1}{2n}\sum_{k=0}^{n-1}\left[\log\!\left((\tilde{\theta}_i^\top \bar{\Phi}(x(k)))^2 + c\right)+\frac{r_{i,k}^2}{(\tilde{\theta}_i^\top \bar{\Phi}(x(k)))^2 + c}\right]+\lambda \|\tilde{\theta}_i\|_1,
\label{eq:nonlinear_rowwise_objective}
\end{equation}
where $r_{i,k}$ is defined as in Eq.\eqref{Eq:r_ik}. 
Thus, the nonlinear model can be handled by the same proximal-gradient scheme with monotone backtracking, with the only modification being the replacement of $x(k)$ by $\bar{\Phi}(x(k))$ in the smooth quasi-likelihood term.

\section{Main Theoretical Results} 

As shown in the previous section, we consider data generated by a continuous-time SDE and observed at discrete times whose distribution may be non-Gaussian due to multiplicative noise. This setting allows us to relax several standard assumptions that are commonly adopted in causal graph inference; however, it also creates fundamental technical challenges for both the formulation and theoretical analysis.

Since SDEs admit a well-developed mathematical theory, we do not impose ad hoc regularity assumptions directly on the observed data. Instead, we derive the needed properties from the underlying stochastic dynamics. A basic requirement is \emph{well-posedness}: the SDE must admit a unique strong solution on the observation horizon \([0,T]\). Without such well-posedness, the data-generating model is not mathematically well defined, and subsequent causal graph recovery lacks a sound foundation. For the theoretical analysis in this section, we work in the regime where the mechanistic drift is affine in the state, or can be locally approximated by one; that is, along the relevant operating trajectory, \(\dot{x}(t)\approx Bx(t)+b\). This includes the case of a truly linear drift, as well as a first-order local linearization of a nonlinear drift. Such a local linearization step is classical in the analysis of dynamical systems (see, e.g., \cite{haermanline}). Under this assumption, the resulting multiplicative SDE can be analyzed using standard tools for linear SDEs. We stress, however, that this restriction is made only to obtain a tractable theory: the proposed algorithm itself is not limited to linear drifts and is also applied in our experiments to nonlinear benchmark systems, including Lotka-Volterra and Lorenz dynamics. We then state the following assumption.
\begin{assumption}
\label{SDE_gen}
The process $X_t$ satisfies the following conditions for all $t \in [0,T]$:
\begin{enumerate}[label=\textbf{B(\arabic*)}, leftmargin=2cm]
    \item \label{ASS:LLC}\textbf{Local Lipschitz continuity.} 
        The physics information function (i.e., the drift)  $\bm{g}(t, \bm x(t), \bm\gamma)$ and each causal fitting function (i.e., diffusion function) $S_A(\bm x(t))$ 
        are locally Lipschitz in $x$, uniformly in $t$. That is, for every $R>0$, 
        there exists a constant $L_R>0$ for all $x,y \in \mathbb{R}^d \ \text{with}\ \|x\|,\|y\|\le R$ such that
        \[
        \|{g}(t, x, \bm\gamma)-{g}(t,  y, \bm\gamma)\| + 
        \|S_A(x)-S_A(y)\|   \le   L_R \|x-y\|.
        \]
    \item \textbf{Linear growth.} \label{ASS:LG}
    There exists a constant $C>0$ such that
    \[
    \|{g}(t, x, \bm\gamma)\|^2 + \|S_A(x)\|^2 \le C(1+\|x\|^2), 
    \quad \forall x\in\mathbb{R}^p,
    \]
    \item \textbf{Initial condition.}  \label{ASS:IC}
    The initial state $X_0$ is square-integrable, i.e. $\mathbb{E}\|X_0\|^2 < \infty$. Moreover, we assume $\mathbb E[X_0X_0^\top] \succ 0$ and $X_0$ is independent of $W$. 
\end{enumerate}
\end{assumption}
Under the linear or locally linearizable drift described above, we obtain the following basic structural property. The proofs of Proposition~\ref{prop: WD} and all remaining propositions and lemmas are provided in the Supplementary Information (see \S\ref{sec: sde_Lemma}).

\begin{proposition}
\label{prop: WD}
Suppose the initial condition, the function $g$, and $S_f$ satisfy the assumptions before. Then our SDE (Linear multiplicative SDE) has a unique solution $(X_t)_{t\in[0,T]}$ such that $\int_0^t \mathbb E\|X_s\|^2 ds< \infty$ for all $t \in [0,T]$. Moreover, there exist a process $\Xi_t$ is invertible a.s. for all $t \in [0,T]$ such that $X_t = \Xi_t X_0$. Hence, we have $\lambda_{min} (\mathbb{E}\hat\Sigma) > 0$ where $\hat\Sigma = \frac{1}{n}\sum_{k=0}^{n-1} X_{t_k}X_{t_k}^\top$. 
\end{proposition}
In the remainder of this section, we develop two technical approaches to analyze the SDE, depending on whether the SDE is stable or unstable. For the stable SDE case, we adopt Lyapunov stability (i.e., second moment stability), a powerful tool in control theory. Importantly, it also yields several useful properties for statistical analysis. For unstable SDEs, the properties induced by stability may fail. Nevertheless, our theoretical analysis can still be carried out under standard high-dimensional assumptions commonly used in Lasso theory. Building on the Theorem~3.6 of \cite{Bierkens2010}, we can establish Lyapunov stability and provide sufficient conditions on the matrices $A$ and $B$. The full derivations of these sufficient conditions are given in \S\ref{sec: sde_state}.

\subsection{Statistical Guarantees under Stable SDEs}
\label{sec: sde}

To obtain explicit stability-based moment bounds, we now specialize to the linear setting.
\begin{equation}
\label{Eq:equiv_form}
    dX_t = B X_t\,dt + S_A(X_t)\,dW_t,
\end{equation}
where \(B\in\mathbb{R}^{p\times p}\) and $S_A(x)=\mathrm{diag}(A_1^\top x,\ldots,A_p^\top x)$. Equivalently, it can be written as $dX_t = B X_t\,dt + \sum_{i=1}^p G_i X_t\,dW_t^{(i)},$ where $G_i := e_i A_i^\top \in \mathbb{R}^{p\times p}$ and $ A_i^\top$ is the $i$-th row of $A$ and $e_i \in \mathbb R^p$. The following proposition gives a sufficient condition for Lyapunov-type stability in this linear multiplicative setting.

\begin{proposition}[Lyapunov stability with sufficient conditions]
\label{thm:gen_lyap_AB}
Consider the linear multiplicative SDE in the form of Eq.\eqref{Eq:equiv_form} with \(B,A\in\mathbb{R}^{p\times p}\). Suppose there exist constants \(m\ge 1\) and \(\omega<0\) such that $\|e^{Bt}\|_2 \le m e^{\omega t}$ for all $t\ge 0$, and $2\omega + m^2 \sum_{i=1}^p \|G_i\|_2^2 < 0.$ Equivalently, since \(G_i=e_iA_i^\top\), it suffices that $2\omega + m^2 \|A\|_F^2 < 0.$ Then, for every symmetric negative definite matrix \(M\prec 0\), there exists a unique symmetric positive definite matrix \(\Sigma\in\mathbb{S}_{++}^p\) such that
\begin{equation}\label{eq:gen_lyap_eq_Sigma}
    B\Sigma + \Sigma B^\top + \sum_{i=1}^p G_i \Sigma G_i^\top = M.
\end{equation}
Moreover, \(\Sigma\) admits the integral representation as $\Sigma = \int_0^\infty e^{Bt} \left( \sum_{i=1}^p G_i \Sigma G_i^\top - M \right) e^{B^\top t}\,dt.$
\end{proposition}
This proposition provides a deterministic sufficient condition on the drift matrix \(B\) and diffusion matrix \(A\) ensuring a Lyapunov-type second-moment bound for the linear theory branch. We use this condition only to derive the moment and covariance controls needed in the subsequent statistical analysis. In particular, it is weaker than requiring a stationary distribution and serves here as a tractable stability condition for the finite-horizon, nonstationary setting.

\begin{proposition}
\label{Ass:xbound}
Let the SDE satisfy the Assumption~\ref{SDE_gen} and $T<\infty$ (i.e. locally Lipschitz, linear growth condition, and the square-integrable initial condition). Then, the SDE defined in Eq.\eqref{GSDE}, admits a unique strong solution with continuous sample paths. If the conditions in Proposition~\ref{thm:gen_lyap_AB} hold, there exists $C > 0$ such that $\mathbb P\Big(\sup_{0\le s\le T}\|X_s\|\ge K\Big)
\le \frac{C\cdot\mathbb E\|X_0\|^2}{K^2}.$
\end{proposition}

\begin{proposition}
\label{prop: cov_decay}
Let $(X_t)_{t\ge 0}$ be the unique strong solution to the linear SDE defined in Eq.\ref{Eq:equiv_form}.  If the SDE satisfies the conditions as stated in Proposition~\ref{thm:gen_lyap_AB}, we have $\sup_{t\ge 0}\mathbb{E}\|X_t\|_2^2 < \infty$. For $t\ge 0$ and $h\ge 0$, define the cross-time covariance matrix
\begin{equation}
\label{eq:cross_cov_def}
C(t,h) := \mathrm{Cov}(X_t,X_{t+h})
:= \mathbb{E}\left[(X_t-\mathbb{E}X_t)(X_{t+h}-\mathbb{E}X_{t+h})^\top\right].
\end{equation}
Then there exist constants $C_0,c_0>0$ such that $\|C(t,h)\|_2 \le C_0 e^{-c_0 h}$ for all $t\ge 0, h\ge 0.$
\end{proposition}
After establishing these properties, we now turn to our main statistical guarantees. In particular, we establish finite-sample error bounds for stationary points of the empirical objective under the SDE-induced random design, and we characterize the resulting implications for causal graph recovery. Given the optimization problem in Eq.\eqref{Eq:linear_pa}, one may ask whether it is possible to identify the true structure by solving it. In high-dimensional statistics, a common structural condition used to ensure identifiability is sparsity, as formalized in Assumption~\ref{sparse}. Without this assumption, minimizing Eq.\eqref{Eq:linear_pa} generally does not yield identifiable or statistically tractable parent recovery. Even under sparsity, exact parent selection is not guaranteed in the presence of strong correlations among variables. In particular, if some non-parent variables are highly correlated with the true parent set, they can explain the same variation in the data and lead to spurious edge selection under $\ell_1$ regularization. To rule out these extreme cases and enable support recovery, we consider a standard incoherence condition, which is widely used in the analysis of Lasso and related graph recovery methods \citep{Meinshausen2006,jalali11a2011}. The incoherence condition is typically imposed on the Fisher information matrix (i.e., the Hessian of the population loss evaluated at the ground truth $\theta_i^*$). We first define the (population) Fisher-type information matrix as $Q^* \coloneqq \mathbb{E}\left[\nabla^2_{\theta}\mathcal{L}(\theta_i^\ast; X)\right]$. For index sets $A, B\subseteq V$, we denote the corresponding sub-block by $Q^*_{AB} \coloneqq \mathbb{E}\left[\nabla_A \nabla_B\mathcal{L}(\theta_i^\ast; X)\right]$, where $\nabla_A$ denotes the gradient restricted to coordinates in $A$. Let $V_i \coloneqq \mathrm{pa}(i)$ be the parent set of node $i$, and let $V_i^c \coloneqq V\setminus V_i$ denote its complement. We then impose the following condition.

\begin{assumption}[Incoherence]
\label{ass: Incoh}
There exists $\alpha \in (0,1]$ such that $\|Q^*_{V^c_iV_i}(Q^*_{V_iV_i})^{-1} \|_{\infty} \le 1-\alpha$.
\end{assumption}
Under Lyapunov stability, incoherence, and the additional regularity conditions stated below, the nodewise parent-selection problem becomes statistically analyzable. Theorem~\ref{Thm: LASSO} shows that, for any fixed node $i$, stationary points of Eq.\eqref{Eq:linear_pa} satisfying the required regularity conditions recover a subset of the true parent set, and exact support recovery follows under an additional minimum signal strength condition.

\begin{theorem}[Support Recovery Under Stable SDEs]
\label{Thm: LASSO}

Fix a node $i\in[p]$, and consider the row-wise estimator in Eq.\eqref{Eq:linear_pa} under the linear stable SDE model in Eq.\eqref{Eq:equiv_form} assuming that Assumption~\ref{SDE_gen} holds. Suppose Assumptions~\ref{sparse} and \ref{ass: Incoh} hold for the true parameter $\theta_i^*$ with parent set $pa(i)$. If the regularization satisfies
    \begin{equation}
\lambda_n \ge \frac{4(2-\alpha)}{\alpha} (\frac{8K}{\sqrt c} \sqrt{\frac{2\log(p)}{n}} )
    \end{equation}
Then, there exist positive constants $c_1$, $c_2$, such that if the number of samples $n$ scales as $n \gtrsim C\tilde{c}^2\log (p)|pa(i)^3|,$ where $\tilde c = \frac{3^2(2-\alpha)^2 K}{\alpha_1^2 c(\alpha_1^2+256\tau^2 |pa(i)|^2)}$ and universal constant $C > 0$. Then, we guarantee the following statements hold with probability at least $1-c_1\exp(-c_2\lambda_n^2n)$: 
\begin{enumerate}
    \item Any stationary point $\hat\theta_i$ of Eq.\eqref{Eq:linear_pa} satisfying the KKT conditions and strict dual feasibility, and the empirical loss satisfies the LRSC condition under the choice of $c$ specified in Lemma~\ref{lem:RSC}, yields an estimated parent set contained in the true parent set, such that $\hat{\mathcal N}(i) \subseteq \mathrm{pa}(i)$ and $\hat{\mathcal N}(i)\coloneqq\{j\neq i:\hat\theta_{ij}\neq 0\}$ is the estimated parent set.
   
    \item Moreover, there exists a $\tau > 0$ defined by Lemma~\ref{lem:RSC}, and define $\theta_{\min} \coloneqq \min_{j\in\mathrm{pa}(i)} |\theta^*_{ij}| $ and $s \coloneqq \max_{i} |\mathrm{pa}(i)|$. If the minimum signal strength condition
    \begin{equation}
         \theta_{\min}\ge \frac{3\sqrt{s}}{\alpha_1-16\tau s}\lambda_n
    \end{equation}
holds and provided that $\alpha_1-16\tau s > 0 $ holds as shown in Lemma~\ref{lem:err_rsc}, all true parents are included (i.e. no false negatives), that is, $ \mathrm{pa}(i)\subseteq \hat{\mathcal N}(i)$, and hence $\hat{\mathcal N}(i) = \mathrm{pa}(i).$
\end{enumerate}
\end{theorem}
The proof sketch for Theorem~\ref{Thm: LASSO} is provided below, and all lemmas used in the proof are included in \S\ref{sec:lemma} and the proof of lemmas also provide in \S\ref{sec: Lemma}.

Proving the main theorem is challenging due to several obstacles. First, the objective function is inherently non-convex. Fortunately, despite global non-convexity, we find that the objective satisfies a local restricted strong convexity (LSRC) property with the help of the stabilization constant $c$, which enables us to establish stability and convergence guarantees. Moreover, multiplicative noise in the SDE yields a non-stationary trajectory, challenging methods that assume stationarity. Nevertheless, Lyapunov stability implies useful structure; specifically, Lyapunov stability provides a stability-induced property that offers reusable prior knowledge for learning from such data.
\begin{proof}
\textit{For the first statement}, we prove the first statement via a primal-dual witness technique using the KKT conditions. The first step is to construct the primal candidate using a restricted oracle subproblem by restricting non-parent coordinates to zero such that 
\begin{equation}
    \hat \theta_{i(r)} \coloneqq \arg\min_{\{\theta_{ik} = 0 \mid k\notin pa(i)\}} \big\{ \mathcal L(\theta_i)+\lambda_n\|\theta_i\|_1\big\}, 
\end{equation}
where $\hat \theta_{i(r)}$ is the optimal solution of the oracle subproblem. Then, we need to construct the dual candidate (i.e., $(\hat Z_{i(r)})_{pa(i)}$) by using Lemma~\ref{KKT} for all $j\in pa(i)$ such that  
\begin{equation}
    (\hat Z_{i(r)})_j \in \partial |(\hat\theta_{i(r)})_j|
    =
    \begin{cases}
    \operatorname{sign}((\hat\theta_{i(r)})_j), & (\hat\theta_{i(r)})_j\neq 0,\\
    [-1,1], & (\hat\theta_{i(r)})_j=0.
    \end{cases}
\end{equation}
By the restricted KKT condition, the pair
\((\hat\theta_{i(r)},\hat Z_{i(r)})\) satisfies the stationarity condition on \(pa(i)\). It remains to verify the off-support dual feasibility condition, namely \(|(\hat Z_{i(r)})_k|<1\) for all \(k\notin pa(i)\). By Lemma~\ref{lem: PDW}, the PDW-constructed pair \((\hat\theta_{i(r)},\hat Z_{i(r)})\) satisfies strict dual feasibility with high probability. Hence the PDW estimator is a full-problem stationary point with no false positives. Moreover, any stationary point $\hat\theta_i$ of Eq.\eqref{Eq:linear_pa}, satisfying the stated strict dual feasibility condition, must satisfies $\hat\theta_{ik}=0$ for all $k\notin pa(i)$ (no false positives). This is true since, if $\hat\theta_{ik}\neq 0$ for some $k\notin pa(i)$, then the subgradient characterization of the $\ell_1$ norm gives $(\hat z_i)_k=\operatorname{sign}(\hat\theta_{ik})\in\{-1,1\}$, which contradicts strict dual feasibility $(|(\hat z_i)_k|<1)$.

\textit{For the second statement}, it remains to show that there are no false negatives. Given the LRSC from Lemma~\ref{lem:RSC}, and boundedness of population estimator from Lemma~\ref{lem:prop} holds for $n$ sufficiently large, we can have the error bound by using Lemma~\ref{lem:h_lips} and Lemma~\ref{lem:err_rsc}. Then, we define $\theta_{\min} \coloneqq \min_{j\in\mathrm{pa}(i)} |\theta^*_{ij}|.$ To show the correct support recovery, it suffices to show that $  \|\hat \theta_{i} - \theta^*_{i}\|_\infty \le \frac{\theta_{\min}}{2}$. Then, we can show it by using the fact that $\|\cdot\|_\infty \le\|\cdot\|_2$ and Lemma~\ref{lem:err_rsc} such that $\|\hat \theta_{i} - \theta^*_{i}\|_\infty\le\frac{3\lambda_n\sqrt{s}}{2(\alpha_1-16\tau s)}$. We then have that $\frac{2}{\theta_{\min}}\|\hat \theta_{i} -  \theta^*_{i}\|_\infty\le\frac{2}{\theta_{\min}}\frac{3\lambda_n\sqrt{s}}{2(\alpha_1-16\tau s)}\le 1$, and then $\theta_{\min}\ge\frac{3\lambda_n\sqrt{s}}{\alpha_1-16\tau s}.$ Therefore, every true parent coefficient remains nonzero on $\mathrm{pa}(i)$. Combined with part~(1), which already established that $\hat{\theta}_{ik}=0$ for all $k\notin \mathrm{pa}(i)$, we conclude that $\operatorname{supp}(\hat{\theta}_i)=\mathrm{pa}(i)$. 
\end{proof}
\vspace{-1em}
\begin{remark}
The minimum signal strength condition is commonly referred to as a beta-min condition in the high-dimensional LASSO literature (see, e.g., \cite{Meinshausen2006,jalali11a2011}); here, it ensures that each true parent effect is sufficiently separated from the estimation error level. Moreover, Theorem~\ref{Thm: LASSO} is stated for stationary points satisfying the KKT conditions, LRSC, and strict dual feasibility. Although this may appear conditional, uniqueness can be recovered under a stronger sample-size regime. This is consistent with the distinction commonly made in high-dimensional nonconvex estimation: restricted curvature is sufficient for support recovery, whereas uniqueness of stationary points requires stronger full-curvature conditions (see, e.g., \cite{Loh2017supprot}). Specifically, if the sample size is further enlarged so that $\alpha_1 - p\tau_n >0,$ where $\tau_n$ is from Lemma~\ref{lem:RSC} such that $\tau_n = \tau r_n$, then the LRSC bound in Lemma~\ref{lem:RSC} implies ordinary local strong convexity on the LRSC region. Indeed, since $\|\Delta\|_1^2 \le p\|\Delta\|_2^2,$ we have, for all directions \(\Delta\), $\Delta^\top \nabla^2\ell_i(\theta)\Delta \ge (\alpha_1-p\tau_n)\|\Delta\|_2^2.$ Consequently, \(\ell_i(\theta)+\lambda_n\|\theta\|_1\) is locally strongly convex on this region, because the \(\ell_1\) penalty is convex. Therefore, there is at most one stationary point in the LRSC region. Since the PDW solution \(\hat\theta_{i(r)}\) lies in this region and satisfies the full KKT conditions with high probability by Lemma~\ref{lem: PDW}, it coincides with the unique stationary point in this region. Since \(\tau_n=O(\sqrt{\log p/n})\), the condition \(\alpha_1-p\tau_n>0\) is satisfied when \(n\gtrsim p^2\log p\) with a sufficiently large constant.
\end{remark}

% Theorem~\ref{Thm: LASSO} is stated for a fixed node \(i\). Applying the theorem row-wise to all \(i\in[p]\) and taking a union bound yields full graph support recovery, with the corresponding adjustment of the overall failure probability. 

\subsection{Statistical Guarantees under Unstable SDEs}
\label{Sec: unsta}

As shown above, Lyapunov stability provides two central ingredients for our analysis: (i) \emph{high-probability boundedness} of the SDE solution over $[0,T]$ (Proposition~\ref{Ass:xbound}), and (ii) \emph{decay of temporal dependence} (e.g., covariance or mixing decay; Proposition~\ref{prop: cov_decay}), which together yield the concentration bounds required for statistical recovery.  We now turn to the unstable regime, where Lyapunov stability may be weakened or fail. A natural question is how to address the resulting technical challenges without these two properties? To obtain a finite-horizon guarantee in this setting, we begin with a standard moment assumption on the initial condition.

\begin{assumption}[Finite $q$-th moment of the initial condition]
\label{ass:finiteq_inital}
There exists $q \ge 2$ such that $\mathbb{E}\,\|X_0\|^{q} < \infty.$ Equivalently, $X_0 \in L^{q}$.
\end{assumption}
Assumption~\ref{ass:finiteq_inital} is standard for finite-horizon moment estimates in SDE theory (see e.g., \cite{oksendal2010stochastic} and \cite{Higham2002}). Then, under the Lipschitz and linear growth conditions in Assumption~\ref{SDE_gen}, together with Assumption~\ref{ass:finiteq_inital}, the following finite-horizon  bound holds.

\begin{proposition}
\label{ass:localization}
Let $X_t\in\mathbb{R}^p$ solve the SDE  defined in Eq.\eqref{Eq:equiv_form} for $t\in[0,T]$, and satisfy the Assumption~\ref{SDE_gen} and Assumption~\ref{ass:finiteq_inital}. Then, there exists a constant $C=C(q,L,p)>0$ such that
\begin{equation*}
    \mathbb{E}\Big[\sup_{t\in[0,T]}\|X_t\|^q\Big] \le C\,e^{CT}\Bigl(1+\mathbb{E}\|X_0\|^q\Bigr).
\end{equation*}
Consequently, we have $ \mathbb{P}\left(\sup_{t\in[0,T]}\|X_t\|\le K(T,\delta)\right)\ge 1-\delta$ for any $\delta\in(0,1)$, with $K(T,\delta)
:=\Biggl(\frac{C\,e^{CT}\bigl(1+\mathbb{E}\|X_0\|^q\bigr)}{\delta}\Biggr)^{1/q}$. For any sampling times $0=t_0<\cdots<t_{n-1}\le T$, $\mathbb{P}\left(\max_{0\le k\le n-1}\|X(t_k)\|\le K(T,\delta)\right)\ge 1-\delta$, and also $\max_k\|X(t_k)\|_\infty\le K(T,\delta)$ holds with probability at least $1-\delta$ since $\|x\|_\infty\le \|x\|$.
\end{proposition}

\begin{remark}
\label{remark: sec moment}
Proposition~\ref{ass:localization} implies a standard finite-horizon localization event. For any $\delta\in(0,1)$, let $K=K(T,\delta)$ be as in Proposition~\ref{ass:localization} and define $\mathcal E_K := \{\max_{0\le k\le n-1}\|X(t_k)\|_\infty \le K\}$. Since $\max_k \|X(t_k)\|_\infty \le \sup_{t\in[0,T]}\|X_t\|_\infty$ and $\|x\|_\infty\le \|x\|_2$, we have $\mathbb P(\mathcal E_K)\ge 1-\delta$. We will work on $\mathcal E_K$ to enable subsequent concentration arguments. Moreover, for technical convenience in our concentration analysis on the finite horizon $[0,T]$, we work with a stopped version of the process. Define the stopping time $\tau_K:=\inf\{k:\|X_k\|_\infty>K\}\wedge n,$ and the stopped variables $\tilde X_k:=X_{k\wedge\tau_K}$ and $\tilde Z_k^{(ab)}:=\tilde X_{k,a}\tilde X_{k,b}$. Then $|\tilde Z_k^{(ab)}|\le K^2$ holds almost surely for all $k$, and hence $\bigl\|\mathbb E[\tilde Z_k\mid\mathcal F_{k-1}]\bigr\|_\infty \le K^2$ a.s. Moreover, on the event $\mathcal E_K=\{\max_{0\le k\le n-1}\|X_k\|_\infty\le K\} =\{\tau_K=n\}$, we have $\tilde Z_k=Z_k$ for all $k\le n-1$. Therefore, any high-probability bound established for $\tilde Z_k$ transfers to $Z_k$ up to an additive $\mathbb P(\mathcal E_K^c)$ term.
\end{remark}
To proceed with support recovery, however, boundedness alone is not sufficient, as we also need a lower-curvature condition for the predictable Gram matrix along sparse directions. We therefore impose the following restricted-eigenvalue assumption.

\begin{assumption}
\label{ass:pred_curvature}
Let $s:=|\mathrm{supp}(\theta^\star)|$ and a constant $c_0> 0$. Define $\mathcal R(s;c_0):=\Big\{\Delta\in\mathbb R^p:\ \|\Delta\|_1\le c_0\sqrt{s}\,\|\Delta\|_2\Big\}.$ There exists a deterministic constant $m>0$ such that, with high probability 
\begin{equation*}
\inf_{\Delta\in\mathcal R(s;c_0)\setminus\{0\}}
\frac{\Delta^\top\tilde{\Sigma}\Delta}{\|\Delta\|_2^2}\ge\ m, \quad \text{where}\quad\tilde{\Sigma}:=\frac1n\sum_{k=1}^{n}\mathbb E[X_kX_k^\top\mid\mathcal F_{k-1}].
\end{equation*}
\end{assumption}

\begin{remark}
    Assumption~\ref{ass:pred_curvature} is a restricted-eigenvalue condition on the predictable Gram matrix. Unlike the global requirement $\lambda_{\min}(\tilde{\Sigma})>0$, it only enforces a uniform curvature lower bound along near-sparse directions $\mathcal R(s;c_0)=\{\Delta:\|\Delta\|_1\le c_0\sqrt{s}\|\Delta\|_2\}$. This assumption plays the role of a persistent-excitation condition for support recovery in dependent, pathwise settings and is standard in high-dimensional sparse analysis (see, e.g., \cite{Bickel_2009}.
\end{remark} 
Under Assumption~\ref{ass:finiteq_inital}, together with the corresponding property, and Assumption~\ref{ass:pred_curvature}, we obtain the following theorem for unstable multiplicative SDEs. The required technical details are summarized in \S\ref{sec:lemma}, and all proofs are collected in \S\ref{sec: Lemma}

\begin{theorem}[Support Recovery for Unstable Multiplicative SDE]
\label{Thm:lasso_uns}
Let $X_t \in \mathbb{R}^p$ solve the multiplicative linear SDE in Eq.\eqref{Eq:equiv_form} on the finite horizon $[0,T]$, and fix a node $i \in [p]$. Let $S_i := pa(i) = supp(\theta_i^*)$ and $s_i := |S_i|$. Assume that Assumptions~\ref{sparse}, \ref{SDE_gen}, \ref{ass: Incoh}, \ref{ass:finiteq_inital}, and \ref{ass:pred_curvature} hold, and let $K$ be defined as in Proposition~\ref{ass:localization} so that, with probability at least $1-\delta$, $\max_{0\le k\le n-1}\|X(t_k)\|_\infty \le K$. Then there exist positive constants $c_1,c_2,C$ such that whenever
\begin{equation}
    \lambda_n \ge C \sqrt{\frac{\log p}{n}} \qquad\text{and}\qquad n \gtrsim s_i^3 \log p,
\end{equation}
the following statements hold with probability at least $1-c_1 e^{-c_2 n \lambda_n^2}-\delta$. 

\begin{enumerate}
\item By Proposition~\ref{ass:localization}, any stationary point $\hat\theta_i$ of Eq.\eqref{Eq:linear_pa} satisfying the stationarity conditions in Lemma~\ref{KKT}, strict dual feasibility, and the LRSC condition in Lemma~\ref{lem:RSC_unstable} satisfies $\widehat N(i) := \{j\neq i : \hat\theta_{ij}\neq 0\} \subseteq pa(i).$
\item Moreover, let \(\theta_{\min,i}:=\min_{j\in\operatorname{pa}(i)}|\theta^*_{ij}|\).
If the minimum signal strength condition \(\theta_{\min,i}\ge 3\sqrt{s_i}\lambda_n/(\alpha_1-16\tau s_i)\)
holds, then \(\operatorname{supp}(\hat\theta_i)=\operatorname{pa}(i)\).
% Moreover, let $\theta_{\min,i} := \min_{j\in pa(i)} |\theta_{ij}^*|.$ We have $supp(\hat\theta_i)=pa(i)$ if the following minimum signal strength condition holds
% \[
% \theta_{\min,i} \ge \frac{3\sqrt{s_i}}{\alpha_1-16\tau s_i}\lambda_n
% \]
\end{enumerate}
\end{theorem}
\begin{proof}
The argument follows the proof of Theorem~\ref{Thm: LASSO}, with the stability-based ingredients replaced by the finite-horizon localization as stated in Proposition~\ref{ass:localization} and restricted-curvature bounds as stated in Assumption~\ref{ass:pred_curvature}. In particular, Proposition~\ref{ass:localization} is proved by using Assumption~\ref{ass:finiteq_inital}, and yields the event $\max_{0\le k\le n-1}\|X(t_k)\|_\infty\le K$ with high probability.  

By Proposition~\ref{ass:localization}, the process remains in a finite-horizon localized event with high probability. On this event, Lemma~\ref{lem:prop} controls the score and Fisher-information terms, Lemma~\ref{lem:RSC_unstable} yields the LRSC condition, and Lemma~\ref{lem: PDW} shows that the constructed primal-dual pair satisfies the stationary-point conditions characterized in Lemma~\ref{KKT}. 

Moreover, the same dual-feasibility argument as in the proof of Theorem~\ref{Thm: LASSO} yields strict dual feasibility on the localized event. Therefore, the same primal-dual witness argument as in the stable case implies $\hat N(i)\subseteq pa(i)$. Under the minimum signal strength condition, the usual support separation argument yields $supp(\hat\theta_i)=pa(i)$.
\end{proof}
\vspace{-1em}
\subsection{ODE Misspecification}
\label{Sec: ODE miss}

Although our study assumes that part of the physical information is known, as noted earlier, our algorithm remains effective even when this component is misspecified. This situation arises naturally in real-world applications. A representative example is SIR modeling for infectious disease dynamics: while the model structure may be known, the inferred parameters often provide only a coarse characterization of the underlying process. Therefore, in this section, we theoretically investigate how misspecification in the physical component affects our method, as well as the role of the regularization term.

As seen from Eq.\eqref{Eq:A_learning}, the ODE information enters the formulation through $r_{i,k}$, so any misspecification in that component directly induces bias in $r_{i,k}$. To model this effect, we introduce an additive perturbation $|\delta|>0$ in the drift term at each time step, so that
\begin{equation}
\label{Eq:ODE_miss}
     r_{i,k} = \frac{X_{i,k+1}-X_{i,k}-(B_i^\top X_k+b_i +\delta)\Delta t}{\sqrt{\Delta t}}  = \theta^\top_{i}X_{k}z_{k+1} - \delta\sqrt{\Delta t}, \quad z_{k+1} \sim \mathcal N(0,1) 
\end{equation} 
Then, we have the following theorem. We emphasize that this result is established under the stable SDE regime. The unstable regime can be treated in a similar manner. All technical details are provided in Sec.~\ref{Sec:TD_odemiss}, and all proofs are deferred to Sec.~\ref{sec: ODE_miss_Lemma}.

\begin{theorem}[Support recovery under ODE misspecification]
\label{Thm:Lasso_miss}
Let $\tilde{\ell}_i^{(\delta)}(\theta)$ be the misspecified empirical loss defined by Eq.\eqref{Eq:ODE_miss} and assume the graphical model in Eq.\eqref{Eq:linear_pa} with multiplicative linear SDE in Eq.\eqref{GSDE}. Suppose Lyapunov stability and Assumptions~\ref{sparse}-\ref{ass: Incoh} hold. Fix a node $i\in[p]$, let $s_i = |pa(i)|$ and $s := \max_{i\in[p]} |pa(i)|$. We define
\begin{equation*}
\begin{split}
  & A_\delta := \left(\frac{3\sqrt3}{8}+C\right)\frac{K}{\sqrt c}+ \frac{3\sqrt3\,K}{8c^{3/2}}\delta^2\Delta t + C\frac{K}{c}\,\delta\sqrt{\Delta t}, \\ &\alpha_\delta
:= m\left( \frac{2}{9c} - \frac{B^2K^2+\delta^2\Delta t}{c^2} \right) 
,\qquad\bar D_{\max}^{(\delta)} := \frac{K^3}{2} \left( \frac{24}{c^{3/2}} + \frac{48(R+\delta\sqrt{\Delta t})^2}{c^{5/2}} \right),
\end{split}
\end{equation*}
where $A_\delta$ is from Lemma~\ref{lemma:ODE_miss gradient}; $\alpha_\delta$ is from Lemma~\ref{lemma:ODE_miss_lrsc}; and $\bar D_{\max}^{(\delta)}$ is from Lemma~\ref{lemma:ODE_miss_h_lip}. If the regularization satisfies
\begin{equation}
\lambda_n \ge \frac{4(2-\alpha)}{\alpha} A_\delta \sqrt{\frac{\log p}{n}},
\end{equation}
and the sample size satisfies $n \ge C\,\tilde c_\delta\, s^3 \log p,$ where $\tilde c_\delta := \max\left\{ \frac{H_\delta^2}{\alpha_\delta^2},
\left( \frac{400(2-\alpha)^2\,\bar D_{\max}^{(\delta)}\,A_\delta}{\alpha^2\alpha_\delta^2} \right)^2 \right\}$ and $H_\delta :=\left(\frac{2}{9c}-\frac{B^2K^2+\delta^2\Delta t}{c^2}\right)+\frac{C(1+|\delta|\sqrt{\Delta t})}{c^2}$,
the no-false-positive conclusion holds. If, in addition, the similar minimum signal strength condition as in Theorem~\ref{Thm: LASSO} holds, then the exact-support conclusion holds. Moreover, if the misspecification level satisfies
\begin{equation}
\label{eq:delta_small_range}
|\delta| \le \delta_0 := \min\left\{\sqrt{\frac{c/9-B^2K^2}{\Delta t}}, \frac{\sqrt{R^2+c/2}-R}{\sqrt{\Delta t}}, \sqrt{\frac{c}{\Delta t}} \right\},
\end{equation}
all conditions reduce to the same functional form as in the well-specified case of Theorem~\ref{Thm: LASSO}.
\end{theorem}

\begin{proof}
To show these quantities, we establish the following expressions, which are used in the PDW argument

\begin{equation}
 (i):\frac{1}{\lambda_n}\|\nabla \tilde{\ell}_i^{(\delta)}(\theta_i^*)\|_\infty \le \frac{\alpha}{4(2-\alpha)}, \qquad(ii):\frac{1}{\lambda_n}\|R_{pa(i)}^{(\delta)}\|_\infty\le \frac{\alpha}{4(2-\alpha)},
\end{equation}
where $R_{pa(i)}^{(\delta)}:= \bigl(\nabla^2 \tilde \ell_i^{(\delta)}(\theta_i^*) - \nabla^2 \tilde \ell_i^{(\delta)}(\tilde\theta_i)\bigr)\Delta$ and \(\Delta:=\hat\theta_i-\theta_i^*\). For more details about $R_{pa(i)}^{(\delta)}$, see Lemma~\ref{lem: PDW}. From Lemma~\ref{lemma:ODE_miss gradient}, we know the gradient's upper bound, and we can directly obtain the following sufficient condition to obtain (i).

\begin{equation}
\lambda_n \ge \frac{4(2-\alpha)}{\alpha}\left(\frac{3\sqrt3}{8}+C\right)\frac{K}{\sqrt c}+ \frac{3\sqrt3\,K}{8c^{3/2}}\delta^2\Delta t + C\frac{K}{c}\,\delta\sqrt{\Delta t} \sqrt{\frac{\log p}{n}}.
\end{equation}
The detailed derivation is given in Lemma~\ref{cor:lambda_misspec}. By the above inequality and the LRSC condition in Lemma~\ref{lemma:ODE_miss_lrsc}, the same cone condition used in Lemma~\ref{lem:err_rsc} is still hold. Thus, the corresponding $\ell_1$ and $\ell_2$ bound still hold. Then, by Lemma~\ref{lemma:ODE_miss_h_lip} and the fact that $\|\cdot\|_\infty \le\|\cdot\|_2$, we have $\|R_{pa(i)}^{(\delta)}\|_\infty \le \bar D_{\max}^{(\delta)}\|\Delta\|_1\|\Delta\|_2.$ Substituting the above \(\ell_1,\ell_2\) bounds yields a sufficient upper bound on \(\lambda_n\) ensuring (ii). The full procedure is attached in Lemma~\ref{cor:lambda_misspec}. Combining this upper bound with the lower bound from (i) gives the sample size condition \(n \ge C\tilde c_\delta s^3\log p\). Therefore, the primal-dual witness argument proceeds exactly as in Theorem~\ref{Thm: LASSO} and yields no false positives, yielding $\hat N(i)\subseteq pa(i)$. Finally, by imposing a similar minimum signal strength condition, every true parent is retained, yielding $\hat N(i)=pa(i)$. Moreover, if \(\delta\le \delta_0\), then the quantities \(A_\delta\), \(\alpha_\delta\), and \(\bar D_{\max}^{(\delta)}\) reduce to the same functional form as in the well-specified case, so the requirements on \(\lambda_n\) and \(n\) recover the same functional dependence as in Theorem~\ref{Thm: LASSO}, up to universal constants.
\end{proof}
\vspace{-1em}
\subsection{The Role of c-stabilization}
\label{Sec:c_role}
In the previous section, we introduced the stabilization parameter \(c\) into the Gaussian quasi-likelihood to ensure that the loss is well posed without requiring additional assumptions on the diffusion matrix. We then established Lasso support recovery for the population target induced by the \(c\)-stabilized loss. From a formal statistical perspective, the stabilized optimization output should be viewed as a surrogate for the corresponding population-level optimizer. We did not state this interpretation earlier, for clarity of exposition. 

This raises the central theoretical question of characterizing \(c\) for which the support of the original population target \(\theta^*\) coincides with that of the surrogate population target \(\theta^{\ast,c}\), namely, $\operatorname{supp}(\theta^*)=\operatorname{supp}(\theta^{*,c})$. Hence, in this section, we first show that support recovery for the surrogate target can be transferred to the original population target through a bridging argument. We next further analyze the role of $c$. A key point is that $c$ is a genuine trade-off parameter: it must be large enough to guarantee well-posedness of the loss and to make the RSC analysis valid, but it cannot be chosen arbitrarily large, since the effective RSC margin appears in the denominator of the recovery bound. Consequently, the true parameter signal must be sufficiently strong to avoid excessive shrinkage.

We first define  $\theta_S^*:=\theta_{i,S_i}^\star\in\mathbb R^{s_i}$ for $i\in\{1,\dots,p\}$ and $S_i = pa(i)$. Then, we first introduce the following formulations for the restricted expected empirical risk and its Jacobian with respect to $\vartheta\in\mathbb R^{s_i}$, where $s_i = |S_i|$. For \(\vartheta\in\mathbb R^{s_i}\), define $z_{i,k}(\vartheta):=\vartheta^\top X_{k,S_i}$ for $ k=0,\dots,n-1$. Then the restricted expected empirical risk can be written as
\begin{equation*}
    R_{n,i,c}^{(S_i)}(\vartheta) = \frac{1}{2n}\sum_{k=0}^{n-1} \mathbb E\!\left[\log\!\big(z_{i,k}(\vartheta)^2+c\big) + \frac{r_{i,k}^2}{z_{i,k}(\vartheta)^2+c} \right],
\end{equation*}
where $r_{i,k}$ is defined as in Eq.\eqref{Eq:A_learning}. We then define the gradient $F_i(\vartheta,c)$ as 
\begin{equation*}
    F_i(\vartheta,c):=\nabla R_{n,i,c}^{(S_i)}(\vartheta) = \frac{1}{n}\sum_{k=0}^{n-1} \mathbb E\!\left[ X_{k,S_i}\, z_{i,k}(\vartheta) \frac{z_{i,k}(\vartheta)^2+c-r_{i,k}^2} {\big(z_{i,k}(\vartheta)^2+c\big)^2} \right].
\end{equation*}
Using $\mathbb E[r_{i,k}^2\mid X_k]=(\theta_{i,S_i}^{*\top}X_{k,S_i})^2 :=(Z_{i,k}^\star)^2$, we obtain 
\begin{equation}
\label{Eq:pop_F_i}
    F_i(\theta_{i,S_i}^*,c) = \nabla R_{n,i,c}^{(S_i)}(\theta_{i,S_i}^*) = \frac{c}{n}\sum_{k=0}^{n-1} \mathbb E\left[ \frac{X_{k,S_i}Z_{i,k}^*}{\big((Z_{i,k}^*)^2+c\big)^2}\right].
\end{equation}
Moreover, we can define its Jacobian as $H_i(\vartheta,c):=D_\vartheta F_i(\vartheta,c) =\nabla^2 R_{n,i,c}^{(S_i)}(\vartheta)$. Since the derivation is analogous to Lemma~\ref{lem:prop}, we omit the full formulation from the main text for readability. We now impose the following assumption over the $\theta_S^*$.

\begin{assumption}[Local restricted regularity at $c=0$]
\label{ass:local_restricted_regularity}
Assume there exists an open neighborhood $ U_i\subset \mathbb R^{s_i}$ of  $\theta_{S}^*$ and an open interval \(I_i\subset\mathbb R\) with \(0\in I_i\) such that:
\begin{enumerate}
    \item $F_i(\vartheta,c)$ is continuously differentiable (e.g., $C^1$) on $U_i\times I_i$.
    \item $ F_i(\theta_S^*,0)=0$.
    \item The restricted Hessian at $(\theta_S^*,0)$ is invertible, that is $ H_i(\theta_S^*,0)=D_\vartheta F_i(\theta_S^*,0)$  is nonsingular.
\end{enumerate}
\end{assumption}
To avoid possible confusion, we briefly clarify each part of Assumption~\ref{ass:local_restricted_regularity}. These conditions are assumptions on the restricted population quantities around the true restricted parameter $\theta^*_{i,S_i}$. The first condition is a local regularity assumption on the restricted population score map around $\theta^*_{i,S_i}$. The second condition states that the true restricted parameter is a root of the restricted population score at $c=0$, namely $F_i(\theta^*_{i,S_i},0)=0$. Under the stated model, whenever the expression is well defined, this follows directly from Eq.\eqref{Eq:pop_F_i}. The third condition is a standard local nondegeneracy assumption on the true restricted parameter, requiring the restricted Hessian at $(\theta^*_{i,S_i},0)$ to be nonsingular. Then, we have the following corollary, which establishes a sufficient condition for bridging support recovery from the surrogate population target to the original population target.
\begin{corollary}[Support recovery for the population parameter]
\label{cor:exact-support-original}
Under Assumption~\ref{ass:local_restricted_regularity}, fix $c\in[0,c_{0,i}]$ for $c_{0,i} > 0$, and define $\bar\theta_{i,c}^*:=\iota_{S_i}(\vartheta_{i,c}^*)\in\mathbb R^p$. From Theorem~\ref{Thm: LASSO}, we know that  $\operatorname{supp}(\hat\theta_i)=\operatorname{supp}(\bar\theta_{i,c}^*)$ when $\beta_{\min,i}^{(c)}>T_{n,i}$. Let $\beta_{\min,i}\coloneqq \min_{j\in S_i} |[\theta_i^*]_j|$, and if it satisfies 
\[
\beta_{\min,i}>T_{n,i}+C_{H,i}\Delta_{i,c},
\]
where $C_{H,i}$ is a finite constant and $\Delta_{i,c} =\|F_i(\theta_S^*,c)\|_\infty$, we have that $\operatorname{supp}(\hat\theta_i)=S_i=\operatorname{supp}(\theta_i^*)$. 
\end{corollary}

\begin{proof}
By Lemma~\ref{lem:uniform_inverse_hessian} and the reverse triangle inequality, for every $j\in S_i$,
\[
|[\vartheta_{i,c}^*]_j| \ge |[\theta_S^*]_j| - \|\vartheta_{i,c}^*-\theta_S^*\|_\infty \ge \beta_{\min,i}-C_{H,i}\Delta_{i,c}.
\]
Hence, we know that $\beta_{\min,i}^{(c)} \ge \beta_{\min,i}-C_{H,i}\Delta_{i,c}.$ Therefore, if we pick $\beta_{\min,i}>T_{n,i}+C_{H,i}\Delta_{i,c},$ we still have $\beta_{\min,i}^{(c)}>T_{n,i}$.  By the surrogate stabilized support-recovery theorem, we have $\operatorname{supp}(\hat\theta_i)=\operatorname{supp}(\bar\theta_{i,c}^*)$. It remains to show that $\operatorname{supp}(\bar\theta_{i,c}^*)=S_i.$ By construction, we have that $[\bar\theta_{i,c}^*]_{S_i^c}=0,$ so no false positives occur outside $S_i$. On the other hand, for every $j\in S_i$,
\[
|[\vartheta_{i,c}^*]_j| \ge \beta_{\min,i}-C_{H,i}\Delta_{i,c} > T_{n,i}\ge 0,
\]
and in particular $|[\vartheta_{i,c}^*]_j|>0$.  Hence every active coordinate remains nonzero, so $\operatorname{supp}(\bar\theta_{i,c}^*)=S_i$. Combining the two displays yields $\operatorname{supp}(\hat\theta_i)=S_i=\operatorname{supp}(\theta_i^*).$
\end{proof} 
As discussed earlier, Corollary~\ref{cor:exact-support-original} admits a natural interpretation consistent with our previous observations. The true signal must be sufficiently strong to prevent any true parents from disappearing after stabilization, and the result makes this requirement precise through an explicit quantitative bound. The key idea of the proof is a support-transfer argument. Lemma~\ref{lem:uniform_inverse_hessian} shows that the stabilized population target remains uniformly close to the original population parameter. This closeness result is established using an \textit{implicit function theorem argument}, as developed in Lemma~\ref{lem:local_restricted_branch}. 

Having established support recovery for the population parameter, we next examine the role of the stabilization parameter $c$. In particular, the LRSC condition imposes a lower bound on 
$c$, which makes the nonemptiness of the feasible stabilization set an important question.
\begin{proposition}[Feasible and optimal choice of $c$]
\label{prop:fesibale_c}
Let $\beta_{\min,i}$ be as in Corollary~\ref{cor:exact-support-original} for node \(i\), define
\[
\mathcal C_i := \Bigl\{ c\ge c_{\min,i}: \beta_{\min,i} > T_{n,i}(c)+C_{H,i}\Delta_{i,c}c: = \phi_i(c) \Bigr\},
\]
where $c_{\min,i}\ge 9\|\theta_i^\star\|_2^2K^2$ required by Lemma~\ref{lem:RSC}. For all \(c\ge c_{\min,i}\), we have $T_{n,i}(c):=\frac{27\sqrt{s}\lambda_n}{m-16sr_n}\,c,$ and $\Delta_{i,c}=\|F_i(\theta_S^\star,c)\|_\infty \le \frac{9}{16\sqrt{3}}Kc^{-1/2}.$ If \(n\) is large enough and 
\begin{equation}
    \label{Eq:beta_suff}
    \beta_{\min,i} > \frac{\sqrt3}{16}\frac{C_{H,i}}{\|\theta_i^\star\|_2} + \frac{54\sqrt{s}\lambda_n}{m}\,c_{\min,i},
\end{equation}
then \(\mathcal C_i\neq\varnothing\). Hence, the minimum signal strength condition in Corollary~\ref{cor:exact-support-original} is attainable. Moreover, the upper bound \(\phi_i(c)\le \frac{54\sqrt{s}\lambda_n}{m}c + \frac{9}{16\sqrt{3}}C_{H,i}Kc^{-1/2}\) is uniquely minimized at
\begin{equation*}
c_{*,i} = \left( \frac{C_{H,i}Km} {192\sqrt{3}\sqrt{s}\lambda_n} \right)^{2/3}.
\end{equation*}
Accordingly, the optimal admissible choice is obtained by projecting \(c_{*,i}\) onto the admissible interval \(\mathcal C_i\).
\end{proposition}
As a result, the sufficient minimum signal strength condition decomposes into two parts: a bridge part, which is uniformly controlled on the admissible local $c$-range, and a recovery part, which depends explicitly on $c$. Thus, larger $c$ helps stabilization, whereas excessively large $c$ weakens effective curvature and therefore requires a stronger minimum signal strength signal. That is the core trade-off in the theorem. Moreover, we provide a sufficient condition as a lower bound for $\beta_{\min,i}$ to ensure the admissible set $\mathcal C_i$ is nonempty. Under this condition, we provide a theory-driven optimal $c_{*,i}$, and in particular obtain the guaranteed admissible subrange $[c_{\min,i}, c_{*,i}] \subset \mathcal C_i$. Notably, since \(c_{*,i}\asymp \lambda_n^{-2/3}\), under the scaling, \(\lambda_n \asymp \sqrt{\log p / n}\), from Theorem~\ref{Thm: LASSO}, we have $c_{*,i}\asymp \left(\frac{n}{\log p}\right)^{1/3}$. This indicates that, as the sample size increases, the admissible choice of \(c\) becomes less restrictive.

\section{Numerical Experiments}

In this section, we first evaluate the performance of our algorithm on simulated nonstationary time series generated by multiplicative SDEs under both stable and unstable regimes and across two graph families: DAGs (Sec.~\ref{Sec:DAG}) and directed loop graphs (Sec.~\ref{Sec:DLP}). To keep the simulation section concise, we present representative results in the main text: DAGs under the stable system and directed loop graphs under the unstable system. Full results for all graph families and regimes are provided in Appendix~\ref{sec:full_simu}. We benchmark our method against four baselines using three evaluation metrics. We then present a scenario (Sec.~\ref{Sec:miss+c}) demonstrating the robustness of our algorithm to misspecification of the physics prior knowledge, and we examine the admissible region of the stabilization constant and the tolerance established in Sec.~\ref{Sec:c_role}.

\textbf{Experimental setup.}
We evaluate graph recovery under two synthetic graph families: random DAGs generated from an Erd\H{o}s--R\'enyi model, following a standard simulation protocol widely used in DAG-learning benchmarks \citep{pamfil2020dynotears, Zheng2018}, and directed loop graphs with reciprocal chain interactions. Nonzero causal weights are sampled from bounded uniform distributions. Given the causal matrix $A$, we generate trajectories from the stochastic dynamical model in Sec.~\ref{sec: Setup} using Euler--Maruyama discretization on $[0,5]$ with $\Delta t=0.01$. We compare \texttt{SCD} with four temporal causal discovery baselines: \texttt{DYNOTEARS}, \texttt{PCMCI}, \texttt{SCOTCH}, and \texttt{Neural-GC}. Graph recovery is evaluated using Structural Hamming Distance (SHD), True Positive Rate (TPR), and False Discovery Rate (FDR), where higher TPR and lower SHD and FDR indicate better recovery. Full experimental details are provided in Appendix~\ref{sec:app:experiment-details}.

Because our setting involves nonstationary multiplicative SDEs and directed loop structures, it differs from the assumptions underlying several standard temporal causal discovery methods, such as acyclicity, stationarity, linear VAR dynamics, or particular noise specifications. We therefore view these baselines as informative reference points under a unified evaluation protocol rather than as methods specifically designed for our model class. A detailed discussion of baseline behavior and assumption mismatches is provided in \ref{sec:full_simu}.

\subsection{Directed Acyclic Graphs} 
\label{Sec:DAG}

\begin{table}[t]
\centering
\small
\setlength{\tabcolsep}{7pt}
\renewcommand{\arraystretch}{1.02}
\caption{Graph recovery results on \textit{DAGs} (Stable System): mean $\pm$ std over 10 runs.}
\label{tab:dag_stable}
\begin{tabular}{llcccc}
\toprule
\textbf{Metric} & \textbf{Method} & $p = 5$ & $p = 10$  & $p = 15$ & $p = 20$ \\
\midrule
\multirow{5}{*}{\textbf{SHD} $\downarrow$}
& \texttt{SCD} (ours)   & $1.900 \pm 1.300 $ & $5.200 \pm 1.327$ & $19.400 \pm 1.625$ & $69.300 \pm 3.315$ \\
& \texttt{DYNOTEARS}    & $10.400 \pm 0.548$ & $50.200 \pm 6.301$ & $94.800 \pm 4.764$ & $241.200 \pm 25.470$ \\
& \texttt{PCMCI}        & $9.400 \pm 1.817$  & $41.600 \pm 11.567$ & $91.800 \pm 22.610$ & $170.400 \pm 28.510$ \\
& \texttt{SCOTCH}       & $8.400 \pm 1.949$ & $38.000 \pm 4.359$ & $85.800 \pm 14.255$ & $161.400 \pm 25.294$ \\
& \texttt{Neural-GC}    & $12.200 \pm 3.564$ & $53.000 \pm 3.742$ & $97.600 \pm 3.507$ & $215.500 \pm 7.778$ \\
\midrule
\multirow{5}{*}{\textbf{TPR} $\uparrow$}
& \texttt{SCD} (ours)   & $0.902 \pm 0.077$ & $0.847 \pm 0.049$ & $0.896 \pm 0.031$ & $0.836 \pm 0.016$ \\
& \texttt{DYNOTEARS}    & $0.407 \pm 0.051$ & $0.095 \pm 0.040$ & $0.196 \pm 0.034$ & $0.221 \pm 0.059$ \\
& \texttt{PCMCI}        & $0.481 \pm 0.062$ & $0.225 \pm 0.081$ & $0.272 \pm 0.070$ & $0.332 \pm 0.041$ \\
& \texttt{SCOTCH}       & $0.563 \pm 0.147$ & $0.117 \pm 0.066$ & $0.190 \pm 0.049$ & $0.322 \pm 0.030$ \\
& \texttt{Neural-GC}    & $0.356 \pm 0.158$ & $0.142 \pm 0.047$ & $0.247 \pm 0.038$ & $0.318 \pm 0.027$ \\
\midrule
\multirow{5}{*}{\textbf{FDR} $\downarrow$}
& \texttt{SCD} (ours)   & $0.111 \pm 0.111$ & $0.175 \pm 0.047$ & $0.294 \pm 0.013$ & $0.371 \pm 0.037$ \\
& \texttt{DYNOTEARS}    & $0.622 \pm 0.127$ & $0.750 \pm 0.117$ & $0.722 \pm 0.059$ & $0.607 \pm 0.103$ \\
& \texttt{PCMCI}        & $0.356 \pm 0.214$ & $0.388 \pm 0.227$ & $0.592 \pm 0.156$ & $0.613 \pm 0.168$ \\
& \texttt{SCOTCH}       & $0.422 \pm 0.093$ & $0.787 \pm 0.130$ & $0.773 \pm 0.138$ & $0.716 \pm 0.109$ \\
& \texttt{Neural-GC}    & $0.556 \pm 0.208$ & $0.512 \pm 0.248$ & $0.533 \pm 0.158$ & $0.414 \pm 0.105$ \\
\bottomrule
\end{tabular}
\end{table}

We first evaluate the performance of our method on DAGs by varying the number of variables as  $p \in \{5,10,15,20\}$. The stable-system results are reported in  Table~\ref{tab:dag_stable}, and the unstable-system results are provided in Appendix~\ref{sec:full_simu}. We report graph recovery for the additional coupling component,  i.e., the unknown graph embedded in the diffusion term $A$. To ensure a fair comparison, we exclude from all baseline estimates the edges specified by the prior knowledge.

From Table~\ref{tab:dag_stable}, we observe that SHD increases for all methods as the dimension  grows from $p=5$ to $p=20$. This trend is expected, since the number of possible directed edges scales on the order of $p^2$, making graph recovery progressively more challenging at higher dimensions. Across all dimensions, \texttt{SCD} achieves substantially lower SHD than the baselines. For example, when $p=20$, \texttt{SCD} attains SHD $69.300$, whereas the baseline methods range from $161.400$ to $241.200$. Moreover, \texttt{SCD} maintains high TPR, between $0.836$ and $0.902$, while keeping FDR comparatively low. In particular, at $p=20$, \texttt{SCD} has FDR $0.371$, compared with baseline FDR values ranging from $0.414$ to $0.716$; at $p=5$, \texttt{SCD} has FDR $0.111$, compared with baseline values ranging from $0.356$ to $0.622$. These results indicate that \texttt{SCD} recovers a large fraction of true edges while controlling false discoveries more effectively than the competing methods.

\subsection{Directed Loop Graphs}
\label{Sec:DLP}
We further evaluate performance on \textit{directed loop graphs} under both stable and unstable regimes. The unstable-system results are reported in Table~\ref{tab:loop_unstable}, while the corresponding stable-system results are provided in Appendix~\ref{sec:full_simu}. These structures are motivated by feedback-loop patterns that arise in many dynamical systems \citep{mooij2013}. As before, the results evaluate recovery of the additional coupling component, i.e., the unknown graph embedded in the diffusion term $A$.

\begin{table}[t]
\centering
\small
\setlength{\tabcolsep}{7pt}
\renewcommand{\arraystretch}{1.02}
\caption{Graph recovery results on \textit{Directed loop graphs} (Unstable System): mean $\pm$ std over 10 runs.}
\label{tab:loop_unstable}
\begin{tabular}{llcccc}
\toprule
\textbf{Metric} & \textbf{Method} & $p = 5$ & $p = 10$  & $p = 15$ & $p = 20$ \\
\midrule

\multirow{5}{*}{\textbf{SHD} $\downarrow$}
& \texttt{SCD} (ours)   & $0.800 \pm 0.400$ & $1.900 \pm 1.053$ & $6.500 \pm 3.590$ & $13.500 \pm 4.153$ \\
& \texttt{DYNOTEARS}    & $9.800 \pm 1.643$ & $29.800 \pm 5.541$ & $44.800 \pm 16.115$ & $58.600 \pm 6.269$ \\
& \texttt{PCMCI}        & $8.600 \pm 1.140$ & $29.800 \pm 3.271$ & $81.300 \pm 3.055$ & $135.600 \pm 15.258$ \\
& \texttt{SCOTCH}       & $11.200 \pm 1.789$ & $38.400 \pm 3.912$ & $68.400 \pm 10.065$ & $119.800 \pm 18.102$ \\
& \texttt{Neural-GC}    & $10.700 \pm 1.258$ & $36.200 \pm 13.226$ & $56.000 \pm 13.614$ & $89.200 \pm 5.263$ \\
\midrule

\multirow{5}{*}{\textbf{TPR} $\uparrow$}
& \texttt{SCD} (ours)   & $0.978 \pm 0.044$ & $0.959 \pm 0.037$ & $0.950 \pm 0.064$ & $0.902 \pm 0.057$ \\
& \texttt{DYNOTEARS}    & $0.167 \pm 0.236$ & $0.150 \pm 0.109$ & $0.098 \pm 0.075$ & $0.162 \pm 0.072$ \\
& \texttt{PCMCI}        & $0.477 \pm 0.078$ & $0.251 \pm 0.040$ & $0.141 \pm 0.026$ & $0.099 \pm 0.016$ \\
& \texttt{SCOTCH}       & $0.326 \pm 0.081$ & $0.149 \pm 0.044$ & $0.114 \pm 0.039$ & $0.121 \pm 0.025$ \\
& \texttt{Neural-GC}    & $0.306 \pm 0.130$ & $0.240 \pm 0.017$ & $0.153 \pm 0.050$ & $0.133 \pm 0.031$ \\
\midrule

\multirow{5}{*}{\textbf{FDR} $\downarrow$}
& \texttt{SCD} (ours)   & $0.075 \pm 0.061$ & $0.062 \pm 0.043$ & $0.183 \pm 0.097$ & $0.274 \pm 0.104$ \\
& \texttt{DYNOTEARS}    & $0.800 \pm 0.274$ & $0.811 \pm 0.178$ & $0.893 \pm 0.116$ & $0.858 \pm 0.076$ \\
& \texttt{PCMCI}        & $0.575 \pm 0.143$ & $0.667 \pm 0.111$ & $0.619 \pm 0.109$ & $0.679 \pm 0.096$ \\
& \texttt{SCOTCH}       & $0.625 \pm 0.177$ & $0.744 \pm 0.145$ & $0.771 \pm 0.135$ & $0.663 \pm 0.073$ \\
& \texttt{Neural-GC}    & $0.562 \pm 0.375$ & $0.542 \pm 0.295$ & $0.750 \pm 0.179$ & $0.747 \pm 0.094$ \\
\bottomrule
\end{tabular}
\end{table}

We observe similar trends on directed loop graphs under the unstable system. As $p$ increases from $5$ to $20$, SHD generally increases across methods, reflecting the increasing difficulty of graph recovery in higher dimensions. Nevertheless, \texttt{SCD} consistently achieves the lowest SHD while maintaining very high TPR across all dimensions. In particular, \texttt{SCD} attains TPR between $0.902$ and $0.978$, indicating strong sensitivity to true edges even in the presence of feedback loops. Its FDR remains comparatively low, increasing from $0.075$ at $p=5$ to $0.274$ at $p=20$, which is consistent with the harder high-dimensional setting.

In contrast, the baselines tend to suffer from markedly higher FDR and substantially lower TPR, especially at larger $p$. For example, at $p=20$ in Table~\ref{tab:loop_unstable}, \texttt{SCD} attains SHD $13.5$ with TPR $0.902$ and FDR $0.274$. By comparison, the baseline methods have SHD values ranging from $58.6$ to $135.6$, TPR below $0.17$, and FDR ranging from $0.663$ to $0.858$. Relative to the DAG experiments, the loop-graph setting is generally more challenging for the baselines, which is consistent with a model-assumption mismatch. Several baselines are primarily designed for acyclic, lagged, or stationarity-based causal structures, whereas directed loop graphs contain intrinsic feedback interactions that are not fully captured by purely acyclic formulations. The corresponding stable-loop results are provided in Appendix~\ref{sec:full_simu} and show the same qualitative pattern.

\subsection{ODE Misspecification and Role of c}
\label{Sec:miss+c}
\begin{figure}[tbp]
    \centering
    \includegraphics[width=0.5\linewidth]{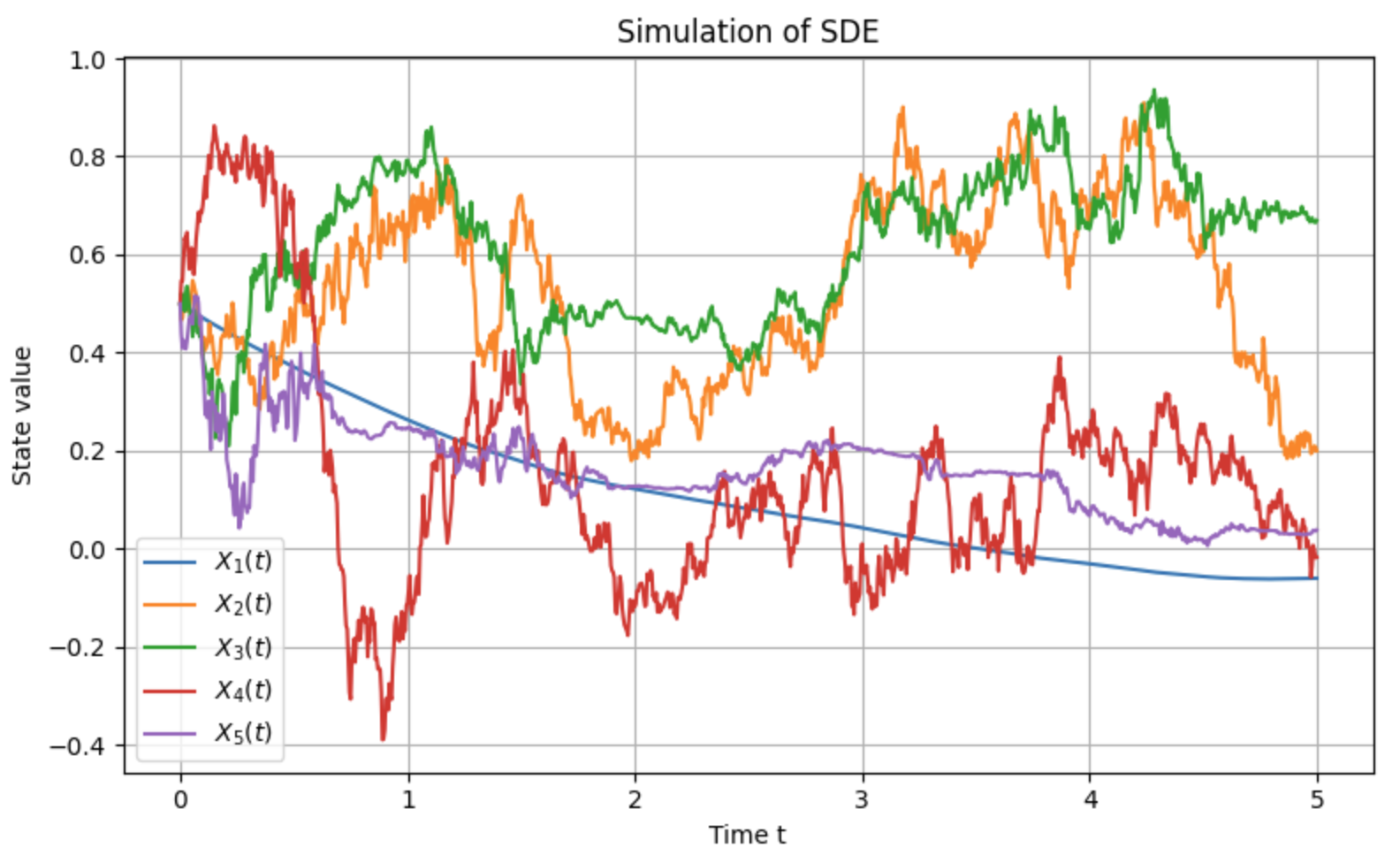}
    \caption{Visualization of the simulated dataset generated under an unstable regime with directed loops in the diffusion graph and feedback loops and self-loops in the ODE component. The resulting time series is nonstationary.}
       \label{fig: blockSDE}
\end{figure}

As shown in Sec~\ref{Sec: ODE miss}, our method does not require ODE information when it is unavailable. In this section, we further examine the robustness of our algorithm under ODE (drift) misspecification. Specifically, we use the nominal drift $g(x_k,\gamma,t) \approx Bx_k + b$, while the true drift deviates from this form by a deterministic additive perturbation, that is $g^{\star}(x_k) = Bx_k + b + \delta_k,$ where $\|\delta_k\|_2 \le \varepsilon_{\mathrm{mis}}.$ In this section, we consider a five-dimensional system under an unstable SDE regime with directed loops in both the diffusion graph and the ODE component. Specifically, in the ODE prior, we specify cross-interactions between two variables, forming a two-node feedback loop: $x_1 \to x_2$ and $x_2 \to x_1$. Thus, ODE information is provided for only $2$ out of $5$ variables, and the specified ODE contains both feedback cycles and self-loops. Accordingly, the ODE drift matrix $B$ has a non-diagonal $2\times2$ block in its top-left corner, while all remaining entries are set to zero (i.e., the other variables have no prescribed ODE dynamics). This choice yields intentionally weak mechanistic guidance for the remaining variables.  A visualization of the simulated data is shown in Fig.\ref{fig: blockSDE}, and the resulting time series remains nonstationary. We first report the results for the no-misspecification case and compare our method with four benchmarks. The results are reported in Tab.\ref{tab:picd_d5_compact}. We observe that \texttt{SCD} achieves the lowest SHD, the lowest FDR, and the highest TPR among all benchmarks.

\begin{table}[tbp]
\centering
\caption{Graph recovery results for directed loop diffusion graphs under an unstable regime, with an ODE prior containing feedback loops and self-loops (mean $\pm$ std over 10 runs).}
\label{tab:picd_d5_compact}
\begin{tabular}{lccc}
\toprule
\textbf{Method} & \textbf{SHD} $\downarrow$ & \textbf{TPR} $\uparrow$ & \textbf{FDR} $\downarrow$ \\
\midrule
\texttt{SCD} (ours)  & $1.100 \pm 1.044$ & $0.903 \pm 0.085$ & $0.025 \pm 0.075$ \\
\texttt{DYNOTEARS} & $8.000 \pm 0.471$ & $0.150 \pm 0.337$ & $0.975 \pm 0.053$ \\
\texttt{PCMCI} & $9.300 \pm 1.947$ & $0.422 \pm 0.115$ & $0.588 \pm 0.156$ \\
\texttt{SCOTCH } & $10.200 \pm 1.483$ & $0.354 \pm 0.099$ & $0.625 \pm 0.177$ \\
\texttt{Neural-GC} & $9.100 \pm 2.079$ & $0.459 \pm 0.222$ & $0.662 \pm 0.167$ \\
\bottomrule
\end{tabular}
\end{table}

\begin{table}[tbp]
\centering
\caption{Stress test under drift misspecification (mean $\pm$ std over 10 runs).}
\label{tab:misspec_stress}
\begin{tabular}{rccc}
\hline
\textbf{Miss-entry \%} & \textbf{SHD} $\downarrow$ & \textbf{FDR} $\downarrow$ & \textbf{TPR} $\uparrow$ \\
\hline
25\%  & $1.700 \pm 1.900$ & $0.025 \pm 0.075$ & $0.856 \pm 0.140$ \\
50\%  & $1.700 \pm 1.900$ & $0.025 \pm 0.075$ & $0.856 \pm 0.140$ \\
75\%  & $1.800 \pm 2.135$ & $0.025 \pm 0.075$ & $0.850 \pm 0.150$ \\
100\% & $1.900 \pm 1.921$ & $0.025 \pm 0.075$ & $0.839 \pm 0.144$ \\
\hline
\end{tabular}
\end{table}

The results are reported in Tab.\ref{tab:misspec_stress}. We observe that with one misspecified drift edge, performance degrades slightly (TPR drops from $0.90$ to $\approx 0.85$) and then stabilizes when two edges are misspecified. Moreover, when three or four drift edges are misspecified, the FDR remains unchanged and the SHD increases only mildly (from $1.7$ to $1.8$ and $1.9$), while the TPR decreases from $0.856$ to $0.850$ and $0.839$. Overall, these results indicate that our algorithm is robust for parent selection under moderate ODE misspecification. After presenting the robustness results, we turn to the role of the stabilization parameter \(c\). Motivated by Proposition~\ref{prop:fesibale_c}, we first conduct a sensitivity analysis over a range of \(c\) values to show that the performance of our algorithm is not driven by a delicately tuned choice of \(c\). We next examine empirically whether larger \(n\) makes the admissible choice of \(c\) less restrictive. The experimental setup is based on the stable DAG diffusion with \(p=5\). 

From Fig.\ref{fig:c_n_tradeoff}, we observe that the recovery performance is best over a moderate range of \(c\) values. When \(c\) is too small, SHD is high while TPR is relatively low, indicating poor graph recovery. As \(c\) increases to a moderate range, SHD and FDR decrease substantially and TPR approaches one, suggesting stable recovery performance. However, when \(c\) becomes excessively large, performance deteriorates again: SHD and FDR increase, while TPR drops sharply. Overall, the figure indicates that the proposed method is not driven by a narrowly tuned choice of \(c\), but rather performs well over a nontrivial interval of moderate values. The empirical results are also consistent with the theoretical picture described in Proposition~\ref{prop:fesibale_c}. When \(c\) is too small, the local RSC condition may fail, and graph recovery deteriorates accordingly. When \(c\) becomes excessively large, the signal induced by the true parameters is overly diluted; without a corresponding increase in the true parameter strength (followed by the minimum signal strength condition), the recovery result again worsens.

\begin{figure}[tbp]
    \centering
    \includegraphics[width=1.0\linewidth]{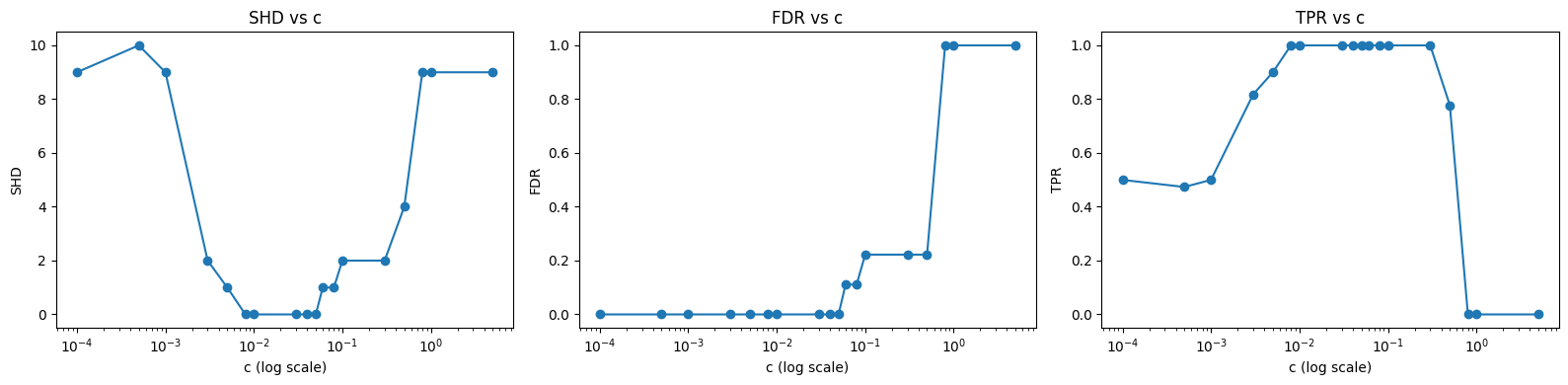}
    \caption{Recovery performance versus the stabilization parameter $c$ at $n=500$.}
    \label{fig:c_n_tradeoff}
\end{figure}

From Fig.\ref{fig:c_n_tradeoff}, we also observe that our previous choice of \(c\) in the stable SDE setting with DAG diffusion is not the best-performing choice. In particular, recovery performance degrades sharply when \(c>0.5\). We therefore examine whether increasing the sample size \(n\) can make the admissible choice of \(c = 0.5\) less restrictive. For this purpose, we consider additional values of \(c\), including a smaller value \(c=0.3\) and a larger value \(c=0.8\). In this experiment, we consider four sample sizes, \(n \in \{500,1000,1500,2000\}\). 

The full results are reported in Tab.\ref{tab:c_n_tradeoff}. We observe that, for each fixed value of \(c\), recovery performance generally improves as \(n\) increases: SHD and FDR tend to decrease, while TPR tends to increase. In addition, the degradation caused by larger \(c\) becomes less severe when \(n\) is larger. For example, when \(c=0.8\), SHD decreases from \(7.2\) to \(5.4\), FDR decreases from \(0.800\) to \(0.508\), and TPR increases from \(0.600\) to \(0.719\) as \(n\) grows from \(500\) to \(2000\). These results provide empirical support for the prediction of Proposition~\ref{prop:fesibale_c} that, as the sample size \(n\) increases, the guaranteed admissible upper end of \(c\) becomes less restrictive.

\begin{table}[tbp]
\centering
\caption{Empirical recovery performance under different stabilization parameters $c$ and sample sizes $n$. Results are reported as mean $\pm$ standard deviation over 10 runs.}
\label{tab:c_n_tradeoff}
\begin{tabular}{llcccc}
\toprule
$c$ & Metric & $n=500$ & $n=1000$ & $n=1500$ & $n=2000$ \\
\midrule
\multirow{3}{*}{0.3}
& \textbf{SHD} $\downarrow$ & $3.400 \pm 1.497$ & $3.300 \pm 0.881$ & $3.000 \pm 0.500$ & $2.700 \pm 0.745$ \\
& \textbf{FDR} $\downarrow$ & $0.289 \pm 0.074$ & $0.286 \pm 0.055$ & $0.278 \pm 0.056$ & $0.244 \pm 0.067$ \\
& \textbf{TPR} $\uparrow$& $0.898 \pm 0.123$ & $0.908 \pm 0.089$ & $0.935 \pm 0.065$ & $0.955 \pm 0.060$ \\
\midrule
\multirow{3}{*}{0.5}
& \textbf{SHD} $\downarrow$ & $4.300 \pm 1.479$ & $4.000 \pm 1.581$ & $3.500 \pm 0.500$ & $3.300 \pm 0.943$ \\
& \textbf{FDR} $\downarrow$ & $0.361 \pm 0.182$ & $0.361 \pm 0.092$ & $0.306 \pm 0.048$ & $0.296 \pm 0.052$ \\
& \textbf{TPR} $\uparrow$& $0.882 \pm 0.118$ & $0.887 \pm 0.121$ & $0.897 \pm 0.060$ & $0.905 \pm 0.067$ \\
\midrule
\multirow{3}{*}{0.8}
& \textbf{SHD} $\downarrow$ & $7.200 \pm 1.470$ & $7.000 \pm 1.095$ & $6.600 \pm 1.497$ & $5.400 \pm 1.841$ \\
& \textbf{FDR} $\downarrow$ & $0.800 \pm 0.163$ & $0.733 \pm 0.151$ & $0.689 \pm 0.191$ & $0.508 \pm 0.229$ \\
& \textbf{TPR} $\uparrow$& $0.600 \pm 0.490$ & $0.710 \pm 0.369$ & $0.717 \pm 0.371$ & $0.719 \pm 0.306$ \\
\bottomrule
\end{tabular}
\end{table}

\section{Application}
\label{Sec:APP}
In this section, we first present two studies based on real-world dynamical systems. We then show the reconstruction of a stochastic SIR model using real COVID-19 data.
\subsection{Stochastic Lotka-Volterra System}
This experiment demonstrates that the proposed method does not require fully specified governing equations. Instead, it can leverage partial mechanistic structure while remaining robust to omitted drift components, which is a practically important regime in scientific discovery problems. We begin with the two-dimensional stochastic Lotka Volterra system
\begin{equation}
\label{eq:lv_sde_basic}
\begin{aligned}
dV_t &= \bigl(aV_t - bV_tR_t\bigr)\,dt + \sigma_1(V_t)\,dW_{1,t},\\
dR_t &= \bigl(cV_tR_t - dR_t\bigr)\,dt + \sigma_2(R_t)\,dW_{2,t},
\end{aligned}
\end{equation}
where $V_t$ and $R_t$ denote the prey and predator populations, respectively, and $W_t=(W_{1,t},W_{2,t})^\top$ is a two-dimensional Brownian motion. Then, we introduce the following equations: $\sigma_1(X_t) = f_1\big(\phi_1(P_{1,t}), \phi_2(P_{2,t})\big)$ and $\sigma_2(X_t) = f_2\big(\phi_3(P_{1,t}), \phi_4(P_{2,t})\big)$, where $P_1,P_2$ are two variables introduced into the system, and each $\phi_i$ is either a linear function or a nonlinear dictionary element. In our setting, we assume that $f_1, f_2$ are additive, that is,  $f_i(u,v)= u+v$. Then, to study causal discovery under partially known mechanistic dynamics, we augment the system with two additional components and consider the four-dimensional process such that $X_t = (V_t, R_t, P_{1,t}, P_{2,t})^\top \in \mathbb{R}^4$, and the true data-generating dynamics are given by
\begin{equation}
\label{eq:true_augmented_lv_sde}
    dX_t = b_{\mathrm{true}}(X_t)dt + \bigl((A^*)^\top \phi (X_t)\bigr)\odot dW_t,
\end{equation}
where $b_{\mathrm{true}}(X_t) = (aV_t- bV_tR_t,cV_tR_t - dR_t, eP_{1,t},fP_{2,t})^\top$, $\phi(X_t)$ denotes the dictionary-transformed state as defined in Sec~\ref{Sec: Nonlinear}, $W_t \in \mathbb{R}^4$ is a four-dimensional Brownian motion, and $A^*\in \mathbb R^{4\times4}$ which is the adjacency matrix of either simulated DAG and direct loop graph with a zero upper-left $2\times 2$ block. In our experiments, we set $a=1.5,b=1, c=1, d=3, e=f=-2$. The ODE information in the data-generating process is computed using the Runge--Kutta method \citep{butcher2016numerical}. Because our SDE-based estimator defines residuals under the Euler--Maruyama discretization, \textit{the estimator operates under both partial mechanistic misspecification and discretization mismatch}. A visualization of a generated dataset is shown in Fig.\ref{fig:LVSDE}. The left panel presents the phase portrait of the observed causally perturbed prey--predator coordinates, while the right panel shows the time trajectories of the full state.
\begin{figure}
    \centering
    \includegraphics[width=0.8\linewidth]{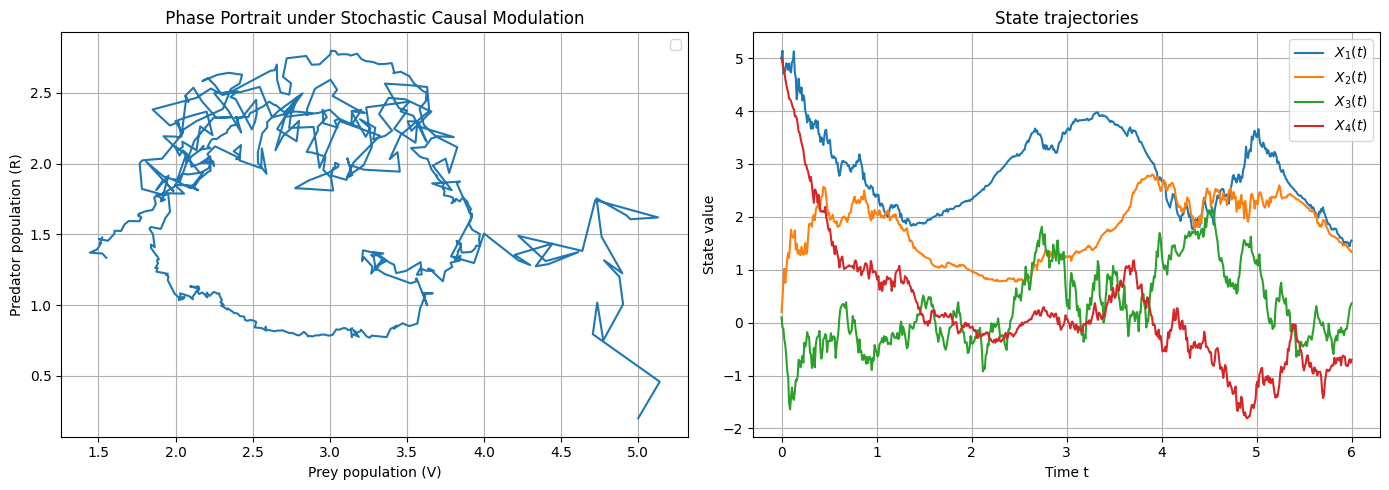}
    \caption{Visualization of a representative sample path from the stochastic Lotka--Volterra system, illustrating its nonstationary temporal behavior. The visualization is generated under a linear diffusion setting with a loop graph structure.}
    \label{fig:LVSDE}
\end{figure}
\begin{table}[t]
\centering
\caption{Performance comparison on the LV system under different graph and functional settings. All results are reported as mean $\pm$ standard deviation over 10 runs.}
\label{tab:lv_results}
\resizebox{0.8\textwidth}{!}{%
\begin{tabular}{llcccc}
\toprule
\multirow{2}{*}{\textbf{Metric}} & \multirow{2}{*}{\textbf{Method}} 
& \multicolumn{2}{c}{DAG} 
& \multicolumn{2}{c}{Loop Graph} \\
\cmidrule(lr){3-4} \cmidrule(lr){5-6}
& & Linear & Nonlinear & Linear & Nonlinear \\
\midrule

\multirow{5}{*}{\textbf{SHD} $\downarrow$}
& \texttt{SCD} (ours)       & $0.800 \pm 0.777$ & $1.300 \pm 0.579$ & $1.600 \pm 0.800$ & $2.700 \pm 1.030$ \\
& \texttt{PCMCI}      & $6.900 \pm 1.345$ & $7.100 \pm 1.595$ & $6.600 \pm 2.011$ & $7.000 \pm 1.000$ \\
& \texttt{DYNOTEARS}  & $9.600 \pm 0.787$ & $9.300 \pm 0.949$ & $8.100 \pm 1.287$ & $7.300 \pm 1.380$ \\
& \texttt{Neural-GC}        & $8.100 \pm 1.464$ & $8.800 \pm 0.463$ & $6.100 \pm 0.316$ & $6.100\pm 0.378$ \\
& \texttt{SCOTCH}     & $2.500 \pm 2.014$ & $5.000 \pm 1.871$ & $6.100 \pm 0.876$ & $4.800 \pm 2.950$ \\
\midrule

\multirow{5}{*}{\textbf{TPR} $\uparrow$}
& \texttt{SCD} (ours)        & $0.901 \pm 0.105$ & $0.899 \pm 0.093$ & $0.900 \pm 0.122$ & $0.811 \pm 0.149$ \\
& \texttt{PCMCI}      & $0.342 \pm 0.103$ & $0.301 \pm 0.126$ & $0.451 \pm 0.130$ & $0.437 \pm 0.072$ \\
& \texttt{DYNOTEARS}  & $0.071 \pm 0.089$ & $0.127 \pm 0.093$ & $0.258 \pm 0.195$ & $0.384 \pm 0.115$ \\
& \texttt{Neural-GC}        & $0.340 \pm 0.077$ & $0.341 \pm 0.014$ & $0.495 \pm 0.014$ & $0.494 \pm 0.017$ \\
& \texttt{SCOTCH}      & $0.548 \pm 0.422$ & $0.573 \pm 0.292$ & $0.415 \pm 0.236$ & $0.390 \pm 0.261$ \\
\midrule

\multirow{5}{*}{\textbf{FDR} $\downarrow$}
& \texttt{SCD} (ours)        & $0.034 \pm 0.076$ & $0.148 \pm 0.124$ & $0.133 \pm 0.103$ & $0.190 \pm 0.165$ \\
& \texttt{PCMCI}      & $0.629 \pm 0.138$ & $0.660 \pm 0.165$ & $0.417 \pm 0.239$ & $0.429 \pm 0.163$ \\
& \texttt{DYNOTEARS}  & $0.914 \pm 0.107$ & $0.840 \pm 0.126$ & $0.717 \pm 0.236$ & $0.500 \pm 0.304$ \\
& \texttt{Neural-GC}        & $0.343 \pm 0.151$ & $0.200 \pm 0.000$ & $0.017 \pm 0.053$ & $0.024 \pm 0.063$ \\
& \texttt{SCOTCH}      & $0.380 \pm 0.274$ & $0.680 \pm 0.110$ & $0.800 \pm 0.153$ & $0.533 \pm 0.342$ \\
\bottomrule
\end{tabular}%
}
\end{table}
 During the learning procedure, we assume that only the Lotka--Volterra component of the drift is known at inference time to evaluate robustness to incomplete mechanistic knowledge. Specifically, the estimator is provided with the reduced drift, that is, $b_{\mathrm{known}}(X_t) = (aV_t - bV_tR_t, cV_tR_t - dR_t,0,0)^\top$, so that the stable drift terms in the last two coordinates are omitted during estimation. The goal is to recover the support of the unknown diffusion-interaction matrix $A^*$ from discrete observations generated from Eq.\eqref{eq:true_augmented_lv_sde}.

The results are reported in Tab.\ref{tab:lv_results}. Across all four settings, we observe that \texttt{SCD} achieves the smallest SHD, indicating the most accurate overall graph recovery. In the DAG setting, \texttt{SCD} attains SHD 0.800 and 1.300 under the linear and nonlinear specifications, respectively, improving substantially over the best competing baseline, \texttt{SCOTCH} , with corresponding SHD values 2.500 and 5.000. A similar pattern holds for loop graphs, where \texttt{SCD} again yields the lowest SHD (1.600 for linear and 2.700 for nonlinear), while all competing methods incur substantially larger reconstruction error. These results indicate that the proposed method remains accurate under both nonlinear stochastic modulation and cyclic graph structure. This improvement is supported by a favorable TPR--FDR tradeoff. \texttt{SCD} maintains high TPR across all settings (0.901, 0.899, 0.900, and 0.811) while keeping FDR relatively low (0.034, 0.148, 0.133, and 0.190). Overall, the deterioration from DAG to loop graphs and from linear to nonlinear settings is modest relative to the baselines, providing empirical evidence that \texttt{SCD} is robust to graph complexity and nonlinear diffusion effects under partially specified mechanistic dynamics.

\subsection{Stochastic Chaotic Lorenz System}
To further demonstrate that the proposed method does not rely on fully specified governing equations, we consider the stochastic chaotic Lorenz system. Specifically, we begin with the classical three-dimensional stochastic chaotic Lorenz system.
\begin{equation}
    \begin{split}
        dx_t &= \alpha (y_t - x_t)\,dt + \sigma_1(x_t)\,dW_{1,t}, \\
dy_t &= \bigl(x_t(\rho - z_t) - y_t\bigr)\,dt + \sigma_2(y_t)\,dW_{2,t}, \\
dz_t &= \bigl(x_t y_t - \beta z_t\bigr)\,dt + \sigma_3(z_t)\,dW_{3,t},
    \end{split}
\end{equation}
where $(x_t,y_t,z_t)$ denotes the system state, $(\alpha,\rho,\beta)$ are the Lorenz parameters, $W_t=(W_{1,t},W_{2,t},W_{3,t})^\top$ is a three-dimensional Brownian motion, and the functions $\sigma_i$ are defined analogously to those in the stochastic Lotka--Volterra system. In this experiment, we also introduce two additional variables, defined analogously to those in the Lotka--Volterra setting. We then model the true data-generating process as
\begin{equation}
\label{eq:lorenz_sde}
dX_t = b_{\mathrm{true}}(X_t)\,dt + \bigl((A^*)^\top \phi(X_t)\bigr)\odot dW_t,
\end{equation}
where $X_t=(x_t,y_t,z_t, P_{1,t}, P_{2,t})^\top \in \mathbb{R}^5$, $b_{\mathrm{true}}(X_t) = \bigl(\alpha (y_t-x_t), x_t(\rho-z_t)-y_t,\; x_t y_t-\beta z_t, eP_{1,t}, fP_{2,t}\bigr)^\top$, $W_t \in \mathbb{R}^5$ is a five-dimensional Brownian motion, and $A^*\in \mathbb R^{5\times5}$ denotes the adjacency matrix of the simulated graph, which is taken to be either a DAG or a directed-loop graph, with a zero upper-left $3\times 3$ block. In our experiments, we set $\alpha=10,\rho=28, \beta=\frac{8}{3}, e=f=-1$. A visualization of a generated dataset is shown in Fig.\ref{fig:LorenSDE}. The left panel presents the 3D phase portrait of the observed causally perturbed Lorenz coordinates, which exhibits a noisy butterfly-like phase portrait, while the right panel shows the time trajectories of the full state.

\begin{figure}
    \centering
    \includegraphics[width=0.8\linewidth]{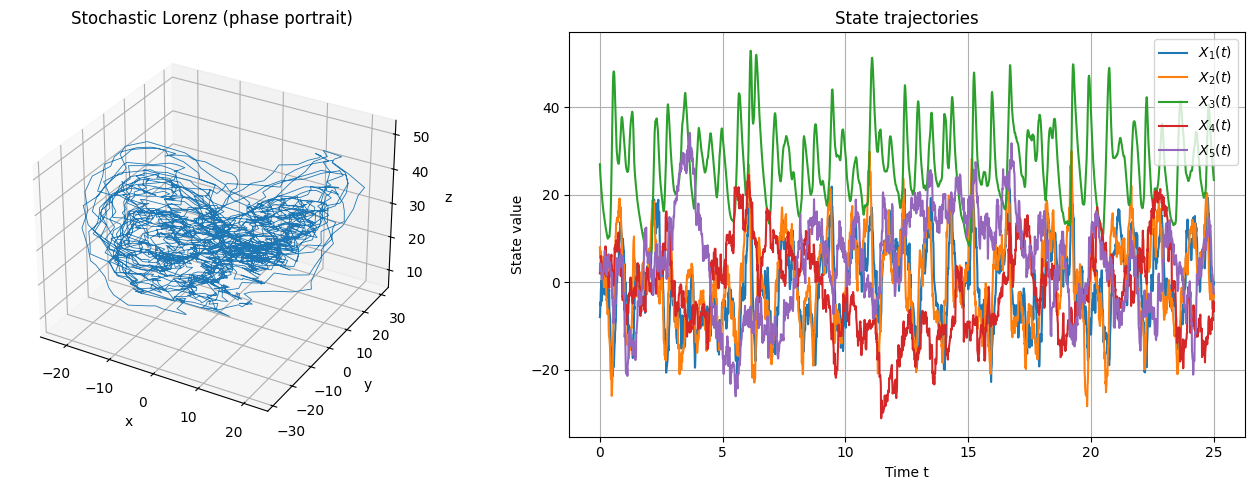}
    \caption{Visualization of a representative sample path from the stochastic chaotic Lorenz system, illustrating its nonstationary temporal behavior. The visualization is generated under a linear diffusion setting with a loop graph structure.}
    \label{fig:LorenSDE}
\end{figure}

During the learning procedure, we also remove some drift components at inference time. Specifically, the estimator is provided with a reduced Lorenz drift, that is, $b_{\mathrm{known}}(X_t) = \bigl(\alpha (y_t-x_t), x_t(\rho-z_t)-y_t,\; x_t y_t-\beta z_t,0,0)$.  The goal is to recover the support of the unknown diffusion-interaction matrix \(A^*\) from discrete observations generated according to Eq.\eqref{eq:lorenz_sde}.

\begin{table}[t]
\centering
\caption{Performance comparison on the Chaotic Lorenz system under different graph and functional settings. All results are reported as mean $\pm$ standard deviation over 10 runs. NGC is marked as NA because the public implementation does not work for our setting.}
\label{tab:lorenz_results}
\resizebox{0.8\textwidth}{!}{%
\begin{tabular}{llcccc}
\toprule
\multirow{2}{*}{\textbf{Metric}} &\multirow{2}{*}{\textbf{Method}} 
& \multicolumn{2}{c}{DAG}
& \multicolumn{2}{c}{Loop Graph} \\
\cmidrule(lr){3-4} \cmidrule(lr){5-6}
& & Linear & Nonlinear & Linear & Nonlinear \\
\midrule

\multirow{5}{*}{\textbf{SHD} $\downarrow$}
& \texttt{SCD} (ours)        & $1.200 \pm 1.249$ & $1.500 \pm 1.025$ & $2.200 \pm 1.077$ & $3.100 \pm 0.831$ \\
& \texttt{PCMCI}      & $8.400 \pm 1.430$ & $8.900 \pm 2.183$ & $10.000 \pm 1.886$ & $8.300 \pm 2.497$ \\
& \texttt{DYNOTEARS}  & $10.900 \pm 1.729$ & $7.600 \pm 2.271$ & $10.100 \pm 1.287$ & $9.100 \pm 1.197$ \\
& \texttt{Neural-GC}        & NA & NA & NA & NA \\
& \texttt{SCOTCH}      & $7.900 \pm 2.378$ & $6.700 \pm 2.359$ & $7.900 \pm 0.994$ & $6.500 \pm 1.434$ \\
\midrule

\multirow{5}{*}{\textbf{TPR} $\uparrow$}
& \texttt{SCD} (ours)        & $0.900 \pm 0.144$ & $0.908 \pm 0.160$ & $0.895 \pm 0.072$ & $0.858 \pm 0.152$ \\
& \texttt{PCMCI}      & $0.371 \pm 0.295$ & $0.357 \pm 0.301$ & $0.426 \pm 0.280$ & $0.455 \pm 0.252$ \\
& \texttt{DYNOTEARS}  & $0.043 \pm 0.069$ & $0.162 \pm 0.217$ & $0.233 \pm 0.139$ & $0.202 \pm 0.168$ \\
& \texttt{Neural-GC}        & NA & NA & NA & NA \\
& \texttt{SCOTCH}      & $0.443 \pm 0.207$ & $0.527 \pm 0.259$ & $0.460 \pm 0.192$ & $0.636 \pm 0.124$ \\
\midrule

\multirow{5}{*}{\textbf{FDR} $\downarrow$}
& \texttt{SCD} (ours)        & $0.069 \pm 0.088$ & $0.222 \pm 0.104$ & $0.171 \pm 0.131$ & $0.253 \pm 0.074$ \\
& \texttt{PCMCI}      & $0.675 \pm 0.200$ & $0.846 \pm 0.120$ & $0.690 \pm 0.130$ & $0.621 \pm 0.178$ \\
& \texttt{DYNOTEARS}  & $0.925 \pm 0.127$ & $0.863 \pm 0.169$ & $0.752 \pm 0.163$ & $0.740 \pm 0.214$ \\
& \texttt{Neural-GC}        & NA & NA & NA & NA \\
& \texttt{SCOTCH}      & $0.558 \pm 0.181$ & $0.664 \pm 0.173$ & $0.576 \pm 0.079$ & $0.469 \pm 0.146$ \\
\bottomrule
\end{tabular}%
}
\end{table}
Tab.\ref{tab:lorenz_results} shows that \texttt{SCD} again achieves the best support recovery across all four settings on the stochastic chaotic Lorenz system. It attains the smallest SHD in every case, with values 1.200, 1.500, 2.200, and 3.100 under DAG-linear, DAG-nonlinear, loop-linear, and loop-nonlinear settings, respectively. In contrast, the competing baselines exhibit substantially larger reconstruction error across all settings. This advantage is supported by a favorable TPR--FDR tradeoff: \texttt{SCD} maintains high TPR (0.900, 0.908, 0.895, and 0.858) while keeping FDR relatively low (0.069, 0.222, 0.171, and 0.253). Although performance deteriorates moderately from DAG to loop graphs and from linear to nonlinear settings, the degradation remains modest relative to the baselines.

Taken together with the stochastic Lotka--Volterra results, these findings indicate that \texttt{SCD} is consistently robust across two qualitatively different stochastic dynamical systems: an oscillatory prey--predator system and a chaotic Lorenz system. In both cases, \texttt{SCD} yields the smallest SHD together with strong recall and controlled false discovery, suggesting that the proposed method is stable across graph structure, functional form, and system complexity. \texttt{Neural-GC} is reported as NA in Tab.\ref{tab:lorenz_results} because the public implementation does not operate reliably in this setting. However, the loop-nonlinear setting is clearly the most challenging regime for \texttt{SCD}, yielding its worst SHD (3.100), lowest TPR (0.858), and highest FDR (0.253). This trend is consistent with the stochastic Lotka--Volterra results, where the nonlinear loop-graph case is likewise the hardest setting. Overall, these results suggest that SCD is robust across both oscillatory and chaotic nonstationary systems, although recovery becomes more difficult when cyclic structure and nonlinear diffusion effects are present simultaneously.

\subsection{Stochastic SIR model from COVID-19 Data}

From the previous two sections, we have shown how a known dynamical system becomes stochastic through causal diffusion modulation. For example, starting from the classical Lorenz system, the noisy and increasingly irregular butterfly-shaped trajectory can be understood from the perspective of the underlying causal interaction structure. However, in the previous sections, the diffusion term does not directly involve interactions among the coordinates appearing in the known ODE component. This raises a natural question: \textit{can the proposed framework also recover diffusion interactions that directly depend on variables from the known mechanistic system itself?}

In this section, we demonstrate that the proposed framework can infer additional stochastic structure beyond the known ODE dynamics by using real COVID data to infer a diffusion graph whose dominant dependence pattern can be compared against an epidemiologically motivated reference structure. Before introducing the data and preprocessing steps, we first describe the classical stochastic SIR formulation and the diffusion structure that we aim to learn, which is motivated by the SIR literature.

In standard epidemic models \citep{ALLEN2017128, Greenwood2009}, the susceptible compartment is often treated as dependent through the conservation law, that is, $S_t= N - I_t-R_t $, where $I_t$ denotes the infected compartment and $R_t$ denotes the removed compartment, which may include recoveries, deaths, or both. The stochastic modeling is typically written in terms of the bivariate state \((S_t,I_t)\). In particular, the stochastic terms are usually driven by the infection and recovery mechanisms, and therefore depend on quantities such as \(S_t I_t / N\) and \(I_t\). Consequently, the primary stochastic coupling is concentrated in the \(S\)- and \(I\)-equations, reflecting strong dependence of \(dS_t\) and \(dI_t\) on \(S_t\) and \(I_t\). By contrast, these standard formulations do not typically posit an explicit diffusion component in which \(R_t\) acts as an independent stochastic driver. Thus, from the perspective of the epidemic-modeling literature, the clearest prior support is for a diffusion structure dominated by the \((S,I)\) block, with no direct evidence for an explicit \(R\)-driven diffusion effect. Motivated by this structure, we consider the following stochastic SIR model for $X_t= [S_t, I_t, R_t]^\top$ such that
\begin{equation}
\label{eq:sir}
\begin{split}
   dX_t&=f(X_t)\,dt+(GX_t)\odot dW_t\quad\text{where } f(X_t)=
\begin{bmatrix}
-\beta S_t I_t \\
\beta S_t I_t - \gamma I_t\\
\gamma I_t
\end{bmatrix} \text{ and }
G = \begin{bmatrix}
        a,b, 0\\
        c,d, 0\\
        0,0,0
    \end{bmatrix}
\end{split}
\end{equation}
where $dW_t$ is a three-dimensional Brownian motion. The matrix $G$ is chosen to encode the dominant parent-set structure suggested by standard stochastic SIR models: stochastic dependence is concentrated in the $(S_t,I_t)$ block, with no explicit $R_t$-driven diffusion effect.

% The matrix $G$ is chosen to capture the key diffusion structure suggested by the authoritative stochastic SIR literature, namely that the dominant stochastic dependence is concentrated in the $(S_t,I_t)$ components, while there is no explicit evidence for $R_t$ as an independent driver in the diffusion term.

We briefly describe the data preprocessing and ODE-prior construction here, and provide the full preprocessing details in Appendix~\ref{sec:prepro}. We use the U.S. COVID-19 county-level confirmed-case and death time-series data from the Johns Hopkins University Center for Systems Science and Engineering (JHU CSSE) \citep{Dong2020AnIW}. County-level records are aggregated to the state level and aligned by date. From the cumulative confirmed cases $C_t$, we compute daily new cases and define the infectious compartment $I_t$ as a rolling sum over a fixed infectious window. The removed and susceptible compartments are then constructed as $R_t=C_t-I_t$ and $S_t=N-I_t-R_t $. We set $N=100000$ and normalize all compartments by $N$, yielding an SIR-type trajectory satisfying $S_t+I_t+R_t=1$. The study period is 2020-05-06 to 2021-05-01; the processed trajectory is shown in Fig.\ref{fig:sir}.

We then fit the classical SIR ODE in Eq.~\eqref{eq:sir} to the constructed trajectory by estimating $(\beta,\gamma)$ via least-squares optimization, using a numerical ODE solver and L-BFGS. The fitted ODE serves as an approximate mechanistic drift prior for the subsequent stochastic diffusion-structure inference. Since the compartments are reconstructed from reported case counts and the ODE parameters are estimated from data, this prior should be interpreted as approximate rather than ground truth. After preprocessing the data and preparing the misspecified physics prior, we then apply our method and compare it with benchmark methods to evaluate the inferred graph structure against a literature-motivated epidemiological reference pattern. For all methods, we apply a threshold of $0.1$ to obtain binary adjacency patterns. The resulting inferred diffusion graphs are shown in Fig.\ref{fig:SIR1}. 
% We apply our method and benchmark approaches, threshold all estimated adjacency matrices at $0.1$, and report the inferred diffusion graphs in Fig.~\ref{fig:SIR1}.
\begin{figure}[t]
    \centering
    \includegraphics[width=0.6\linewidth]{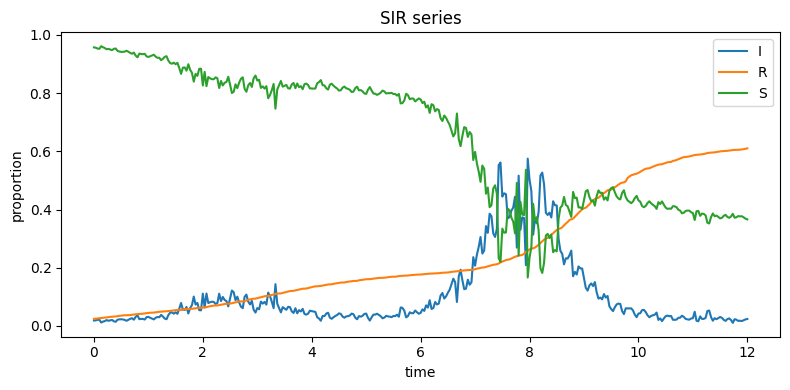}
    \caption{Visualization of the preprocessed SIR data from JHU CSSE. All components are expressed as population densities obtained by dividing by $N$.}
    \label{fig:sir}
\end{figure}
\begin{figure}[t]
\centering
\includegraphics[width=0.8\linewidth]{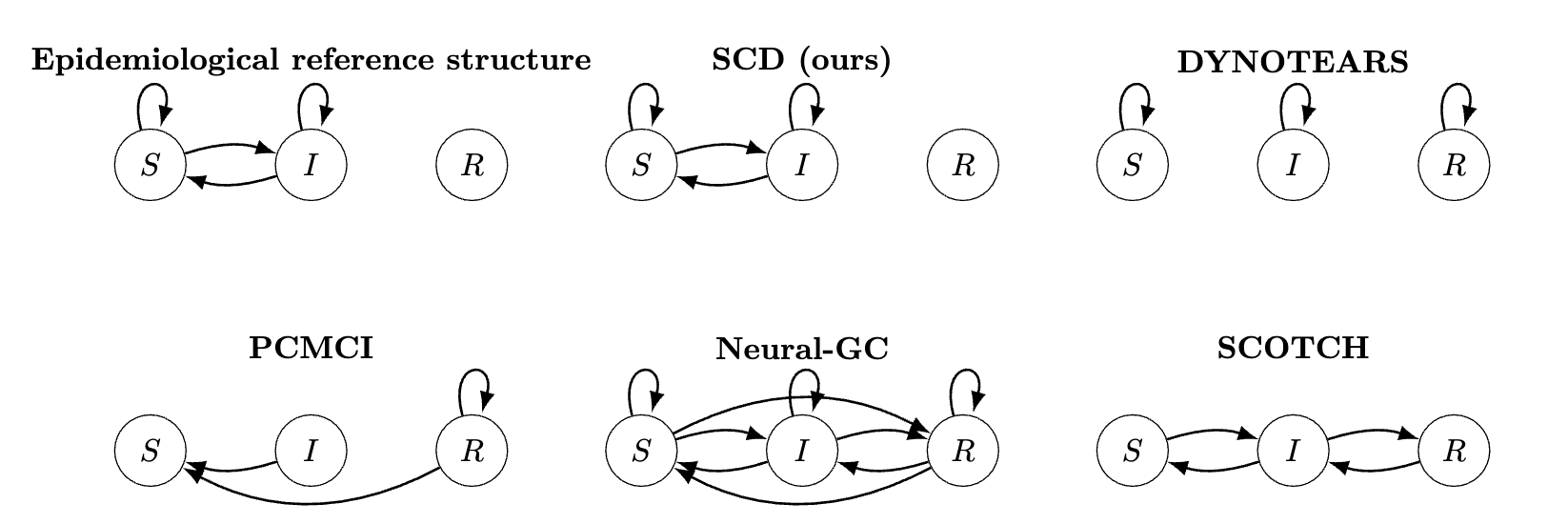}
\caption{Directed graph representations of the literature-motivated reference structure and the estimated adjacency patterns for the five methods on the COVID-19 SIR data. Here, an edge $j\to i$ is drawn whenever $A_{ij}=1$.}
\label{fig:SIR1}
\end{figure}
Fig.\ref{fig:SIR1} compares the thresholded diffusion graphs across methods on the COVID-19 SIR data. Among all methods, \texttt{SCD} is the most consistent with the literature-motivated reference structure: it yields a sparse graph concentrated in the $(S,I)$ block, with self-loops on $S$ and $I$, bidirectional coupling between these two compartments, and no explicit $R$-driven diffusion effect. This pattern also agrees with the stochastic SIR prior that the dominant stochastic dependence is driven by the susceptible and infected compartments, while the removed compartment does not act as an independent driver. By contrast, \texttt{DYNOTEARS} degenerates to self-loops only and fails to recover cross-compartment interactions, \texttt{PCMCI} introduces an $R \to S$ effect together with an $R$ self-loop, \texttt{Neural-GC} yields a nearly fully dense graph that is difficult to interpret, and \texttt{SCOTCH} recovers a chain-like structure $S \leftrightarrow I \leftrightarrow R$, suggesting additional coupling through $R$. Overall, these results indicate that SCD provides the most epidemiologically interpretable diffusion graph, and is the method most closely aligned with the literature-supported reference structure in this real-data setting.
\section{Conclusion}
\label{Sed:Dis}
We proposed a physics-informed framework for causal discovery in stochastic dynamical systems that incorporates available mechanistic knowledge through an SDE formulation. This perspective allows known ODE structure to be used as an inductive bias while explicitly modeling stochastic uncertainty in the system dynamics. Based on this formulation, we developed a learning procedure built on a stabilized Gaussian quasi-likelihood objective, which remains applicable when the mechanistic component is only partially known or misspecified. Taken together, these results position SDEs as a principled statistical framework for causal discovery when both prior physical structure and stochastic variability matter.

Several promising directions remain for future work. (1) A natural extension is to incorporate richer forms of mechanistic knowledge beyond standard ODEs, including delay differential equations, dynamical systems with time-varying parameters, and more general nonlinear functional specifications. (2) Stochastic partial differential equations (SPDEs) provide a natural next step for causal learning from spatiotemporal datasets, where spatial interactions and temporal evolution must be modeled jointly. (3) Although our current method supports certain forms of nonlinear causal discovery through functional dictionaries, developing more flexible nonlinear extensions remains an important direction. (4) Building on the same stochastic-calculus foundation, particularly It\^o's lemma, the framework could be extended toward causal representation learning by coupling latent SDE dynamics with an observation model. More broadly, our motivation is to incorporate physical knowledge into causal discovery, recognizing that causal conclusions are most reliable when the underlying scientific phenomenon is meaningfully represented by the model. Extending this idea beyond differential-equation-based formulations to richer forms of physical structure may further improve the reliability and applicability of causal discovery in scientific settings, although developing such a general and systematic integration remains challenging.

\acks{This work was supported by start-up funds from Northeastern University. The authors declare no competing interests related to financial activities outside the submitted work.}


\newpage

\appendix

\section{Additional Details}
\subsection{Optimization Algorithm and Convergence Justification}
\label{sec:opt_con}
we only briefly summarize the procedure below. For each node \(i\), define the row-wise objective
\[
F_i(\theta_i)=f_i(\theta_i)+\lambda\|\theta_i\|_1,
\]
where \(f_i(\theta_i)\) denotes the smooth part of Eq.\eqref{Eq:linear_pa}. The row-wise objective is optimized by proximal gradient iterations of the form
\[
\theta_i^{(m+1)}
=
\operatorname{prox}_{\eta_m\lambda\|\cdot\|_1}
\Bigl(
\theta_i^{(m)}-\eta_m \nabla f_i(\theta_i^{(m)})
\Bigr),
\]
where \(\eta_m>0\) is selected by monotone backtracking. Since \(c>0\), the stabilized smooth term \(f_i\) is continuously differentiable. In our implementation, the backtracking step accepts only updates that decrease the objective value, so that
\[
F_i(\theta_i^{(m+1)}) \le F_i(\theta_i^{(m)}).
\]
Since $F_i$ is bounded below, the objective values converge along the accepted iterates. By the standard sufficient-decrease property of proximal gradient with backtracking (see, e.g., \cite{DeMarchi2022}), we also have $\sum_{m =0}^\infty \|\theta_i^{(m+1)}-\theta_i^{(m)}\|_2^2 <\infty$, and thus $\|\theta_i^{(m+1)}-\theta_i^{(m)}\|_2\to 0$. The proximal optimality condition implies
\[
0 \in \nabla f_i\!\left(\theta_i^{(m)}\right)
+ \frac{1}{\eta_m}\left(\theta_i^{(m+1)} - \theta_i^{(m)}\right)
+ \lambda\,\partial \left\|\theta_i^{(m+1)}\right\|_1 .
\]
Therefore, any accumulation point \(\bar{\theta}_i\) of an infinite accepted iterate sequence is a first-order stationary point:
\[
0 \in \nabla f_i(\bar{\theta}_i) + \lambda\,\partial \|\bar{\theta}_i\|_1 .
\]

\subsection{More on Mixing Condition}
\label{sec:mix}
The following example is closely related to the predator--prey benchmark used in our paper. We want to note that this example is included only as a sanity check, illustrating that Assumption~\ref{ass:alpha-mixing-EM} is compatible with multiplicative-noise SDE models and can be verified under model-specific conditions. We consider a stochastic Lotka--Volterra system of the form
\[
\begin{aligned}
dV_t &=
(aV_t-bV_tR_t)\,dt+\sigma_1(V_t)\,dW_{1,t},\\
dR_t &=(cV_tR_t-dR_t)\,dt+\sigma_2(R_t)\,dW_{2,t},
\end{aligned}
\qquad (V_t,R_t)\in\mathbb R_+^2 .
\]
As a canonical multiplicative-noise special case, one may take $\sigma_1(V)=\eta_1 V, \sigma_2(R)=\eta_2 R,$ which gives
\[
\begin{aligned}
dV_t &= V_t(a-bR_t)\,dt+\eta_1 V_t\,dW_{1,t},\\
dR_t &= R_t(cV_t-d)\,dt+\eta_2 R_t\,dW_{2,t}.
\end{aligned}
\]
\textit{The diffusion coefficients are state-dependent} and degenerate on the coordinate boundaries. The factors \(V_t\) and \(R_t\) also preserve the positive orthant in the continuous-time model, making this a standard autonomous multiplicative-noise population system. For related stochastic predator--prey models, model-specific conditions are known under which the associated Markov process admits an invariant measure and has suitable long-run stability properties \citep{Huu2016}. In regimes where the fixed-\(\Delta\) skeleton chain $\{(V_{k\Delta},R_{k\Delta})\}_{k\ge 0}$ is geometrically \(\beta\)-mixing, and hence geometrically \(\alpha\)-mixing \citep{Meitz2021}. 

% From deterministic or otherwise nonstationary initial conditions, the process is generally not stationary, but geometric convergence to the invariant law provides the corresponding short-memory mechanism, with constants depending on the initial distribution or its Lyapunov moment.

\subsection{Experimental Setting Details}
\label{sec:app:experiment-details}
\noindent \textbf{Graph Generation:} \textit{For DAGs}, we generate random graphs using the \textit{Erd\H{o}s--R\'enyi (ER)} model with prescribed expected total degree \(pr\). Specifically, we set \(pr=\left\lceil \frac{p}{2}\right\rceil\), where \(p\) denotes the number of nodes. Each potential edge is sampled independently from a Bernoulli distribution with probability \(pr/p\). The resulting graph is represented as an adjacency matrix, and acyclicity is enforced by restricting the support to a strictly lower-triangular form, yielding a valid DAG. Finally, we apply a random permutation to the node labels to remove any trivial ordering effects, producing a randomized topology suitable for downstream experiments. After generating the binary adjacency matrix, we assign nonzero edge weights independently from a uniform distribution over \([-0.6,-0.3] \cup [0.3,0.6]\). \textit{For directed loop graphs}, we construct the graph on the node set \(\{0,1,\dots,p-1\}\) as follows. For each \(i=1,\dots,p-2\), we introduce reciprocal edges \(i \to i+1\) and \(i+1 \to i\), with weights sampled independently from the uniform distribution \(\mathrm{Unif}(w_{\min}, w_{\max})\), where \(w_{\min}=0.3\) and \(w_{\max}=1.3\). In addition, we connect the two chain endpoints to the sink node \(0\) by adding the edges \(1 \to 0\) and \((p-1) \to 0\), with weights drawn from the same distribution.

\noindent\textbf{Data Generation:} The data generation process involves two steps. First, we construct a weighted graph $G$ corresponding to the causal matrix $A$, and specify the ODE information corresponding to $B$ (and $b$, in the linearized form). For the causal structure, we follow the generation procedures described above for both graph families. Since our primary goal in this section is to evaluate the performance and robustness of graph inference, we use a decoupled first-order linear autonomous ODE system for the physics prior to facilitate verification of the stability condition in Sec.~\ref{sec: sde} and to clearly assess the effect of misspecification. Experiments based on real-world physics priors are deferred to Sec.~\ref{Sec:APP}. We simulate ODEs that include all self-loops (i.e., diagonal entries) and assign the weights uniformly over $[-1.5, -0.2]$. For the unstable regime, we generate the true ODE by setting a randomly selected subset of diagonal drift entries to zero. For the robustness examination, we include an additional experiment in which a subset of off-diagonal interaction terms in the ODE is included, reflecting real-world mechanistic priors that further restrict the search space. Second, given the ODE information and the specified graph structure, we generate time series data using the Euler--Maruyama approximation for the stochastic differential equations described in Sec.~\ref{sec: Setup}. Time series are generated on $[0,5]$ with 500 time steps using the Euler--Maruyama scheme ($\Delta t = 0.01$) for both the stable and unstable settings. 

\noindent\textbf{Benchmark Methods and Evaluation Metrics:} Since our method targets causal graph recovery from time-series data, we compare our method, described in Sec.~\ref{sec: method} and referred to as \texttt{SCD}, with four causal discovery methods for temporal data. The first baseline, \texttt{DYNOTEARS} \citep{pamfil2020dynotears}, is a score-based method for temporal DAG learning. The second baseline, \texttt{PCMCI} \citep{Runge2019PCMCI}, is a constraint-based method for time-series causal discovery based on conditional independence tests. The third baseline, \texttt{SCOTCH} \citep{wang2024neural}, is a variational-inference-based causal discovery method that combines neural SDEs with structure learning. The last baseline, \texttt{Neural-GC} \citep{Tank2021NGC}, is a Granger-causality-based method that uses a regularized neural network for causal discovery. We evaluate graph recovery using Structural Hamming Distance (SHD), True Positive Rate (TPR), and False Discovery Rate (FDR). SHD counts edge disagreements with the ground truth, including missing, extra, and misoriented edges, with lower values indicating better recovery. TPR measures the fraction of true edges correctly identified, whereas FDR measures the fraction of predicted edges that are false positives. Thus, higher TPR and lower SHD and FDR indicate better performance.

\noindent\textbf{Settings:} For the experimental setup, we tune the penalty parameters $\lambda$ and $c$ according to our theoretical guidance. Empirically, we find that $c=0.1$ and $\lambda=3$ perform well in low-dimensional settings. In high-dimensional settings, only a slight increase in $\lambda$ is needed. For \texttt{PCMCI} and \texttt{DYNOTEARS}, we use a one-lag ($\tau=1$) graph, since the loop structure can be represented in a lagged form under their modeling assumptions. We then aggregate the estimated lagged and contemporaneous graphs by taking the union of edges to obtain the final directed graph used for evaluation. All other settings for the remaining baselines are taken directly from the corresponding papers or their open-source implementations. After obtaining the estimated graphs, we apply thresholding to all methods. For our method, we find that a threshold in the range $[0.20, 0.25]$ works well across all settings. For \texttt{DYNOTEARS}, the recommended thresholds in the original implementation (e.g., $0.3$ for the contemporaneous graph and $0.1$ for the lagged graph) produced degenerate outputs with no selected edges in our experiments; therefore, we use smaller thresholds, namely $0.03$ for the contemporaneous graph and $0.01$ for the lagged graph. For \texttt{SCOTCH}, since its edge scores are probabilities produced by a logistic model, we threshold the output by fixing the threshold at $0.5$. All other baselines are thresholded according to their recommended settings.

We note that some baselines are developed under assumptions that do not fully align with our setting. In particular, several methods are primarily designed for DAGs or rely on acyclicity-related identifiability, whereas our directed loop graphs contain feedback loops (cyclic structures). Moreover, certain time-series baselines depend on stronger conditions such as stationarity, linear VAR dynamics, or specific noise models. When these assumptions are violated, these methods may be disadvantaged in our setting. We nevertheless report these baselines under a unified evaluation protocol as informative reference points for comparison.

\subsection{More Simulation Results}
\label{sec:full_simu}
\textbf{DAG with unstable system.} The DAG results under the unstable system are reported in Table~\ref{tab:dag_unstable}. The qualitative pattern is consistent with the stable-system results: \texttt{SCD} continues to achieve lower SHD, higher TPR, and lower FDR than the competing methods across dimensions. As $p$ increases from $5$ to $20$, the recovery problem becomes more challenging for all methods, but \texttt{SCD} remains comparatively stable. Its TPR ranges from $0.817$ to $0.894$, while its FDR ranges from $0.089$ to $0.404$. At $p=20$, \texttt{SCD} attains SHD $57.600$ and FDR $0.404$, whereas the baselines have SHD values ranging from $151.000$ to $174.200$ and FDR values ranging from $0.612$ to $0.911$. These results suggest that \texttt{SCD} remains robust even when the stability-aligned assumptions are weakened.

\textbf{Directed loop graph with stable system.} Table~\ref{tab:loop_stable} reports the directed loop graph results under the stable system. The qualitative pattern is consistent with the unstable-loop results in the main text. As $p$ increases from $5$ to $20$, SHD generally increases across methods, while \texttt{SCD} continues to achieve the lowest SHD and maintains high TPR. For example, at $p=20$, \texttt{SCD} achieves SHD $15.0$ with TPR $0.897$ and FDR $0.316$, whereas the baselines typically exhibit FDR in the range $0.73$--$0.90$, TPR below $0.13$, and substantially larger SHD. At lower dimensions, \texttt{SCD} also controls false discoveries effectively, with FDR $0.087$ at $p=5$ and $0.143$ at $p=10$. These results further support that \texttt{SCD} remains effective for recovering feedback-loop structures under both stable and unstable regimes.

Overall, the baselines exhibit substantial performance degradation in our experimental setting. \texttt{DYNOTEARS} typically shows very high FDR with low TPR, which results in large SHD. This is plausible in our setting because, although the underlying SDE dynamics are linear in form, time integration and multiplicative noise can induce dependence patterns that deviate from the stationary linear VAR-style assumptions underlying \texttt{DYNOTEARS}. \texttt{PCMCI} attains only modest TPR with moderate-to-high FDR and increased variability at larger $p$, which is also expected since it is primarily designed for stationary time series. \texttt{SCOTCH} likewise yields high FDR and low-to-moderate TPR as $p$ grows. One contributing factor is a mismatch in diffusion structure, as \texttt{SCOTCH} is developed for SDE models with a fully diagonal diffusion term, whereas our data-generating process can induce degenerate diffusion components (e.g., near-zero diffusion in some coordinates), violating this structural assumption. As shown by their data visualizations, \texttt{SCOTCH} is commonly demonstrated on settings with strongly increasing (potentially explosive) trajectories. In contrast, our simulations exhibit comparatively moderate diffusion magnitudes, which may represent a distribution shift relative to their typical evaluation regime and could partly explain the degraded performance. Finally, \texttt{Neural-GC} is grounded in a Granger-causality-style notion of temporal precedence; since feedback and contemporaneous interactions are present in data, this mismatch in causal semantics can also lead to degraded performance. These results indicate that \texttt{SCD} is well suited for extracting graph structure from this class of time series and remains robust when moving from stable to unstable dynamics.

\begin{table}[tbp]
\centering
\small
\setlength{\tabcolsep}{8pt}
\renewcommand{\arraystretch}{1.02}
\caption{Graph recovery results on \textit{DAGs} (Unstable System): mean $\pm$ std over 10 runs.}
\label{tab:dag_unstable}
\begin{tabular}{llcccc}
\toprule
\textbf{Metric} & \textbf{Method} & $p = 5$ & $p = 10$  & $p = 15$ & $p = 20$ \\
\midrule
\multirow{5}{*}{\textbf{SHD} $\downarrow$}
& \texttt{SCD} (ours)   & $2.800 \pm 1.720$ & $5.800 \pm 0.748$ & $22.700 \pm 3.801$ & $57.600 \pm 5.341$ \\
& \texttt{DYNOTEARS}    & $11.200 \pm 0.837$ & $26.600 \pm 5.225$ & $68.000 \pm 5.958$ & $157.000 \pm 3.082$ \\
& \texttt{PCMCI}        & $9.000 \pm 1.000$ & $29.200 \pm 9.935$ & $81.000 \pm 8.718$ & $169.200 \pm 7.085$ \\
& \texttt{SCOTCH}       & $8.400 \pm 2.191$ & $29.600 \pm 3.647$ & $80.000 \pm 6.856$ & $151.000 \pm 12.021$ \\
& \texttt{Neural-GC}    & $8.600 \pm 1.140$ & $41.800 \pm 5.541$ & $85.200 \pm 12.008$ & $174.200 \pm 10.640$ \\
\midrule
\multirow{5}{*}{\textbf{TPR} $\uparrow$}
& \texttt{SCD} (ours)   & $0.817 \pm 0.116$ & $0.864 \pm 0.047$ & $0.889 \pm 0.064$ & $0.894 \pm 0.036$ \\
& \texttt{DYNOTEARS}    & $0.245 \pm 0.160$ & $0.135 \pm 0.125$ & $0.123 \pm 0.046$ & $0.185 \pm 0.033$ \\
& \texttt{PCMCI}        & $0.401 \pm 0.247$ & $0.251 \pm 0.203$ & $0.289 \pm 0.052$ & $0.329 \pm 0.031$ \\
& \texttt{SCOTCH}       & $0.513 \pm 0.189$ & $0.222 \pm 0.085$ & $0.250 \pm 0.032$ & $0.325 \pm 0.056$ \\
& \texttt{Neural-GC}    & $0.503 \pm 0.108$ & $0.052 \pm 0.041$ & $0.192 \pm 0.059$ & $0.247 \pm 0.036$ \\
\midrule

\multirow{5}{*}{\textbf{FDR} $\downarrow$}
& \texttt{SCD} (ours)   & $0.089 \pm 0.083$ & $0.237 \pm 0.073$ & $0.359 \pm 0.052$ & $0.404 \pm 0.046$ \\
& \texttt{DYNOTEARS}    & $0.844 \pm 0.127$ & $0.875 \pm 0.125$ & $0.949 \pm 0.018$ & $0.911 \pm 0.026$ \\
& \texttt{PCMCI}        & $0.711 \pm 0.230$ & $0.738 \pm 0.149$ & $0.604 \pm 0.064$ & $0.612 \pm 0.061$ \\
& \texttt{SCOTCH}       & $0.489 \pm 0.243$ & $0.650 \pm 0.151$ & $0.710 \pm 0.097$ & $0.779 \pm 0.033$ \\
& \texttt{Neural-GC}    & $0.511 \pm 0.217$ & $0.887 \pm 0.120$ & $0.776 \pm 0.118$ & $0.788 \pm 0.027$ \\
\bottomrule
\end{tabular}
\end{table}

\begin{table}[tbp]
\centering
\small
\setlength{\tabcolsep}{8pt}
\renewcommand{\arraystretch}{1.02}
\caption{Graph recovery results on \textit{Directed loop graphs} (Stable System): mean $\pm$ std over 10 runs.}
\label{tab:loop_stable}
\begin{tabular}{llcccc}
\toprule
\textbf{Metric} & \textbf{Method} & $p = 5$ & $p = 10$  & $p = 15$ & $p = 20$ \\
\midrule

\multirow{5}{*}{\textbf{SHD} $\downarrow$}
& \texttt{SCD} (ours)   & $0.900 \pm 0.639$ & $4.400 \pm 1.990$ & $6.500 \pm 2.821$ & $15.000 \pm 2.646$ \\
& \texttt{DYNOTEARS}    & $10.200 \pm 1.924$ & $23.000 \pm 3.674$ & $74.200 \pm 10.895$ & $62.200 \pm 5.070$ \\
& \texttt{PCMCI}        & $8.400 \pm 2.074$  & $28.800 \pm 9.783$ & $60.000 \pm 15.050$ & $65.200 \pm 34.629$ \\
& \texttt{SCOTCH}       & $9.800 \pm 0.447$ & $36.000 \pm 6.708$ & $61.000 \pm 3.606$ & $127.500 \pm 32.889$ \\
& \texttt{Neural-GC}    & $10.800 \pm 1.643$ & $34.800 \pm 3.899$ & $96.200 \pm 18.660$ & $95.600 \pm 12.720$ \\
\midrule

\multirow{5}{*}{\textbf{TPR} $\uparrow$}
& \texttt{SCD} (ours)   & $0.975 \pm 0.051$ & $0.899 \pm 0.073$ & $0.941 \pm 0.055$ & $0.897 \pm 0.035$ \\
& \texttt{DYNOTEARS}    & $0.404 \pm 0.070$ & $0.250 \pm 0.151$ & $0.156 \pm 0.042$ & $0.120 \pm 0.019$ \\
& \texttt{PCMCI}        & $0.371 \pm 0.260$ & $0.115 \pm 0.106$ & $0.145 \pm 0.062$ & $0.116 \pm 0.107$ \\
& \texttt{SCOTCH}       & $0.343 \pm 0.062$ & $0.172 \pm 0.051$ & $0.109 \pm 0.012$ & $0.094 \pm 0.008$ \\
& \texttt{Neural-GC}    & $0.420 \pm 0.045$ & $0.145 \pm 0.089$ & $0.160 \pm 0.021$ & $0.093 \pm 0.036$ \\
\midrule

\multirow{5}{*}{\textbf{FDR} $\downarrow$}
& \texttt{SCD} (ours)   & $0.087 \pm 0.089$ & $0.143 \pm 0.072$ & $0.184 \pm 0.065$ & $0.316 \pm 0.061$ \\
& \texttt{DYNOTEARS}    & $0.475 \pm 0.105$ & $0.844 \pm 0.133$ & $0.636 \pm 0.081$ & $0.900 \pm 0.022$ \\
& \texttt{PCMCI}        & $0.675 \pm 0.338$ & $0.811 \pm 0.191$ & $0.779 \pm 0.125$ & $0.879 \pm 0.164$ \\
& \texttt{SCOTCH}       & $0.725 \pm 0.137$ & $0.744 \pm 0.093$ & $0.833 \pm 0.041$ & $0.730 \pm 0.092$ \\
& \texttt{Neural-GC}    & $0.100 \pm 0.105$ & $0.756 \pm 0.228$ & $0.443 \pm 0.090$ & $0.816 \pm 0.100$ \\
\bottomrule
\end{tabular}
\end{table}

\subsection{COVID-19 Data and ODE Preprocessing}
\label{sec:prepro}

We use the U.S. COVID-19 time-series data provided by the Johns Hopkins University Center for Systems Science and Engineering (JHU CSSE) \citep{Dong2020AnIW}, specifically the county-level confirmed-case and death data time series COVID19 confirmed US and time series COVID-19 deaths US. For each table, we remove unused metadata columns, reshape the data from wide to long format over calendar dates, and convert the date field to a standard datetime variable. We then restrict attention to a single state and aggregate county-level observations to the state level by summing over counties on each date. The confirmed-case and death series are merged by date to obtain a unified state-level panel. From the cumulative confirmed-case series, we first compute daily new confirmed counts. We then construct the infectious compartment $I_t$ as a rolling sum of recent new confirmed cases over a fixed infectious window, which serves as a proxy for active infections. The removed compartment is defined as $R_t = C_t - I_t$, where $C_t$ denotes the cumulative confirmed count, and the susceptible compartment is defined by $S_t=N - I_t - R_t$. This yields a constructed SIR trajectory on a common normalized scale. Finally, we truncate the sample to the study period from 2020-05-06 to 2021-05-01 and map the observation dates to an equally spaced time grid for downstream modeling. After preprocessing, the resulting trajectories exhibit the canonical qualitative behavior of an SIR system and, by construction, satisfy the conservation relation $S_t+I_t+R_t=N$. In our experiment, we set $N = 100000$, and both $R_t$ and $I_t$ are represented on the normalized scale obtained by dividing by $N$. The processed data are shown in Fig.\ref{fig:sir}.

As the ODE structure is given, we first fit the deterministic SIR drift parameters to the preprocessed data by solving the classical SIR system in Eq.\eqref{eq:sir} using the observed initial condition $(S_0,I_0,R_0)$. Specifically, we implement the SIR vector field in JAX and solve the ODE numerically over the observation grid using the Tsit5 solver from \texttt{diffrax}, an explicit Runge--Kutta method. The transmission and removal parameters $(\beta,\gamma)$ are then estimated by minimizing the sum of squared residuals between the ODE solution and the observed SIR trajectories via L-BFGS optimization. After optimization, the estimated parameters $\hat{\beta}$ and $\hat{\gamma}$ are used to reconstruct the fitted deterministic trajectory on a refined time grid. This fitted ODE solution serves as the mechanistic drift component for the subsequent stochastic modeling and diffusion-structure inference. We emphasize that the physics prior is misspecified, as its parameters are estimated in a purely data-driven manner and therefore cannot be numerically verified against the true underlying dynamics. 

\section{Technical Details}
\subsection{Technical details for Proposition~\ref{thm:gen_lyap_AB}}
\label{sec: sde_state}

For transparency, we first present the proof in a simplified single-channel linear multiplicative-noise setting, which serves as a technical reference for the main argument. The statement in the main text corresponds to the multi-channel formulation
\[
dX_t = B X_t\,dt + \sum_{i=1}^p G_i X_t\,dW_t^{(i)},
\]
where \(G_i=e_iA_i^\top\). The proof of the displayed main-text version follows by the same argument, replacing the single diffusion contribution by the sum of the corresponding channel-wise terms. We therefore record the core stability argument here in the simplified form and use it as a blueprint for Proposition~\ref{thm:gen_lyap_AB}. Before we move to the technical detail, we first show the Definition of stability
\begin{definition}
Consider the stochastic differential equation with $t \ge 0$
\begin{equation}
\begin{split}
&dx(t) = f(x(t),t) dt + g(x(t),t) dW(t), \\
&x(0) = x_0 
\end{split}
\end{equation}
The solution $x(t,x_0)$ is called pathwise asymptotically stable about equilibrium point $x_e\in\mathbb{R}^p$ if
\begin{equation}
    \lim_{t\to\infty}\|x(t;x_0)-x_e\|=0 \quad \text{almost surely},
\end{equation}
and pathwise exponentially stable if
\begin{equation}
    \limsup_{t\to\infty}\frac{1}{t}\log\|x(t;x_0)-x_e\|<0 \quad \text{almost surely},
\end{equation}
for all initial conditions $x_0\in\mathbb{R}^p$.
\end{definition}
\begin{remark}
The above definitions are asymptotic (as $t\to\infty$). In practice, we only observe data over a finite horizon $[0,T]$ with $T<\infty$. Rather than relying on asymptotic behavior, we only use the stability-implied properties needed for our statistical analysis.
\end{remark}
\begin{proposition}[Lyapunov bound for the pathwise Lyapunov exponent]
\label{thm:lyap-exponent-bound}
Consider the linear stochastic differential equation in Eq.\ref{GSDE}. Suppose there exist a positive definite matrix $\Sigma \in \mathbb{R}^{d\times d}$ and a real number
$\lambda \in \mathbb{R}$ such that
\begin{equation}
\label{eq:lyap-ineq}
    \langle \Sigma x(t), x(t)\rangle \bigr[2\langle \Sigma A x(t),x(t)\rangle + \langle \Sigma Bx(t),Bx(t)\rangle - 2 \lambda \langle \Sigma x(t), x(t)\rangle\bigl] \le 2\langle \Sigma x(t),B x(t)\rangle^2
\end{equation}
Then the solution $x(t)$ of \eqref{GSDE} satisfies
\begin{equation}
\label{eq:lyap-exponent-claim}
\limsup_{t\to\infty} \frac{1}{t} \log |x(t)|   \le   \lambda
\quad \text{almost surely}.
\end{equation}
\end{proposition}
Thus, the core idea is to show Eq.\eqref{eq:lyap-ineq} hold. We then impose the following Lemma

\begin{lemma}
\label{lem:35}
Consider the linear SDE  with drift matrix $A$ and diffusion matrices $B_i$, $i=1,\dots,k$.
Suppose there exist a positive definite matrix $Q\in\mathbb{R}^{n\times n}$, a constant $\lambda\in\mathbb{R}$, and scalars
$b_i\in\mathbb{R}$, $i=1,\dots,k$, such that
\begin{equation}
\label{eq:cond13}
\Big(A+\sum_{i=1}^k b_i B_i\Big)^\top Q
+ Q\Big(A+\sum_{i=1}^k b_i B_i\Big)
+ \sum_{i=1}^k B_i^\top Q B_i
+ \Big(\frac12\sum_{i=1}^k b_i^2 - 2\lambda\Big) Q \preceq 0 .
\end{equation}
Then condition \eqref{eq:lyap-ineq} holds, and hence
\begin{equation}
\label{eq:limsup14}
\limsup_{t\to\infty}\frac{1}{t}\log\|x(t)\|\le \lambda
\qquad\text{a.s.}
\end{equation}
\end{lemma}

\begin{lemma}\label{lem:lyap_integral}
If the conditions as stated in Proposition~\ref{thm:gen_lyap_AB} hold, for any matrix $N$, the matrix $P:=\int_0^\infty e^{Bt} N e^{B^\top t}\,dt$ is well-defined and satisfies $B P + P B^\top = -N.$
\end{lemma}

\subsection{Technical details for Sec~\ref{sec: sde} and Sec~\ref{Sec: unsta}}
\label{sec:lemma}
Before introducing the lemmas, we emphasize our key technical details. The challenge is that the objective in Eq.\eqref{Eq:A_learning} is nonconvex. Accordingly, basic optimization theory typically guarantees only that the algorithm returns a stationary point $\hat\theta_i$. However, stationarity alone is generally insufficient to ensure statistical identifiability: it can be difficult to relate an arbitrary stationary point $\hat\theta_i$ to the ground truth $\theta_i^*$ and to establish a nonasymptotic error bound.

We introduce a local restricted strong convexity (LRSC) condition to address this challenge and to establish identifiability. This technical is inspired by the restricted strong convexity (RSC) condition in \cite{poling2015}. Specifically, they restrict the feasible set so that the nonconvex loss becomes strongly convex throughout the restricted region, yielding a global curvature guarantee within that set. 

Technically, they establish the RSC based on $\|\Delta\|_2\le 1$ and $\|\Delta\|_2\ge 1$ for any $\Delta\in \mathbb R^p$. Then, by tuning the $\lambda$ and restricted feasible set to ensure the stationary points can be within $\|\hat\theta_i- \theta_i^*\|\le 1$. Then, they can introduce the a cone condition to establish the establish a nonasymptotic error bound. For more technical details, see \cite{poling2015}. 

However, our objective is based on a Gaussian quasi-likelihood whose curvature depends on the unknown ground-truth parameter. As a result, a global curvature guarantee over a fixed restricted feasible set is not sufficient for our analysis. Fortunately, the SDE stability condition together with Proposition~\ref{Ass:xbound}, provides the necessary control on the ground-truth dynamics and associated quantities. Under these properties, we find that tuning the slack parameter $c$ ensures the required regularity holds in the region of interest. Therefore, by restricting the feasible set we can establish the cone condition, which in turn enables a local restricted strong convexity (LRSC) argument. Then, we list all lemmas that we used for proving our Theorem~\ref{Thm: LASSO}. All proofs of the lemmas are included in \S\ref{sec: Lemma}.

The argument proceeds by first establishing several lemmas that characterize stationary points of Eq.\eqref{Eq:linear_pa} and establish the key high-probability properties needed for our analysis.
\begin{lemma}
 \label{KKT}
For each node $i \in [p]$ with the loss function given by Eq.\eqref{Eq:linear_pa}, the gradient of $\mathcal{L}_i(\theta_i)$ is given by
\begin{equation}
    \nabla_{\theta_i}\mathcal{L}_i(\theta_i)
=\frac{1}{n}\sum_{k=0}^{n-1} 
u_{i,k}(\theta_i) X_k, 
\end{equation}
where $u_{i,k}(\theta_i) =\frac{(\theta_i^\top X_k)\big(\theta_i^\top X_k)^2+c-r_{i,k}^2\big)}{((\theta_i^\top X_k)^2+c)^2}$. Consider the Lasso problem and define $G(\theta_i)\coloneqq \nabla_{\theta_i}\mathcal L_i(\theta_i)$,
\begin{equation}
\label{Lem_eq: stationary}
    \hat{\theta}_i=\arg\min_{\theta_i}  
\mathcal L_i(\theta_i)+\lambda\|\theta_i\|_1 .
\end{equation}
Then $\hat{\theta}_i$ is a stationary point if and only if for every $b\in[p]\setminus\{i\}$,
\begin{equation}
    \begin{cases}
G_b(\hat{\theta}_i)=-\lambda\mathrm{sign}(\hat{\theta}_{ib}), & \hat{\theta}_{ib}\neq 0,\\
|G_b(\hat{\theta}_i)|\le \lambda, & \hat{\theta}_{ib}=0,
\end{cases}
\end{equation}
where $\hat \theta_{ib}$ is the $b$-th entry of vector $\theta_i$. 
\end{lemma}
The Lemma~\ref{KKT} characterizes stationary points of the $\ell_1$-regularized objective via first-order optimality KKT conditions. It also provides an explicit expression for the gradient of $\mathcal{L}_i$ and the corresponding stationarity conditions for the Lasso problem. Then, the following lemma provides high-probability bounds for the gradient and Hessian.

\begin{lemma}
\label{lem:prop}
Given the objective function from Eq.\ref{Eq:linear_pa} for all $i\in [p]$, the Gradient and Hessian at population level $\theta^*_i$ has the following properties:
\begin{enumerate}
    \item Denote that $s \coloneqq |pa(i)|$ and by Proposition~\ref{Ass:xbound}. We have 
    \begin{equation}
    \label{lemeq:4_1}
\|\nabla_{\theta_i}\ell_i(\theta_i^*)\|_{\infty}
= O_{\mathbb P}\Big(\sqrt{\tfrac{\log p}{n}}\Big).
    \end{equation}
    \item Let $V_i \coloneqq pa(i) \subset [p]$ be the set of all parents of the node $i$, $Q^n_{V_iV_i}$ be the empirical and $Q^*_{V_iV_i}$ be the population fisher information matrix. We have
    \begin{equation}
    \label{lemeq:4_2}
\|Q^n_{V_iV_i}-Q^*_{V_iV_i}\|_\infty=O_{\mathbb P}\!\left(s\sqrt{\frac{\log s}{n}}+s\frac{(\log n)^2\log s}{n}\right).
    \end{equation}
    \item By Assumption~\ref{ass: Incoh} and Proposition~\ref{prop: WD}, the empirical Fisher information also has the following property such that, with high probability,
    \begin{equation}
    \begin{split}
    &\big\|Q^n_{V_i^cV_i} (Q^n_{V_iV_i})^{-1}\big\|_\infty \le  1-\alpha+ O_{\mathbb P}\left(\frac{s^{3/2}}{C_{\min}}\,r_n + \frac{s^2}{C^2_{\min}}r^2_n\right),
    \end{split}
\end{equation}
where $r_n= s\sqrt{\frac{\log p}{n}}+ s\frac{(\log n)^2\log p}{n}$ and $C_{\min}$ is from Proposition~\ref{prop: WD}. For sufficiently large $n$ and $\frac{s^{3/2}}{C_{\min}}\,r_n + \frac{s^2}{C^2_{\min}}r^2_n = o(1)$, we have 
\[
\left\|Q^n_{V_i^cV_i}(Q^n_{V_iV_i})^{-1}\right\|_\infty \le 1-\frac{\alpha}{2}
\]
\end{enumerate}
Here $\|\cdot\|_\infty$ denotes the matrix $\ell_\infty$ norm (maximum absolute row sum) in Eq.\eqref{lemeq:4_2} and the vector $\ell_\infty$ norm in  Eq.\eqref{lemeq:4_1}.
\end{lemma}
The next two lemmas establish support recovery for the $\ell_1$-regularized estimator. Lemma~\ref{lem:True_supp} states that the estimated parents match the ground truth with high probability, while Lemma~\ref{lem: PDW} provides a primal dual witness (PDW) verification of the required KKT conditions.

\begin{lemma}
 \label{lem:True_supp}
     Let $\hat \theta_i$ be given by Eq.\ref{Eq:linear_pa} for $i \in V$ with the $\lambda$ chosen as in Theorem~\ref{Thm: LASSO} and for some $c > 0$ for all $j \in pa(i)$ such that 
     \begin{equation}
     \begin{split}
     \mathbb P (|\operatorname{sign}(\hat \theta_{ij})| = |\operatorname{sign}( \theta^*_{ij})|)= 1- O(\exp (-cn)). 
     \end{split}
    \end{equation}
\end{lemma}

\begin{lemma}
\label{lem: PDW}
Given  a subgradient $\hat{Z}_i \in \partial \|\hat{\theta}_{i,S_i}\|_1$ whose entries satisfy $(\hat{Z}_i)_j = \mathrm{sign}(\hat{\theta}_{ij})$ for $ \hat{\theta}_{ij}\neq 0$; otherwise, $|(\hat{Z}_i)_j|\le 1$ for $\hat{\theta}_{ij}=0$. Fix a node $i \in [p]$. The primal–dual pair $(\hat\Theta_{i}, \hat Z_{i})$ for the oracle solution of the restricted subproblem satisfies the strict dual feasibility condition with probability at least $1 - c_1 e^{-c_2 n}$, for some constants $c_1, c_2 > 0$, for all sufficiently large $n$.
\end{lemma}
Before proving the next lemma, we recall a technical fact. Given the empirical loss given by Eq.\eqref{Eq:linear_pa}, the loss is smooth and derivatives are well-defined. Therefore, we can define $z = (a_i^\top \bm x(k))^2$, and rewrite it as $\phi_k(z)=\log(z+c)+\frac{r_{i,k}^2}{z+c}$. There exists $ M<\infty $ such that $M \coloneqq \max_{0\le k\le n-1}\ \sup_{z\ge0}\big|\phi_k'''(z)\big|$ with high probability. We now establish a high-probability Hessian Lipschitz property.

\begin{lemma}[High-probability Hessian Lipschitz]
\label{lem:h_lips}
Given $c>0$ and empirical loss function by Eq.\eqref{Eq:linear_pa}. Suppose the conditions from Proposition.\ref{Ass:xbound} and Assumption~\ref{SDE_gen} hold. For node $i$, any $\theta_i,\theta_i'\in\mathbb R^p$ will satisfy the following with probability at least $1 - c_3 e^{-c_4\beta t}$,
\begin{equation}
\big\|\nabla^2\ell(\theta_i)-\nabla^2\ell(\theta_i')\big\|_\infty\le
\bar D_{\max}\|\theta_i-\theta_i'\|_1
\end{equation}
where $\|\cdot\|_\infty$ for matrices denotes the entrywise maximum norm and $\bar D_{\max}
\coloneqq\frac{K^3}{2}\Big(\frac{24}{c^{3/2}}+\frac{48R^2}{c^{5/2}}\Big).$
\end{lemma}
Before we move to the next lemma, we recall a technical fact that $r_{i,k} = {\theta_i^*}^\top X_k\epsilon_k$. Moreover, since the stability from Proposition~\ref{thm:gen_lyap_AB} show us that $m^2\|A\|^2 - 2\omega < 0$ for $m\ge 1$ and $w>0$ . Since $\theta^*_i$ is a row vector from $A$ and the $\|X_k\| \le K$ with high probability, we can say there exist $\mathcal B> 0 $ such that $|\theta_i^*X_k | \le K\mathcal B$ with high probability. Then, we conclude the following lemma 
\begin{lemma}[Local Restricted Strong Convexity]
\label{lem:RSC}
Given the empirical loss $\ell(\theta)$ in Eq.\eqref{Eq:linear_pa}, suppose Lyapunov stability holds and choose the tuning parameter such that $c\ge 9 \mathcal B^2 K^2$. From Proposition~\ref{Ass:xbound} and Proposition~\ref{prop: WD}, for any $\theta$ satisfying $\|\theta\|_2 \le \sqrt{\frac{c}{2K^2}}$, there exist constants $\alpha_1>0$ and $\tau \ge 0 $ with $r_n \asymp \sqrt{\frac{\log p}{n}}$ such that, with high probability, for all $\Delta\in\mathbb R^p$,
\begin{equation}
\Delta^\top \nabla^2 \ell(\theta)\Delta \ge \alpha_1\|\Delta\|_2^2 -\tau r_n\|\Delta\|_1^2 
\end{equation}
\end{lemma}
Consequently, we obtain the following lemma.
\begin{lemma}
\label{lem:err_rsc}
Let $\hat\theta$ be any stationary point with the penalty chosen as in Lemma~\ref{lem:WR_bound}, and
suppose the tuning parameter $c$ and feasible set satisfy the conditions of Lemma~\ref{lem:RSC}. Then, $\Delta \coloneqq  \hat\theta-\theta^*$ satisfies
\begin{equation}
\|\Delta\|_2 \le \frac{3\lambda\sqrt{s}}{2(\alpha-16\tau s)},
\qquad\|\Delta\|_1 \le \frac{6\lambda s}{\alpha-16\tau s}.
\end{equation}
\end{lemma}

\begin{lemma}[LRSC for Unstable Multiplicative SDE]
\label{lem:RSC_unstable}
Given the empirical loss $\ell(\theta)$ in Eq.\eqref{Eq:linear_pa}, suppose Assumption~\ref{ass:localization} and Assumption~\ref{ass:pred_curvature} hold and choose the tuning parameter such that $c\ge 9 \|\theta^*\|_1^2 K^2$. For any $\theta$ satisfying $\|\theta\|_1 \le \sqrt{\frac{c}{2K^2}}$, there exist constants $\alpha_1>0$ and $\tau \ge 0 $ with $r_n = O( \sqrt{\frac{\log p}{n}} + \frac{\log p}{n})$ such that, with high probability, for all $\Delta\in \mathcal R(s;c_0)$,
\begin{equation}
\Delta^\top \nabla^2 \ell(\theta)\Delta \ge
\alpha_1\|\Delta\|_2^2
 -\tau r_n\|\Delta\|_1^2 
\end{equation}
\end{lemma}

\subsection{Technical details for Section~\ref{Sec: ODE miss}}
\label{Sec:TD_odemiss}
\begin{lemma}[Gradient Bound Under ODE Misspecification]
\label{lemma:ODE_miss gradient}
Suppose the residual under ODE misspecification satisfies Eq.\eqref{Eq:ODE_miss}, and let the SDE solution $X_k$ satisfy the Proposition~\ref{Ass:xbound}.  Then there exist constants $C, C_1, C_2 >0$, with probability at least $1- C_1p^{-C_2}$, such that
\[
\|\nabla \tilde \ell^{(\delta)}_i(\theta_i^\star)\|_\infty
\le\left(\left(\frac{3\sqrt{3}}{8}+C\right)\frac{K}{\sqrt c} + \left( \frac{3\sqrt{3}}{8}\frac{K\delta^2\Delta t}{c^{3/2}} + C\frac{K}{c}\delta\sqrt{\Delta t} \right)\right)\sqrt{\frac{\log p}{n}}.
\]
In particular, if $\delta \lesssim \frac{\sqrt c}{\sqrt{\Delta t}}$ with $n$ large enough, then
\[
\|\nabla \tilde \ell_i(\theta_i^\star)\|_\infty = O_{\mathbb P}\left(\sqrt{\frac{\log p}{n}}\right).
\]
\end{lemma}

\begin{lemma}[Local Restricted Strong Convexity Under ODE Misspecification]
\label{lemma:ODE_miss_lrsc}
Given the misspecified empirical loss \(\tilde{\ell}_i^{(\delta)}(\theta)\) under Eq.\eqref{Eq:ODE_miss}, suppose Lyapunov stability holds and choose the tuning parameter such that $c \ge 9\bigl(B^2K^2+\delta^2\Delta t\bigr)$. From Proposition~\ref{Ass:xbound} and Proposition~\ref{prop: WD}, for any \(\theta\) satisfying $\|\theta\|_2 \le \sqrt{\frac{c}{2K^2}}$, there exist constants \(\alpha_\delta>0\) and \(\tau_\delta\ge 0\) with $r_n \asymp \sqrt{\frac{\log p}{n}},$ such that, with high probability, for all \(\Delta\in\mathbb{R}^p\),
\[
\Delta^\top \nabla^2 \tilde{\ell}_i^{(\delta)}(\theta)\Delta \ge \alpha_\delta \|\Delta\|_2^2-\tau_\delta r_n \|\Delta\|_1^2 .
\]
where $\alpha_\delta = m\gamma_{c,\delta}$, $\gamma_{c,\delta}:=\frac{2}{9c}-\frac{B^2K^2+\delta^2\Delta t}{c^2}$, and $\tau_\delta r_n = \gamma_{c,\delta}r_n + \frac{C(1+|\delta|\sqrt{\Delta t})}{c^2}\sqrt{\frac{\log p}{n}}$.  In particular, if $\delta \le \sqrt{\frac{c/9-B^2K^2}{\Delta t}}$, we have the following, which is same as LRSC form from Lemma~\ref{lem:RSC}.
\begin{equation}
\Delta^\top \nabla^2 \tilde{\ell}_i^{(\delta)}(\theta)\Delta\ge \alpha_1\|\Delta\|_2^2 -\tau r_n\|\Delta\|_1^2 
\end{equation}
\end{lemma}
\begin{lemma}[High-probability Hessian Lipschitz under ODE misspecification]
\label{lemma:ODE_miss_h_lip}
Given the misspecified empirical loss \(\tilde{\ell}_i^{(\delta)}\) under Eq.\eqref{Eq:ODE_miss}, suppose the conditions from Proposition~\ref{Ass:xbound} and Assumption~\ref{SDE_gen} hold. For each node $i$, any $\theta_i,\theta_i'\in\mathbb R^p$ will satisfy the following with probability at least $1 - c_3 e^{-c_4\beta t}$,
\[
\bigl\|\nabla^2 \tilde{\ell}_i^{(\delta)}(\theta_i)-\nabla^2 \tilde{\ell}_i^{(\delta)}(\theta'_i)\bigr\|_\infty
\le
\bar D_{\max}^{(\delta)}\|\theta_i-\theta_i'\|_1,
\]
where $\bar D_{\max}^{(\delta)} := \frac{K^3}{2}\left(\frac{24}{c^{3/2}} + \frac{48\bigl(R+\delta\sqrt{\Delta t}\bigr)^2}{c^{5/2}} \right)$. 
In particular, if $\delta \le \frac{1}{\sqrt{\Delta t}}\left(\sqrt{R^2+\frac{c}{2}}-R\right)$, we have $\bar D_{\max}^{(\delta)}\le\frac{K^3}{2}\Big(\frac{48}{c^{3/2}}+\frac{48R^2}{c^{5/2}}\Big)$, which has the same functional form as the Hessian Lipschitz bound in Lemma~\ref{lem:h_lips}.
\end{lemma}

\begin{lemma}[Hessian block concentration under ODE misspecification]
\label{lemma:ODE_miss_hessian}
Given the misspecified residual satisfies Eq.\eqref{Eq:ODE_miss}, and assume Proposition~\ref{Ass:xbound}  and Assumption~\ref{ass:alpha-mixing-EM} hold. Let \(V_i \coloneqq \operatorname{pa}(i)\subset [p]\) and \(s\coloneqq|V_i|\). Define the empirical and population Hessian blocks under the misspecified loss by
\[
\widetilde Q_{AB}^{\,n,(\delta)} \coloneqq\left[\frac{1}{n}\sum_{k=1}^n \nabla^2 \widetilde \ell_{i,k}^{(\delta)}(\theta_i^*)\right]_{AB},\qquad\widetilde Q_{AB}^{\,*,(\delta)}\coloneqq\left[\frac{1}{n}\sum_{k=1}^n \mathbb E\!\left(\nabla^2 \widetilde \ell_{i,k}^{(\delta)}(\theta_i^*)\right)\right]_{AB},
\]
for any index sets \(A,B\subset[p]\). Then there exists a constant \(C_\delta \asymp 1 + |\delta|\sqrt{\frac{\Delta t}{c}} + \frac{\delta^2\Delta t}{c}\), such that
\[
\bigl\|\widetilde Q_{V_iV_i}^{\,n,(\delta)}-\widetilde Q_{V_iV_i}^{\,*,(\delta)}\bigr\|_\infty = O_{\mathbb P}\!\left( C_\delta\left(s\sqrt{\frac{\log s}{n}}+s\frac{(\log n)^2\log s}{n}\right)\right),
\]
In particular, if \(\delta \le \sqrt{\frac{c}{\Delta t}}\), we have 
\begin{equation}
\bigl\|\widetilde Q_{V_iV_i}^{\,n,(\delta)}-\widetilde Q_{V_iV_i}^{\,*,(\delta)}\bigr\|_\infty=O_{\mathbb P}\!\left(s\sqrt{\frac{\log s}{n}}+s\frac{(\log n)^2\log s}{n}\right).
\end{equation}
\end{lemma}
\subsection{Technical details for Section~\ref{Sec:c_role}}
\label{Sec:IFT}
\begin{lemma}[Local restricted branch]
\label{lem:local_restricted_branch}
Under Assumption~\ref{ass:local_restricted_regularity}, there exist $c_{0,i}>0$, an open neighborhood $U_i\subset \mathbb R^{s_i}$ of $\theta_S^\star$, and a unique $C^1$ map $\psi_i:(-c_{0,i},c_{0,i})\to U_i$ for all $ |c|<c_{0,i}$ such that
\[
\psi_i(0)=\theta_S^\star \quad\text{and}\quad F_i(\psi_i(c),c)=0.
\]
In particular, for every $c\in[0,c_{0,i})$, the restricted score equation $F_i(\vartheta,c)=0$ admits a unique solution in $U_i$, namely $\vartheta_{i,c}^\star:=\psi_i(c)$. Moreover, there exist $r_i>0$ and $\bar c_{0,i}\in(0,c_{0,i}]$ such that the above conclusions hold with
\[
U_i = B(\theta_S^\star,r_i) \quad\text{for all } c\in[0,\bar c_{0,i}).
\]
\end{lemma}

\begin{lemma}
\label{lem:uniform_inverse_hessian}
Under Assumption~\ref{ass:local_restricted_regularity}, there exist $\widetilde c_{0,i}\in(0,c_{0,i}]$, an open neighborhood $\widetilde U_i\subseteq U_i$ of $\theta_S^\star$, and a finite constant $C_{H,i}>0$ such that
\[
\sup_{\vartheta\in \widetilde U_i, c\in[0,\widetilde c_{0,i}]} \|H_i(\vartheta,c)^{-1}\|_{\infty} \le C_{H,i}.
\]
Then, for every $c\in[0,\widetilde c_{0,i}]$, we have $\|\vartheta_{i,c}^\star-\theta_S^\star\|_\infty \le C_{H,i}\,\Delta_{i,c}$, where $\Delta_{i,c}:=\|F_i(\theta_S^\star,c)\|_\infty =\|F_{i,c}^{(S_i)}(\theta_{i,S_i}^\star)\|_\infty$.
\end{lemma}

\section{Proofs for Main Result and Technical Details}
\subsection{Proofs for the SDE statements}
\label{sec: sde_Lemma}
\subsubsection{Proof of Proposition~\ref{prop: WD}}
\begin{proof}
To show existence and uniqueness of solutions, we use the Picard iteration, which is the same successive approximation procedure one uses for analyzing the ODE. We construct a sequence of processes $(X^{(n)})_{n\in \mathbb N}$ recursively as 
\begin{equation}
    X^{(0)} = X_0, \qquad X_t^{(n+1)} \coloneqq X_0 + \int_0^t g(X^n_s)ds+ \int_0^t S_f(X^n_s)dW_s
\end{equation}
Each of these processes is evidently adapted. Moreover, by using the Assumption~\ref{ASS:LG}, we can prove that each $X_t^{(n)}$ has continuous samples path, and satisfies $\int_0^t \mathbb E[X_s]^2 < \infty$ for all $0\le t\le T$. We then define $\Delta_t^n \coloneqq\mathbb E[(X^{(n+1)} -X^{(n)})^2]$, and we can have the following
\begin{equation}
\begin{split}
    X^{(n+1)} -X^{(n)} &= \int_0^t [g(X^n_s)- g(X^{n-1}_s)]ds + \int_0^t [S_f(X^n_s)-S_f(X^{n-1}_s)]dW_s\\
     \implies \Delta_t^n &=  \mathbb E[(X^{(n+1)} -X^{(n)})]^2 \\&\leq 2\mathbb E[\int_0^t [g(X^n_s)- g(X^{n-1}_s)]ds]^2+ 2\mathbb E[\int_0^t [S_f(X^n_s)-S_f(X^{n-1}_s)]dW_s]^2\\
     &\leq 2t \int_0^t \mathbb E[(g(X^n_s)- g(X^{n-1}_s))^2]ds+ 2\int_0^t \mathbb E[(S_f(X^n_s)-S_f(X^{n-1}_s))^2]ds\\ 
     & \le 2TL^2 \int_0^t\mathbb E[(X^n_s- X^{n-1}_s)^2]ds+ 2L^2\int_0^t \mathbb E[(X^n_s- X^{n-1}_s)^2]ds \\
     & = 2L^2(T+1)\int_0^t\mathbb E[(X^n_s- X^{n-1}_s)^2]ds \coloneqq C_0 \int_0^t\Delta_t^{n-1}ds
\end{split}
\end{equation}
The third inequity holds by using Assumption~\ref{ASS:LLC}. Hence, we can see that $\Delta_t^n = C_0 \int_0^t\Delta_t^{n-1}ds$. To derive a closed form of upper bound, we can compute $\Delta_t^0$ by using the Assumption~\ref{ASS:IC} as 
\begin{equation}
    \begin{split}
        \Delta_t^0 = \mathbb E[(X^{(1)} -X^{(0)})^2]  &=\mathbb E[\int_0^t g(X^0_s)ds+\int_0^t S_f(X^0_s)dW_s]^2\\
        & \leq 2 \mathbb E[t(g(X_0)^2+S_f(X_0)^2]\\
        & \leq 2K^2(1+T) \mathbb E[(X_0)^2]t\coloneqq C_1t
    \end{split}
\end{equation}
Therefore, we can have $\Delta_t^n \le C_1\frac{(C_ot)^n}{n!}$.Hence, we can have an arbitrary case for any $k,n\in\mathbb N$ by using Cauchy–Schwarz inequality as 
\begin{equation}
\begin{split}
    |X_t^{(n+k)} - X_t^{(n)}|
&\le \sum_{j=1}^{k} |X_t^{(n+j)} - X_t^{(n+j-1)}| \\
&\le (\sum_{j=1}^{k} \frac{1}{2^{n+j-1}} )^{1/2}
(\sum_{j=1}^{k} 2^{n+j-1} | X_t^{(n+j)} - X_t^{(n+j-1)}|^2)^{1/2} \\
 \implies\mathbb E[(X^{(n+k)} -X^{(n)})^2] &\le \sum_{j=1}^{k} \frac{1}{2^{n+j-1}} \sum_{j=1}^{k} 2^{n+j-1} \Delta_t^{n+j-1}\\
 &\le 2C_1\sum_{j =1}^{k} \frac{(2C_0t)^{n+j-1}}{(n+j-1)!} \le 2C_1\sum_{j =n-1}^{\infty} \frac{(2C_0t)^{j}}{j!}
\end{split}
\end{equation}
Therefore, we know that 
\begin{equation}
\begin{split}
       & \sup_{k\ge 0}\mathbb E[(X^{(n+k)} -X^{(n)})^2] \le 2C_1\sum_{j =n-1}^{\infty} \frac{(2C_0t)^{j}}{j!} \\
       &\sup_{t\in[0,T]}\sup_{k\ge 0}\mathbb E[(X^{(n+k)} -X^{(n)})^2] \to 0, \quad \text{as} \quad n\to \infty\\
       \implies& \lim_{n\to \infty} \sup_{t\in[0,T]} \mathbb E[(X^{(n+k)} -X^{(n)})^2] = 0
\end{split}
\end{equation}
Therefore, This process is adapted, satisfies $\int_0^t \mathbb E[X_s]^2 < \infty$ for all $0\le t\le T$, and solve the SDE that we define. To show the uniqueness, Let $(\tilde{X}_t)_{t \in [0,T]}$ be another solution of the SDE. Then, we define $Z_t \coloneqq X_t - \tilde{X}_t$, and we obsevre that
\begin{equation}
    Z_t = \int_0^t \big( g(X_s) - g(\tilde{X}_s) \big)  ds 
        + \int_0^t \big( S_f(X_s) - S_f(\tilde{X}_s) \big)  dW_s, 
\end{equation}
and the same argument as before gives
\begin{equation}
    \mathbb{E}[Z_t^2] 
    \le 2K^2(1+T) \int_0^t \mathbb{E}[Z_s^2]  ds,
    \qquad 0 \le t \le T. 
\end{equation}
By Gronwall's inequality then gives $\mathbb{E}[Z_t^2] 
\le e^{2K^2(1+T)t}  \mathbb{E}[Z_0^2]$, so $\mathbb{E}[Z_t^2] = 0$ for all $t \in [0,T]$. Then, we need to show the existence of $\Phi_t$ and the representation and invertibility. Define the fundamental matrix process $(\Phi_t)_{t\in[0,T]}$ as the unique strong solution to the matrix SDE
\begin{equation}\label{eq:Phi_SDE}
d\Phi_t = B\Phi_t\,dt + \sum_{r=1}^m A_r\Phi_t\,dW_t^{(r)}, 
\qquad \Phi_0 = I_d .
\end{equation}
Existence and pathwise uniqueness hold since \eqref{eq:Phi_SDE} is linear with globally Lipschitz (constant) coefficients (e.g., after vectorizing $\Phi_t\in\mathbb{R}^{d\times d}$ into $\mathbb{R}^{d^2}$). Set $\widetilde X_t := \Phi_t X_0$. Since $X_0$ is time-constant,
\[
d\widetilde X_t = (d\Phi_t)X_0
= B\Phi_tX_0\,dt +  A\Phi_tX_0\,dW_t^{(r)}
= B\widetilde X_t\,dt + A\widetilde X_t\,dW_t^{(r)},
\]
and $\widetilde X_0 = \Phi_0 X_0 = X_0$. By pathwise uniqueness of the state SDE, we conclude that
\begin{equation}\label{eq:X_representation}
X_t = \Phi_t X_0 \qquad \text{a.s. for all } t\in[0,T].
\end{equation}
To show $\Phi_t$ is invertible a.s., define $(\Psi_t)_{t\in[0,T]}$ as the unique strong solution of
\begin{equation}\label{eq:Psi_SDE}
d\Psi_t
= -\Psi_t B\,dt - \Psi_t A\,dW_t^{(r)} + \Psi_t A^2\,dt,
\qquad \Psi_0 = I_d .
\end{equation}
Let $Y_t := \Psi_t\Phi_t$. By It\^o's product rule,
\[
d(\Psi_t\Phi_t) = (d\Psi_t)\Phi_t + \Psi_t(d\Phi_t) + (d\Psi_t)(d\Phi_t).
\]
Substituting \eqref{eq:Phi_SDE}--\eqref{eq:Psi_SDE}, the $dW$ terms cancel. Moreover, using
$dW_t^{(r)}dW_t^{(s)} = \delta_{rs}\,dt$, the quadratic covariation term satisfies
\[
(d\Psi_t)(d\Phi_t)
=\Big(- \Psi_t A\,dW_t^{(r)}\Big)\Big( A\Phi_t\,dW_t^{(s)}\Big)
= - \Psi_t A \Phi_t\,dt,
\]
which cancels the drift contribution $+\sum_{r=1}^m \Psi_t A_r^2 \Phi_t\,dt$ coming from $(d\Psi_t)\Phi_t$.
Hence $d(\Psi_t\Phi_t)=0$, so $Y_t\equiv Y_0=\Psi_0\Phi_0=I_d$ on $[0,T]$ a.s., i.e.,
\begin{equation}\label{eq:inverse_identity}
\Psi_t\Phi_t \equiv I_d \qquad \text{a.s. for all } t\in[0,T].
\end{equation}
Therefore $\Phi_t$ is invertible a.s. for each $t$, and $\Phi_t^{-1}=\Psi_t$. Then, we can show the positive definiteness of the population Gram. Let $\Sigma_0 := \mathbb{E}[X_0X_0^\top]\succ 0$. For any $v\neq 0$, using \eqref{eq:X_representation} and conditioning on $\Phi_t$,
\begin{align}
\mathbb{E}\left[(v^\top X_t)^2 \mid \Phi_t\right]
&= \mathbb{E}\left[(v^\top \Phi_t X_0)^2 \mid \Phi_t\right] \nonumber= \mathbb{E}\left[( (\Phi_t^\top v)^\top X_0 )^2 \mid \Phi_t\right] \nonumber = (\Phi_t^\top v)^\top \Sigma_0 (\Phi_t^\top v) \nonumber\\
&\ge \lambda_{\min}(\Sigma_0)\,\|\Phi_t^\top v\|_2^2 .
\label{eq:cond_lower_bd}
\end{align}
Taking expectations yields
\begin{equation}\label{eq:Sigma_t_lower_bd}
v^\top \Sigma_t v
:= v^\top \mathbb{E}[X_tX_t^\top] v
= \mathbb{E}\left[(v^\top X_t)^2\right]
\ge \lambda_{\min}(\Sigma_0)\,\mathbb{E}\|\Phi_t^\top v\|_2^2 .
\end{equation}
Since $\Phi_t$ is invertible a.s. by \eqref{eq:inverse_identity}, we have $\Phi_t^\top v\neq 0$ a.s. for every fixed $v\neq 0$, hence $\mathbb{E}\|\Phi_t^\top v\|_2^2>0$. Therefore $v^\top \Sigma_t v>0$ for all $v\neq 0$, i.e., $\Sigma_t\succ 0$. Finally,
\begin{equation}\label{eq:pop_gram_pd}
v^\top \mathbb{E}\widehat{\Sigma}\, v
= \frac{1}{n}\sum_{k=0}^{n-1} v^\top \Sigma_{t_k} v
> 0
\qquad \forall\, v\neq 0,
\end{equation}
so $\mathbb{E}\widehat{\Sigma}\succ 0$ and consequently $\lambda_{\min}(\mathbb{E}\widehat{\Sigma})>0$.
\end{proof}
\subsubsection{Proof of Proposition~\ref{thm:gen_lyap_AB}}
\begin{proof}
We break the argument into three steps: (i) construct $\Sigma$ as a fixed point,
(ii) show it solves \eqref{eq:gen_lyap_eq_Sigma}, and (iii) prove $\Sigma\succ 0$ and uniqueness. Define an operator $\mathcal{T}$ on $\mathbb{R}^{n\times n}$ by
\begin{equation}\label{eq:T_operator}
\mathcal{T}(X):=\int_0^\infty e^{At}\big(BXB^\top - M\big)e^{A^\top t}\,dt.
\end{equation}
First, the integral is well-defined by using $\|e^{Bt}\|\le m e^{\omega t}\quad \forall t\ge 0$,
\[
\|e^{At}(BXB^\top-M)e^{A^\top t}\|
\le \|e^{At}\|^2\big(\|B\|^2\|X\|+\|M\|\big)
\le m^2 e^{2\omega t}\big(\|B\|^2\|X\|+\|M\|\big),
\]
and since $\omega<0$, $\int_0^\infty e^{2\omega t}dt<\infty$. Next, for any $X,Y$,
\begin{align*}
\|\mathcal{T}(X)-\mathcal{T}(Y)\|
&= \left\|\int_0^\infty e^{At}\,B(X-Y)B^\top\,e^{A^\top t}\,dt\right\|\\
&\le \int_0^\infty \|e^{At}\|\,\|B\|\,\|X-Y\|\,\|B\|\,\|e^{A^\top t}\|\,dt\\
&\le \|X-Y\| \cdot m^2\|B\|^2\int_0^\infty e^{2\omega t}\,dt\\
&= \|X-Y\|\cdot \frac{m^2\|B\|^2}{-2\omega}.
\end{align*}
By $m^2\|A\|^2 + 2\omega < 0$, we have $\frac{m^2\|B\|^2}{-2\omega}<1$, so $\mathcal{T}$ is a contraction.
Hence, by Banach's fixed-point theorem, there exists a unique fixed point $\Sigma$ satisfying
$\Sigma=\mathcal{T}(\Sigma)$ (i.e. the integral representation). Then,  we will show that the fixed point solves the generalized Lyapunov equation. 
By Lemma~\ref{lem:lyap_integral} with $N := M - B\Sigma B^\top.$ From integral representation,
\[
\Sigma=\int_0^\infty e^{At}\big(B\Sigma B^\top - M\big)e^{A^\top t}\,dt
= -\int_0^\infty e^{At}\big(M - B\Sigma B^\top\big)e^{A^\top t}\,dt
= -\int_0^\infty e^{At} N e^{A^\top t}\,dt.
\]
Hence Lemma~\ref{lem:lyap_integral} gives
\[
A\Sigma + \Sigma A^\top = -(-N)=N = M - B\Sigma B^\top,
\]
which rearranges to \eqref{eq:gen_lyap_eq_Sigma}. Uniqueness follows already from the contraction argument (fixed point is unique). To show $\Sigma\succ 0$ when $M\prec 0$, consider the Picard iteration with induction :
\[
\Sigma_0 := 0,\qquad
\Sigma_{j+1} := \mathcal{T}(\Sigma_j)
= \int_0^\infty e^{At}\big(B\Sigma_j B^\top - M\big)e^{A^\top t}\,dt.
\]
Since $M\prec 0$, we have $-M\succ 0$, hence
\[
\Sigma_1=\int_0^\infty e^{At}(-M)e^{A^\top t}\,dt \succ 0
\]
because $v^\top \Sigma_1 v = \int_0^\infty (e^{A^\top t}v)^\top(-M)(e^{A^\top t}v)\,dt>0$ for $v\neq 0$. For inductive step, if $\Sigma_j\succeq 0$, then $B\Sigma_j B^\top\succeq 0$, so $B\Sigma_j B^\top - M \succ 0$, implying $\Sigma_{j+1}\succ 0$. Therefore every $\Sigma_j\succ 0$ for $j\ge 1$. Finally, $\Sigma_j\to \Sigma$ in operator norm because $\mathcal{T}$ is a contraction. Since the cone $\mathbb{S}_{++}^n$ is open and the limit of $\Sigma_j\succ 0$ cannot cross the boundary under norm convergence unless eigenvalues approach $0$, we can also argue directly:
for any $v\neq 0$,
\[
v^\top \Sigma v = \lim_{j\to\infty} v^\top \Sigma_j v \ge v^\top \Sigma_1 v > 0,
\]
so $\Sigma\succ 0$. This completes the proof.
\end{proof}
\subsubsection{Proof of Proposition~\ref{Ass:xbound}}
\begin{proof}
Define the quadratic Lyapunov function $V(x):=x^\top P x$ and the process
$V_t:=V(X_t)=X_t^\top P X_t$. Since $P\succ0$, $V_t\ge 0$ for all $t$ and
$\mathbb E[V_0]\le \lambda_{\max}(P)\mathbb E\|X_0\|^2<\infty$.
Let $f(x)=x^\top P x$. Then $\nabla f(x)=2Px$ and $\nabla^2 f(x)=2P$.
For the linear SDE $dX_t=AX_t\,dt+BX_t\,dW_t$, It\^{o}'s formula yields
\[
dV_t
= \langle \nabla f(X_t),\, AX_t\rangle\,dt
+ \langle \nabla f(X_t),\, BX_t\rangle\,dW_t
+ \frac12 \mathrm{tr}\Big[(BX_t)(BX_t)^\top \nabla^2 f(X_t)\Big]dt.
\]
Substituting $\nabla f$ and $\nabla^2 f$ and simplifying,
\[
dV_t
= X_t^\top(A^\top P+PA)X_t\,dt
+ X_t^\top(PB+B^\top P)X_t\,dW_t
+ X_t^\top(B^\top P B)X_t\,dt,
\]
hence
\begin{equation}\label{eq:dV}
dV_t
= X_t^\top(A^\top P+PA+B^\top P B)X_t\,dt
+ X_t^\top(PB+B^\top P)X_t\,dW_t.
\end{equation}
By the assumption $A^\top P+PA+B^\top P B\preceq -Q\preceq 0$, the drift term in
\eqref{eq:dV} is nonpositive. To justify conditional expectations, introduce the
localizing stopping times $\tau_m:=\inf\{t\ge 0:\|X_t\|\ge m\}\wedge T$.
Then $(V_{t\wedge \tau_m})_{t\in[0,T]}$ is integrable and the stochastic integral
$\int_0^{t\wedge \tau_m} X_s^\top(PB+B^\top P)X_s\,dW_s$ is a true martingale with
zero mean. Therefore, for any $0\le s\le t\le T$,
\[
\mathbb E\left[V_{t\wedge \tau_m}\mid \mathcal F_s\right]
= V_{s\wedge \tau_m}
+ \mathbb E\left[\int_{s\wedge \tau_m}^{t\wedge \tau_m}
X_u^\top(A^\top P+PA+B^\top P B)X_u\,du \,\Big|\, \mathcal F_s\right]
\le V_{s\wedge \tau_m}.
\]
Hence $(V_{t\wedge \tau_m})_{t\in[0,T]}$ is a nonnegative supermartingale. Next, note that $\tau_m\uparrow T$ as $m\to\infty$ and $X$ has continuous sample paths,
so for each fixed $t\in[0,T]$ we have $t\wedge \tau_m \to t$ and thus $V_{t\wedge \tau_m}\to V_t$ almost surely. By conditional Fatou's lemma, for any
$0\le s\le t\le T$,
\[
\mathbb E\left[V_t \mid \mathcal F_s\right]
=\mathbb E\left[\liminf_{m\to\infty} V_{t\wedge \tau_m}\,\Big|\,\mathcal F_s\right]
\le \liminf_{m\to\infty}\mathbb E\left[V_{t\wedge \tau_m}\mid \mathcal F_s\right]
\le \liminf_{m\to\infty} V_{s\wedge \tau_m}
= V_s,
\]
which shows that $(V_t)_{t\in[0,T]}$ is a nonnegative supermartingale as well. For any $u>0$, Doob's maximal inequality for nonnegative supermartingales gives
\[
\mathbb P\left(\sup_{0\le r\le T} V_r \ge u\right)
\le \frac{\mathbb E[V_0]}{u}.
\]
Choose $u:=\lambda_{\min}(P)K^2$. Since $V_r=X_r^\top P X_r \ge \lambda_{\min}(P)\|X_r\|^2$,
\[
\left\{\sup_{0\le r\le T}\|X_r\|\ge K\right\}
\subseteq
\left\{\sup_{0\le r\le T}V_r \ge \lambda_{\min}(P)K^2\right\}.
\]
Thus,
\[
\mathbb P\left(\sup_{0\le r\le T}\|X_r\|\ge K\right)
\le
\mathbb P\left(\sup_{0\le r\le T}V_r \ge \lambda_{\min}(P)K^2\right)
\le
\frac{\mathbb E[V_0]}{\lambda_{\min}(P)K^2}.
\]
Finally, $V_0=X_0^\top P X_0 \le \lambda_{\max}(P)\|X_0\|^2$ implies
$\mathbb E[V_0]\le \lambda_{\max}(P)\mathbb E\|X_0\|^2$, yielding
\[
\mathbb P\left(\sup_{0\le r\le T}\|X_r\|\ge K\right)
\le
\frac{\lambda_{\max}(P)}{\lambda_{\min}(P)}\cdot \frac{\mathbb E\|X_0\|^2}{K^2}.
\]
The discrete-time statement follows since $\max_{k\le n}\|X_{t_k}\|\le \sup_{0\le r\le T}\|X_r\|$.
\end{proof}

\subsubsection{Proof of Proposition~\ref{prop: cov_decay}}
For notational simplicity, the proof is written in a compressed linear form. The argument only uses that, conditional on \(\mathcal F_t\), the diffusion increment over \([t,t+h]\) has mean zero, so the same conditional-mean identity applies to the multi-channel representation in Eq.\ref{Eq:equiv_form}.
\begin{proof}
Write $\mu_t := \mathbb{E}X_t$ and $\Sigma_t := \mathrm{Cov}(X_t)
= \mathbb{E}\left[(X_t-\mu_t)(X_t-\mu_t)^\top\right]$.
Fix $t\ge 0$ and define $\mathcal{F}_t:=\sigma\{W_s: 0\le s\le t\}$.
Consider the process $Y_h := X_{t+h}$ for $h\ge 0$. Then $Y_h$ satisfies, in the ``lag'' variable $h$,
\begin{equation}
\label{eq:lag_sde}
dY_h = B Y_h\,dh + A Y_h\,d\widetilde W_h,\qquad
\widetilde W_h:=W_{t+h}-W_t,
\end{equation}
where $(\widetilde W_h)_{h\ge 0}$ is a Brownian motion independent of $\mathcal{F}_t$.
Taking conditional expectations in \eqref{eq:lag_sde} yields the ODE
\[
\frac{d}{dh}\,\mathbb{E}[Y_h\mid \mathcal{F}_t]
= B\,\mathbb{E}[Y_h\mid \mathcal{F}_t],\qquad
\mathbb{E}[Y_0\mid \mathcal{F}_t]=X_t,
\]
hence
\begin{equation}
\label{eq:cond_mean_future}
\mathbb{E}[X_{t+h}\mid \mathcal{F}_t] = e^{Bh}X_t,\qquad \forall\,h\ge 0.
\end{equation}
Taking expectations in \eqref{eq:cond_mean_future} gives $\mu_{t+h}=e^{Bh}\mu_t$, so
\begin{equation}
\label{eq:cond_centered_future}
\mathbb{E}[X_{t+h}-\mu_{t+h}\mid \mathcal{F}_t]
= e^{Bh}(X_t-\mu_t),\qquad \forall\,h\ge 0.
\end{equation}
Using the tower property and \eqref{eq:cond_centered_future}, we obtain
\begin{equation}
    \begin{split}
C(t,h)
&= \mathbb{E}\left[(X_t-\mu_t)(X_{t+h}-\mu_{t+h})^\top\right]\\
&= \mathbb{E}\left[(X_t-\mu_t)\,\mathbb{E}\left[(X_{t+h}-\mu_{t+h})^\top\mid \mathcal{F}_t\right]\right]\\
&= \mathbb{E}\left[(X_t-\mu_t)(X_t-\mu_t)^\top\right] e^{B^\top h}
= \Sigma_t e^{B^\top h}.
    \end{split}
\end{equation}
Therefore,
\begin{equation}
\label{eq:C_norm_bound}
\|C(t,h)\|_2 \le \|\Sigma_t\|_2\,\|e^{B^\top h}\|_2
= \|\Sigma_t\|_2\,\|e^{Bh}\|_2.
\end{equation}
Since $\Sigma_t\succeq 0$, we have $\|\Sigma_t\|_2 \le \mathrm{tr}(\Sigma_t)
= \mathbb{E}\|X_t-\mu_t\|_2^2 \le \mathbb{E}\|X_t\|_2^2$.
Combining this with $\|e^{Bh}\|_2 \le m e^{\omega h}$ for all  and the uniform second-moment bound yields
\[
\|C(t,h)\|_2
\le \mathbb{E}\|X_t\|_2^2 \cdot m e^{\omega h}
\le \left(\sup_{s\ge 0}\mathbb{E}\|X_s\|_2^2\right) m e^{\omega h}
= C_0 e^{-c_0 h},
\]
with $c_0:=-\omega>0$ and $C_0:=m\,\sup_{s\ge 0}\mathbb{E}\|X_s\|_2^2$. This proves the cross covariance decay.
\end{proof}

\subsubsection{Proof of Proposition~\ref{ass:localization}}
\begin{proof}
Write the mild form for $t\in[0,T]$:
\[
X_t=X_0+\int_0^t BX_s\,ds+\int_0^t AX_s\,dW_s.
\]
Taking $\sup_{u\in[0,t]}$ and using $(a+b+c)^q\le 3^{q-1}(a^q+b^q+c^q)$, we get
\begin{align*}
\sup_{u\le t}\|X_u\|^q
&\le 3^{q-1}\|X_0\|^q
+3^{q-1}\sup_{u\le t}\Big\|\int_0^u BX_s\,ds\Big\|^q
+3^{q-1}\sup_{u\le t}\Big\|\int_0^u AX_s\,dW_s\Big\|^q.
\end{align*}
For the drift term, Hölder yields
\[
\sup_{u\le t}\Big\|\int_0^u BX_s\,ds\Big\|^q
\le \Big(\int_0^t \|BX_s\|\,ds\Big)^q
\le t^{q-1}\int_0^t \|BX_s\|^q\,ds.
\]
For the martingale term, by the Burkholder--Davis--Gundy inequality (BDG), there exists
$C_q>0$ such that
\[
\mathbb{E}\Big[\sup_{u\le t}\Big\|\int_0^u AX_s\,dW_s\Big\|^q\Big]
\le C_q\,\mathbb{E}\Big[\Big(\int_0^t \|AX_s\|_{\mathrm F}^2\,ds\Big)^{q/2}\Big].
\]
Since $q\ge 2$, we use $(\int_0^t f(s)^2 ds)^{q/2}\le t^{\frac q2-1}\int_0^t f(s)^q ds$ to obtain
\[
\mathbb{E}\Big[\sup_{u\le t}\Big\|\int_0^u AX_s\,dW_s\Big\|^q\Big]
\le C_q\,t^{\frac q2-1}\int_0^t \mathbb{E}\|AX_s\|_{\mathrm F}^q\,ds.
\]
Combining the above bounds and taking expectation, we get
\[
\mathbb{E}\Big[\sup_{u\le t}\|X_u\|^q\Big]
\le C\mathbb{E}\|X_0\|^q
+ C\int_0^t \mathbb{E}\|BX_s\|^q\,ds
+ C\int_0^t \mathbb{E}\|AX_s\|_{\mathrm F}^q\,ds,
\]
where $C$ absorbs $3^{q-1}$ and the $t$-dependent factors (note $t\le T$).
By linear growth assumption,
\[
\|Bx\|^q+\|Ax\|_{\mathrm F}^q \le C\bigl(1+\|x\|^q\bigr),
\]
hence
\[
\mathbb{E}\Big[\sup_{u\le t}\|X_u\|^q\Big]
\le C\bigl(1+\mathbb{E}\|X_0\|^q\bigr)
+ C\int_0^t \mathbb{E}\Big[\sup_{r\le s}\|X_r\|^q\Big]\,ds.
\]
Applying Grönwall's inequality yields
\[
\mathbb{E}\Big[\sup_{t\in[0,T]}\|X_t\|^q\Big]
\le C\,e^{CT}\bigl(1+\mathbb{E}\|X_0\|^q\bigr).
\]
Finally, Markov's inequality gives for any $K>0$,
\[
\mathbb{P}\Big(\sup_{t\in[0,T]}\|X_t\|>K\Big)
\le \frac{\mathbb{E}[\sup_{t\in[0,T]}\|X_t\|^q]}{K^q}.
\]
Choosing $K=K(T,\delta)$ as stated proves
$\mathbb{P}(\sup_{t\in[0,T]}\|X_t\|\le K(T,\delta))\ge 1-\delta$.
The discrete-time and $\|\cdot\|_\infty$ statements follow from
$\max_k\|X(t_k)\|\le \sup_{t\in[0,T]}\|X_t\|$ and $\|x\|_\infty\le \|x\|$.
\end{proof}

\subsubsection{Proof of Proposition~\ref{thm:lyap-exponent-bound}}

\begin{proof}
    Given the initial condition $x(0) \neq 0$, and let the $y(t) = \langle \Sigma x(t), x(t)\rangle$. Then, by Proposition~\ref{prop: WD} and the positiveness of $\Sigma$, $\mathbb P(y(t) = 0)=0$ for all $t\ge0$. By Itô's Lemma, for
    \[
    y(t) = \langle \Sigma x(t), x(t)\rangle = x(t)^{\top}\Sigma x(t),
    \]
    we have
    \begin{equation}
    \begin{split}
    dy(t)
    &= x(t)^{\top}\Sigma dx(t) + dx(t)^{\top}\Sigma x(t) + dx(t)^{\top}\Sigma dx(t) \\
    &= 2\langle \Sigma x(t), dx(t)\rangle + \langle \Sigma dx(t), dx(t)\rangle \\
    &= 2\langle \Sigma x(t), A x(t)\rangle dt + 2\langle \Sigma x(t), B x(t)\rangle dW(t) + \langle \Sigma B x(t), B x(t)\rangle (dW(t))^{2} \\
    &= \big(2\langle \Sigma A x(t), x(t)\rangle 
            + \langle \Sigma B x(t), B x(t)\rangle\big) dt
       + 2\langle \Sigma x(t), B x(t)\rangle dW(t),
    \end{split}
    \end{equation}
    since $(dW(t))^{2} = dt$ in the Ito sense. Then, we can apply Ito's lemma again, we can have 
    \begin{equation}
    \begin{split}
    d\log y(t) &= \frac{1}{y(t)}  dy(t)
   - \frac{1}{2y(t)^2}(dy(t))^2 \\
    &=  \frac{2\langle \Sigma A x(t),x(t)\rangle + \langle \Sigma Bx(t),Bx(t)\rangle}{y(t)}  dt + \frac{2\langle \Sigma x(t),B x(t)\rangle}{y(t)}  dW(t) - \frac{2\langle \Sigma x(t),B x(t)\rangle^2}{y(t)^2}  dt \\
    & = ( \frac{2\langle \Sigma A x(t),x(t)\rangle + \langle \Sigma Bx(t),Bx(t)\rangle}{y(t)} - \frac{2\langle \Sigma x(t),B x(t)\rangle^2}{y(t)^2}) dt      
    + \frac{2\langle \Sigma x(t),B x(t)\rangle}{y(t)}  dW(t). \\
    & \leq 2\lambda dt + \frac{2\langle \Sigma x(t),B x(t)\rangle}{y(t)}  dW(t) \quad \text{a.s.}
    \end{split}
    \end{equation}
The last inequality hold by Eq.\eqref{eq:lyap-ineq}. Then, we define
\begin{equation}
y(t)=\langle \Sigma x(t),x(t)\rangle,\quad 
h(t):=\frac{2\langle \Sigma x(t),Bx(t)\rangle}{y(t)}, \quad M_t := \int_0^t h(s)dW(s), 
\end{equation}
where $M_t$ is the continuous martingale, and its quadratic variation satisfies
\begin{equation}
\langle M\rangle_t 
= \int_0^t h(s)^2  ds 
\le C t ,
\end{equation}
for some constant \(C>0\).  
Therefore, by the strong law of large numbers for continuous martingales,
\begin{equation}
    \lim_{t\to\infty}\frac{M_t}{t} = \lim_{t\to\infty} \frac{1}{t}\int_0^t \frac{2\langle \Sigma x(s),Bx(s)\rangle}{y(s)} dW(s) = 0, \quad \text{a.s.}
\end{equation}
Therefore, we can have $\lim_{t\to \infty}\frac{1}{t}\log y(t) \le 2\lambda$. By the positiveness of $\Sigma$, we have 
\begin{equation}
    \lim_{t\to \infty}\frac{1}{t}\log |x(t)| \le \lambda
\end{equation}
\end{proof}

\subsubsection{Proof of Lemma~\ref{lem:35}}
\begin{proof}
We expand the argument to make the implication \eqref{eq:cond13}$\Rightarrow$\eqref{eq:lyap-ineq} explicit. Fix $i\in\{1,\dots,k\}$ and any nonzero $x\in\mathbb{R}^n$.
Define the scalar
\[
u_i(x):=\frac{\langle Qx, B_i x\rangle}{\langle Qx,x\rangle}.
\]
Since $(u_i(x)+\tfrac12 b_i)^2\ge 0$, we have
\begin{equation}
\label{eq:abc15}
u_i(x)^2 + b_i u_i(x) + \frac14 b_i^2 \ge 0.
\end{equation}
Multiplying \eqref{eq:abc15} by $2\langle Qx,x\rangle^2$ yields
\begin{equation}
\label{eq:abc15_mult}
-\,2\langle Qx, B_i x\rangle^2
\le
2 b_i \langle Qx, B_i x\rangle \langle Qx,x\rangle
+\frac12 b_i^2 \langle Qx,x\rangle^2 .
\end{equation}
Summing \eqref{eq:abc15_mult} over $i=1,\dots,k$ gives
\begin{equation}
\label{eq:sumbound}
-\,2\sum_{i=1}^k\langle Qx, B_i x\rangle^2
\le
2\Big(\sum_{i=1}^k b_i \langle Qx, B_i x\rangle\Big)\langle Qx,x\rangle
+\frac12\Big(\sum_{i=1}^k b_i^2\Big)\langle Qx,x\rangle^2 .
\end{equation}
Take the quadratic form of \eqref{eq:cond13} with vector $x$:
\begin{align}
0 &\ge x^\top\Big[
\Big(A+\sum_{i=1}^k b_i B_i\Big)^\top Q
+ Q\Big(A+\sum_{i=1}^k b_i B_i\Big)
+ \sum_{i=1}^k B_i^\top Q B_i
+ \Big(\frac12\sum_{i=1}^k b_i^2 - 2\lambda\Big) Q
\Big]x \nonumber\\
&= 2\langle QAx,x\rangle
+ 2\sum_{i=1}^k b_i \langle Q B_i x, x\rangle
+ \sum_{i=1}^k \langle Q B_i x, B_i x\rangle
+ \Big(\frac12\sum_{i=1}^k b_i^2 - 2\lambda\Big)\langle Qx,x\rangle .
\label{eq:quadform13}
\end{align}
Rearranging \eqref{eq:quadform13} gives, for all $x\in\mathbb{R}^n$,
\begin{equation}
\label{eq:rearrange13}
2\langle QAx,x\rangle + \sum_{i=1}^k \langle Q B_i x, B_i x\rangle - 2\lambda \langle Qx,x\rangle
\le
-2\sum_{i=1}^k b_i \langle Qx, B_i x\rangle
-\frac12\Big(\sum_{i=1}^k b_i^2\Big)\langle Qx,x\rangle .
\end{equation}
Multiply \eqref{eq:rearrange13} by $\langle Qx,x\rangle\ge 0$ to get
\begin{equation}
    \begin{split}
        &\langle Qx,x\rangle\Big(
2\langle QAx,x\rangle + \sum_{i=1}^k \langle Q B_i x, B_i x\rangle - 2\lambda \langle Qx,x\rangle
\Big)\\
&\qquad\le
-2\Big(\sum_{i=1}^k b_i \langle Qx,B_i x\rangle\Big)\langle Qx,x\rangle
-\frac12\Big(\sum_{i=1}^k b_i^2\Big)\langle Qx,x\rangle^2.
\label{eq:mul_rearrange}
    \end{split}
\end{equation}
Finally, apply \eqref{eq:sumbound} to the right-hand side of \eqref{eq:mul_rearrange} to bound it by
$2\sum_{i=1}^k \langle Qx, B_i x\rangle^2$. This yields the following
\[
\langle Qx,x\rangle\Big(
2\langle QAx,x\rangle + \sum_{i=1}^k \langle Q B_i x, B_i x\rangle - 2\lambda \langle Qx,x\rangle
\Big)
\le
2\sum_{i=1}^k \langle Qx, B_i x\rangle^2,
\qquad \forall x\in\mathbb{R}^n.
\]
Therefore, Proposition~\ref{thm:lyap-exponent-bound} applies and gives \eqref{eq:limsup14}.
\end{proof}

\subsubsection{Proof of Lemma~\ref{lem:lyap_integral}}
\begin{proof}
Let $F(t):=e^{At}Ne^{A^\top t}$. Then $F$ is differentiable and
\[
\frac{d}{dt}F(t)=A e^{At}Ne^{A^\top t} + e^{At}Ne^{A^\top t}A^\top
= A F(t)+F(t)A^\top.
\]
Integrating from $0$ to $T$ gives
\[
\int_0^T (A F(t)+F(t)A^\top)\,dt = F(T)-F(0).
\]
As $T\to\infty$, $\|F(T)\|\le \|e^{AT}\|^2\|N\|\le m^2e^{2\omega T}\|N\|\to 0$ since $\omega<0$.
Thus, letting $P=\int_0^\infty F(t)\,dt$ yields $AP+PA^\top = -F(0)=-N$.
\end{proof}

\subsection{Proofs for the Statements in \S\ref{sec:lemma}}
\label{sec: Lemma}
\subsubsection{Proof of Lemma~\ref{KKT}}
\begin{proof}
Fix $k$ and let $z:=\theta_i^\top X_k$. Define the single-sample loss
\begin{equation}
\ell(z)
=\frac12\Bigg[\log(z^2+c)+\frac{r_{i,k}^2}{z^2+c}\Bigg].
\end{equation}
Differentiating w.r.t.\ $z$ yields
\begin{equation}
    \frac{d\ell}{dz}=\frac12\left(\frac{2z}{z^2+c}-\frac{2zr_{i,k}^2}{(z^2+c)^2}\right)
=\frac{z\big((z^2+c)-r_{i,k}^2\big)}{(z^2+c)^2}.
\end{equation}
Since $z=\theta_i^\top X_k$, we have $\partial z/\partial\theta_i=X_k$. By the chain rule,
\begin{equation}
    \nabla_{\theta_i}\mathcal L_i(\theta_i)=\frac1n\sum_{k=0}^{n-1}\frac{d\ell}{dz}\Big|_{z=\theta_i^\top X_k}\,X_k
=\frac1n\sum_{k=0}^{n-1}u_{i,k}(\theta_i)\,X_k,
\end{equation}
where
\begin{equation}
u_{i,k}(\theta_i) = \frac{(\theta_i^\top X_k)\big((\theta_i^\top X_k)^2+c-r_{i,k}^2\big)}
{\big((\theta_i^\top X_k)^2+c\big)^2}.
\end{equation}
Define $G(\theta_i):=\nabla_{\theta_i}\mathcal L_i(\theta_i)$. For the Lasso objective $f(\theta_i):=\mathcal L_i(\theta_i)+\lambda\|\theta_i\|_1$,
stationarity is $0\in \nabla\mathcal L_i(\hat\theta_i)+\lambda\,\partial\|\hat\theta_i\|_1$.
Equivalently, there exists $e\in\partial\|\hat\theta_i\|_1$ such that
$G(\hat\theta_i)+\lambda e=0$, where $e_b=\mathrm{sign}(\hat\theta_{i,b})$ if $\hat\theta_{i,b}\neq 0$
and $e_b\in[-1,1]$ if $\hat\theta_{i,b}=0$. This yields the coordinate-wise conditions in 
\begin{equation}
\begin{cases}
G_b(\hat\theta_i)=-\lambda\,\mathrm{sign}(\hat\theta_{i,b}), & \hat\theta_{i,b}\neq 0,\\
|G_b(\hat\theta_i)|\le \lambda, & \hat\theta_{i,b}=0,
\end{cases}
\end{equation}
where $\hat\theta_{i,b}$ denotes the $b$-th entry of $\hat\theta_i$.
\end{proof}
\subsubsection{Proof of Lemma~\ref{lem:prop}}
\begin{itemize}
    \item \textbf{Proof of the first statement:}
\begin{proof}
Given the $c$-stabilized population parameter for $\theta^{*,c}_i$ for $i\in [p]$, for the notation simplicity, we use $\theta^{*}_i$, and we have the following population risk 
\begin{equation}
\begin{split}
     \theta^*_i& \in\arg\min_{\theta: \theta_{ii} = 0}\mathbb E[
\bigl( \log\big((\theta^\top_i X_k)^2+c\big)
+\frac{r_{i,k}^2}{(\theta^\top_i X_k)^2+c}\bigr)].\\
\end{split}
\end{equation}
Let $u_{i,k} := \theta_i^\top X_k$ and $u_{i,k}^* := \theta_i^{*\top}X_k$. Evaluating at $\theta_i=\theta_i^*$ yields
\begin{equation}
\nabla_{\theta_i}\ell_{i,k}(\theta_i^*) =
2\left(\frac{u_{i,k}^*}{(u_{i,k}^*)^2+c}
-\frac{u_{i,k}^*r_{i,k}^2}{\big((u_{i,k}^*)^2+c\big)^2}\right)X_k.
\end{equation}
Under the population model $r_{i,k}=u_{i,k}^* z_{i,k}$ with $z_{i,k}\sim\mathcal N(0,1)$,
we have $r_{i,k}^2=(u_{i,k}^*)^2 z_{i,k}^2$, hence
\begin{equation}
\begin{split}
\nabla_{\theta_i}\ell_{i,k}(\theta_i^*)
& =
\frac{2u_{i,k}^*}{\big((u_{i,k}^*)^2+c\big)^2}
\Big(\big((u_{i,k}^*)^2+c\big)-(u_{i,k}^*)^2 z_{i,k}^2\Big)X_k\\
&=
\frac{2u_{i,k}^*}{\big((u_{i,k}^*)^2+c\big)^2}
\Big(c+(u_{i,k}^*)^2(1-z_{i,k}^2)\Big)X_k. 
\end{split}
\end{equation}
Since $r_{i,k}=u_{i,k}^* z_k$ with $z_k\sim\mathcal N(0,1)$. Then the $j$-th coordinate of the per-sample gradient at $\theta_i^*$ is
\begin{equation}
g_{k,j}:=\big[\nabla_{\theta_i}\ell_{i,k}(\theta_i^*)\big]_j
= \underbrace{\frac{2cu_k^*}{(u_k^{*2}+c)^2}X_{k,j}}_{\coloneqq b_{k,j}} + \underbrace{\frac{2u_k^{*3}}{(u_k^{*2}+c)^2}(1-z_k^2)X_{k,j}}_{\coloneqq\tilde g_{k,j}}.
\end{equation}
Then, we rewrite it as 
\begin{equation}
g_{k,j} \coloneqq \big[\nabla_{\theta_i}\ell_{i,k}(\theta_i^\ast)\big]_j = b_{k,j}+\tilde g_{k,j},
\end{equation}
Conditioning on $X_k$, we can have that 
$\mathbb E[\tilde g_{k,j}\mid X_k]=0$ since $\mathbb E(1-z_k^2)=0$. By first-order optimality on the free coordinates $j\neq i$, and under the usual
interchange of differentiation and expectation, we can have that $\mathbb E [g_{k,j}] = 0$. Thus we can have $\mathbb E[b_{k,j}] = 0$. Additionally, we observe that 
\begin{equation}
\label{Eq: fact1}
    \left|\frac{2u^3}{(u^2+c)^2}\right|
=2\cdot\frac{|u|}{u^2+c}\cdot\frac{u^2}{u^2+c}
\le 2\cdot\frac{1}{2\sqrt c}\cdot 1
=\frac{1}{\sqrt c},
\end{equation}
Hence, with $\|X_k\|_\infty\le K$ and $1-z_k^2$ is sub-exponential, we have that 
\begin{equation}
|\tilde g_{k,j}|
\le \frac{|X_{k,j}|}{\sqrt c}|1-z_k^2|
\le \frac{K}{\sqrt c}|1-z_k^2| \implies \|\tilde g_{k,j}\|_{\psi_1}
\le
\frac{8K}{\sqrt c} := C_1
\end{equation}
Moreover, conditional on $\{X_k\}_{k=1}^n$, the variables $\{\tilde g_{k,j}\}_{k=1}^n$ are independent across $k$, since each $\tilde g_{k,j}$ depends on the Gaussian innovation only through $z_k$, and $\{z_k\}_{k=1}^n$ are i.i.d. and independent of $\{X_k\}_{k=1}^n$. Together with $\mathbb{E}[\tilde g_{k,j}\mid X_k]=0$, this allows us to apply the conditional Bernstein inequality. Thus, $\tilde g_{k,j}\in SE(\nu^2,\sigma)$ with $\nu^2\asymp C_1^2$ and $\sigma\asymp C_1$. We have for any $t>0$ by Bernstein concentration,

\begin{equation}
\mathbb P\left(|\frac1n\sum_{k=1}^n \tilde g_{k,j}|\ge t\mid\{X_k\}\right)\le2\exp\left(- \frac n2\min\left\{\frac{t^2}{\nu^2},\frac{t}{\sigma}\right\}
\right).
\end{equation}
Then, taking a union bound over $j\in[p]\setminus\{i\}$, we obtain
\begin{equation}
\begin{split}
&\mathbb P\left(\left.\|\frac1n\sum_{k=1}^n \tilde g_{k,j}\|_\infty\ge t\right|\{X_k\}\right)
\\ &= \mathbb P\left(\max_{j\in[p]}|\frac1n\sum_{k=1}^n \tilde g_{k,j}|\ge t\mid\{X_k\}\right) \le 2\exp\left( -\frac{n}{2}\min\left\{\frac{t^2}{\nu^2},\frac{t}{\sigma}\right\} +\log p \right).   
\end{split}
\end{equation}
By the tower property, we can have that 
\begin{equation}
\begin{split}
\label{Eq: g_JK bound}
\mathbb P\left(\|\frac1n\sum_{k=1}^n \tilde g_{k,j}\|_\infty\ge t\right)& = \mathbb E\left[ \mathbb P\left(\max_{j\in[p]}|\frac1n\sum_{k=1}^n \tilde g_{k,j}|\ge t\mid\{X_k\}\right) \right]\\&\le 2\exp\left( -\frac{n}{2}\min\left\{\frac{t^2}{\nu^2},\frac{t}{\sigma}\right\} +\log p \right).   
\end{split}
\end{equation}
Then, we define $B := (B_j)_{j\in [p]\setminus\{i\}}$, where \(B_j:=\frac{1}{n}\sum_{k=1}^n b_{k,j}\). It remains to control \(B\). By the same calculus argument as above in Eq.\eqref{Eq: fact1}, we can have the following under \(\|X_k\|_\infty\le K\)
\begin{equation}
|b_{k,j}| = \left| \frac{2cu_{i,k}^\ast}{\big((u_{i,k}^\ast)^2+c\big)^2}X_{k,j} \right| \le \frac{|X_{k,j}|}{\sqrt c} \le \frac{K}{\sqrt c}.
\end{equation}
Since \(\mathbb E[b_{k,j}]=0\), the variables \(b_{k,j}\) are centered and bounded, and therefore
sub-exponential with
\begin{equation}
    \|b_{k,j}\|_{\psi_1}\le C\frac{K}{\sqrt c}
=: C_2
\end{equation}
for some universal constant \(C>0\). Under Assumption~\ref{ass:alpha-mixing-EM}, the sequence \(\{X_k\}_{k\ge1}\) is geometrically \(\alpha\)-mixing. Since \(b_{k,j}\) is a measurable function of \(X_k\), \(\{ b_{k,j}\}_{k\ge1}\) also inherits geometric \(\alpha\)-mixing. Hence, by a Bernstein-type inequality for centered bounded geometrically \(\alpha\)-mixing sequences \citep{merlevede2009bernstein}, there exist constants \(C_1,C_2>0\), depending only on the mixing constants in Assumption~\ref{ass:alpha-mixing-EM}, such that for every \(t>0\),

\begin{equation}
\mathbb P\left(|B_j|\ge t\right) = \mathbb P\left( \left|\frac1n\sum_{k=1}^n b_{k,j}\right|\ge t \right) \le 2\exp\left( -\frac{n}{2}\min\left\{\frac{t^2}{\nu_B^2},\frac{t}{\sigma_B}\right\} \right),
\end{equation}
where \(\nu_B^2\asymp C_2^2\) and \(\sigma_B\asymp C_2\). Taking a union bound over \(j\in[p]\setminus\{i\}\), we obtain
\begin{equation}
\mathbb P\left(\|B\|_\infty\ge t\right) \le 2\exp\left( -\frac{n}{2}\min\left\{\frac{t^2}{\nu_B^2},\frac{t}{\sigma_B}\right\} +\log p \right).
\end{equation}
Combining this with the bound for \(\|\frac{1}{n}\sum_{k=1}^n \tilde g_{k,j}\|_\infty\) established in Eq.\eqref{Eq: g_JK bound}, we have
\begin{equation}
\mathbb P\left( \|\nabla_{\theta_i}\ell_i(\theta_i^\ast)\|_\infty \ge 2t
\right) \le \mathbb P(\|B\|_\infty\ge t)+\mathbb P(\|\frac1n\sum_{k=1}^n \tilde g_{k,j}\|_\infty\ge t)
\end{equation}
and therefore
\begin{equation}
\label{lemeq:grad_bound}
\mathbb P\left(
\|\nabla_{\theta_i}\ell_i(\theta_i^\ast)\|_\infty \ge 2t
\right)
\le
4\exp\left(
-\frac{n}{2}\min\left\{\frac{t^2}{\nu_\ast^2},\frac{t}{\sigma_\ast}\right\}
+\log p
\right),
\end{equation}
where \(\nu_\ast^2\asymp K^2/c\) and \(\sigma_\ast\asymp K/\sqrt c\). Choosing $t = C_0\frac{K}{\sqrt c} \left( \sqrt{\frac{\log p}{n}}+\frac{\log p}{n} \right)$ with \(C_0\) sufficiently large yields
\begin{equation}
\|\nabla_{\theta_i}\ell_i(\theta_i^\ast)\|_\infty \le C_1\frac{K}{\sqrt c} \left( \sqrt{\frac{\log p}{n}}+\frac{\log p}{n} \right)
\end{equation}
with probability at least \(1-4p^{-c_2}\), for some constants \(C_3,c_2>0\).
In particular, when \(n\gtrsim \log p\),
\begin{equation}
\|\nabla_{\theta_i}\ell_i(\theta_i^\ast)\|_\infty \le C_3\frac{K}{\sqrt c}\sqrt{\frac{\log p}{n}}.
\end{equation}
Since $K$ and $c$ are treated as fixed constants throughout, the factor $K/\sqrt{c}$ is absorbed into the generic constant. This completes the proof.
\end{proof}

\item \textbf{Proof of the second statement:}
\begin{proof}
Given the population objective, we can have the Hessian by directly computing such that
\begin{equation}
\nabla_\theta^2 \tilde\ell_{i,k}(\theta)=
\frac{2\Big(-(\theta^\top X_k)^4 + 3(\theta^\top X_k)^2 r_{i,k}^2 + c(c-r_{i,k}^2)\Big)}
{\big((\theta^\top X_k)^2+c\big)^3}
X_k X_k^\top.
\end{equation}
Let \(a_k^* := \theta^{*\top}X_k\) and \(u_k^* := (a_k^*)^2 + c\).
Assume the calibrated model \(r_{i,k}=\sqrt{u_k^*}\,z_{k+1}\) with
\(z_{k+1}\sim\mathcal N(0,1)\), so that \(r_{i,k}^2=u_k^* z_{k+1}^2\) and
\(\mathbb E[z_{k+1}^2]=1\). For population hessian by evaluating the loss at \(\theta=\theta^*\), we have 
\begin{equation}
\nabla_\theta^2\tilde\ell_{i,k}(\theta^*)
=
\frac{2\Big(-(a_k^*)^4 + 3(a_k^*)^2 r_{i,k}^2 + c(c-r_{i,k}^2)\Big)}{(u_k^*)^3}\,X_kX_k^\top.
\end{equation}
Substituting \(r_{i,k}^2=u_k^*z_{k+1}^2\) yields
\begin{equation}
    \nabla_\theta^2\tilde\ell_{i,k}(\theta^*)
=
\frac{2\Big(-(a_k^*)^4 + u_k^*\big(3(a_k^*)^2-c\big)z_{k+1}^2 + c^2\Big)}{(u_k^*)^3}\,X_kX_k^\top.
\end{equation}
Define the conditional population Hessian as \(H_{i,k}^*:=\mathbb E\left[\nabla_\theta^2\tilde\ell_{i,k}(\theta^*)\mid X_k\right]\). Since \(\mathbb E[z_{k+1}^2]=1\), we can have 
\begin{equation}
\label{eq:107}
H_{i,k}^* = \frac{2\Big(-(a_k^*)^4 + u_k^*\big(3(a_k^*)^2-c\big) + c^2\Big)}{(u_k^*)^3}\,X_kX_k^\top.
\end{equation}
Moreover, we observe that 
\begin{equation}
\label{eq:108}
-(a_k^*)^4 + u_k^*\big(3(a_k^*)^2-c\big) + c^2 =
-(a_k^*)^4 + \big((a_k^*)^2+c\big)\big(3(a_k^*)^2-c\big)+c^2
=2(a_k^*)^2u_k^*,
\end{equation}
and hence $H_{i,k}^* = \frac{4(a_k^*)^2}{(u_k^*)^2}\,X_kX_k^\top$. We then introduce the centered Hessian such that $Z_{i,k}:=\nabla_\theta^2\tilde\ell_{i,k}(\theta^*)-H_{i,k}^*$. Then, we have 
\begin{equation}
Z_{i,k} = \frac{2\big(3(a_k^*)^2-c\big)}{(u_k^*)^2}\,(z_{k+1}^2-1)\,X_kX_k^\top,
\end{equation}
For each entry \((j,\ell)\), we have 
\begin{equation}
[Z_{i,k}]_{j\ell} = \frac{2\big(3(a_k^*)^2-c\big)}{(u_k^*)^2}\,(z_{k+1}^2-1)\,X_{k,j}X_{k,\ell}.
\end{equation}

Therefore, conditioning on \(X_k\), we have 
\begin{equation}
    \begin{split}
\mathrm{Var}\left([Z_{i,k}]_{j\ell}\mid X_k\right)
& = \left(\frac{2\big(3(a_k^*)^2-c\big)}{(u_k^*)^2}X_{k,j}X_{k,\ell}\right)^2
\mathrm{Var}(z_{k+1}^2-1) \\
 \mathrm{Var}\left([Z_{i,k}]_{j\ell}\mid X_k\right) & = \frac{8\big(3(a_k^*)^2-c\big)^2}{(u_k^*)^4}\,X_{k,j}^2X_{k,\ell}^2 \le 
\frac{8\cdot 9(u_k^*)^2}{(u_k^*)^4}K^4
=
\frac{72K^4}{(u_k^*)^2}
\le
\frac{72K^4}{c^2},
\end{split}
\end{equation}
The second equality holds since \(z\sim\mathcal N(0,1)\) implies \(\mathrm{Var}(z^2-1)=\mathrm{Var}(z^2)=3-1=2\). With Proposition~\ref{Ass:xbound},  we have \(X_{k,j}^2X_{k,\ell}^2\le K^4\) with high probability. Then, we have that $|3(a_k^*)^2-c|\le 3(a_k^*)^2+c \le 3u_k^*$, which implies $\big(3(a_k^*)^2-c\big)^2\le 9(u_k^*)^2$. The last inequality follows from \(u_k^*=(a_k^*)^2+c\ge c\).
Moreover, let \(p\coloneqq|pa(i)|\) and condition on \(\{X_k\}_{k=1}^n\). We have 
\begin{equation}
\label{eq:89_fix}
\mathrm{Var}\left(\frac1n\sum_{k=1}^n [Z_{i,k}]_{j\ell}\mid \{X_k\}\right)
=\frac{1}{n^2}\sum_{k=1}^n \mathrm{Var}\left([Z_{i,k}]_{j\ell}\mid X_k\right)
\le \frac{72K^4}{n c^2},
\end{equation}
Then, we have 
\begin{equation}
\mathbb E\left[\big\|\frac1n\sum_{k=1}^n Z_{i,k}\big\|_\infty \mid \{X_k\}\right]
\le s\sum_{\ell=1}^s \max_{j\le s}\mathbb E\left[\big|\frac1n\sum_{k=1}^n [Z_{i,k}]_{j\ell}\big|\mid \{X_k\}\right]
\lesssim
s^2\sqrt{\frac{72K^4}{n c^2}}
\end{equation}
Then, our goal is to show that the following high-probability bound holds with probability conditioning on $X_k$
\begin{equation}
\label{eq:62}
\mathbb P\left(\big\|\frac{1}{n}\sum_{k=1}^n Z_{i,k}\big\|_\infty \le t\ \mid \{X_k\}\right)
\ge 1-\delta,
\qquad \forall\,0<\delta<1,
\end{equation}
by picking $t=s^2\sqrt{\frac{72K^4}{n c^2\delta}}$. Since \(\{z_{k+1}\}\) are i.i.d. and independent of \(\{X_k\}\), the variables \(\{[Z_{i,k}]_{j\ell}\}_{k=1}^n\) are independent and mean-zero conditional on \(\{X_k\}\). We have 
\begin{equation}
\|[Z_{i,k}]_{j\ell}\|_{\psi_1}\le |\frac{2\big(3(a_k^*)^2-c\big)}{(u_k^*)^2}\,X_{k,j}X_{k,\ell}|\|z_{k+1}^2-1\|_{\psi_1}
\le
\frac{18K^2}{u_k^*}
\le
\frac{18K^2}{c}=: C_2.
\end{equation}
This inequity holds since \(z\sim\mathcal N(0,1)\), the random variable \(z^2-1\) is sub-exponential, so \(\|z^2-1\|_{\psi_1}\le 3\). Thus, \([Z_{i,k}]_{j\ell}\in SE(\nu^2,\alpha)\) with parameters \(\nu^2\asymp C_2^2\) and \(\alpha\asymp C_2\). By Bernstein's inequality for sums of independent mean-zero sub-exponential random variables, for any \(t>0\),
\begin{equation}
\mathbb P\left(\big|\frac1n\sum_{k=1}^n [Z_{i,k}]_{j\ell}|>t\big| \mid  \ \{X_k\}\right)
\le
2\exp\left(
-\frac{n}{2}\min\left\{\frac{t^2}{\nu^2},\frac{t}{\alpha}\right\}
\right),
\end{equation}
Next, using \(\|A\|_\infty \le p\max_{j,\ell\le p}|A_{j\ell}|\), we have
\begin{equation}
\label{117}
    \begin{split}
\mathbb P\left(\big\|\frac1n\sum_{k=1}^n Z_{i,k}\big\|_\infty>t \mid \{X_k\}\right)&
\le \mathbb P\left(\max_{j,\ell\le s}\big|\frac1n\sum_{k=1}^n [Z_{i,k}]_{j\ell}\big|>\frac{t}{s}\mid \{X_k\}\right)\\
& \le
\sum_{j,\ell\le s}
\mathbb P\left(\left.\left|\frac1n\sum_{k=1}^n [Z_{i,k}]_{j\ell}\right|>\frac{t}{s}\ \right|\ \{X_k\}\right)\\
&\le 2s^2\exp\left(
-\frac{n}{2}\min\left\{\frac{t^2}{\nu^2 s^2},\frac{t}{\alpha s}\right\}
\right),
    \end{split}
\end{equation}
From Assumption~\ref{ass:alpha-mixing-EM}, we know that $\{X_k\}_{k\ge 1}$ is geometrically $\alpha$-mixing, such that $\alpha(m)\le C_0 e^{-c_0 m}$ for some constants $C_0,c_0>0$ and $m>1$. By our pervious definition, $Q^*_{V_iV_i}$ is the population Hessian block evaluated at the ground truth such that $Q^*_{V_iV_i} = n^{-1}\sum_{k=1}^n\mathbb E[H_{i,k}^*]$. For each $(j,\ell)\in V_i\times V_i$, recall that
\[
[H_{i,k}^*]_{j\ell}
=
\frac{4(a_k^*)^2}{((a_k^*)^2+c)^2}X_{k,j}X_{k,\ell}.
\]
For all $t\in\mathbb R$, define $\phi(t):=\frac{4t^2}{(t^2+c)^2}$, and we know that $0\le \phi(t)\le \frac{1}{c}$. It follows that
\[
\left|[H_{i,k}^*]_{j\ell}\right| \le \frac{1}{c}|X_{k,j}X_{k,\ell}| \le \frac{K^2}{c} \implies \left|\mathbb E\!\left([H_{i,k}^*]_{j\ell}\right)\right|
\le \mathbb E\!\left|[H_{i,k}^*]_{j\ell}\right|
\le \frac{K^2}{c}.
\]
Now, we define centered variable such that
\[
Y_{k,j\ell} := [H_{i,k}^*]_{j\ell} - \mathbb E\!\left([H_{i,k}^*]_{j\ell}\right) \implies|Y_{k,j\ell}| \le \left|[H_{i,k}^*]_{j\ell}\right| + \left|\mathbb E\!\left([H_{i,k}^*]_{j\ell}\right)\right| \le \frac{2K^2}{c}
\]
Moreover, since $Y_{k,j\ell}$ is a measurable function of $X_k$, the sequence
$\{Y_{k,j\ell}\}_{k=1}^n$ is also geometrically $\alpha$-mixing. Therefore, by a Bernstein-type inequality for centered bounded geometrically
$\alpha$-mixing sequences from Theorem 2 of \cite{merlevede2009bernstein}, there exists a constant $C>0$ such that, for every $x>0$,
\[
\mathbb P\!\left(
\left|\sum_{k=1}^n Y_{k,j\ell}\right|>x \right) \le 2\exp\!\left( - \frac{C x^2}{n + x(\log n)^2}\right),
\]
where the constant $C$ may depend on $K$, $c$, and the mixing parameters but not on $n$, $x$, $j$, or $\ell$. Equivalently, for every $u>0$,
\[
\mathbb P\!\left(\left|\frac1n\sum_{k=1}^n Y_{k,j\ell}\right|>u\right)\le
2\exp\!\left(-Cn\min\!\left\{u^2,\,\frac{u}{(\log n)^2}\right\}\right).
\]
Restoring the scale $K^2/c$, this can be written as
\[
\mathbb P\!\left(\left|\frac1n\sum_{k=1}^n Y_{k,j\ell}\right|>u\right)\le2\exp\!\left(-Cn\min\!\left\{\frac{u^2c^2}{K^4},\,\frac{uc}{K^2(\log n)^2}\right\}\right).
\]
Let $s:=|V_i|$. Since $\|A\|_\infty \le s\max_{j,\ell\in V_i}|A_{j\ell}|$,  we obtain by a union bound that
\begin{equation}
\begin{split}
    \mathbb P\!\left( \left\| \frac1n\sum_{k=1}^n H_{i,k}^*-Q^*_{V_iV_i}\right\|_\infty > t\right)&\le\sum_{j,\ell\in V_i} \mathbb P\!\left(\left|\frac1n\sum_{k=1}^n Y_{k,j\ell}\right|>\frac{t}{s}\right)\\
    &\le 2s^2\exp\!\left(-Cn\min\!\left\{\frac{t^2c^2}{s^2K^4},\,\frac{tc}{sK^2(\log n)^2}\right\}\right).
\end{split}
\end{equation}
Consequently,
\[
\left\|\frac1n\sum_{k=1}^n H_{i,k}^*-Q^*_{V_iV_i}\right\|_\infty=O_{\mathbb P}\!\left(\frac{K^2}{c}\,s\sqrt{\frac{\log s}{n}}+\frac{K^2}{c}\,s\frac{(\log n)^2\log s}{n}\right).
\]
Since $K$ and $c$ are treated as fixed constants, this simplifies to
\[
\left\| \frac1n\sum_{k=1}^n H_{i,k}^*-Q^*_{V_iV_i} \right\|_\infty = O_{\mathbb P}\!\left( s\sqrt{\frac{\log s}{n}} +s\frac{(\log n)^2\log s}{n}\right).
\]
Combining this with the bound in \eqref{117} yields
\[
\|Q^n_{V_iV_i}-Q^*_{V_iV_i}\|_\infty=O_{\mathbb P}\!\left(s\sqrt{\frac{\log s}{n}}+s\frac{(\log n)^2\log s}{n}\right).
\]
For the off-support block, we first bound only the fluctuation term. Using
\[
\|A\|_{\infty} \le s \max_{j \in V_i^c,\ell \in V_i} |A_{j\ell}|,
\]
together with a union bound over the $(p-s)s$ entries in $V_i^c \times V_i$, we obtain
\begin{equation}
\mathbb{P}\!\left( \left\|\frac{1}{n}\sum_{k=1}^n Z_{i,k;V_i^cV_i}\right\|_{\infty} > t\right) \le2(p-s)s \exp\!\left(-\frac{n}{2}\min\left\{\frac{t^2}{\nu^2 s^2},\frac{t}{\alpha s}\right\}\right).
\end{equation}
Moreover, applying the same mixing argument as above to
\[
\frac{1}{n}\sum_{k=1}^n\Big(H^*_{i,k;V_i^cV_i}-\mathbb{E}[H^*_{i,k;V_i^cV_i}]\Big),
\]
yields a bound of the same order. Therefore,
\begin{equation}
\left\|
Q^n_{V_i^cV_i}-Q^*_{V_i^cV_i}\right\|_{\infty}=O_{\mathbb{P}}\!\left(s\sqrt{\frac{\log\!\big((d-s)s\big)}{n}}+s\,\frac{(\log n)^2\log\!\big((d-s)s\big)}{n}\right).
\end{equation}
\end{proof}

\item \textbf{Proof of the third statement:}
\begin{proof}
From the Assumption~\ref{ass: Incoh}, and Proposition~\ref{prop: WD}, we know that there exists $\alpha\in(0,1]$ such that $\big\|Q^*_{V_i^cV_i}(Q^*_{V_iV_i})^{-1}\big\|_\infty \le 1-\alpha$;  there exists $C_{\min}>0$ such that $\lambda_{\min}(Q^*_{V_iV_i})\ge C_{\min}$. Moreover, from the second statement of Lemma~\ref{lem:prop}, we have, for any $t>0$, there exist constants $\tilde \nu,\tilde a, \tilde C>0$ with \(s\coloneqq|pa(i)|\) such that
\begin{equation}
\label{eq:concVV}
    \begin{split}
      \mathbb{P}&\left(\|Q^n_{V_iV_i}-Q^*_{V_iV_i}\|_\infty>t\right)\\
&\le\mathbb{P}\!\left(\left\|\frac1n\sum_{k=1}^n H^*_{i,k}-Q^*_{V_iV_i}\right\|_\infty>\frac{t}{2}\right)+ \mathbb{P}\!\left(\left\|\frac1n\sum_{k=1}^n Z_{i,k}\right\|_\infty>\frac{t}{2} \right)\\
&\le2s^2\exp\!\left(-C_1 n \min\!\left\{\frac{t^2 c^2}{4 s^2 K^4},\frac{t c}{2 s K^2 (\log n)^2}\right\}\right)+2s^2\exp\!\left(-\frac{n}{2}\min\!\left\{\frac{t^2}{4\nu^2 s^2},\frac{t}{2 a s}\right\}\right).\\
&\le 4s^2\exp\!\left( -\widetilde C\, n \min\left\{ \frac{t^2}{\widetilde \nu^{\,2}s^2}, \frac{t}{\widetilde a\, s(\log n)^2}\right\} \right),
    \end{split}
\end{equation}
Moreover, we can derive the following for off-block version.
\begin{equation}
\label{eq:concVcV}
\mathbb P\!\left(\|Q^n_{V_i^cV_i}-Q^*_{V_i^cV_i}\|_\infty>t\right)\le 4(d-s)s\exp\!\left( -\widetilde C\, n \min\left\{\frac{t^2}{\widetilde \nu^{\,2}s^2}, \frac{t}{\widetilde a\, s(\log n)^2} \right\}\right).
\end{equation}
Define the deviations $\Delta_{AB}:=Q^n_{AB}-Q^*_{AB}$. We use the decomposition such that
\begin{align}
Q^n_{V_i^cV_i}(Q^n_{V_iV_i})^{-1}
&=
\underbrace{Q^*_{V_i^cV_i}(Q^*_{V_iV_i})^{-1}}_{T_0}
+\underbrace{Q^*_{V_i^cV_i}\Big((Q^n_{V_iV_i})^{-1}-(Q^*_{V_iV_i})^{-1}\Big)}_{T_1}
\nonumber\\
&\quad
+\underbrace{\Delta_{V_i^cV_i}(Q^*_{V_iV_i})^{-1}}_{T_2}
+\underbrace{\Delta_{V_i^cV_i}\Big((Q^n_{V_iV_i})^{-1}-(Q^*_{V_iV_i})^{-1}\Big)}_{T_3}.
\label{eq:Tdecomp}
\end{align}
By the triangle inequality and Assumption~\ref{ass: Incoh},
\begin{equation}
\big\|Q^n_{V_i^cV_i}(Q^n_{V_iV_i})^{-1}\big\|_\infty
\le \|T_0\|_\infty+\sum_{j=1}^3\|T_j\|_\infty
\le (1-\alpha)+\sum_{j=1}^3\|T_j\|_\infty.
\label{eq:tri}
\end{equation}
Thus it suffices to show $\|T_j\|_\infty\le \alpha/6$ for $j=1,2,3$. We define the event such that $\mathcal E(t)\coloneqq\Big\{\|\Delta_{V_iV_i}\|_\infty\le t\Big\}\cap\Big\{\|\Delta_{V_i^cV_i}\|_\infty\le t\Big\}$. By the second statement of Lemma~\ref{lem:prop} and a union bound,
\begin{equation}
\mathbb P(\mathcal E(t)^c)
\le 4(s^2(d-s)s)\exp\!\left( -\widetilde C\, n \min\left\{ \frac{t^2}{\widetilde \nu^{\,2}s^2}, \frac{t}{\widetilde a\, s(\log n)^2}\right\} \right)
\label{eq:Etail}
\end{equation}
For any $s\times s$ matrix $M$, $\|M\|_2\le \|M\|_1\le s\|M\|_\infty$.
Hence on $\mathcal E(t)$, $\|Q^n_{V_iV_i}-Q^*_{V_iV_i}\|_2 \le s\|Q^n_{V_iV_i}-Q^*_{V_iV_i}\|_\infty \le s t$.
We pick $t\le C_{\min}/(2s)$, then by Weyl's inequality and Proposition~\ref{prop: WD}, we have that here exists $C_{\min}>0$ such that $\lambda_{\min}(Q^*_{V_iV_i})\ge C_{\min}$. Therefore, we have
\begin{equation}
    \lambda_{\min}(Q^n_{V_iV_i})
\ge \lambda_{\min}(Q^*_{V_iV_i})-\|Q^n_{V_iV_i}-Q^*_{V_iV_i}\|_2
\ge C_{\min}-st
\ge \frac{C_{\min}}{2}.
\end{equation}
Consequently, we have that $\|(Q^*_{V_iV_i})^{-1}\|_2\le \frac1{C_{\min}}$ and $\|(Q^n_{V_iV_i})^{-1}\|_2\le \frac2{C_{\min}}$. Moreover, since $\|A\|_1\le \sqrt{s}\|A\|_2$,
\begin{equation}
\|(Q^*_{V_iV_i})^{-1}\|_1\le \frac{\sqrt{s}}{C_{\min}},
\qquad
\|(Q^n_{V_iV_i})^{-1}\|_1\le \frac{2\sqrt{s}}{C_{\min}}.
\label{eq:inv1}
\end{equation}
Using $\|AB\|_\infty\le \|A\|_\infty\|B\|_1$ and \eqref{eq:inv1}, we have 
\begin{equation}
    \|T_2\|_\infty
=\|\Delta_{V_i^cV_i}(Q^*_{V_iV_i})^{-1}\|_\infty
\le \|\Delta_{V_i^cV_i}\|_\infty \|(Q^*_{V_iV_i})^{-1}\|_1
\le t\cdot \frac{\sqrt{s}}{C_{\min}}.
\end{equation}
Thus $\|T_2\|_\infty\le \alpha/6$ holds if $t\le \alpha C_{\min}/(6\sqrt{s})$. Using the identity $A^{-1}-B^{-1}=B^{-1}(B-A)A^{-1}$ with
$A=Q^n_{V_iV_i}$ and $B=Q^*_{V_iV_i}$, we have that $(Q^n_{V_iV_i})^{-1}-(Q^*_{V_iV_i})^{-1}
=(Q^*_{V_iV_i})^{-1}(Q^*_{V_iV_i}-Q^n_{V_iV_i})(Q^n_{V_iV_i})^{-1}$. Let $R:=Q^*_{V_i^cV_i}(Q^*_{V_iV_i})^{-1}$, so $\|R\|_\infty\le 1-\alpha$ by Assumption~\ref{ass: Incoh}.
Then, using $\|AB\|_\infty\le \|A\|_\infty\|B\|_1$ and $\|M\|_1\le s\|M\|_\infty$,
\begin{equation}
    \begin{split}
        \|T_1\|_\infty
&=\|R\,(Q^*_{V_iV_i}-Q^n_{V_iV_i})(Q^n_{V_iV_i})^{-1}\|_\infty\\
&\le \|R\|_\infty\cdot \|Q^*_{V_iV_i}-Q^n_{V_iV_i}\|_1\cdot \|(Q^n_{V_iV_i})^{-1}\|_1\\
&\le (1-\alpha)\cdot (s t)\cdot \frac{2\sqrt{s}}{C_{\min}}
=\frac{2(1-\alpha)}{C_{\min}}\, s^{3/2}\,t.
    \end{split}
\end{equation}
Hence $\|T_1\|_\infty\le \alpha/6$ holds if
$t\le \alpha C_{\min}/\big(12(1-\alpha)s^{3/2}\big)$. Similarly, we have that $\|T_3\|_\infty
\le \|\Delta_{V_i^cV_i}\|_\infty\cdot
\|(Q^n_{V_iV_i})^{-1}-(Q^*_{V_iV_i})^{-1}\|_1$. Using the same inverse-difference identity and $\|ABC\|_1\le \|A\|_1\|B\|_1\|C\|_1$,
\begin{equation}
    \begin{split}
        \|(Q^n_{V_iV_i})^{-1}-(Q^*_{V_iV_i})^{-1}\|_1
&\le \|(Q^*_{V_iV_i})^{-1}\|_1\cdot \|Q^*_{V_iV_i}-Q^n_{V_iV_i}\|_1\cdot \|(Q^n_{V_iV_i})^{-1}\|_1\\
&\le \frac{\sqrt{s}}{C_{\min}}\cdot (s t)\cdot \frac{2\sqrt{s}}{C_{\min}}
=\frac{2 s^2}{C_{\min}^2}\,t.
    \end{split}
\end{equation}
Therefore, $\|T_3\|_\infty \le t\cdot \frac{2 s^2}{C_{\min}^2}\,t
=\frac{2 s^2}{C_{\min}^2}\,t^2.$ In particular, if $t\le \alpha C_{\min}/(12 s^{3/2})$, then
\begin{equation}
    \|T_3\|_\infty \le \frac{2 s^2}{C_{\min}^2}\Big(\frac{\alpha C_{\min}}{12 s^{3/2}}\Big)^2
=\frac{\alpha^2}{72s}\le \frac{\alpha}{6},
\end{equation}
where the last inequality holds since $\alpha\le 1$ and $s\ge 1$. Choose $t=t^\star$ so that all the above conditions hold, namely
$t\le C_{\min}/(2s)$ and $\|T_j\|_\infty\le \alpha/6$ for $j=1,2,3$.
Then by \eqref{eq:tri},
\begin{equation}
    \big\|Q^n_{V_i^cV_i}(Q^n_{V_iV_i})^{-1}\big\|_\infty
\le (1-\alpha)+\frac{\alpha}{2}
=1-\frac{\alpha}{2}.
\end{equation}
Combining with Eq.\eqref{eq:Etail} yields the stated high-probability bound.
\end{proof}
\end{itemize}

\subsubsection{Proof of Lemma~\ref{lem:True_supp}}
\begin{proof}
Fix \(i\in V\) and \(j\in \mathrm{pa}(i)\). Define the smooth part of the oracle objective by
\[
\ell_{n,i}^{(c)}(\theta) := \frac{1}{2n}\sum_{k=0}^{n-1}
\left[ \log\bigl((\theta^\top X_k)^2+c\bigr) + \frac{r_{i,k}^2}{(\theta^\top X_k)^2+c} \right],
\]
and let $Q_{n,i}^{(c)}(\theta) := \ell_{n,i}^{(c)}(\theta)+\lambda\|\theta\|_1.$
The oracle estimator is
\[
\hat\theta_i^\lambda \in \arg\min_{\theta:\ \theta_k=0,\ \forall k\notin \mathrm{pa}(i)} Q_{n,i}^{(c)}(\theta).
\]
Let $E_{ij}:=\{\hat\theta_{ij}^\lambda=0\}$. Since \(j\in \mathrm{pa}(i)\), we have \(\theta_{ij}^*\neq 0\). Hence
\[
\bigl|\operatorname{sign}(\hat\theta_{ij}^\lambda)\bigr| \neq \bigl|\operatorname{sign}(\theta_{ij}^*)\bigr| \qquad\Longleftrightarrow\qquad \hat\theta_{ij}^\lambda=0.
\]
Therefore, it suffices to show
\[
\mathbb P(E_{ij})\le C\exp(-cn^\varepsilon).
\]
On the event \(E_{ij}\), the oracle estimator \(\hat\theta_i^\lambda\) belongs to the constrained set
\[
\Theta_i^{(j)}(0)
:=
\{\theta:\theta_j=0,\ \theta_k=0,\ \forall k\notin \mathrm{pa}(i)\}.
\]
Moreover, because \(\hat\theta_i^\lambda\) minimizes \(Q_{n,i}^{(c)}\) over the full oracle feasible set, it also minimizes the same objective over the smaller set \(\Theta_i^{(j)}(0)\). Hence, on \(E_{ij}\), the estimator \(\hat\theta_i^\lambda\) is itself a zero-constrained minimizer. By the KKT condition for the free oracle problem, its \(j\)-th coordinate being zero implies
\begin{equation}
\label{eq:true_supp_kkt_zero_fixed}
\left|G_{ij}^{(c)}(\hat\theta_i^\lambda)\right|
\le \lambda,
\qquad
G_{ij}^{(c)}(\theta):=\bigl[\nabla \ell_{n,i}^{(c)}(\theta)\bigr]_j.
\end{equation}
Next, we define the zero-constrained reference point $\theta_i^{*(j=0)} := \theta_i^*-\theta_{ij}^*e_j$, and the coordinate segment
\[
\vartheta_t := \theta_i^{*(j=0)}+t\theta_{ij}^*e_j, \qquad t\in[0,1].
\]
Let \(r_{\mathrm{loc},n}>0\) be the radius, and define the events $\mathcal E_{n,i}$ to be the high-probability event from Lemma~\ref{lem:zero_constrained_localization},
\[
\mathcal L_{n,ij}:=
\left\{ \sup_{\theta:\ \|\theta-\theta_i^{*(j=0)}\|_\infty\le r_{\mathrm{loc},n}} \left\| \bigl[\nabla^2 \ell_{n,i}^{(c)}(\theta)\bigr]_{j,\cdot} \right\|_1 \le L_c \right\},
\]
\[
\mathcal R_{n,ij} := \left\{
\inf_{t\in[0,1]} e_j^\top \nabla^2 \ell_{n,i}^{(c)}(\vartheta_t)e_j
\ge m_{n,c} \right\}, \qquad m_{n,c}:=\alpha_1-\tau r_n,
\]
and $\mathcal S_{n,ij} := \left\{ |G_{ij}^{(c)}(\theta_i^*)|\le \eta_n \right\}$.
Set
\[
\mathcal A_{n,ij} := \mathcal E_{n,i}\cap \mathcal L_{n,ij}\cap \mathcal R_{n,ij}\cap \mathcal S_{n,ij}.
\]
By Lemma~\ref{lem:zero_constrained_localization}, the argument in the proof of Lemma~\ref{lem:h_lips}, Lemma~\ref{lem:segment_curvature}, and the first statement of Lemma~\ref{lem:prop}, each of these events holds with probability at least \(1-C\exp(-cn^\varepsilon)\). Hence, by the union bound, there exist constants, $C.c >0$ such that 
\begin{equation}
\label{eq:true_supp_good_event_fixed}
\mathbb P(\mathcal A_{n,ij})
\ge
1-C\exp(-cn^\varepsilon).
\end{equation}
We now work on the event \(E_{ij}\cap \mathcal A_{n,ij}\), and compare \(G_{ij}^{(c)}(\hat\theta_i^\lambda)\) with \(G_{ij}^{(c)}(\theta_i^{*(j=0)})\). Since \(E_{ij}\cap \mathcal E_{n,i}\subseteq
\{\|\hat\theta_i^\lambda-\theta_i^{*(j=0)}\|_\infty\le r_{\mathrm{loc},n}\}\) by Lemma~\ref{lem:zero_constrained_localization}, we have on \(E_{ij}\cap \mathcal A_{n,ij}\),
\[
\|\hat\theta_i^\lambda-\theta_i^{*(j=0)}\|_\infty\le r_{\mathrm{loc},n}.
\]
Therefore, by the mean value theorem and the definition of \(\mathcal L_{n,ij}\),

\begin{equation}
\label{eq:true_supp_reduce_to_ref_fixed}
    \begin{split}
        \left| G_{ij}^{(c)}(\hat\theta_i^\lambda) - G_{ij}^{(c)}(\theta_i^{*(j=0)}) \right|
&\le \sup_{\theta:\ \|\theta-\theta_i^{*(j=0)}\|_\infty\le r_{\mathrm{loc},n}} \left\| \bigl[\nabla^2 \ell_{n,i}^{(c)}(\theta)\bigr]_{j,\cdot} \right\|_1 \|\hat\theta_i^\lambda-\theta_i^{*(j=0)}\|_\infty\\
&\le L_c\,r_{\mathrm{loc},n}. \\
\implies\left|G_{ij}^{(c)}(\hat\theta_i^\lambda)\right| \ge &\left|G_{ij}^{(c)}(\theta_i^{*(j=0)})\right| - L_c\,r_{\mathrm{loc},n}. 
    \end{split}
\end{equation}
Then, we need to show the lower boundedness of \( |G_{ij}^{(c)}(\theta_i^{*(j=0)})| \).  Since $\theta_i^*-\theta_i^{*(j=0)}=\theta_{ij}^*e_j,$ a one-dimensional Taylor expansion of \(G_{ij}^{(c)}\) along the segment \(\{\vartheta_t:t\in[0,1]\}\) gives
\begin{equation}
\label{eq:true_supp_taylor_fixed}
G_{ij}^{(c)}(\theta_i^{*(j=0)}) = G_{ij}^{(c)}(\theta_i^*) - \left( \int_0^1 e_j^\top \nabla^2 \ell_{n,i}^{(c)}(\vartheta_t)e_j\,dt \right)\theta_{ij}^*.
\end{equation}
Therefore, on \(\mathcal R_{n,ij}\cap \mathcal S_{n,ij}\),
\begin{align}
\left|G_{ij}^{(c)}(\theta_i^{*(j=0)})\right|
&\ge
\left( \int_0^1 e_j^\top \nabla^2 \ell_{n,i}^{(c)}(\vartheta_t)e_j\,dt \right)|\theta_{ij}^*| - |G_{ij}^{(c)}(\theta_i^*)| \notag\\
&\ge m_{n,c}\,|\theta_{ij}^*|-\eta_n.
\label{eq:true_supp_ref_lb_fixed}
\end{align}
Combining \eqref{eq:true_supp_reduce_to_ref_fixed} and \eqref{eq:true_supp_ref_lb_fixed}, we obtain on \(E_{ij}\cap \mathcal A_{n,ij}\),
\begin{equation}
\label{eq:true_supp_grad_lb_fixed}
\left|G_{ij}^{(c)}(\hat\theta_i^\lambda)\right|
\ge m_{n,c}\,|\theta_{ij}^*|-\eta_n-L_c\,r_{\mathrm{loc},n}.
\end{equation}
Let $\theta_{\min}:=\min_{j\in \mathrm{pa}(i)}|\theta_{ij}^*|.$ Assume the minimum signal strength condition is strong enough that, for all sufficiently large \(n\),
\begin{equation}
\label{eq:true_supp_beta_min_fixed}
m_{n,c}\,\theta_{\min} - \eta_n - L_c\,r_{\mathrm{loc},n} \ge \lambda+\delta_0 \qquad\text{for some }\delta_0>0.
\end{equation}
Then \eqref{eq:true_supp_grad_lb_fixed} implies that on \(E_{ij}\cap \mathcal A_{n,ij}\),
\begin{equation}
\label{eq:true_supp_strict_kkt_violation_fixed}
\left|G_{ij}^{(c)}(\hat\theta_i^\lambda)\right| > \lambda.
\end{equation}
But \eqref{eq:true_supp_kkt_zero_fixed} yields, on \(E_{ij}\), $|G_{ij}^{(c)}(\hat\theta_i^\lambda)|\le \lambda$, which contradicts \eqref{eq:true_supp_strict_kkt_violation_fixed}. Therefore, $E_{ij}\cap \mathcal A_{n,ij}=\varnothing$. Hence, by \eqref{eq:true_supp_good_event_fixed},
\begin{equation}
    \begin{split}
        & \mathbb P(E_{ij}) \le \mathbb P(\mathcal A_{n,ij}^c) \le C\exp(-cn^\varepsilon) \\
        &\implies\mathbb P\left( \bigl|\operatorname{sign}(\hat\theta_{ij}^\lambda)\bigr| =
    \bigl|\operatorname{sign}(\theta_{ij}^*)\bigr| \right) \ge 1-C\exp(-cn^\varepsilon).
    \end{split}
\end{equation}
Finally, applying the union bound over \(j\in \mathrm{pa}(i)\) gives the desired no-false-negative conclusion for node \(i\). Moreover, a corresponding global statement over all true edges follows from one additional union bound over $i\in V$ 
\end{proof}

\subsubsection{Proof of Lemma~\ref{lem: PDW}}

\begin{proof}
For node $i$, let $S_i \coloneqq pa(i)\subset[p]$ denote its true parent set. We first construct a restricted estimator (primal witness) by solving the optimization problem only over $S_i$:
\begin{equation}
\hat\theta_{i,S_i} \coloneqq \arg\min_{\beta_{S_i}} \Big\{\ell_i(\beta_{S_i}, 0_{S_i^c}) + \lambda_n \|\beta_{S_i}\|_1\Big\},
\quad
\hat\theta_{i,S_i^c} \coloneqq 0.
\end{equation}
By the fundamental theorem of calculus, define the averaged Hessian along the line segment between $\hat{\theta}_i$ and $\theta_i^*$ such that $\nabla^2 \ell(\tilde{\theta}_i)\coloneqq
\int_{0}^{1} \nabla^{2}\ell\big(\theta_i^* + t(\hat{\theta}_i-\theta_i^*)\big)dt $. Then, we additionally have the following
\begin{equation}
\begin{split}
      \nabla L( \hat \theta_i) -  \nabla L(\theta_i^*) =  \nabla^2 L(\tilde\theta_i)(\hat \theta_i - \theta_i^*)
\end{split}
\end{equation}
Since $\hat\theta_{i,S_i}$ minimizes the restricted problem, the KKT conditions yield that $0 \in\nabla L( \hat \theta_i) + \lambda\partial z_i(j)$, where $z_i(j)\coloneqq sign(\hat \theta_{ij})$ if $\hat \theta_{ij} \neq 0$; $|z_i(j)| \le 1$, otherwise. Then, we introduce $W_{pa(i)} \coloneqq \nabla L(\theta_i^*)$, and $R_{pa(i)} \coloneqq (\nabla^2 L(\theta^*_i) - \nabla^2 L(\tilde\theta_i))(\hat \theta_i - \theta_i^*) $, $Q^n \coloneqq \nabla^2 L(\theta^*_i)$, $\Delta \coloneqq\hat \theta_i - \theta_i^*$, and $z_i(j) \coloneqq \hat Z_i$
and have 
\begin{equation}
\begin{split}
        \nabla^2 L(\theta^*_i)(\hat \theta_i - \theta_i^*)& =  \nabla^2 L(\tilde\theta_i)(\hat \theta_i - \theta_i^*)+   (\nabla^2 L(\theta^*_i)(\hat \theta_i - \theta_i^*)- \nabla^2 L(\tilde\theta_i)(\hat \theta_i - \theta_i^*)) \\
        &=   \nabla L( \hat \theta_i) -  \nabla L(\theta_i^*) +(\nabla^2 L(\theta^*_i) - \nabla^2 L(\tilde\theta_i))(\hat \theta_i - \theta_i^*)
        \\& = -\lambda  z_i(j) - \nabla L(\theta_i^*) +(\nabla^2 L(\theta^*_i) - \nabla^2 L(\tilde\theta_i))(\hat \theta_i - \theta_i^*)\\
          \implies Q^n\Delta &=  -\lambda  z_i(j) - W_{pa(i)} +R_{pa(i)}    
\end{split}
\end{equation}
We decompose \(Q^n\) by introducing the $S_i^c \in [p] \setminus pa(i)$ as  

\begin{equation}
\begin{split}
        Q^n & =
\begin{pmatrix}
Q^n_{S_iS_i} & Q^n_{S_iS_i^c}\\ 
Q^n_{S_i^cS_i} & Q^n_{S_i^cS_i^c}
\end{pmatrix} 
\begin{pmatrix}
\Delta_{S_i}\\ \Delta_{S_i^c}
\end{pmatrix} = 
-\lambda_n
    \begin{pmatrix}
(z_i(j))_{S_i}\\ (z_i(j))_{S_i^c}
\end{pmatrix}
-\begin{pmatrix}
(W_{pa(i)} )_{S_i}\\ (W_{pa(i)} )_{S_i^c}
\end{pmatrix}
+\begin{pmatrix}
(R_{pa(i)})_{S_i}\\ (R_{pa(i)})_{S_i^c}
\end{pmatrix}.\\
& Q^n_{S_iS_i}\Delta_{S_i}
+ Q^n_{S_iS_i^c}\Delta_{S_i^c}=
-\lambda_n (\hat Z_i)_{S_i}
- (W_{pa(i)} )_{S_i}
+ (R_{pa(i)})_{S_i}, \\
&
Q^n_{S_i^cS_i}\Delta_{S_i}
+ Q^n_{S_i^cS_i^c}\Delta_{S_i^c}
=
-\lambda_n (\hat Z_i)_{S_i^c}
- (W_{pa(i)} )_{S_i^c}
+ (R_{pa(i)})_{S_i^c}. 
\end{split}
\end{equation}
Since $\hat\theta_{i,S_i^c}=0$ by using the restricted estimator and $\theta^*_{i,S_i^c}=0$ because $S_i$ is the true support, we have $\Delta_{S_i^c}=(\hat\theta_i-\theta_i^*)_{S_i^c}=0.$ Consequently, we have that $Q^n_{S_iS_i^c}\Delta_{S_i^c}=0$ and $Q^n_{S_i^cS_i^c}\Delta_{S_i^c}=0$.

\begin{equation}
    \begin{split}
       Q^n_{S_iS_i}\Delta_{S_i}
&=
-\lambda_n (\hat Z_i)_{S_i}
- (W_{pa(i)} )_{S_i}
+ (R_{pa(i)})_{S_i},
\\
\Delta_{S_i}
&=
\big(Q^n_{S_iS_i}\big)^{-1}
\Big(
-\lambda_n (\hat Z_i)_{S_i}
- (W_{pa(i)} )_{S_i}
+ (R_{pa(i)})_{S_i}
\Big)
    \end{split}
\end{equation}
We then have 
\begin{equation}
\begin{split}
 & Q^n_{S_i^cS_i}\Delta_{S_i}
    =-\lambda_n (\hat Z_i)_{S_i^c}
- (W_{pa(i)} )_{S_i^c}
+ (R_{pa(i)})_{S_i^c} \\
&Q^n_{S_i^cS_i}\big(Q^n_{S_iS_i}\big)^{-1}
\Big(
-\lambda_n (\hat Z_i)_{S_i}
- (W_{pa(i)} )_{S_i}
+ (R_{pa(i)})_{S_i}
\Big)
=
-\lambda_n (\hat Z_i)_{S_i^c}
- (W_{pa(i)} )_{S_i^c}
+ (R_{pa(i)})_{S_i^c}. \\
& -\lambda_n
Q^{n}_{S_i^c S_i}(Q^{n}_{S_i S_i})^{-1}
\Bigl(
(\hat Z_{i})_{S_i}
+ \tfrac{1}{\lambda_n}(W_{pa(i)} )_{S_i}
- \tfrac{1}{\lambda_n}(R_{pa(i)})_{S_i}
\Bigr)
=
- \lambda_n(\hat Z_{i})_{S_i^c}
- (W_{pa(i)} )_{S_i^c}
+ (R_{pa(i)})_{S_i^c} \\
&Q^{n}_{S_i^c S_i}(Q^{n}_{S_i S_i})^{-1}
\Bigl(
(\hat Z_{i})_{S_i}
+ \tfrac{1}{\lambda_n}(W_{pa(i)} )_{S_i}
- \tfrac{1}{\lambda_n}(R_{pa(i)})_{S_i}
\Bigr)
=
(\hat Z_{i})_{S_i^c}
+ \tfrac{1}{\lambda_n}(W_{pa(i)} )_{S_i^c}
- \tfrac{1}{\lambda_n}(R_{pa(i)})_{S_i^c} \\
&
(\hat Z_{i})_{S_i^c}
=
Q^{n}_{S_i^c S_i}(Q^{n}_{S_i S_i})^{-1}
(\hat Z_{i})_{S_i}
+\frac{1}{\lambda_n}
Q^{n}_{S_i^c S_i}(Q^{n}_{S_i S_i})^{-1}
\bigl( (W_{pa(i)} )_{S_i} - (R_{pa(i)})_{S_i} \bigr)
- \frac{1}{\lambda_n}
\bigl( (W_{pa(i)} )_{S_i^c} - (R_{pa(i)})_{S_i^c} \bigr).
\end{split}
\end{equation}
By the KKT conditions for the $\ell_1$ penalty,  the dual vector $\hat Z_i$ satisfies $(\hat Z_i)_{S_i} \in \partial \|\Theta_{S_i}\|_{1}$, and hence, for every active index $u\in S_i$, $\|(\hat Z_i)_u\|_1 = 1$. Then, we next show the dual norm $\|(\hat Z_i)_{S_i^c}\|_\infty <1$

 \begin{equation}
 \begin{split}
      \|(\hat Z_{i})_{S_i^c}\|_\infty & \le  \|Q^{n}_{S_i^c S_i}(Q^{n}_{S_i S_i})^{-1}\|_\infty \| (\hat Z_{i})_{S_i}\|_\infty + \| - \frac{1}{\lambda_n}(W_{pa(i)} )_{S_i^c} + \frac{1}{\lambda_n} (R_{pa(i)})_{S_i^c} \bigr)\|_\infty\\
      & \qquad 
+ \|Q^{n}_{S_i^c S_i}(Q^{n}_{S_i S_i})^{-1}\|_\infty \| - \frac{1}{\lambda_n}(W_{pa(i)} )_{S_i} +\frac{1}{\lambda_n} (R_{pa(i)})_{S_i} \bigr)\|_\infty \\
               & \le  \|Q^{n}_{S_i^c S_i}(Q^{n}_{S_i S_i})^{-1}\|_\infty \| (\hat Z_{i})_{S_i}\|_\infty + \| - \frac{1}{\lambda_n}W_{pa(i)}  + \frac{1}{\lambda_n} R_{pa(i)} \bigr)\|_\infty  \\
      & \qquad +\|Q^{n}_{S_i^c S_i}(Q^{n}_{S_i S_i})^{-1}\|_\infty \| - \frac{1}{\lambda_n}W_{pa(i)}  + \frac{1}{\lambda_n} R_{pa(i)} \|_\infty \\
&\le \|Q^{n}_{S_i^c S_i}(Q^{n}_{S_i S_i})^{-1}\|_\infty \big ( 1 + \| - \frac{1}{\lambda_n}W_{pa(i)}  + \frac{1}{\lambda_n} R_{pa(i)} \bigr)\|_\infty\big )  \\
& \qquad+1 + \| - \frac{1}{\lambda_n}W_{pa(i)} + \frac{1}{\lambda_n} R_{pa(i)} \|_\infty-1 \\
& \le  \big (\|Q^{n}_{S_i^c S_i}(Q^{n}_{S_i S_i})^{-1}\|_\infty +1 \big ) \big( 1 +\frac{1}{\lambda_n}
\|W_{pa(i)}\|_\infty + \frac{1}{\lambda_n}\| R_{pa(i)} \|_\infty
\big) -1 \\
& \le (2-\alpha) (1 + \frac{\alpha}{4(2-\alpha)}+ \frac{\alpha}{4(2-\alpha)}) - 1 = 1-\alpha +\frac{\alpha}{2} =1 -\frac{\alpha}{2}  < 1
\end{split}
\end{equation}
For the last inequality, we have shown that the empirical fisher $\|Q^{n}_{S_i^c S_i}(Q^{n}_{S_i S_i})^{-1}\|_\infty$ can be bounded by $1-\alpha$ with high probability. By Lemma~\ref{lem:WR_bound}, we show that the $\frac{1}{\lambda_n} \| (W_{pa(i)} )_{S_i^c}\|_\infty  $ and $\frac{1}{\lambda_n}\| (R_{pa(i)})_{S_i^c} \|_\infty$ can be bounded by $\frac{\alpha}{4(2-\alpha)}$ with high probability. 
\end{proof}

\subsubsection{Proof of Lemma~\ref{lem:h_lips}}
\begin{proof}
Write $\ell$ in composite form:
\begin{equation}
\ell(\theta_i)=\frac{1}{2n}\sum_{k=0}^{n-1}\phi_k(\theta_i^\top X_k),
\quad
\phi_k(z)=\log(z^2+c)+\frac{r_{i,k}^2}{z^2+c}.
\end{equation}
Let $r_k:=r_{i,k}$. Direct differentiation gives
\begin{equation}
    \begin{split}
     &\phi_k'(z)=\frac{2z}{z^2+c}-\frac{2r_k^2 z}{(z^2+c)^2}, \quad \phi_k''(z)=\frac{2(c-z^2)}{(z^2+c)^2}-\frac{2r_k^2(c-3z^2)}{(z^2+c)^3}, \\ & \qquad\phi_k^{(3)}(z)
=\frac{4z\Big(z^4-(2c+6r_k^2)z^2-3c^2+6cr_k^2\Big)}{(z^2+c)^4}.
    \end{split}
\end{equation}
Using triangle inequality and $(z^2+c)\ge c$, for all $z\in\mathbb R$,
\begin{equation}
    |\phi_k^{(3)}(z)|
\le
\frac{4|z|}{(z^2+c)^4}\Big(|z|^4+(2c+6r_k^2)|z|^2+(3c^2+6cr_k^2)\Big)
\le
\frac{24}{c^{3/2}}+\frac{48r_k^2}{c^{5/2}}.
\end{equation}
From Lemma~\ref{lem:residual_bound}, we know that $|r_k|\le R$ with high probability, hence
\begin{equation}
\label{Aeq: star}
\sup_{z\in\mathbb R}|\phi_k^{(3)}(z)|
\le M_c \coloneqq
\frac{24}{c^{3/2}}+\frac{48R^2}{c^{5/2}}.
\end{equation}
By the chain rule, we have $\nabla^2\ell(\theta_i)=\frac{1}{2n}\sum_{k=0}^{n-1}\phi_k''(\theta_i^\top X_k)X_kX_k^\top$, so
\begin{equation}
    \nabla^2\ell(\theta_i)-\nabla^2\ell(\theta_i')
=
\frac{1}{2n}\sum_{k=0}^{n-1}
\Big(\phi_k''(\theta_i^\top X_k)-\phi_k''(\theta_i'^\top X_k)\Big)X_kX_k^\top.
\end{equation}
For each $k$, by the 1D mean value theorem and Eq.\eqref{Aeq: star},
\begin{equation}
    \big|\phi_k''(\theta_i^\top X_k)-\phi_k''(\theta_i'^\top X_k)\big|
\le
M_c\big|(\theta_i-\theta_i')^\top X_k\big|
\le
M_c\|\theta_i-\theta_i'\|_1\|X_k\|_\infty,
\end{equation}
where the last inequality uses $\ell_1$--$\ell_\infty$ duality. Taking entrywise maximum norms,
\begin{equation}
    \big\|\nabla^2\ell(\theta_i)-\nabla^2\ell(\theta_i')\big\|_\infty
\le
\frac{1}{2n}\sum_{k=0}^{n-1}
\big|\phi_k''(\theta_i^\top X_k)-\phi_k''(\theta_i'^\top X_k)\big|\cdot \|X_kX_k^\top\|_\infty.
\end{equation}
From Proposition~\ref{Ass:xbound}, $\|X_k\|_\infty\le K$ and $\|X_kX_k^\top\|_\infty  \le \|X_k\|_\infty^2 \le K^2.$ Therefore, we have
\begin{equation}
    \big\|\nabla^2\ell(\theta_i)-\nabla^2\ell(\theta_i')\big\|_\infty
\le
\frac{1}{2n}\sum_{k=0}^{n-1}
\Big(M_c\|\theta_i-\theta_i'\|_1\|X_k\|_\infty\Big)\cdot K^2
\le
\frac{M_c}{2}K^3\|\theta_i-\theta_i'\|_1.
\end{equation}
Substituting $M_c=\frac{24}{c^{3/2}}+\frac{48R^2}{c^{5/2}}$ yields the claim with $\bar D_{\max}:=\frac{K^3}{2}\Big(\frac{24}{c^{3/2}}+\frac{48R^2}{c^{5/2}}\Big).$
\end{proof}

\subsubsection{Proof of Lemma~\ref{lem:RSC}}
\begin{proof}
Let $u_k := \theta^\top X_k$ and $r_k := \theta_i^{*\top}X_k\varepsilon_{k+1}$. The Hessian representation can be written as
\begin{equation}
\begin{split}
\nabla^2 \ell(\theta_i)
&=\frac{1}{2n}\sum_{k=0}^{n-1}\phi_{i,k}''(u_k)X_kX_k^\top=\frac{1}{n}\sum_{k=0}^{n-1}
\left[
\frac{c-u_k^2}{(u_k^2+c)^2}
-\frac{r_{i,k}^2(c-3u_k^2)}{(u_k^2+c)^3}
\right] X_kX_k^\top\\
\end{split}
\end{equation}
Thus, we have 
\begin{equation}
\label{eq:quad-hessian-main}
    \begin{split}
      \Delta^\top \nabla^2 \ell(\theta_i) \Delta
&=\frac{1}{n}\sum_{k=0}^{n-1}(\Delta^\top X_k)^2
\left[
\frac{c-u_k^2}{(u_k^2+c)^2}
-\frac{r_{i,k}^2(c-3u_k^2)}{(u_k^2+c)^3}
\right]
\end{split}
\end{equation}
From Proposition~\ref{Ass:xbound}, the event that $\mathcal E_K \coloneqq \{\sup_{0\le s\le T}\|X_s\|\le K\}$ exists with high probability. 
On the event $\mathcal E_K$, since $\|\theta\|_2\le \sqrt{c/(2K^2)}$, we have the following for all $k$
\[
u_k^2=(\theta^\top X_k)^2 \le \|\theta\|_2^2\|X_k\|_2^2 \le \frac{c}{2K^2}K^2 =\frac c2.
\]
Hence, we observe that $c-u_k^2\ge \frac c2$ and $u_k^2+c\le \frac{3c}{2}$. Therefore
\[
\frac{c-u_k^2}{(u_k^2+c)^2} \ge \frac{c/2}{(3c/2)^2} = \frac{2}{9c}.
\]
Also, since $u_k^2\le c/2$, we also have that $|c-3u_k^2|\le c$ and $u_k^2+c\ge c$. So, the following holds
\[
-\frac{r_{i,k}^2(c-3u_k^2)}{(u_k^2+c)^3} \ge  -\frac{r_{i,k}^2|c-3u_k^2|}{(u_k^2+c)^3} \ge -\frac{r_{i,k}^2}{c^2}.
\]
Substituting these two bounds into Eq.\eqref{eq:quad-hessian-main} yields
\begin{equation}\label{eq:hessian-lower-Mn}
\begin{split}
 &\Delta^\top \nabla^2\ell_i(\theta)\Delta \ge \frac{2}{9c}\cdot \frac1n\sum_{k=0}^{n-1}(\Delta^\top X_k)^2 - \frac1{c^2}\cdot \frac1n\sum_{k=0}^{n-1}r_{i,k}^2(\Delta^\top X_k)^2.\\   
 &\implies \Delta^\top \nabla^2\ell_i(\theta)\Delta \ge \frac{2}{9c}\,\Delta^\top \hat\Sigma \Delta - \frac1{c^2}\,\Delta^\top M_n\Delta.
\end{split}
\end{equation}
Next, conditional on $X$, we have $\mathbb E[M_n\mid X] = \frac1n\sum_{k=0}^{n-1}(\theta_i^{*\top}X_k)^2X_kX_k^\top$. Since $\epsilon_{k+1} \sim \mathcal{N}(0,1)$, we know that $\epsilon_{k+1}^2$ is chi-square, and we have $\mathbb E[\epsilon_{k+1}^2\mid X]=1$. On $\mathcal E_K$, we have that 
\[
(\theta_i^{*\top}X_k)^2 \le \|\theta_i^\star\|_2^2\|X_k\|_2^2 \le B^2K^2.
\]
Hence $\mathbb E[M_n\mid X] \preceq B^2K^2\,\hat\Sigma.$
Therefore, for any $\Delta\in\mathbb R^p$,
\[
\Delta^\top M_n\Delta = \Delta^\top \mathbb E[M_n\mid X]\Delta + \Delta^\top\bigl(M_n-\mathbb E[M_n\mid X]\bigr)\Delta
\]
From the Lemma~\ref{lem:Mn-concentration}, we know that $\mathcal E_M \coloneqq\|M_n-\mathbb E[M_n\mid X]\|_{\max} \le \Omega(C,K) \sqrt{\frac{\log(2p^2)}{n}}:=q_n$ with high probability. Thus, on the event $\mathcal E_K\cap \mathcal E_M$, we have 
\[
\Delta^\top M_n\Delta \le B^2K^2\,\Delta^\top \hat\Sigma\Delta + \|M_n-\mathbb E[M_n\mid X]\|_\infty \|\Delta\|_1^2
\le
B^2K^2\,\Delta^\top \hat\Sigma\Delta + q_n\|\Delta\|_1^2.
\]
Plugging this into \eqref{eq:hessian-lower-Mn} gives
\[
\Delta^\top \nabla^2\ell_i(\theta)\Delta
\ge
\left(
\frac{2}{9c}-\frac{B^2K^2}{c^2}
\right)
\Delta^\top \hat\Sigma\Delta
-
\frac{q_n}{c^2}\|\Delta\|_1^2
=
\gamma_c\,\Delta^\top \hat\Sigma\Delta
-
\frac{q_n}{c^2}\|\Delta\|_1^2.
\]
It remains to show the lower boundedness of $\Delta^\top \hat\Sigma\Delta$. From Proposition~\ref{prop: WD}, we have $m\coloneqq\lambda_{\min}(\mathbb E\hat\Sigma)>0$. Then, we have $\Delta^\top\mathbb E\hat\Sigma\Delta \ge m\|\Delta\|_2^2$. Also,
\[
\Delta^\top \hat\Sigma\Delta = \Delta^\top \mathbb E[\hat\Sigma]\Delta + \Delta^\top(\hat\Sigma-\mathbb E[\hat\Sigma])\Delta.
\]
Using the standard bound such that $|\Delta^\top A\Delta|\le \|A\|_\infty\|\Delta\|_1^2$, we obtain the following on $\mathcal E_\Sigma \coloneqq\{ \hat\Sigma-\mathbb E[\hat\Sigma]\le r_n\}$,
\[
\Delta^\top \hat\Sigma\Delta \ge m\|\Delta\|_2^2-r_n\|\Delta\|_1^2.
\]
Hence, on $\mathcal E_K\cap \mathcal E_\Sigma\cap \mathcal E_M$,
\[
\Delta^\top \nabla^2\ell_i(\theta)\Delta \ge \gamma_c\bigl(m\|\Delta\|_2^2-r_n\|\Delta\|_1^2\bigr) - \frac{q_n}{c^2}\|\Delta\|_1^2 = m\gamma_c\|\Delta\|_2^2 - \left( \gamma_c r_n+\frac{q_n}{c^2} \right)\|\Delta\|_1^2.
\]
This proves
\[
\Delta^\top \nabla^2\ell_i(\theta)\Delta \ge \alpha\|\Delta\|_2^2-\tau_n\|\Delta\|_1^2,
\]
with $\alpha=m\gamma_c$ and $\tau_n=\gamma_c r_n+\frac{q_n}{c^2}$.  Finally, if $c\ge 9B^2K^2$, then $\frac{B^2K^2}{c^2}\le \frac{1}{9c}$. 
so
\[
\gamma_c \equiv \frac{2}{9c}-\frac{B^2K^2}{c^2} \ge \frac{2}{9c}-\frac{1}{9c} = \frac{1}{9c}.
\]
Thus
\[
\Delta^\top \nabla^2\ell_i(\theta)\Delta \ge \frac{m}{9c}\|\Delta\|_2^2 - \left( \frac{r_n}{9c}+\frac{q_n}{c^2} \right)\|\Delta\|_1^2.
\]
Since $\log(2p^2)=\log 2+2\log p \asymp \log p$, it follows that $q_n \asymp \sqrt{\frac{\log(2p^2)}{n}} \asymp\sqrt{\frac{\log p}{n}}$. To complete the proof, we need to show the $r_n \asymp\sqrt{\frac{\log p}{n}}$, let $\hat\Sigma = \frac1n\sum_{k=0}^{n-1}X_kX_k^\top$ and $\Sigma_k=\mathbb E[X_kX_k^\top]$.For the $a,b$ entries, we define the centered variable such that  $Z_k^{(ab)}\coloneqq X_{k,a}X_{k,b}-\mathbb E[X_{k,a}X_{k,b}]$ so that $\hat\Sigma_{ab}-\mathbb E\hat\Sigma_{ab}=\frac1n\sum_{k=0}^{n-1}Z_k^{(ab)}$. By construction, \(\mathbb{E}[Z_k^{(ab)}]=0\). Moreover, since \(Z_k^{(ab)}\) is a measurable function of \(X_k\), the sequence \((Z_k^{(ab)})_{k=0}^{n-1}\) inherits the same geometric \(\alpha\)-mixing rate as \((X_k)_{k=0}^{n-1}\). Under the boundedness on \(X_k\), we have that $|Z_k^{(ab)}|\le 2K^2:= M>0$ uniformly over $k,a,b$, and we define
\[
v_{ab}^2 := \sup_{i\ge 0} \left( \operatorname{Var}(Z_i^{(ab)}) + 2\sum_{j>i}\bigl|\operatorname{Cov}(Z_i^{(ab)},Z_j^{(ab)})\bigr| \right),
\]
By Proposition~\ref{prop: cov_decay}, there exist $C_0,c_0>0$ such that
$\sum_{h\ge0}\sup_t|Cov(Z_t^{(ab)},Z_{t+h}^{(ab)})|\le \sum_{h\ge0}C_0e^{-c_0h}<\infty$. Hence, $v^2:=\sup_{a,b} v_{ab}^2 < \infty$. Then, by Theorem 2 of \cite{merlevede2009bernstein}, there exists a constant \(c>0\) such that for all \(x>0\),
\begin{equation}
\begin{split}
       \mathbb{P}\left(\left|\sum_{k=0}^{n-1} Z_k^{(ab)}\right|\ge x\right)\le\exp\left(-\frac{c x^2}{v^2 n + M^2 + xM(\log n)^2}\right),  
\end{split}
\end{equation}
Taking \(x = nu\) and for \(n\) sufficiently large, we obtain
\begin{equation}
    \begin{split}
    \mathbb{P}\left( \left|\widehat{\Sigma}_{ab}-\mathbb{E}\widehat{\Sigma}_{ab}\right|\ge u\right)&\le\exp\left(-\frac{c n^2 u^2}{v^2 n + M^2 + n u M(\log n)^2}\right) \\
    &\le \exp\left( - c n \min\left(\frac{u^2}{v^2},\frac{u}{M(\log n)^2}\right)\right).
    \end{split}
\end{equation}
Applying a union bound over all \(p^2\) entries yields
\[
\|\widehat{\Sigma}-\mathbb{E}\widehat{\Sigma}\|_\infty= O_{\mathbb{P}}\left(\sqrt{\frac{\log p}{n}}\vee\frac{(\log n)^2\log p}{n}\right).
\]
Given the scale of $n$ such that $n\gtrsim {(\log n)^4}\log p$, the first term dominates, and hence
\[
\|\widehat{\Sigma}-\mathbb{E}\widehat{\Sigma}\|_\infty = O_{\mathbb{P}}\left( \sqrt{\frac{\log p}{n}} \right).
\]
Therefore, with probability at least \(1-\delta\), there exist a constant \(C>0\), such that 
\[
\|\widehat{\Sigma}-\mathbb{E}\widehat{\Sigma}\|_\infty \le r_n, \qquad r_n = C\sqrt{\frac{\log p}{n}},
\]
\end{proof}

\subsubsection{Proof of Lemma~\ref{lem:err_rsc}}
\begin{proof}
By stationarity, choose $z\in\partial\|\hat\theta\|_1$ such that
$\nabla\ell(\hat\theta)+\lambda z=0$. Taking inner product with $\Delta$ gives $\langle \nabla\ell(\hat\theta),\Delta\rangle = -\lambda\langle z,\Delta\rangle$. By convexity of $\|\cdot\|_1$, for any $z\in\partial\|\hat\theta\|_1$, we have that $\|\theta^*\|_1 \ge \|\hat\theta\|_1 + \langle z,\theta^*-\hat\theta\rangle
$, which implies that $ \langle z,\Delta\rangle \ge \|\hat\theta\|_1-\|\theta^*\|_1.$ Hence, we know that $\langle \nabla\ell(\hat\theta),\Delta\rangle
\le
\lambda(\|\theta^*\|_1-\|\hat\theta\|_1).$ Therefore, 
\begin{equation}
\label{Eq:reC_2ound}
\begin{split}
        \langle \nabla\ell(\hat\theta)-\nabla\ell(\theta^*),\Delta\rangle
&=
\langle \nabla\ell(\hat\theta),\Delta\rangle-\langle \nabla\ell(\theta^*),\Delta\rangle\\
&\le
\lambda(\|\theta^*\|_1-\|\hat\theta\|_1)+\|\nabla\ell(\theta^*)\|_\infty\|\Delta\|_1\\
&\le
\lambda(\|\theta^*\|_1-\|\hat\theta\|_1)+\frac{{\lambda}}{4}\|\Delta\|_1 ,
    \end{split}
\end{equation}
where the last inequality uses the score bound from Lemma~\ref{lem:WR_bound} with $\alpha <1$. Given the true parent set $S$ and the non parent $S^c$, we have 
\begin{equation}
\begin{split}
    \|\theta^*\|_1-\|\hat\theta\|_1 &\le 
        \|\theta^*\|_1- \|\hat\theta_S\|_1 +\|\hat\theta_{S^c}\|_1  \\ & =  \|\theta^*\|_1 -\|\theta_S^* +\Delta_S\|_1 + \|\Delta_{S^c}\|_1\\
        &\le \|\Delta_S\|_1-\|\Delta_{S^c}\|_1.
\end{split}
\end{equation}
Plugging this into the Eq.\eqref{Eq:reC_2ound} yields
\begin{equation}
    \langle \nabla\ell(\hat\theta)-\nabla\ell(\theta^*),\Delta\rangle
\le
\frac{5\lambda}{4}\|\Delta_S\|_1-\frac{\lambda3}{4}\|\Delta_{S^c}\|_1.
\end{equation}
Then, we can have the following as well with the size of true parent set $|S| \coloneqq |pa(i)|$ 
\begin{equation}
\langle \nabla\ell(\hat\theta)-\nabla\ell(\theta^*),\Delta\rangle
\le
\frac{5\lambda}{4}\|\Delta_S\|_1
\le
\frac{5\lambda}{4}\sqrt{s}\|\Delta\|_2.
\end{equation}
Since we know that $\|\theta_i^*\|_2 \le B$ and $\|\hat\theta_i\|_2\le 2\sqrt{BL_\epsilon}$ from our algorithm and Lemma~\ref{lem:RSC}, we have 
\begin{equation}
   \|\Delta\|_2= \|\hat\theta_i - \theta^*_i\|_2 \le \|\hat\theta_i\|_2 +\|\theta^*\|_2\le B+ 2 \sqrt{BL_\epsilon}. 
\end{equation}
Moreover, since $\|\Delta\|_1 \le \sqrt{p}\|\Delta\|_2$, we have $\|\Delta\|_1 \le \sqrt{p}(B + 2 \sqrt{BL_\epsilon}) \coloneqq r_0$. By RSC condition from Lemma~\ref{lem:RSC}, we have the following

\begin{equation}
\begin{split}
  &  \alpha\|\Delta\|_2^2-\tau r_n\|\Delta\|_1^2 \le \langle \nabla\ell(\hat\theta)-\nabla\ell(\theta^*),\Delta\rangle \le \frac{5\lambda}{4}\|\Delta_S\|_1-\frac{\lambda3}{4}\|\Delta_{S^c}\|_1. \\
  & \implies -\tau r_nr_0\|\Delta\|_1 \le \frac{5\lambda}{4}\|\Delta_S\|_1-\frac{\lambda3}{4}\|\Delta_{S^c}\|_1. \\
  &  -\tau r_nr_0\|\Delta_S\|_1 -\tau r_nr_0\|\Delta_{S^c}\|_1 \le \frac{5\lambda}{4}\|\Delta_S\|_1-\frac{\lambda3}{4}\|\Delta_{S^c}\|_1.\\
  & \frac{\lambda3}{4}\|\Delta_{S^c}\|_1. - \tau r_nr_0\|\Delta_{S^c}\|_1 \le \frac{5\lambda}{4}\|\Delta_S\|_1 + \tau r_nr_0\|\Delta_S\|_1\\
  &    \|\Delta_{S^c}\|_1 \le 
    \frac{\frac{5\lambda}{4} + \tau r_nr_0}{\frac{3\lambda}{4} - \tau r_nr_0}\|\Delta_S\|_1
\end{split}
\end{equation}
Thus, we need to show 
\begin{equation}
\frac{3\lambda}{4} - \tau r_nr_0 > 0 \implies\lambda > \frac43 \tau r_nr_0 
\end{equation}
Therefore, it is sufficient if we have $\lambda >4 \tau r_nr_0 $. Moreover, we have that $r_n \asymp \sqrt{\frac{\log p}{n}}$ and $\tau = \frac{7}{72c}$.
By given the $\lambda$ that we choose, we need the following
\begin{equation}
\begin{split}
     & \lambda_n \ge4 \tau r_nr_0\\
        & \lambda_n \ge \frac{4(2-\alpha)}{\alpha} ( \frac{8K}{c} \sqrt{\frac{2}{n}\log p} + \frac{\sqrt{8}K}{c}\sqrt{\frac{|pa(i)|}{n}})  \ge\sqrt{p}(B+ 2 \sqrt{BL_\epsilon})\frac{7}{72c}
        \sqrt{\frac{\log p}{n}}
\end{split}
\end{equation}
Thus, we can have 
\begin{equation}
\begin{split}
        & \|\Delta_{S^c}\|_1 \le 3\|\Delta_S\|_1 \\
        & \implies \|\Delta\|_1  = \|\Delta_{S^c}\|_1 + \|\Delta_{S}\|_1 \le 4\|\Delta_{S}\|_1 \le 4\sqrt{s} \|\Delta\|_2
\end{split}
\end{equation}
Therefore, we have 

\begin{equation}
\begin{split}
& (\alpha-16\tau r_ns)\|\Delta\|_2^2\le \langle \nabla\ell(\hat\theta)-\nabla\ell(\theta^*),\Delta\rangle  \le
\frac{5\lambda}{4}\sqrt{s}\|\Delta\|_2\\ 
& \implies \|\Delta\|_2\le \frac{5\lambda\sqrt{s}}{(4\alpha-64\tau r_ns)}, \qquad\|\Delta\|_1 \le \frac{5\lambda s}{(\alpha-16\tau r_ns)}
\end{split}
\end{equation}
Since $r_n \asymp\sqrt{\frac{\log p}{n}}$, the sufficient condition for $n$ is that there exists a constant $C>0$ such that 
\begin{equation}
\begin{split}
  \alpha-16\tau r_ns >0  \implies r_n <\frac{\alpha}{16\tau s} \implies  C\sqrt{\frac{\log p}{n}}<  \frac{\alpha}{16\tau s}
  \implies  n >\log p(\frac{16\tau s}{C\alpha})^2
\end{split}
\end{equation}
Finally, under the assumed sample size condition such that $n \gtrsim \log p$, we have $r_n \lesssim 1$. Therefore, we may upper bound $r_n$ by a universal constant and rewrite it in a looser but simpler form in later section.
\end{proof}

\subsubsection{Proof of Lemma~\ref{lem:RSC_unstable}}

\begin{proof}
Since the empirical loss $\ell(\theta)$ is unchanged, the same argument applies by using the conditions and the H\"{o}lder inequality. In particular,

\begin{equation}
\Delta^\top \nabla^2 \ell(\theta)\Delta \ge \frac{7}{72c}\cdot \frac1n\sum_{k=0}^{n-1}(\Delta^\top X_k)^2 = \frac{7}{72c}\Delta^\top\Big(\frac1n\sum_{k=0}^{n-1}X_kX_k^\top\Big)\Delta.
\end{equation}
Then, let $\hat{\Sigma}:=\frac1n\sum_{k=0}^{n-1}X_kX_k^\top$, $Z_k:=X_kX_k^\top$, and define the natural filtration $\mathcal F_k:=\sigma(X_0,\ldots,X_k)$ with the convention that $\mathcal F_{-1}$ is the trivial $\sigma$-field. Introduce the predictable compensator $\tilde{\Sigma}:=\frac1n\sum_{k=0}^{n-1}\mathbb{E}[Z_k\mid \mathcal F_{k-1}].$ With $M_k:=Z_k-\mathbb{E}[Z_k\mid\mathcal F_{k-1}]$, we have $\mathbb{E}[M_k\mid\mathcal F_{k-1}]=0$ and hence $\hat{\Sigma}-\tilde{\Sigma}=\frac1n\sum_{k=0}^{n-1}M_k$. Consequently, for any $\Delta\in \mathbb R^p$,
\[
\Delta^\top\hat{\Sigma}\Delta
=\Delta^\top\tilde{\Sigma}\Delta+\Delta^\top(\hat{\Sigma}-\tilde{\Sigma})\Delta .
\]
Using $|\Delta^\top A\Delta|\le \|A\|_\infty\|\Delta\|_1^2$ (where $\|A\|_\infty:=\max_{i,j}|A_{ij}|$),
on the event $\|\hat{\Sigma}-\tilde{\Sigma}\|_\infty\le r_n$ we obtain
\begin{equation}
\Delta^\top\hat{\Sigma}\Delta
\ge \Delta^\top\tilde{\Sigma}\Delta - r_n\|\Delta\|_1^2.
\label{eq:LRSC_from_tilde_unstable}
\end{equation}
By Assumption~\ref{ass:pred_curvature}, define the curvature event such that $\mathcal G_{\mathrm{curv}}
:=\left\{
\inf_{\Delta\in\mathcal R(s;c_0)\setminus\{0\}}
\frac{\Delta^\top\tilde{\Sigma}\Delta}{\|\Delta\|_2^2}
\ge m
\right\}$. On $\mathcal G_{\mathrm{curv}}$ we have $\Delta^\top\tilde{\Sigma}\Delta\ge m\|\Delta\|_2^2$ for all $\Delta\in \mathcal R(s;c_0)$. Hence on $\Omega := \mathcal G_{\mathrm{curv}}\cap\{\|\hat{\Sigma}-\tilde{\Sigma}\|_\infty\le r_n\}$,  combining with \eqref{eq:LRSC_from_tilde_unstable} yields
\[
\Delta^\top\hat{\Sigma}\Delta
\ge m\|\Delta\|_2^2 - r_n\|\Delta\|_1^2,
\qquad \forall \Delta\in\mathcal R(s;c_0).
\]
It remains to upper bound $\|\hat{\Sigma}-\tilde{\Sigma}\|_\infty$.
For $(a,b)\in[p]\times[p]$, define the martingale difference
\[
\xi_k^{(ab)}:=M_k^{(ab)}
= Z_k^{(ab)}-\mathbb{E}[Z_k^{(ab)}\mid \mathcal F_{k-1}],
\qquad \mathbb{E}[\xi_k^{(ab)}\mid\mathcal F_{k-1}]=0,
\]
so that $\hat{\Sigma}_{ab}-\tilde{\Sigma}_{ab}=\frac1n\sum_{k=0}^{n-1}\xi_k^{(ab)}$. By Assumption~\ref{ass:localization} and the standard stopping argument in Remark~\ref{remark: sec moment}, the following event holds with high probability
\begin{equation}
\mathcal E_K:=
\Big\{\max_{0\le k\le n-1}\|X_k\|_\infty\le K\Big\}
\cap
\Big\{\max_{0\le k\le n-1}\big\|\mathbb{E}[Z_k\mid \mathcal F_{k-1}]\big\|_\infty\le K^2\Big\}.
\label{eq:EK_aug_pred}
\end{equation}
Define partial sums $S_0^{(ab)}:=0$ and $S_m^{(ab)}:=\sum_{k=0}^{m-1}\xi_k^{(ab)}$ for $m=1,\ldots,n$. Define the shifted filtration $\mathcal G_0:=\mathcal F_{-1}$ and
$\mathcal G_m:=\mathcal F_{m-1}$ for $m\ge 1$.  Then $(S_m^{(ab)})_{m=0}^n$ is a martingale with respect to $(\mathcal G_m)_{m=0}^n$. Let $\sigma$ be the first violating index: \[ \sigma:=\inf\Big\{0\le k\le n-1:\ \|X_k\|_\infty>K\ \text{ or }\ \big\|\mathbb{E}[Z_k\mid\mathcal F_{k-1}]\big\|_\infty>K^2\Big\}\wedge n. \] 
Set $\tau:=\sigma+1$. Since $\{\tau\le m\}=\{\sigma\le m-1\}\in \mathcal F_{m-1}=\mathcal G_m$, $\tau$ is a stopping time with respect to $(\mathcal G_m)_{m=0}^n$. Consider the stopped process such that $\bar S_m^{(ab)}:=S_{m\wedge(\tau-1)}^{(ab)}=S_{m\wedge\sigma}^{(ab)}$ for $m=0,1,\ldots,n,$ which is a martingale with respect to $(\mathcal G_m)_{m=0}^n$. Its increments satisfy, for $m\ge 1$,
\[
\bar S_m^{(ab)}-\bar S_{m-1}^{(ab)} = \xi_{m-1}^{(ab)}\,\mathbf 1_{\{m\le \sigma\}}.
\]
On $\{m\le \sigma\}$ we have $m-1<\sigma$, hence $\|X_{m-1}\|_\infty\le K$, $\big\|\mathbb{E}[Z_{m-1}\mid\mathcal F_{m-2}]\big\|_\infty\le K^2$, and therefore $|\xi_{m-1}^{(ab)}|\le K^2+K^2=2K^2$. If $m>\sigma$ the increment is $0$. Thus, almost surely, $\big|\bar S_m^{(ab)}-\bar S_{m-1}^{(ab)}\big|\le b:=2K^2$, for $m=1,\ldots,n.$ Moreover, the predictable quadratic variation satisfies
\[
V_n^{(ab)}:=\sum_{m=1}^n \mathbb{E}\left[\big(\bar S_m^{(ab)}-\bar S_{m-1}^{(ab)}\big)^2\mid\mathcal F_{m-2}\right].
\]
Since $\big(\bar S_m^{(ab)}-\bar S_{m-1}^{(ab)}\big)^2\le b^2$ a.s., we have $\mathbb{E}\left[\big(\bar S_m^{(ab)}-\bar S_{m-1}^{(ab)}\big)^2\mid\mathcal F_{m-2}\right]\le b^2,$ and hence $V_n^{(ab)}\le \sum_{m=1}^n b^2 = 4nK^4$ a.s. By Freedman's inequality, for any $\eta>0$,
\[
\mathbb{P}\left(
|\bar S_n^{(ab)}|
\ge \sqrt{2V_n^{(ab)}\,\eta}+\frac{b}{3}\eta
\right)\le 2e^{-\eta}.
\]
Using $V_n^{(ab)}\le 4nK^4$ and $b=2K^2$ yields
\[
\mathbb{P}\left(
|\bar S_n^{(ab)}|
\ge \sqrt{8nK^4\,\eta}+\frac{2}{3}K^2\eta
\right)\le 2e^{-\eta}.
\]
Set $\eta=\log(2p^2/\delta)$ and take a union bound over all $(a,b)\in[p]^2$ to obtain
\[
\mathbb{P}\left(
\max_{a,b\in[p]}\Big|\frac1n\bar S_n^{(ab)}\Big|
\le
2\sqrt{2}\,K^2\sqrt{\frac{\log(2p^2/\delta)}{n}}
+\frac{2}{3}K^2\frac{\log(2p^2/\delta)}{n}
\right)\ge 1-\delta.
\]
On $\mathcal E_K$ there is no violation up to $n-1$, hence $\sigma=n$ and therefore
$\bar S_n^{(ab)}=S_n^{(ab)}=\sum_{k=0}^{n-1}\xi_k^{(ab)}$. Consequently, with probability at least $1-\delta-\delta_{\mathrm{loc}}$,
\[
\|\hat{\Sigma}-\tilde{\Sigma}\|_\infty
=\max_{a,b\in[p]}\Big|\frac1n S_n^{(ab)}\Big|
\le r_n,
\]
where $r_n =2\sqrt{2}\,K^2\sqrt{\frac{\log(2p^2/\delta)}{n}}
+\frac{2}{3}K^2\frac{\log(2p^2/\delta)}{n}$.
This completes the proof.
\end{proof}

\subsection{Proofs for the Statements of \S\ref{Sec:TD_odemiss}}
\label{sec: ODE_miss_Lemma}
\subsubsection{Proof of Lemma~\ref{lemma:ODE_miss gradient}}
\begin{proof}
For gradient at true parameter $\theta_i$ with the $r_{ik}$ defined in Eq.\eqref{Eq:ODE_miss}, we obtain, at $\theta_i=\theta_i^\star$ and $u_k^\star =(\theta_i^\star)^\top X_k $ ,
\begin{equation}
\begin{split}
\nabla_{\theta_i}\ell_{i,k}^{\mathrm{mis}}(\theta_i^\star)
&=\left(\frac{2u_k^\star}{u_k^{\star 2}+c}
-\frac{2u_k^\star\big(u_k^\star z_{k+1}-\delta\sqrt{\Delta t}\big)^2}{(u_k^{\star 2}+c)^2}\right)X_k \\
&=\frac{2u_k^\star}{(u_k^{\star 2}+c)^2}
\Big((u_k^{\star 2}+c)-\big(u_k^\star z_{k+1}-\delta\sqrt{\Delta t}\big)^2\Big)X_k\\
&=\frac{2u_k^\star}{(u_k^{\star 2}+c)^2}
\Big(c+u_k^{\star 2}(1-z_{k+1}^2)+2u_k^\star z_{k+1}\delta\sqrt{\Delta t}-\delta^2\Delta t\Big)X_k. 
\end{split}
\end{equation}
Then, we define the decomposition 
\begin{equation}
B_k:=\frac{2u_k^\star}{(u_k^{\star 2}+c)^2}\big(c-\delta^2\Delta t\big)X_k, \quad Z_k:=\frac{2u_k^\star}{(u_k^{\star 2}+c)^2}\Big(u_k^{\star 2}(1-z_{k+1}^2)+2u_k^\star z_{k+1}\delta\sqrt{\Delta t}\Big)X_k. 
\end{equation}
Hence, the empirical gradient satisfies
\begin{equation}
    \nabla \tilde{\mathcal L}_i(\theta_i^\star) = \frac1n\sum_{k=0}^{n-1}B_k+\frac1n\sum_{k=0}^{n-1}Z_k.
\end{equation}
To show the boundedness of $\|\nabla \tilde{\mathcal L}_i(\theta_i^\star)\|_\infty$, we first show the boundedness for $Z_k$ with high probability, which is similar to Lemma~22. Then, we show the boundedness of $B_k$. By conditioning  on the trajectory $\{X_k\}$ (equivalently, on the past history up to time \(k\)), we have $\mathbb E[Z_k\mid X_k]=0$ since $\mathbb E(z_{k+1})=0$ and $\mathbb E(1-z^2_{k+1})=0$. For each coordinate $j\in[p]$, we write
\begin{equation}
Z_{k,j}=\alpha_{k,j}(1-z_{k+1}^2)+\beta_{k,j}z_{k+1},
\end{equation}
where
\begin{equation}
    \alpha_{k,j}:=\frac{2u_k^{\star 3}}{(u_k^{\star 2}+c)^2}X_{k,j},
\qquad
\beta_{k,j}:=\frac{4u_k^{\star 2}}{(u_k^{\star 2}+c)^2}\delta\sqrt{\Delta t}X_{k,j}.
\end{equation}
On the event $\mathcal E:=\{\max_k \|X_k\|_\infty\le K\}$ given by Proposition~\ref{Ass:xbound}, we have the uniform bounds
\begin{equation}
|\alpha_{k,j}|
\le \Big(\sup_{u\in\mathbb R}\frac{2|u|^3}{(u^2+c)^2}\Big)|X_{k,j}|
\le \frac{1}{\sqrt c}|X_{k,j}|
\le \frac{K}{\sqrt c},
\end{equation}
and
\begin{equation}
|\beta_{k,j}| \le \Big(\sup_{u\in\mathbb R}\frac{4u^2}{(u^2+c)^2}\Big)\delta\sqrt{\Delta t}|X_{k,j}|
\le \frac{1}{c}\delta\sqrt{\Delta t}|X_{k,j}|
\le \frac{K}{c}\delta\sqrt{\Delta t}.
\end{equation}
Since $z_{k+1}\sim\mathcal N(0,1)$ is Gaussian and hence sub-Gaussian, we know that $z^2_{k+1}$ is sub-exponential and $(1-z_{k+1}^2)$ is the centered sub-exponential. Hence, conditional on $X_k$, $Z_{k,j}$ is sub-exponential and there exists an absolute constant $C_0>0$ such that on $\mathcal E$,
\begin{equation}\label{eq:step3_psi1_bound}
\|Z_{k,j}\|_{\psi_1}
\le C_0\big(|\alpha_{k,j}|+|\beta_{k,j}|\big)
\le C_0\left(\frac{K}{\sqrt c}+\frac{K}{c}\delta\sqrt{\Delta t}\right)
=: \kappa .
\end{equation}
Therefore, by a standard Bernstein inequality for conditionally sub-exponential martingale differences, there exist absolute constants \(c_1>0\) such that, for any \(t>0\),
\begin{equation}\label{eq:step4_bernstein_single}
\mathbb P\left(\left|\frac{1}{n}\sum_{k=0}^{n-1} Z_{k,j}\right|\ge t\ \Big|\ \mathcal E\right)
\le 2\exp\left(-c_1 n\min\left\{\frac{t^2}{\kappa^2},\frac{t}{\kappa}\right\}\right).
\end{equation}
Taking a union bound over $j\in[p]$ yields
\begin{equation}\label{eq:Z_union}
\mathbb P\left(\left\|\frac{1}{n}\sum_{k=0}^{n-1} Z_{k}\right\|_\infty\ge t\ \Big|\ \mathcal E\right)
\le 2\exp\left(-c_1 n\min\left\{\frac{t^2}{\kappa^2},\frac{t}{\kappa}\right\}+\log p\right).
\end{equation}
Let $M := \left\|\frac{1}{n}\sum_{k=0}^{n-1} Z_k \right\|_\infty$. Then, we can identify the typical stochastic scale of $M$ under $\mathcal E$ by converting the conditional tail bound into an expectation bound. Since $ M \ge 0 $, we can have 
\begin{equation}
\mathbb E[M\mid \mathcal E] = \int_0^\infty \mathbb P(M\ge t\mid \mathcal E)dt \le \int_0^\infty\min \{1,2\exp\left(-c_1 n\min\left\{\frac{t^2}{\kappa^2},\frac{t}{\kappa}\right\}+\log p\right)\}dt
\end{equation}
By integrating the two Bernstein regimes, there exists an absolute constant $C_5>0$ such that
\begin{equation}
\mathbb E[M\mid \mathcal E] \le C_5\kappa\left(\sqrt{\frac{\log p}{n}}+\frac{\log p}{n}\right).
\end{equation}
In particular, for $n$ sufficiently large and since $f(x) =\min\{\frac{x^2}{\kappa^2},\frac{x}{\kappa}\}$ is nondecreasing for $x\ge 0$, we can have that $\mathbb E[M\mid \mathcal E] \le C_5\kappa\sqrt{\frac{\log p}{n}}.$ We now return to the tail bound in Eq.\eqref{eq:Z_union}. For any $\epsilon > 0$, evaluating at $t = C_5\kappa\sqrt{\frac{\log p}{n}} +\epsilon$ yields
\begin{equation}
\label{eq:Z_uniont}
\mathbb P\left( M > C_5\kappa\sqrt{\frac{\log p}{n}}+\epsilon \middle|\mathcal E \right) \le 2\exp\left( -c_1 n \min\left\{\frac{\epsilon^2}{\kappa^2},\frac{\epsilon}{\kappa}\right\} +\log p \right).
\end{equation}
Then, we can show the deterministic remainder term bound for $n^{-1}\sum B_k$ to complete the proof. Recall that $B_k$ is given as 
\begin{equation}
    B_k=\frac{2u_k^\star}{(u_k^{\star 2}+c)^2}\big(c-\delta^2\Delta t\big)X_k = (1-\frac{\delta^2\Delta t}{c}) \frac{2cu_k^\star}{(u_k^{\star 2}+c)^2}X_k = (1-\frac{\delta^2\Delta t}{c}) b_{k}.
\end{equation}
where $b_{k}$ is same as our previous proof, and coordinatewisely $b_{k,j}:=\frac{2cu_k^\star}{\bigl((u_k^\star)^2+c\bigr)^2}X_{k,j},
$ and $B_{k,j}=\left(1-\frac{\delta^2\Delta t}{c}\right)b_{k,j}$. Fix $j\in[p]$. Since $b_{k,j}$ is a measurable function of $X_k$, the sequence $\{b_{k,j}\}_{k=0}^{n-1}$ inherits the geometric $\alpha$-mixing property from $\{X_k\}_{k=0}^{n-1}$. Hence so does $\{B_{k,j}\}_{k=0}^{n-1}$. Moreover, because $\theta_{i,c}^\star$ is the $c$-stabilized population parameter, the population first-order condition gives that $\mathbb{E}[b_{k,j}]=0.$ Therefore,
\[
\mathbb{E}[B_{k,j}] = \left(1-\frac{\delta^2\Delta t}{c}\right)\mathbb{E}[b_{k,j}] =0.
\]
Fix $j\in[p]$, on the event $\mathcal E$ and by rewriting $u_k^\star=(\theta_i^\star)^\top X_k$, we have the following
\begin{align}
|b_{k,j}|
&=\left|\frac{2cu_k^\star}{(u_k^{\star 2}+c)^2}\right|\cdot |X_{k,j}| \le \left(\sup_{u\in\mathbb R}\underbrace{\frac{2c|u|}{(u^2+c)^2}}_{f(u)}\right)|X_{k,j}|\le \frac{3\sqrt{3}}{8}c^{-1/2}K
\label{eq:step5_Bkj_1}
\end{align}
The last inequlity holds by using Lemma.\ref{lem_mis: u/uc bound}, we have $\sup_{u\in\mathbb R}f(u) =\frac{3\sqrt{3}}{8}c^{-3/2}$ Hence, on $\mathcal E$, we have 
\[
|B_{k,j}| \le \left|1-\frac{\delta^2\Delta t}{c}\right|\frac{3\sqrt{3}}{8}\frac{K}{\sqrt c} =:L_B
\]
Therefore, by a Bernstein inequality for centered bounded geometrically
$\alpha$-mixing sequences, there exist absolute constants $C_8,C_9>0$ such that,
for any $t>0$,
\[
\mathbb P\left( \left| \frac1n\sum_{k=0}^{n-1} B_{k,j} \right| \ge t \middle|\mathcal E \right) \le 2\exp\left( -C_8 n\min\left\{\frac{t^2}{L_B^2},\frac{t}{L_B}\right\} \right).
\]
Taking a union bound over $j\in[p]$ yields
\[
\mathbb P\left( \left\| \frac1n\sum_{k=0}^{n-1} B_k \right\|_\infty \ge t \middle|\mathcal E \right) \le2\exp\left(-C_8 n\min\left\{\frac{t^2}{L_B^2},\frac{t}{L_B}\right\}+\log p\right).
\]
Choosing $t=C_9 L_B\sqrt{\frac{\log p}{n}}$, and taking $n$ sufficiently large so that the quadratic regime applies, we obtain
\begin{equation}
    \label{eq:step5_Bk_avg_bound}
    \left\| \frac1n\sum_{k=0}^{n-1} B_kv\right\|_\infty \le C_9 \left|1-\frac{\delta^2\Delta t}{c}\right| \frac{K}{\sqrt c} \sqrt{\frac{\log p}{n}}
\end{equation}
with probability at least $1-Cp^{-c_0}$ on $\mathcal E$, for some constants
$C,c_0>0$. Equivalently,
\[
\left\| \frac1n\sum_{k=0}^{n-1} B_k \right\|_\infty = O_{\mathbb P}\left( \left|1-\frac{\delta^2\Delta t}{c}\right| \frac{K}{\sqrt c} \sqrt{\frac{\log p}{n}} \right) \qquad\text{on }\mathcal E.
\]
Therefore, by combining Eq.\eqref{eq:Z_uniont} with $\epsilon = C_6 \kappa \sqrt{\frac{\log p}{n}}$ and Eq.\eqref{eq:step5_Bk_avg_bound}, we have that, conditional on $\mathcal E$, with probability at least $1-C\exp(-c_0\log p)$ with $ C,c_0, C_7 > 0$,
\begin{align}
\|\nabla \tilde{\mathcal L}_i(\theta_i^\star)\|_\infty
&\le \left\|\frac1n\sum_{k=0}^{n-1}B_k\right\|_\infty +\left\|\frac1n\sum_{k=0}^{n-1}Z_k\right\|_\infty \nonumber\\
&\le \frac{3\sqrt 3}{8}\left|1-\frac{\delta^2\Delta t}{c}\right| \frac{K}{\sqrt c} \sqrt{\frac{\log p}{n}}+ C_7\left(\frac{K}{\sqrt c}+\frac{K}{c}\delta\sqrt{\Delta t}\right)\sqrt{\frac{\log p}{n}}.
\label{eq:final_grad_bound}
\end{align}
Then, rewriting \eqref{eq:final_grad_bound}, we obtain
\[
\|\nabla \tilde{\mathcal L}_i(\theta_{i,c}^\star)\|_\infty \le
\underbrace{\left(\frac{3\sqrt{3}}{8}+C_7\right)\frac{K}{\sqrt c}\, \sqrt{\frac{\log p}{n}}}_{\text{original}} + \underbrace{ \left( \frac{3\sqrt{3}}{8}\frac{K\delta^2\Delta t}{c^{3/2}} + C_7\frac{K}{c}\delta\sqrt{\Delta t} \right)\sqrt{\frac{\log p}{n}}}_{\text{misspecification-induced part}}.
\]
Therefore, if we require the misspecification-induced contribution to be no larger than the baseline stochastic scale, it suffices to impose
\[
\left(
\frac{K\delta^2\Delta t}{c^{3/2}}
+
\frac{K}{c}\delta\sqrt{\Delta t}
\right)\sqrt{\frac{\log p}{n}}
\lesssim
\frac{K}{\sqrt c}\sqrt{\frac{\log p}{n}}.
\]
Equivalently, letting $x:=\frac{\delta\sqrt{\Delta t}}{\sqrt c},$ it suffices to require $x^2+x\lesssim 1.$ Thus, the sufficient condition is
\[
x\lesssim 1 \implies \delta \lesssim \frac{\sqrt c}{\sqrt{\Delta t}}.
\]
Under the above range of $\delta$, the misspecification-induced terms are absorbed into the original stochastic term. Hence, the gradient bound remains of the same order as in the original case, namely
\[
\|\nabla \tilde{\mathcal L}_i(\theta_{i,c}^\ast)\|_\infty
\lesssim
\frac{K}{\sqrt c}\sqrt{\frac{\log p}{n}}.
\]
\end{proof}
\subsubsection{Proof of Lemma~\ref{lemma:ODE_miss_lrsc}}
\begin{proof}
Let $u_k:=\theta^\top X_k$, and $\tilde r_{i,k}$  defined in Eq.\eqref{Eq:ODE_miss}. Then, the Hessian representation under ODE misspecification can be written as

\begin{equation}
    \begin{split}
    &\nabla^2 \tilde{\ell}_i^{(\delta)}(\theta) = \frac{1}{n}\sum_{k=0}^{n-1}\left[\frac{c-u_k^2}{(u_k^2+c)^2}-\frac{\tilde r_{i,k}^2(c-3u_k^2)}{(u_k^2+c)^3} \right]X_kX_k^\top.\\ 
    &\Delta^\top \nabla^2 \tilde{\ell}_i^{(\delta)}(\theta)\Delta = \frac{1}{n}\sum_{k=0}^{n-1} (\Delta^\top X_k)^2 \left[ \frac{c-u_k^2}{(u_k^2+c)^2} - \frac{\tilde r_{i,k}^2(c-3u_k^2)}{(u_k^2+c)^3} \right]. 
    \end{split}
\end{equation}
From Proposition~\ref{Ass:xbound}, we have that the event $\mathcal E_K:=\left\{\sup_{0\le s\le T}\|X_s\|\le K\right\}$ holds with high probability. On \(\mathcal E_K\), since \(\|\theta\|_2\le \sqrt{c/(2K^2)}\), we have for all \(k\),
\[
u_k^2=(\theta^\top X_k)^2 \le \|\theta\|_2^2\|X_k\|_2^2 \le \frac{c}{2K^2}K^2=\frac{c}{2}.
\]
Therefore, we have that $c-u_k^2 \ge \frac{c}{2}$ and $u_k^2+c \le \frac{3c}{2}$, and hence
\[
\frac{c-u_k^2}{(u_k^2+c)^2} \ge \frac{c/2}{(3c/2)^2} = \frac{2}{9c}.
\]
Also, since \(u_k^2\le c/2\), we have \(|c-3u_k^2|\le c\) and \(u_k^2+c\ge c\), so
\[
-\frac{\tilde r_{i,k}^2(c-3u_k^2)}{(u_k^2+c)^3} \ge -\frac{\tilde r_{i,k}^2|c-3u_k^2|}{(u_k^2+c)^3} \ge -\frac{\tilde r_{i,k}^2}{c^2}.
\]
Substituting these bounds yields
\[
\Delta^\top \nabla^2 \tilde{\ell}_i^{(\delta)}(\theta)\Delta \ge \frac{2}{9c}\cdot \frac{1}{n}\sum_{k=0}^{n-1}(\Delta^\top X_k)^2 - \frac{1}{c^2}\cdot \frac{1}{n}\sum_{k=0}^{n-1}\tilde r_{i,k}^2(\Delta^\top X_k)^2.
\]
Define $\tilde M_n:=\frac{1}{n}\sum_{k=0}^{n-1}\tilde r_{i,k}^2 X_kX_k^\top.$
Equivalently,
\[
\Delta^\top \nabla^2 \tilde{\ell}_i^{(\delta)}(\theta)\Delta \ge \frac{2}{9c}\Delta^\top \hat\Sigma \Delta - \frac{1}{c^2}\Delta^\top \tilde M_n \Delta,
\]
Next, write $\tilde r_{i,k}=r_{i,k}^0-\delta\sqrt{\Delta t}$ and $r_{i,k}^0:=\theta_i^{*\top}X_k z_{k+1}$, and define $M_n:=\frac{1}{n}\sum_{k=0}^{n-1}(r_{i,k}^0)^2X_kX_k^\top$ and $N_n:=\frac{1}{n}\sum_{k=0}^{n-1}r_{i,k}^0X_kX_k^\top$. 
Then
\[
\tilde M_n = M_n-2\delta\sqrt{\Delta t}\,N_n+\delta^2\Delta t\,\hat\Sigma.
\]
Then, we define $\bar{\tilde M}_n:=\bar M_n+\delta^2\Delta t\,\hat\Sigma$, where
\[
\bar M_n:=\frac1n\sum_{k=0}^{n-1}\mathbb E\!\left[(r^0_{i,k})^2X_kX_k^\top\mid \mathcal F_k\right]
=\frac1n\sum_{k=0}^{n-1}(\theta_i^{*\top}X_k)^2X_kX_k^\top,
\]
Since $\tilde M_n=M_n-2\delta\sqrt{\Delta t}\,N_n+\delta^2\Delta t\,\hat\Sigma$, it follows that
\[
\tilde M_n-\bar{\tilde M}_n=(M_n-\bar M_n)-2\delta\sqrt{\Delta t}\,N_n.
\]
On \(\mathcal E_K\), we have $(\theta_i^{*\top}X_k)^2\le \|\theta_i^*\|_2^2\|X_k\|_2^2\le B^2K^2,$ so $\bar{\tilde M}_n\preceq (B^2K^2+\delta^2\Delta t)\hat\Sigma.$ By Lemma~\ref{lem:Mn-concentration}, there exists \(q_n\asymp \sqrt{\frac{\log p}{n}}\), with high probability, such that
\[
\|M_n-\bar M_n\|_\infty\le q_n.
\]
For \(N_n\), each entry is a martingale sum with conditionally centered Gaussian increments:
\[
(N_n)_{ab} = \frac1n\sum_{k=0}^{n-1} (\theta_i^{*\top}X_k)X_{k,a}X_{k,b}z_{k+1}.
\]
On \(\mathcal E_K\), let $\xi_{k,ab}:=\frac1n(\theta_i^{*\top}X_k)X_{k,a}X_{k,b}z_{k+1},$ and $(N_n)_{ab}=\sum_{k=0}^{n-1}\xi_{k,ab}$. Then, \(\{\xi_{k,ab}\}\) is \(\mathcal F_{k+1}\)-measurable and $\mathbb E[\xi_{k,ab}\mid \mathcal F_k]=0$, so \(\{\xi_{k,ab}\}\) is a martingale-difference sequence relative to \(\{\mathcal F_k\}\). Moreover,  on \(\mathcal E_K\), we have 
\begin{equation}
    \begin{split}
        &\operatorname{Var}(\xi_{k,ab}\mid \mathcal F_k) = \frac1{n^2}(\theta_i^{*\top}X_k)^2X_{k,a}^2X_{k,b}^2 \le \frac{B^2K^6}{n^2}\implies\sum_{k=0}^{n-1}\operatorname{Var}(\xi_{k,ab}\mid \mathcal F_k) \le \frac{B^2K^6}{n}.
    \end{split}
\end{equation}
Therefore, by a Gaussian tail bound and a union bound over all \(p^2\) entries, there exist constants \(C,C_1,C_2>0\) such that
\[
\|N_n\|_\infty
\le
C\sqrt{\frac{\log p}{n}}
\]
with probability at least \(1-C_1p^{-C_2}\). Consequently,
\[
\|\tilde M_n-\bar{\tilde M}_n\|_\infty
\le
q_n+2|\delta|\sqrt{\Delta t}\,\|N_n\|_\infty
\le
C(1+|\delta|\sqrt{\Delta t})\sqrt{\frac{\log p}{n}}
=: \tilde q_n.
\]
with high probability. Thus, on the event \(\mathcal E_K\cap \mathcal E_{\tilde M}\), where $\mathcal E_{\tilde M} := \left\{ \|\tilde M_n-\bar{\tilde M}_n\|_\infty \le \tilde q_n
\right\}$. we have
\[
\Delta^\top \tilde M_n \Delta \le
\Delta^\top \bar{\tilde M}_n \Delta + \tilde q_n\|\Delta\|_1^2 \le
(B^2K^2+\delta^2\Delta t)\Delta^\top\hat\Sigma\Delta +\tilde q_n\|\Delta\|_1^2.
\]
Substituting this into the previous lower bound gives
\[
\Delta^\top \nabla^2 \tilde{\ell}_i^{(\delta)}(\theta)\Delta \ge \left( \frac{2}{9c}-\frac{B^2K^2+\delta^2\Delta t}{c^2} \right)\Delta^\top \hat\Sigma\Delta - \frac{\tilde q_n}{c^2}\|\Delta\|_1^2.
\]
That is,
\[
\Delta^\top \nabla^2 \tilde{\ell}_i^{(\delta)}(\theta)\Delta \ge \gamma_{c,\delta}\,\Delta^\top \hat\Sigma\Delta - \frac{\tilde q_n}{c^2}\|\Delta\|_1^2.
\]
Finally, the control of $\Delta^\top \hat\Sigma\Delta \ge m\|\Delta\|_2^2-r_n\|\Delta\|_1^2$ is identical to that in Lemma~\ref{lem:RSC}, since the misspecification affects only the residual term and does not change the argument for \(\hat\Sigma\). Therefore, on \(\mathcal E_{\hat\Sigma}:=\{\|\hat\Sigma-\mathbb E[\hat\Sigma]\|_\infty\le r_n\}\),
\[
\Delta^\top \nabla^2 \tilde{\ell}_i^{(\delta)}(\theta)\Delta
\ge m\gamma_{c,\delta}\|\Delta\|_2^2 - \left( \gamma_{c,\delta}r_n+\frac{\tilde q_n}{c^2} \right)\|\Delta\|_1^2.
\]
This proves
\[
\Delta^\top \nabla^2 \tilde{\ell}_i^{(\delta)}(\theta)\Delta \ge \alpha_\delta\|\Delta\|_2^2-\tau_{n,\delta}\|\Delta\|_1^2,
\]
with
\[
\alpha_\delta=m\gamma_{c,\delta}, \qquad \tau_{n,\delta} = \gamma_{c,\delta}r_n+\frac{C(1+|\delta|\sqrt{\Delta t})}{c^2}\sqrt{\frac{\log p}{n}}.
\]
If \(c\ge 9(B^2K^2+\delta^2\Delta t)\), then
\[
\gamma_{c,\delta} = \frac{2}{9c}-\frac{B^2K^2+\delta^2\Delta t}{c^2} \ge \frac{2}{9c}-\frac{1}{9c} = \frac{1}{9c},
\]
which gives the stated bound. If $\delta \le \sqrt{\frac{c/9 - B^2K^2}{\Delta t}}$, 
then \(B^2K^2+\delta^2\Delta t \le c/9\), and therefore
\[
\gamma_{c,\delta} = \frac{2}{9c}-\frac{B^2K^2+\delta^2\Delta t}{c^2} \ge \frac{2}{9c}-\frac{1}{9c} = \frac{1}{9c}.
\]
Hence $\alpha_\delta = m\gamma_{c,\delta} \ge \frac{m}{9c} =: \alpha_1$
Moreover, since \(r_n \asymp \sqrt{\frac{\log p}{n}}\) and $\tilde q_n \le C(1+|\delta|\sqrt{\Delta t})\sqrt{\frac{\log p}{n}}$, there exists a constant \(\tau \ge 0\) such that $\gamma_{c,\delta}r_n + \frac{\tilde q_n}{c^2} \le \tau r_n.$ Substituting into the previous display yields
\[
\Delta^\top \nabla^2 \tilde{\ell}_i^{(\delta)}(\theta)\Delta \ge \alpha_1 \|\Delta\|_2^2 - \tau r_n \|\Delta\|_1^2,
\]
which is exactly the stated bound.
\end{proof}

\subsubsection{Proof of Lemma~\ref{lemma:ODE_miss_h_lip}}
\begin{proof}
Write $\tilde{\ell}_i^{(\delta)}$ in composite form:
\[
\tilde{\ell}_i^{(\delta)}(\theta_i)= \frac{1}{2n}\sum_{k=0}^{n-1}\phi_{k,\delta}(\theta_i^\top X_k),\qquad \phi_{k,\delta}(z) = \log(z^2+c)+\frac{(r_{i,k}^{(\delta)})^2}{z^2+c}.
\]
Direct differentiation gives
\begin{equation}
\begin{split}
&\phi_{k,\delta}'(z) = \frac{2z}{z^2+c} - \frac{2(r_{i,k}^{(\delta)})^2z}{(z^2+c)^2},\quad
\phi_{k,\delta}''(z) = \frac{2(c-z^2)}{(z^2+c)^2}-\frac{2(r_{i,k}^{(\delta)})^2(c-3z^2)}{(z^2+c)^3},\\
&\phi_{k,\delta}^{(3)}(z) = \frac{4z\bigl(z^4-(2c+6(r_{i,k}^{(\delta)})^2)z^2-3c^2+6c(r_{i,k}^{(\delta)})^2\bigr)}{(z^2+c)^4}.
\end{split}
\end{equation}
Using the triangle inequality and $(z^2+c)\ge c$ for all $z\in\mathbb R$, we obtain
\begin{equation}
\label{Eq: hessian_miss1}
    \bigl|\phi_{k,\delta}^{(3)}(z)\bigr| \le \frac{4|z|}{(z^2+c)^4} \Bigl( |z|^4+(2c+6(r_{i,k}^{(\delta)})^2)|z|^2+(3c^2+6c(r_{i,k}^{(\delta)})^2) \Bigr) \le \frac{24}{c^{3/2}}+\frac{48(r_{i,k}^{(\delta)})^2}{c^{5/2}}.
\end{equation}
On the event $\mathcal E$ given by Proposition~\ref{Ass:xbound}, since $r_{i,k}^{(\delta)} = r_{i,k}^{(0)}-\delta\sqrt{\Delta t}$, we have $|r_{i,k}^{(\delta)}| \le |r_{i,k}^{(0)}|+\delta\sqrt{\Delta t} \le R+\delta\sqrt{\Delta t}$, where $R$ is from Lemma~\ref{lem:residual_bound}. Plug in back  Eq.\eqref{Eq: hessian_miss1}, we can have 
\begin{equation}
\begin{split}
\sup_{z\in\mathbb R}\bigl|\phi_{k,\delta}^{(3)}(z)\bigr| \le M_{c,\delta} := \frac{24}{c^{3/2}} + \frac{48\bigl(R+\delta\sqrt{\Delta t}\bigr)^2}{c^{5/2}}.
\end{split}
\end{equation}
By the chain rule,
\begin{equation}
\begin{split}
    & \nabla^2\tilde{\ell}_i^{(\delta)}(\theta_i) = \frac{1}{2n}\sum_{k=0}^{n-1}\phi_{k,\delta}''(\theta_i^\top X_k)\,X_kX_k^\top \\
    &\implies \nabla^2\tilde{\ell}_i^{(\delta)}(\theta_i)-\nabla^2\tilde{\ell}_i^{(\delta)}(\theta_i') =
\frac{1}{2n}\sum_{k=0}^{n-1} \Bigl( \phi_{k,\delta}''(\theta_i^\top X_k)-\phi_{k,\delta}''((\theta_i')^\top X_k)\Bigr)X_kX_k^\top.
\end{split}
\end{equation}
For each $k$, by the one-dimensional mean value theorem,
\[
\bigl|
\phi_{k,\delta}''(\theta_i^\top X_k)
-
\phi_{k,\delta}''((\theta_i')^\top X_k)
\bigr|
\le
M_{c,\delta}\,
|(\theta_i-\theta_i')^\top X_k|
\le
M_{c,\delta}\,\|\theta_i-\theta_i'\|_1\|X_k\|_\infty,
\]
where the last step uses $\ell_1$--$\ell_\infty$ duality. Taking entrywise maximum norms,
\[
\bigl\|\nabla^2\tilde{\ell}_i^{(\delta)}(\theta_i)-\nabla^2\tilde{\ell}_i^{(\delta)}(\theta_i')\bigr\|_\infty
\le
\frac{1}{2n}\sum_{k=0}^{n-1}
\bigl|
\phi_{k,\delta}''(\theta_i^\top X_k)
-
\phi_{k,\delta}''((\theta_i')^\top X_k)
\bigr|
\cdot
\|X_kX_k^\top\|_\infty.
\]
On $\mathcal E$, we have $\|X_k\|_\infty\le K$ and hence $\|X_kX_k^\top\|_\infty \le \|X_k\|_\infty^2 \le K^2.$
Therefore,
\[
\bigl\|\nabla^2\tilde{\ell}_i^{(\delta)}(\theta_i)-\nabla^2\tilde{\ell}_i^{(\delta)}(\theta_i')\bigr\|_\infty
\le
\frac{1}{2n}\sum_{k=0}^{n-1}
\Bigl(M_{c,\delta}\|\theta_i-\theta_i'\|_1\|X_k\|_\infty\Bigr)K^2
\le
\frac{M_{c,\delta}}{2}K^3\|\theta_i-\theta_i'\|_1.
\]
This proves the claimed bound with $\bar D_{\max}^{(\delta)}=\frac{K^3}{2}M_{c,\delta}$. For the final claim, we take 
\begin{equation}
\begin{split}
&\delta \le \frac{1}{\sqrt{\Delta t}}\left(\sqrt{R^2+\frac{c}{2}}-R\right)\equiv \delta \sqrt{\Delta t}+R\le\sqrt{R^2+\frac{c}{2}}\\ 
\end{split}
\end{equation}
We can observe that 
\[
(R+\delta\sqrt{\Delta t})^2 = R^2+2R\delta\sqrt{\Delta t}+\delta^2\Delta t
\le R^2+\frac{c}{2} \implies
2R\delta\sqrt{\Delta t}+\delta^2\Delta t \le \frac{c}{2},
\]
Hence,
\[
\bar D_{\delta,\max}
=
\frac{K^3}{2}
\left(
\frac{24}{c^{3/2}}
+
\frac{48(R+\delta\sqrt{\Delta t})^2}{c^{5/2}}
\right)
\le
\frac{K^3}{2}
\left(
\frac{24}{c^{3/2}}
+
\frac{48R^2}{c^{5/2}}
+
\frac{24}{c^{3/2}}
\right)
=
\frac{K^3}{2}
\left(
\frac{48}{c^{3/2}}
+
\frac{48R^2}{c^{5/2}}
\right).
\]
Therefore, the Hessian Lipschitz bound remains of the same form as in Lemma~\ref{lem:h_lips}.
\end{proof}

\subsubsection{Proof of Lemma~\ref{lemma:ODE_miss_hessian}}
\begin{proof}
Let $a_k^* \coloneqq \theta_i^{*\top}X_k$, $u_k^* \coloneqq (a_k^*)^2 + c$, and $m_\delta \coloneqq \delta\sqrt{\Delta t}$.
Under Eq.\eqref{Eq:ODE_miss}, at \(\theta=\theta_i^*\) we have
\[
r_{i,k}^{(\delta)} = \sqrt{u_k^*}\,z_{k+1}-m_\delta \implies (r_{i,k}^{(\delta)})^2 = u_k^* z_{k+1}^2 - 2m_\delta \sqrt{u_k^*}z_{k+1} + m_\delta^2,
\]
where $z_{k+1}\sim\mathcal N(0,1)$. By direct differentiation,
\[
\nabla^2 \widetilde \ell_{i,k}^{(\delta)}(\theta_i^*)=\frac{2\Bigl(-(a_k^*)^4+\bigl(3(a_k^*)^2-c\bigr (r_{i,k}^{(\delta)})^2+c^2\Bigr)}{(u_k^*)^3}X_kX_k^\top.
\]
Therefore, conditioning on \(X_k\) and using \(\mathbb E[z_{k+1}]=0\), \(\mathbb E[z_{k+1}^2]=1\), we define $H_{i,k}^{*,(\delta)}\coloneqq\mathbb E\!\left[\nabla^2 \widetilde \ell_{i,k}^{(\delta)}(\theta_i^*)\mid X_k\right]$, and we obtain
\[
H_{i,k}^{*,(\delta)} =\mathbb E\!\left[\nabla^2 \widetilde \ell_{i,k}^{(\delta)}(\theta_i^*)\mid X_k\right]=\frac{2\Bigl(-(a_k^*)^4+\bigl(3(a_k^*)^2-c\bigr)(u_k^*+m_\delta^2)+c^2\Bigr)}{(u_k^*)^3}X_kX_k^\top.
\]
Since $u_k^* \coloneqq (a_k^*)^2 + c$, we can have the expanding it as 

\begin{equation}
\begin{split}
    -(a_k^*)^4 + \bigl(3(a_k^*)^2-c\bigr)u_k^* + c^2 & =-(a_k^*)^4 + \bigl(3(a_k^*)^2-c\bigr)((a_k^*)^2 + c) + c^2 \\
    &=-(a_k^*)^4 + \bigl(3(a_k^*)^4 +2(a_k^*)^2c-c^2\bigr) + c^2 \\
    &= 2(a_k^*)^4+2(a_k^*)^2c =2(a_k^*)^2u_k^*.
\end{split}
\end{equation}
This simplifies to
\[
H_{i,k}^{*,(\delta)}=\left( \frac{4(a_k^*)^2}{(u_k^*)^2}+\frac{2\bigl(3(a_k^*)^2-c\bigr)m_\delta^2}{(u_k^*)^3}\right)X_kX_k^\top.
\]
Now, we define the centered fluctuation $Z_{i,k}^{(\delta)}\coloneqq\nabla^2 \widetilde\ell_{i,k}^{(\delta)}(\theta_i^*)- H_{i,k}^{*,(\delta)}$ and substituting the expression above gives
\[
Z_{i,k}^{(\delta)}=\frac{2\bigl(3(a_k^*)^2-c\bigr)}{(u_k^*)^3}\Bigl(u_k^*(z_{k+1}^2-1)-2m_\delta\sqrt{u_k^*}\,z_{k+1}\Bigr)X_kX_k^\top.
\]
Hence for each entry \((j,\ell)\),
\[
[Z_{i,k}^{(\delta)}]_{j\ell}=\frac{2\bigl(3(a_k^*)^2-c\bigr)}{(u_k^*)^3} \Bigl(u_k^*(z_{k+1}^2-1)-2m_\delta \sqrt{u_k^*}\,z_{k+1}\Bigr) X_{k,j}X_{k,\ell}.
\]
Conditioning on \(X_k\), and using \(\operatorname{Var}(z_{k+1}^2-1)=2\), \(\operatorname{Var}(z_{k+1})=1\), \(\operatorname{Cov}(z_{k+1}^2-1,z_{k+1})=0\), we have
\[
\operatorname{Var}\!\left([Z_{i,k}^{(\delta)}]_{j\ell}\mid X_k\right) = \left( \frac{2\bigl(3(a_k^*)^2-c\bigr)}{(u_k^*)^3}X_{k,j}X_{k,\ell} \right)^2 \Bigl(2(u_k^*)^2 + 4m_\delta^2 u_k^*\Bigr).
\]
Using \(|3(a_k^*)^2-c|\le 3u_k^*\), \(u_k^*\ge c\), and \(|X_{k,j}X_{k,\ell}|\le K^2\) with high probability from Proposition~\ref{Ass:xbound},
\[
\operatorname{Var}\!\left([Z_{i,k}^{(\delta)}]_{j\ell}\mid X_k\right) \le\frac{72K^4}{c^2}\left(1+\frac{2m_\delta^2}{c} \right).
\]
Also, since both \(z_{k+1}\) and \(z_{k+1}^2-1\) are sub-exponential, we can have 
\[
\|[Z_{i,k}^{(\delta)}]_{j\ell}\|_{\psi_1} \le C\,\frac{K^2}{c} \left(1+\frac{|m_\delta|}{\sqrt c} \right),
\]
for some universal constant \(C>0\). For each fixed entry $(j,l)$, \([Z_{i,k}^{(\delta)}]_{j\ell}\) is conditionally centered given $X_k$, and its randomness is driven only by the Gaussian innovation, so a standard Bernstein-type concentration argument yields
\[
\mathbb P\!\left(\left|\frac{1}{n}\sum_{k=1}^n [Z_{i,k}^{(\delta)}]_{j\ell}\right| > t\right)\le2\exp\!\left(-c_1 n \min\!\left\{\frac{t^2}{\nu_\delta^2},\frac{t}{\alpha_\delta}\right\}\right),
\]
where $\nu_\delta^2 \asymp \frac{K^4}{c^2}\left(1+\frac{m_\delta^2}{c}\right)$ and $\alpha_\delta \asymp \frac{K^2}{c}\left(1+\frac{|m_\delta|}{\sqrt c}\right)$. Using \(\|A\|_\infty \le s\max_{j,\ell}|A_{j\ell}|\) and a union bound over the \(s^2\) entries in \(V_i\times V_i\), we obtain
\[
\left\|\frac{1}{n}\sum_{k=1}^n Z_{i,k;V_iV_i}^{(\delta)}\right\|_\infty=O_{\mathbb P}\!\left(\frac{K^2}{c}\left(1+\frac{|m_\delta|}{\sqrt c}\right)s\sqrt{\frac{\log s}{n}}\right).
\]
Next, define the centered population-Hessian fluctuation as we did before, $Y_{k,j\ell}^{(\delta)} \coloneqq[H_{i,k}^{*,(\delta)}]_{j\ell}- \mathbb E\bigl([H_{i,k}^{*,(\delta)}]_{j\ell}\bigr)$.  From the explicit form of \(H_{i,k}^{*,(\delta)}\),
\[
|[H_{i,k}^{*,(\delta)}]_{j\ell}|\le\left(\frac{4(a_k^*)^2}{(u_k^*)^2}+\frac{2|3(a_k^*)^2-c|\,m_\delta^2}{(u_k^*)^3}\right) |X_{k,j}X_{k,\ell}|.
\]
Using $\frac{4t^2}{(t^2+c)^2}\le \frac{1}{c}$ and $|3t^2-c|\le 3(t^2+c)$, we obtain
\[
|[H_{i,k}^{*,(\delta)}]_{j\ell}| \le\left(\frac{1}{c}+\frac{6m_\delta^2}{c^2}\right)|X_{k,j}X_{k,\ell}|\le\frac{K^2}{c}\left(1+\frac{6m_\delta^2}{c} \right) \implies |Y_{k,j\ell}^{(\delta)}| \le2\frac{K^2}{c}\left(1+\frac{6m_\delta^2}{c} \right).
\]
Since \(Y_{k,j\ell}^{(\delta)}\) is a measurable function of \(X_k\), the sequence
\(\{Y_{k,j\ell}^{(\delta)}\}_{k=1}^n\) is geometrically \(\alpha\)-mixing by Assumption~\ref{ass:alpha-mixing-EM}. Applying the Bernstein-type inequality for centered bounded geometrically \(\alpha\)-mixing sequences gives
\[
\mathbb P\!\left( \left| \frac{1}{n}\sum_{k=1}^n Y_{k,j\ell}^{(\delta)} \right| > u \right) \le 2\exp\!\left( -c_2 n \min\!\left\{ \frac{u^2}{B_\delta^2}, \frac{u}{B_\delta(\log n)^2}\right\}\right),
\]
where $B_\delta \asymp \frac{K^2}{c}\left(1+\frac{m_\delta^2}{c}\right)$. Applying a union bound over \(V_i\times V_i\) yields
\[
\left\|\frac{1}{n}\sum_{k=1}^n H_{i,k;V_iV_i}^{*,(\delta)}-\widetilde Q_{V_iV_i}^{*,(\delta)}\right\|_\infty=O_{\mathbb P}\!\left(\frac{K^2}{c}\left(1+\frac{m_\delta^2}{c}\right)\left(s\sqrt{\frac{\log s}{n}}+s\,\frac{(\log n)^2\log s}{n}\right)\right).
\]
Finally, by the decomposition
\[
\widetilde Q_{V_iV_i}^{\,n,(\delta)} - \widetilde Q_{V_iV_i}^{\,*,(\delta)} = \frac{1}{n}\sum_{k=1}^n Z_{i,k;V_iV_i}^{(\delta)}+\left(\frac{1}{n}\sum_{k=1}^n H_{i,k;V_iV_i}^{*,(\delta)}-\widetilde Q_{V_iV_i}^{*,(\delta)} \right),
\]
we conclude
\[
\bigl\|\widetilde Q_{V_iV_i}^{\,n,(\delta)}-\widetilde Q_{V_iV_i}^{\,*,(\delta)}\bigr\|_\infty=O_{\mathbb P}\!\left( C_\delta \left(s\sqrt{\frac{\log s}{n}}+s\,\frac{(\log n)^2\log s}{n}\right) \right),
\]
with $C_\delta \asymp 1+\frac{|m_\delta|}{\sqrt c}+\frac{m_\delta^2}{c} = 1+|\delta|\sqrt{\frac{\Delta t}{c}}+\frac{\delta^2\Delta t}{c}.$ The proof for the off-support block \(V_i^c\times V_i\) is identical, except that the union bound is taken over \((p-s)s\) entries instead of \(s^2\), which gives
\[
\bigl\|\widetilde Q_{V_i^cV_i}^{\,n,(\delta)}-\widetilde Q_{V_i^cV_i}^{\,*,(\delta)}\bigr\|_\infty =O_{\mathbb P}\!\left(C_\delta\left(s\sqrt{\frac{\log((p-s)s)}{n}}+s\,\frac{(\log n)^2\log((p-s)s)}{n}\right) \right).
\]
To show the condition of \(|\delta| \le \sqrt{c/\Delta t}\), given $|\delta| \le \sqrt{c/\Delta t}$, we have 
\[
|\delta|\sqrt{\frac{\Delta t}{c}} \le 1 \implies \frac{\delta^2\Delta t}{c}
= \left(|\delta|\sqrt{\frac{\Delta t}{c}}\right)^2 \le 1,
\]
and therefore $C_\delta \asymp 1+|\delta|\sqrt{\frac{\Delta t}{c}}+\frac{\delta^2\Delta t}{c}=O(1).$ Substituting this into the general bound and the off-support block bound for yields 
\[
\bigl\|\widetilde Q_{V_iV_i}^{\,n,(\delta)}-\widetilde Q_{V_iV_i}^{\,*,(\delta)}\bigr\|_\infty=O_{\mathbb P}\!\left(s\sqrt{\frac{\log s}{n}}+s\frac{(\log n)^2\log s}{n} \right).
\]
\end{proof}
\subsection{Proofs for the Statements in \S\ref{Sec:c_role} and \S\ref{Sec:IFT}}

\subsubsection{Proof of Proposition~\ref{prop:fesibale_c}}
\begin{proof}
Define $\phi_i(c):=T_{n,i}(c)+C_{H,i}\Delta_{i,c}$.  Since \(m-16sr_n>m/2\), we have
\[
T_{n,i}(c) = \frac{27\sqrt{s}\lambda_n}{m-16sr_n}\,c \le \frac{54\sqrt{s}\lambda_n}{m}\,c.
\]
Also, for any \(c\ge c_{\min,i}\),
\[
C_{H,i}\Delta_{i,c}
\le
\frac{9}{16\sqrt3}C_{H,i}Kc^{-1/2}
\le
\frac{9}{16\sqrt3}C_{H,i}Kc_{\min,i}^{-1/2}
\le
\frac{\sqrt3}{16}\frac{C_{H,i}}{\|\theta_i^\star\|_2},
\]
where the last inequality follows from $c_{\min,i}\ge 9 \|\theta_i^\star\|_2^2K^2$. Therefore, for all \(c\ge c_{\min,i}\),
\[
\phi_i(c) \le \frac{\sqrt3}{16}\frac{C_{H,i}}{\|\theta_i^\star\|_2} + \frac{54\sqrt{s}\lambda_n}{m}\,c =:g_i(c).
\]
Now \(g_i(c)\) is continuous and increasing on \([c_{\min,i},\infty)\). Under the sufficient condition stated in the statement, $\beta_{\min,i}>g_i(c_{\min,i})$. Hence, by continuity, there exists \(\varepsilon_i>0\), for all $c\in[c_{\min,i},\,c_{\min,i}+\varepsilon_i)$, such that $\beta_{\min,i}>g_i(c)$
Since \(\phi_i(c)\le g_i(c)\), it follows that
\[
\beta_{\min,i}>\phi_i(c) = T_{n,i}(c)+C_{H,i}\Delta_{i,c}
\]
for all \(c\in[c_{\min,i},\,c_{\min,i}+\varepsilon_i)\). Thus \(\mathcal C_i\neq\varnothing\). For \(n\) large enough, we have
\[
T_{n,i}(c)=\frac{27\sqrt{s}\lambda_n}{m-16sr_n}c \le \frac{54\sqrt{s}\lambda_n}{m}c.
\]
Together with $\Delta_{i,c}\le \frac{9}{16\sqrt{3}}Kc^{-1/2}$, it follows that
\[
\phi_i(c) = T_{n,i}(c)+C_{H,i}\Delta_{i,c} \le \frac{54\sqrt{s}\lambda_n}{m}c + \frac{9}{16\sqrt{3}}C_{H,i}Kc^{-1/2} =: \bar\phi_i(c).
\]
We minimize the upper bound \(\bar\phi_i(c)\) over \(c>0\). Write
\[
\bar\phi_i(c)=ac+bc^{-1/2}, \qquad a:=\frac{54\sqrt{s}\lambda_n}{m}, \quad b:=\frac{9}{16\sqrt{3}}C_{H,i}K.
\]
Then $\bar\phi_i'(c)=a-\frac12 b\,c^{-3/2}$ and setting \(\bar\phi_i'(c)=0\) gives
\[
a=\frac12 b\,c^{-3/2} \implies c^{3/2}=\frac{b}{2a}.
\]
Hence the unique critical point is
\[
c_{*,i} = \left(\frac{b}{2a}\right)^{2/3} = \left( \frac{\frac{9}{16\sqrt{3}}C_{H,i}K} {2\cdot \frac{54\sqrt{s}\lambda_n}{m}} \right)^{2/3} = \left( \frac{C_{H,i}Km}{192\sqrt{3}\sqrt{s}\lambda_n} \right)^{2/3}.
\]
Finally,
\[
\bar\phi_i''(c)=\frac34 b\,c^{-5/2}>0 \qquad\text{for all } c>0,
\]
so \(\bar\phi_i\) is strictly convex on \((0,\infty)\). Therefore \(c_{*,i}\) is its unique global minimizer.
\end{proof}

\subsubsection{Proof of Lemma~\ref{lem:local_restricted_branch}}
\begin{proof}
Since $F_i$ is $C^1$ on an open neighborhood of $(\theta_S^\star,0)$, and
\[
D_\vartheta F_i(\theta_S^\star,0)=H_i(\theta_S^\star,0)
\]
is invertible, the implicit function theorem applies at the point $(\theta_S^\star,0)$. Therefore, there exist:
\begin{itemize}
    \item an open interval $I_i\subset \mathbb R$ containing $0$,
    \item an open neighborhood $U_i^\circ\subset \mathbb R^{s_i}$ of $\theta_S^\star$,
    \item a unique $C^1$ map $ \psi_i:I_i\to U_i^\circ$ for all $ c\in I_i$ such that
\begin{equation}
\psi_i(0)=\theta_S^\star \qquad\text{and}\qquad F_i(\psi_i(c),c)=0
\end{equation}
\end{itemize}
Moreover, the implicit function theorem gives local uniqueness such that, if $(\vartheta,c)\in U_i^\circ\times I_i$ satisfies $F_i(\vartheta,c)=0$, then necessarily $\vartheta=\psi_i(c).$ Now choose $c_{0,i}>0$ such that
\[
(-c_{0,i},c_{0,i})\subset I_i.
\]
For $c\in(-c_{0,i},c_{0,i})$, define $\vartheta_{i,c}^\star:=\psi_i(c).$ Then
\[
F_i(\vartheta_{i,c}^\star,c)=0,
\qquad
\vartheta_{i,c}^\star\in U_i^\circ,
\qquad
\vartheta_{i,0}^\star=\theta_S^\star,
\]
and this solution is unique in $U_i^\circ$. Finally, since $U_i^\circ$ is open and contains $\theta_S^\star$, there exists $r_i>0$ such that the open ball
\[
B(\theta_S^\star,r_i)\subset U_i^\circ.
\]
Since $\psi_i$ is continuous and $\psi_i(0)=\theta_S^\star$, after possibly shrinking $c_{0,i}>0$, we may further ensure that
\[
\psi_i(c)\in B(\theta_S^\star,r_i) \qquad\text{for all } |c|<c_{0,i}.
\]
Set $U_i:=B(\theta_S^\star,r_i).$ Then $U_i$ is open, convex, and contains $\theta_S^\star$. Moreover, for every $|c|<c_{0,i}$, the solution $\vartheta_{i,c}^\star=\psi_i(c)$ lies in $U_i$, and uniqueness still holds on $U_i$ since $U_i\subset U_i^\circ$. This proves the lemma.
\end{proof}

\subsubsection{Proof of Lemma~\ref{lem:uniform_inverse_hessian}}
\begin{proof}
Since $F_i$ is $C^1$ on an open neighborhood of $(\theta_S^\star,0)$, its Jacobian $H_i(\vartheta,c)=D_\vartheta F_i(\vartheta,c)$ is continuous near $(\theta_S^\star,0)$. By Assumption~\ref{ass:local_restricted_regularity}, $H_i(\theta_S^\star,0)$ is invertible. Since invertibility is an open property, there exist an open neighborhood $\mathcal U_i\subset \mathbb R^{s_i}\times \mathbb R$ of $(\theta_S^\star,0)$ such that $H_i(\vartheta,c)$ is invertible for every $(\vartheta,c)\in\mathcal U_i$. Shrinking $U_i$ and $c_{0,i}$ if necessary, we pick
\[
\overline{U_i}\times[0,c_{0,i}]
\subset \mathcal U_i.
\]
Hence $H_i(\vartheta,c)$ is invertible for all $(\vartheta,c)\in \overline{U_i}\times[0,c_{0,i}]$. Now the map $(\vartheta,c)\longmapsto H_i(\vartheta,c)^{-1}$ is continuous on $\overline{U_i}\times[0,c_{0,i}]$, since matrix inversion is continuous on the set of invertible matrices. Since $\overline{U_i}$ and $[0,c_{0,i}]$ are compact, we can know that $\overline{U_i}\times[0,c_{0,i}]$ is compact. Then, the function
\[
(\vartheta,c)\longmapsto \|H_i(\vartheta,c)^{-1}\|_{\infty}
\]
attains its maximum there. Define
\[
C_{H,i}
:=
\max_{(\vartheta,c)\in \overline{U_i}\times[0,c_{0,i}]}
\|H_i(\vartheta,c)^{-1}\|_{\infty}
<\infty.
\]
To finish the proof, we need to show the second inequity. Fix $c\in[0,c_{0,i}]$. By Lemma~\ref{lem:local_restricted_branch}, we know that $F_i(\vartheta_{i,c}^\star,c)=0$. 
Since $U_i = B(\theta_S^\star,r_i)$ from Lemma~\ref{lem:local_restricted_branch}, which is convex,  the segment, $\theta_S^\star+t(\vartheta_{i,c}^\star-\theta_S^\star)$ for $t\in[0,1]$, lies in $U_i$. By the fundamental theorem of calculus along this segment,
\[
F_i(\vartheta_{i,c}^\star,c)-F_i(\theta_S^\star,c) = \bar H_{i,c}\,(\vartheta_{i,c}^\star-\theta_S^\star),
\]
where $\bar H_{i,c} := \int_0^1 H_{i,c}\!\bigl(\theta_S^\star+t(\vartheta_{i,c}^\star-\theta_S^\star)\bigr)\,dt.$ Hence
\[
0 = F_i(\theta_S^\star,c) + \bar H_{i,c}(\vartheta_{i,c}^\star-\theta_S^\star),
\]
so that $\vartheta_{i,c}^\star-\theta_S^\star = -\bar H_{i,c}^{-1}F_i(\theta_S^\star,c),$ provided $\bar H_{i,c}$ is invertible. Let $H_{i,0}^\star:=H_i(\theta_S^\star,0)$. Since $H_i$ is continuous at $(\theta_S^\star,0)$ and $H_{i,0}^\star$ is invertible, after shrinking $U_i$ and $c_{0,i}$ if necessary we may ensure
\[
\|\bar H_{i,c}-H_{i,0}^\star\|_{\infty}\le\sup_{\vartheta\in U_i, c\in[0,c_{0,i}]}
\|H_i(\vartheta,c)-H_{i,0}^\star\|_{\infty}
\le
\frac{1}{2\|(H_{i,0}^\star)^{-1}\|_{\infty}}.
\]
By the Neumann-series perturbation argument, $\bar H_{i,c}$ is invertible and $\|\bar H_{i,c}^{-1}\|_{\infty} \le 2\|(H_{i,0}^\star)^{-1}\|_{\infty}$.  After enlarging $C_{H,i}$ if necessary, we have $\|\bar H_{i,c}^{-1}\|_{\infty}\le C_{H,i}$. 
Consequently,
\[
\|\vartheta_{i,c}^\star-\theta_S^\star\|_\infty \le \|\bar H_{i,c}^{-1}\|_{\infty}\,
\|F_i(\theta_S^\star,c)\|_\infty \le C_{H,i}\,\Delta_{i,c}.
\]
This proves the result.
\end{proof}

\section{Auxiliary Lemmas and Proof}

\begin{lemma}
\label{lem:WR_bound}
For any $i\in[p]$,  given $W_{pa(i)} \coloneqq \nabla L(\theta_i^*)$ and defined $\alpha$ as in Lemma~\ref{lem: PDW}. We can have
\begin{equation}
\begin{split}
 \mathbb P& (\frac{1}{\lambda_n}\|W_{pa(i)}\|_\infty > \frac{\alpha}{4(2-\alpha)})\le 2(d-1) \exp\Big( -\frac{n}{2} \min\{ \frac{(\frac{{\lambda_n}\alpha}{4(2-\alpha)})^{2}}{\nu^{2}}, \frac{\frac{{\lambda_n}\alpha}{4(2-\alpha)}}{\sigma} \} \Big).   
\end{split}
\end{equation}
The above probability can be made arbitrarily small by picking $\lambda_n$ large enough such that 
\begin{equation}
\lambda_n \ge \frac{4(2-\alpha)}{\alpha} (\frac{8K}{\sqrt c} \sqrt{\frac{2\log(d)}{n}}) 
\end{equation}
Then, given $R_{pa(i)} \coloneqq (\nabla^2 L(\theta^*_i) - \nabla^2 L(\tilde\theta_i)(\hat \theta_i - \theta_i^*)$ and defined $\alpha$ as in Lemma~\ref{lem: PDW}. There exists $C_1,C_2 >0$ such that
\begin{equation}
   \mathbb P( \frac{1}{\lambda_n}\|R_{pa(i)}\|_\infty> \frac{\alpha}{4(2-\alpha)}) \le C_1 e^{-C_2 K}.
\end{equation}
\end{lemma}
\begin{proof}

By Eq.\eqref{lemeq:grad_bound}, we have the following for $\epsilon>0$

\begin{equation}
\mathbb P\left( \|\nabla_{\theta_i}\ell_i(\theta_i^\ast)\|_\infty \ge 2t
\right)\le4\exp\left(-\frac{n}{2}\min\left\{\frac{t^2}{\nu_\ast^2},\frac{t}{\sigma_\ast}\right\}+\log p\right),
\end{equation}
Let $G_i := \left\|\nabla_{\theta_i}\ell_i(\theta_i^*)\right\|_{\infty}.$ for any \(t>0\),
\[
\mathbb{P}\!\left(G_i \ge 2t\right)\le4\exp\!\left(-\frac{n}{2}\min\left\{\frac{t^2}{\nu_*^2},\,\frac{t}{\sigma_*}\right\}+\log p\right).
\]
To verify the score condition, we want $G_i \le \frac{\lambda_n\alpha}{4(2-\alpha)}$ with high probability. Thus, set $2t = \frac{\lambda_n\alpha}{4(2-\alpha)},$ that is, $t = \frac{\lambda_n\alpha}{8(2-\alpha)}.$ Substituting this choice into the above tail bound yields
\begin{align}
\mathbb{P}\!\left(
\left\|\nabla_{\theta_i}\ell_i(\theta_i^*)\right\|_{\infty}
\ge \frac{\lambda_n\alpha}{4(2-\alpha)}
\right)
&\le
4\exp\!\left(
-\frac{n}{2}
\min\left\{
\frac{\lambda_n^2\alpha^2}{64(2-\alpha)^2\nu_*^2},
\,
\frac{\lambda_n\alpha}{8(2-\alpha)\sigma_*}
\right\}
+\log p
\right).
\label{eq:score-tail-lambda}
\end{align}
Now suppose \(\lambda_n\) is chosen such that
\[
\lambda_n \ge C\,\frac{2-\alpha}{\alpha}\,
\max\left\{
\nu_*\sqrt{\frac{\log p}{n}},
\,
\sigma_*\frac{\log p}{n}
\right\}
\]
for a sufficiently large universal constant \(C>0\). Then there exists a constant \(c_0>1\) such that
\[
\frac{n}{2}
\min\left\{
\frac{\lambda_n^2\alpha^2}{64(2-\alpha)^2\nu_*^2},
\,
\frac{\lambda_n\alpha}{8(2-\alpha)\sigma_*}
\right\}
\ge c_0 \log p.
\]
Hence, from \eqref{eq:score-tail-lambda},
\[
\mathbb{P}\!\left(
\left\|\nabla_{\theta_i}\ell_i(\theta_i^*)\right\|_{\infty}
\ge \frac{\lambda_n\alpha}{4(2-\alpha)}
\right)
\le
4\exp\!\bigl(-(c_0-1)\log p\bigr)
=
4p^{-(c_0-1)}.
\]
Therefore, $\mathbb{P}\!\left(
\left\|\nabla_{\theta_i}\ell_i(\theta_i^*)\right\|_{\infty}
\le \frac{\lambda_n\alpha}{4(2-\alpha)}
\right)
\ge
1-4p^{-(c_0-1)},$ which converges to \(1\) as \(p\to\infty\). This proves the desired score bound. Finally, since  we have $\nu_* \asymp \frac{K}{\sqrt{c}}$ and  $\sigma_* \asymp \frac{K}{\sqrt{c}},$ and in particular the dominant term is of order
\[
\lambda_n
\asymp
\frac{2-\alpha}{\alpha}\,\frac{K}{\sqrt{c}}
\sqrt{\frac{\log p}{n}}.
\]
Thus it suffices to choose $\lambda_n \ge C\,\frac{2-\alpha}{\alpha}\,\frac{K}{\sqrt{c}} \sqrt{\frac{\log p}{n}}$ with $C>0$. Then, we need to show the bound of $\frac{1}{\lambda_n}\|R_{pa(i)}\|_\infty$. Given $R_{pa(i)} \coloneqq (\nabla^2 L(\theta^*_i) - \nabla^2 L(\tilde\theta_i))(\hat \theta_i - \theta_i^*)$ and by Lemma~\ref{lem:h_lips}, the Hessian is locally Lipschitz in the sense that

\begin{equation}
\begin{split}
      \|(\nabla^2 L(\theta^*_i) - \nabla^2 L(\tilde\theta_i))(\hat \theta_i - \theta_i^*)\|_\infty &\le\|(\nabla^2 L(\theta^*_i) - \nabla^2 L(\tilde\theta_i))\|_\infty\|(\hat \theta_i - \theta_i^*)\|_\infty \\
      &\le L_H\|(\hat \theta_i - \theta_i^*)\|_1\|(\hat \theta_i - \theta_i^*)\|_\infty
\end{split}
\end{equation}
Since $\|\cdot\|_\infty \le\|\cdot\|_2$ with the boundedness from Lemma~\ref{lem:err_rsc}, we can have 
\begin{equation}
    \|(\nabla^2 L(\hat\theta_i) - \nabla^2 L(\tilde\theta_i))(\hat \theta_i - \theta_i^*)\|_\infty \le L_H\frac{3\lambda\sqrt{s}}{2(\alpha_1-16\tau s)} \frac{6\lambda s}{\alpha_1-16\tau s} \le \frac{\lambda\alpha}{4(2-\alpha)}.
\end{equation}
We can see $\frac{\lambda s^{3/2}}{(\alpha_1-16\tau s)^2}\le \frac{\alpha}{36L_H(2-\alpha)}$. To get a upper bound, since $    \frac{\lambda s^{3/2}}{(\alpha-16\tau s)^2} \ge  \frac{\lambda s^{3/2}}{\alpha^2+256\tau^2 s^2}$, we can have the upper bound for $\lambda_n$ such that 
\begin{equation}
 \lambda_n \le \frac{\alpha_1\big(\alpha_1^{2}+256\tau^{2}s^{2}\big)}{36L_H(2-\alpha)s^{3/2}},
\end{equation}
To make this the boundedness condition meaningful, we can pick the sample size scale by 
\begin{equation}
\frac{4(2-\alpha)}{\alpha} \frac{8K}{\sqrt c} \sqrt{\frac{2\log(d)}{n}} \le \frac{\alpha_1^{3}+\alpha_1256\tau^{2}s^{2}}{36L_H(2-\alpha)s^{3/2}}
\end{equation}
Then, we can see that 
\begin{equation}
n >\max\{2^{17}\tilde{c} s^3\log p, 2^{13}\tilde c s^4\}, \text{where, } \tilde c = \frac{3^4(2-\alpha)^4 K^2}{\alpha^4 c^2(\alpha^2+256\tau^2 s^2)^2}
\end{equation}
Since $s:=|pa(i)|$, a sufficient sample size condition is
\[
n \ge C\,\tilde c\big(s^3\log p + s^4\big) 
\]
for a universal constant $C>0$.
\end{proof}

\begin{lemma}[Choice of $\lambda_n$ and sample size under ODE misspecification]
\label{cor:lambda_misspec}
Suppose the assumptions of Lemma~17 hold. If
\begin{equation}
\lambda_n \ge \frac{4(2-\alpha)}{\alpha}
\left[
\left(\frac{3\sqrt{3}}{8}+C\right)\frac{K}{\sqrt{c}}+ \frac{3\sqrt{3}K}{8c^{3/2}}\delta^2\Delta t + C\frac{K}{c}\delta\sqrt{\Delta t}\right]\sqrt{\frac{\log p}{n}},
\label{eq:lambda_misspec}
\end{equation}
then, with probability at least $1-C_1p^{-C_2}$,
\begin{equation}
\left\|\nabla \tilde{\ell}_i^{(\delta)}(\theta_i^*)\right\|_\infty \le \frac{\alpha}{4(2-\alpha)}\lambda_n .
\label{eq:score_dominate_misspec}
\end{equation}
\end{lemma}

\begin{proof}

In this proof, we follow the procedure of Lemma~\ref{lem:WR_bound}, and we will show the $\frac{1}{\lambda_n}\|W_{pa(i)}\|_\infty \le \frac{\alpha}{4(2-\alpha)}$ and $\frac{1}{\lambda_n}\|R_{pa(i)}\|_\infty \le \frac{\alpha}{4(2-\alpha)}$ with high probability. By Lemma~17, there exist constants $C,C_1,C_2>0$ such that, with probability at least
$1-C_1p^{-C_2}$,
\begin{equation}
\left\|\nabla \tilde{\ell}_i^{(\delta)}(\theta_i^*)\right\|_\infty \le \left[ \left(\frac{3\sqrt{3}}{8}+C\right)\frac{K}{\sqrt{c}} + \frac{3\sqrt{3}K}{8c^{3/2}}\delta^2\Delta t + C\frac{K}{c}\delta\sqrt{\Delta t} \right] \sqrt{\frac{\log p}{n}}.
\label{eq:grad_bound_misspec_pf}
\end{equation}
Therefore, if $\lambda_n$ satisfies \eqref{eq:lambda_misspec}, then the right-hand side of
\eqref{eq:grad_bound_misspec_pf} is bounded above by $\frac{\alpha}{4(2-\alpha)}\lambda_n,$ which yields
\[
\left\|\nabla \tilde{\ell}_i^{(\delta)}(\theta_i^*)\right\|_\infty \le \frac{\alpha}{4(2-\alpha)}\lambda_n .
\]
This proves \eqref{eq:score_dominate_misspec}. To show the boundness of $\frac{1}{\lambda_n}\|R_{pa(i)}\|_\infty$, we first show the $\ell_1$ and $\ell_2$ bound. The proof follows the same argument as Lemma~35. In particular, one replaces the score bound in Lemma~38 by above and the LRSC condition in Lemma~34 by Lemma~18. The same cone-condition argument then yields
\[
\|\Delta_{S^c}\|_1 \le 3\|\Delta_S\|_1,
\qquad
\|\Delta\|_1 \le 4\sqrt{s}\,\|\Delta\|_2,
\]
and combining this with the misspecified LRSC inequality gives
\[
(\alpha_\delta-16\tau_\delta r_n s)\|\Delta\|_2^2
\le
\frac{5\lambda_n}{4}\sqrt{s}\,\|\Delta\|_2,
\]
We can get the following 
\begin{equation}
 \|\Delta\|_2 \le \frac{5\lambda_n\sqrt{s}}{4(\alpha_\delta-16\tau_\delta r_n s)},\qquad\|\Delta\|_1 \le 4\sqrt{s}\,\|\Delta\|_2
\le \frac{5\lambda_n s}{\alpha_\delta-16\tau_\delta r_n s}.
\end{equation}
The second one is hold by using $\|\Delta\|_1 \le 4\sqrt{s}\,\|\Delta\|_2$.
Let
\[
A_\delta := \left(\frac{3\sqrt3}{8}+C\right)\frac{K}{\sqrt c} + \frac{3\sqrt3 K}{8c^{3/2}}\delta^2\Delta t + C\frac{K}{c}\delta\sqrt{\Delta t}.
\]
Then, it can be rewritten as $\lambda_n \ge \frac{4(2-\alpha)}{\alpha}A_\delta\sqrt{\frac{\log p}{n}}$. Next, by Lemma~19 and the $\ell_1/\ell_2$ bounds,
\[
\|R_{pa(i)}^{(\delta)}\|_\infty \le \bar D_{\max}^{(\delta)}\|\Delta\|_1\|\Delta\|_2 \le \frac{25\bar D_{\max}^{(\delta)}}{4} \frac{\lambda_n^2 s^{3/2}}{(\alpha_\delta-16\tau_\delta r_n s)^2}.
\]
Hence a sufficient condition for $\frac{1}{\lambda_n}\|R_{pa(i)}^{(\delta)}\|_\infty \le \frac{\alpha}{4(2-\alpha)}$ is
\begin{equation}
\lambda_n \le \frac{\alpha(\alpha_\delta-16\tau_\delta r_n s)^2} {25\bar D_{\max}^{(\delta)}(2-\alpha)s^{3/2}}.
\end{equation}
To avoid circularity, impose $32\tau_\delta r_n s \le\alpha_\delta,$ so that $\alpha_\delta-16\tau_\delta r_n s \ge \frac{\alpha_\delta}{2}.$ Then it is enough that
\[
\lambda_n \le \frac{\alpha\alpha_\delta^2} {100\bar D_{\max}^{(\delta)}(2-\alpha)s^{3/2}}.
\]
Combining with the lower bound on $\lambda_n$, we obtain the sufficient sample size condition
\[
n \ge \left[ \frac{400(2-\alpha)^2\bar D_{\max}^{(\delta)}A_\delta} {\alpha^2\alpha_\delta^2} \right]^2 s^3\log p.
\]
Moreover, the positivity condition $16\tau_\delta r_n s \le \frac{\alpha_\delta}{2}$ is ensured whenever
\[
n \ge C\frac{\tau_\delta^2}{\alpha_\delta^2}s^2\log p.
\]
Therefore, a sufficient sample size condition is $n \ge C\,\tilde c_\delta\, s^3\log p$ for a suitable constant $\tilde c_\delta>0$ depending on $\alpha,\alpha_\delta,\tau_\delta,\bar D_{\max}^{(\delta)},A_\delta$.
\end{proof}

\begin{lemma}\label{lem:Mn-concentration}
Let $\|\cdot\|_{\max}$ denote the entrywise maximum norm, and 
\begin{equation}
M_n:=\frac1n\sum_{k=0}^{n-1} r_{i,k}^2 X_kX_k^\top, \quad r_{i,k}:=\theta_i^{*\top}X_k\,\epsilon_{k+1},
\end{equation}
Suppose the SDE in Eq.\eqref{GSDE} satisfies the conditions stated in Proposition~\ref{thm:gen_lyap_AB} and Proposition~\ref{Ass:xbound}. Then, with $n$ larger enough, there exists a universal constant $\Omega(B,K)$ , such that
\begin{equation}
\|M_n-\bar M_n\|_{\max} \le  \Omega(B,K)\sqrt{\frac{\log(2p^2)}{n}}
\end{equation}
with probability $1-2\exp(-c\log(2p^2))-\mathbb P(\mathcal E_K^c)$, where $\mathbb P(\mathcal E_K^c)$ denotes the probability of the complement of the localization event $\mathcal E_K$ defined in Proposition~\ref{Ass:xbound}.
\end{lemma}

\begin{proof}
Fix $(a,b)\in[p]\times[p]$. For all $k = 0,\dots,n-1$, we define
\begin{equation}
D_k^{(ab)} := (\epsilon_{k+1}^2-1)(\theta_i^{*\top}X_k)^2X_{k,a}X_{k,b}
\end{equation}
Then, we define $\bar M_n\coloneqq \frac1n\sum_{k=0}^{n-1}(\theta_i^{*\top}X_k)^2X_kX_k^\top$, we observe that 

\begin{equation}
\bigl(M_n-\bar M_n\bigr)_{ab} = \frac1n\sum_{k=0}^{n-1}D_k^{(ab)}.
\end{equation}
Since $X_k$ is $\mathcal F_k$-measurable and $\epsilon_{k+1}$ is independent of $\mathcal F_k$ with $\mathbb E[\epsilon_{k+1}^2-1\mid \mathcal F_k]=0$, we have

\begin{equation}
\mathbb E[D_k^{(ab)}\mid \mathcal F_k] = (\theta_i^{*\top}X_k)^2X_{k,a}X_{k,b}\, \mathbb E[\epsilon_{k+1}^2-1\mid \mathcal F_k] =0.
\end{equation}
Hence $\{D_k^{(ab)}\}_{k=0}^{n-1}$ is a martingale difference sequence. On the event $\mathcal E_K$, we have $|X_{k,a}|\le \|X_k\|_2\le K$ and $|X_{k,b}|\le \|X_k\|_2\le K$, and we then have $|\theta_i^{*\top}X_k| \le \|\theta_i^\star\|_2\|X_k\|_2 \le BK.$
Therefore, on $\mathcal E_K$, we have 
\begin{equation}
\big|(\theta_i^{*\top}X_k)^2X_{k,a}X_{k,b}\big| \le B^2K^4
\end{equation}
Since $\epsilon_{k+1}\sim N(0,1)$, the variable $\epsilon_{k+1}^2-1$ is centered sub-exponential with universal $\psi_1$ norm. Thus, for some universal constant $C_0>0$,
\begin{equation}
\|D_k^{(ab)}\|_{\psi_1\mid \mathcal F_k} \le C_0 B^2K^4 \qquad\text{on }\mathcal E_K.
\end{equation}
Applying a standard Bernstein inequality for martingale differences with conditionally sub-exponential increments yields
\begin{equation}
\mathbb P\left(\left \{\left| \frac1n\sum_{k=0}^{n-1}D_k^{(ab)} \right|>t\right \} \cap \mathcal E_K \right) \le 2\exp\left[ -c n\min\left( \frac{t^2}{B^4K^8}, \frac{t}{B^2K^4} \right) \right]
\end{equation}
for some universal constant $c>0$. Applying a union bound over all $(a,b)\in[p]\times[p]$ gives with high probability event $\mathbb P(\mathcal E^c_K) \le \delta_k^c$
\begin{equation}
\mathbb P\left(\left\| \frac1n\sum_{k=0}^{n-1}D_k^{(ab)} \right\|_\infty>t \right) \le \exp\left[ -c n\min\left( \frac{t^2}{B^4K^8}, \frac{t}{B^2K^4}\right) +\log(2p^2)\right] +\delta_k^c
\end{equation}
By picking $n$ larger enough, this proves the claim.
\end{proof}

\begin{lemma}
    \label{lem_mis: u/uc bound}
    For any $c \neq 0$, we have the following
    \begin{equation}
        \sup_{u\in\mathbb R}\frac{2|u|}{(u^2+c)^2} = \frac{3\sqrt{3}}{8}\,c^{-3/2}
    \end{equation}
\end{lemma}
\begin{proof}
We first show the upper bound of $f(u)$. Since $f(u)$ is even, $\sup_{u\in\mathbb R} f(u)=\sup_{u\ge 0} f(u)$. For $u\ge 0$, we can have $f'(u)$ by taking differentiation  
\[
f'(u)
= 2(u^2+c)^{-2} + 2u\cdot\big(-2\big)(u^2+c)^{-3}\cdot(2u)
= \frac{2(u^2+c)-8u^2}{(u^2+c)^3}
= \frac{2(c-3u^2)}{(u^2+c)^3}.
\]
Thus the only critical point for $u\ge 0$ is $c-3u^2=0$, and we have $u^\star=\sqrt{\frac{c}{3}}.$ Moreover, $f(0)=0$ and $f(u)\to 0$ as $u\to\infty$, while $f'(u)>0$ for $u<\sqrt{c/3}$ and $f'(u)<0$ for $u>\sqrt{c/3}$, so $u^\star$ is the global maximizer. Evaluate $g$ at $u^\star$, we have the following
\[
g(u^\star)
= \frac{2\sqrt{c/3}}{\big(c/3+c\big)^2}
= \frac{2\sqrt{c/3}}{(4c/3)^2}
= \frac{2\sqrt{c/3}}{16c^2/9}
= \frac{18}{16}\cdot \frac{\sqrt{c/3}}{c^2}
= \frac{9}{8}\cdot \frac{1}{\sqrt{3}}\,c^{-3/2}
= \frac{3\sqrt{3}}{8}\,c^{-3/2}.
\]
Therefore,
\[
\sup_{u\in\mathbb R}\frac{2|u|}{(u^2+c)^2} = \frac{3\sqrt{3}}{8}\,c^{-3/2}.
\] 
\end{proof}

\begin{lemma}
\label{lem:residual_bound}
Let Assumption~\ref{SDE_gen}, and Proposition~\ref{Ass:xbound} hold. Then, for all $i=1,\dots,d$ and $k=0,\dots,n-1$, we have the following with high probability
\[
|r_{i,k}| \le R,
\]
where $R := \frac{2K}{\sqrt{\Delta_{t_k}}} + G_K \sqrt{\Delta_{t_k}}$, and $G_K := \sqrt{C\bigl(1+dK^2\bigr)}$.
\end{lemma}

\begin{proof}
Recall
\[
r_{i,k} := \frac{x_i(k+1)-x_i(k)-g_i(\theta,X_k,t_k)\Delta t_k}{\sqrt{\Delta t_k}}.
\]
Hence, by the triangle inequality,
\[
|r_{i,k}| \le \frac{|x_i(k+1)-x_i(k)| + |g_i(\theta,X_k,t_k)|\,\Delta t_k}{\sqrt{\Delta t_k}}.
\]
On the event $\mathcal E_K$, we have $|x_i(k)|\le K$ and $|x_i(k+1)|\le K$, so \[|x_i(k+1)-x_i(k)|\le |x_i(k+1)|+|x_i(k)|\le 2K.
\]
Moreover, by Assumption~\ref{SDE_gen} (B2), for all $x\in\mathbb R^d$, we have $\|g(\theta,x,t)\|^2 \le C\bigl(1+\|x\|^2\bigr)$. Since $\|X_k\|_\infty \le K$ on $\mathcal E_K$, we have $\|X_k\|^2 \le dK^2$ for the Euclidean norm, and therefore
\[
\|g(\theta,X_k,t_k)\| \le \sqrt{C\bigl(1+dK^2\bigr)} = G_K.
\]
In particular, $|g_i(\theta,X_k,t_k)| \le \|g(\theta,X_k,t_k)\| \le G_K.$ Substituting these bounds into the definition of $r_{i,k}$ gives
\[
|r_{i,k}| \le \frac{2K + G_K \Delta t_k}{\sqrt{\Delta t_k}} = \frac{2K}{\sqrt{\Delta t_k}} + G_K \sqrt{\Delta t_k}c= R.
\]
Since this bound is uniform in $i$ and $k$, the conclusion follows.
\end{proof}

\begin{lemma}
\label{lem:zero_constrained_localization}
Fix \(i\in V\) and \(j\in \mathrm{pa}(i)\), and define
\[
E_{ij}:=\{\hat\theta_{ij}^\lambda=0\},
\qquad
\theta_i^{*(j=0)}:=\theta_i^*-\theta_{ij}^*e_j .
\]
Assume the conditions of Lemma~\ref{lem:err_rsc} hold, and let \(\hat\theta_i^\lambda\) be the oracle estimator (hence a stationary point of the penalized oracle problem). Then, with high probability, on the event \(E_{ij}\),
\begin{equation}
\label{eq:zero_constrained_localization}
\|\hat\theta_i^\lambda-\theta_i^{*(j=0)}\|_\infty
\le
r_{\mathrm{loc},n},
\qquad
r_{\mathrm{loc},n}:=
\frac{3\lambda\sqrt{s}}{2(\alpha-16\tau s)}.
\end{equation}
Equivalently, there exists an event \(\mathcal E_{n,i}\) with $\mathbb P(\mathcal E_{n,i})\ge 1-C\exp(-cn^\varepsilon)$ such that
\[
E_{ij}\cap \mathcal E_{n,i}
\subseteq \left\{ \|\hat\theta_i^\lambda-\theta_i^{*(j=0)}\|_\infty \le r_{\mathrm{loc},n} \right\}.
\]
\end{lemma}

\begin{proof}
By Lemma~\ref{lem:err_rsc}, since \(\hat\theta_i^\lambda\) is a stationary point, the estimation error
\[
\Delta_i:=\hat\theta_i^\lambda-\theta_i^*
\]
satisfies, with high probability,
\begin{equation}
\label{eq:l2_bound_from_rsc}
\|\Delta_i\|_2 \le \frac{3\lambda\sqrt{s}}{2(\alpha-16\tau s)}.
\end{equation}
Let \(\mathcal E_{n,i}\) denote the event on which \eqref{eq:l2_bound_from_rsc} holds. Now work on the event \(E_{ij}=\{\hat\theta_{ij}^\lambda=0\}\). Since $\theta_i^{*(j=0)}=\theta_i^*-\theta_{ij}^*e_j$, we have
\[
\hat\theta_i^\lambda-\theta_i^{*(j=0)} = \hat\theta_i^\lambda-\theta_i^*+\theta_{ij}^*e_j = \Delta_i+\theta_{ij}^*e_j.
\]
Coordinatewise,
\[
\bigl(\hat\theta_i^\lambda-\theta_i^{*(j=0)}\bigr)_k
=
\begin{cases}
\Delta_{ik}, & k\neq j,\\[4pt]
\Delta_{ij}+\theta_{ij}^*, & k=j.
\end{cases}
\]
But on \(E_{ij}\), we have
\[
\Delta_{ij} = \hat\theta_{ij}^\lambda-\theta_{ij}^* = -\theta_{ij}^* \implies\Delta_{ij}+\theta_{ij}^*=0.
\]
Therefore,
\[
\bigl(\hat\theta_i^\lambda-\theta_i^{*(j=0)}\bigr)_k =
\begin{cases}
\Delta_{ik}, & k\neq j,\\
0, & k=j.
\end{cases}
\]
So \(\hat\theta_i^\lambda-\theta_i^{*(j=0)}\) is obtained from \(\Delta_i\) by deleting its \(j\)-th coordinate. In particular,
\[
\|\hat\theta_i^\lambda-\theta_i^{*(j=0)}\|_\infty \le \|\Delta_i\|_\infty \le \|\Delta_i\|_2.
\]
Combining this with \eqref{eq:l2_bound_from_rsc}, we obtain on \(E_{ij}\cap \mathcal E_{n,i}\),
\[
\|\hat\theta_i^\lambda-\theta_i^{*(j=0)}\|_\infty \le \frac{3\lambda\sqrt{s}}{2(\alpha-16\tau s)} = r_{\mathrm{loc},n}.
\]
This proves \eqref{eq:zero_constrained_localization}.
\end{proof}

\begin{lemma}
\label{lem:segment_curvature}
Fix \(i\in V\) and \(j\in \mathrm{pa}(i)\), and define
\[
\theta_i^{*(j=0)}:=\theta_i^*-\theta_{ij}^*e_j, \qquad \vartheta_t:=\theta_i^{*(j=0)}+t\theta_{ij}^*e_j, \quad t\in[0,1].
\]
Assume the conditions of Lemma~\ref{lem:RSC} hold. Then, with high probability,
\begin{equation}
\label{eq:segment_curvature_bound}
\inf_{t\in[0,1]}
e_j^\top \nabla^2 \ell_{n,i}^{(c)}(\vartheta_t)e_j \ge m_{n,c}, \qquad m_{n,c}:=\alpha_1-\tau r_n.
\end{equation}
In particular, for all sufficiently large \(n\), if \(\alpha_1>\tau r_n\), then \(m_{n,c}>0\).
\end{lemma}

\begin{proof}
For each \(t\in[0,1]\),
\[
\vartheta_t = \theta_i^{*(j=0)}+t\theta_{ij}^*e_j
\]
agrees with \(\theta_i^*\) in every coordinate except the \(j\)-th one, where \(\theta_{ij}^*\) is replaced by \(t\theta_{ij}^*\). Hence
\[
\|\vartheta_t\|_2^2 = \sum_{k\neq j}(\theta_{ik}^*)^2+t^2(\theta_{ij}^*)^2 \le
\sum_{k\neq j}(\theta_{ik}^*)^2+(\theta_{ij}^*)^2 = \|\theta_i^*\|_2^2.
\]
Therefore,
\[
\|\vartheta_t\|_2 \le \|\theta_i^*\|_2 \le \sqrt{\frac{c}{2K^2}},
\qquad \forall t\in[0,1].
\]
So every point on the segment \(\{\vartheta_t:t\in[0,1]\}\) belongs to the admissible region of Lemma~\ref{lem:RSC}. Now apply Lemma~\ref{lem:RSC} with \(\theta=\vartheta_t\) and \(\Delta=e_j\). Since $\|e_j\|_2^2=1$ and $\|e_j\|_1^2=1$,  Lemma~\ref{lem:RSC} yields
\[
e_j^\top \nabla^2 \ell_{n,i}^{(c)}(\vartheta_t)e_j \ge \alpha_1\|e_j\|_2^2-\tau r_n\|e_j\|_1^2 =\alpha_1-\tau r_n.
\]
Because this bound is independent of \(t\), we conclude that
\[
\inf_{t\in[0,1]} e_j^\top \nabla^2 \ell_{n,i}^{(c)}(\vartheta_t)e_j \ge\alpha_1-\tau r_n= m_{n,c}.
\]
This proves \eqref{eq:segment_curvature_bound}.
\end{proof}

\bibliography{Ref}

\end{document}